\documentclass[12pt]{article}

\usepackage{arxiv}
\usepackage{xargs}
\usepackage{xcolor}
\usepackage[utf8]{inputenc}
\usepackage[T1]{fontenc}    
\usepackage{hyperref}       
\usepackage{url}            
\usepackage{booktabs}       
\usepackage{amsfonts}       
\usepackage{nicefrac}       
\usepackage{microtype}      
\usepackage{xcolor}         
\usepackage{times}
\usepackage{lipsum}
\usepackage{graphicx}
\usepackage{mathrsfs}
\usepackage{xargs}
\usepackage{enumitem}
\usepackage{aliascnt}
\usepackage{appendix}
\usepackage[english]{babel}
\usepackage{amsmath}
\usepackage{amsthm}
\usepackage{amssymb}
\usepackage{caption}
\usepackage{cancel}
\usepackage{array}
\usepackage{framed}
\usepackage{multirow}
\usepackage{bbm}
\usepackage[ruled,vlined]{algorithm2e}
\usepackage{dsfont}

\usepackage[english]{babel}
\usepackage[sorting=ynt,style=authoryear-comp,natbib=true,maxbibnames=99,maxcitenames=2,uniquelist=false,giveninits]{biblatex}

\usepackage{hyperref}
\hypersetup{colorlinks = true}
\hypersetup{linkcolor=blue}
\usepackage{cleveref}

\crefname{subsubsubappendix}{appendix}{appendices}
\Crefname{subsubsubappendix}{Appendix}{Appendices}

\newtheoremstyle{styledef}
  {6pt}
  {6pt}
  {\itshape}
  {0em}
  {\bfseries}
  {}
  {1.5em}
  {}

\theoremstyle{styledef}

\newenvironment{enumerateList}{ \begin{enumerate}[label=(\roman*),wide=0pt, labelindent=\parindent]}{\end{enumerate}}

\newenvironment{hypZeroREP}[1]{
\begin{enumerate}[align=left,label=\textbf{\sf(#1$_{\mathrm{df}}^{\mathrm{REP}}$)}]\begin{sf}}
  {\end{sf}\end{enumerate}}
\newenvironment{hypZeroDREP}[1]{
\begin{enumerate}[align=left,label=\textbf{\sf(#1$_{\mathrm{df}}^{\mathrm{DREP}}$)}]\begin{sf}}
{\end{sf}\end{enumerate}}

\usepackage[textwidth=70pt]{todonotes}

\newtheorem{thm}{Theorem}
\crefname{thm}{theorem}{theorems}
\Crefname{thm}{Theorem}{Theorems}
\newtheorem{ex}{Example}
\crefname{ex}{example}{examples}
\Crefname{ex}{Example}{Examples}
\newtheorem{lem}{Lemma}
\crefname{lem}{Lemma}{lemmas}
\Crefname{lem}{Lemma}{Lemmas}
\newtheorem{prop}{Proposition}
\crefname{prop}{Proposition}{Propositions}
\Crefname{Prop}{Proposition}{Propositions}

\crefname{coro}{corollary}{corollaries}
\Crefname{Coro}{Corollary}{Corollaries}
\newtheorem{defi}{Definition}
\crefname{defi}{definition}{definitions}
\Crefname{Def}{Definition}{Definitions}
\newtheorem{rem}{Remark}
\crefname{rem}{remark}{remarks}
\Crefname{rem}{Remark}{Remarks}

\newenvironment{hyp}[1]{
  \begin{enumerate}[label=\textbf{\sf(#1\arabic*)},resume=hyp#1]\begin{sf}}
  {\end{sf}\end{enumerate}}
  \crefname{hyp}{}{ass}
  \Crefname{hyp}{}{Ass}

\renewenvironment{leftbar}[2][\hsize]
{
    \def\FrameCommand
    {
        {\color{#2}\vrule width 3pt}
        \hspace{0pt}
    }
    \MakeFramed{\hsize#1\advance\hsize-\width\FrameRestore}
}
{\endMakeFramed}

\newcommandx{\KL}[3][1=\theta, 2=\phi, 3=x]{\mathrm{KL}\lr{#1,#2;#3}}

\newcommand{\tBREP}{\widetilde{B}^{(\mathrm{REP})}(\theta, \phi;x)}

\newcommandx{\barw}[2][1=\theta, 2=\phi]{\overline{w}_{#1, #2}}
\newcommandx{\barell}[2][1=\theta, 2=\phi]{\overline{\ell}_{#1, #2}}
\newcommandx{\repell}[3][1=\theta, 2=\phi, 3=\phi]{\Lambda_{#1,#2}}
\newcommandx{\repelltwo}[3][1=\theta, 2=\phi, 3=\phi]{\Lambda_{#1,#2}}
\newcommandx{\repellstandard}[3][1=\theta, 2=\phi,3=\phi]{\widetilde{\Lambda}_{#1,#2}}
\newcommandx{\repbarell}[2][1=\theta, 2=\phi]{\overline{\Lambda}_{#1,#2}}
\newcommandx{\Wrepell}[1][1=\alpha]{\overline{\mathcal{W}}^{(#1)}}
\newcommandx{\barwL}[2][1=\theta, 2=\phi]{\overline{w}_{#1, \boldsymbol{#2}}}
\newcommandx{\barwa}[2][1=\theta, 2=\phi]{\overline{w}^{(\alpha)}_{#1, #2}}
\newcommandx{\barwba}[3][1=\theta, 2=\Phi, 3=\alpha]{\overline{w}^{(#3)}_{#1, \boldsymbol{#2}}}

\newcommandx{\REPq}{\tilde{q}}

\newcommandx{\gammaA}[1][1=\alpha]{\gamma^{(#1)}}
\newcommandx{\gammaREP}[2][1=\alpha, 2=\psi]{#2\text{-}\mathrm{G}_2^{({#1}, \mathrm{REP})}(\theta, \phi; x)}
\newcommandx{\gammaDREP}[2][1=\alpha, 2=\psi]{#2\text{-}\mathrm{G}_2^{({#1}, \mathrm{DREP})}(\theta, \phi; x)} 
\newcommandx{\VRREP}[2][1=\alpha, 2=\psi]{#2\text{-}\mathrm{G}_1^{({#1}, \mathrm{REP})}(\theta, \phi; x)}
\newcommandx{\VRDREP}[2][1=\alpha, 2=\psi]{#2\text{-}\mathrm{G}_1^{({#1}, \mathrm{DREP})}(\theta, \phi; x)} 

\newcommandx{\corrREP}[2][1=\alpha, 2=\psi]{\mathrm{C}_{#2}(\theta, \phi; x)}
\newcommandx{\corrDREP}[2][1=\alpha, 2=\psi]{\mathrm{C}_{#2}^{({#1}, \mathrm{DREP})}(\theta, \phi; x)} 
\newcommandx{\USNRREP}[3][1=\alpha, 2=\psi, 3=N]{\mathrm{sSNR}_{#2,{#3}}^{({#1}, \mathrm{REP})}(\theta, \phi; x)}
\newcommandx{\USNRDREP}[3][1=\alpha,2=\psi, 3=N]{\mathrm{sSNR}_{#2,#3}^{({#1}, \mathrm{DREP})}(\theta, \phi; x)}

\newcommandx{\grad}[1][1={M,N}]{g_{#1}^{(\alpha)}}
\newcommandx{\gradREP}[3][1={M,N},2=\psi,3=\alpha]{#2\text{-}g_{#1}^{(#3, \mathrm{REP})}(\theta, \phi;x)}
\newcommandx{\gradDREP}[3][1={M,N},2=\psi,3=\alpha]{#2\text{-}g_{#1}^{(#3,\mathrm{DREP})}(\theta, \phi;x)}
\newcommandx{\gradGEN}[3][1={M,N},2=\psi,3=\alpha]{#2\text{-}g_{#1}^{(#3, \mathrm{(D)REP})}(\theta, \phi;x)}
\newcommandx{\vREP}[2][1=\alpha, 2=\psi]{#2\text{-}V^{(#1, \mathrm{REP})}}
\newcommandx{\genvDREP}[2][1=\alpha, 2=\psi]{#2\text{-}V^{(#1, \mathrm{DREP})}}
\newcommandx{\vDREPun}[2][1=\alpha, 2=\psi]{\genvDREP[#1][#2]_1}
\newcommandx{\vDREPdeux}[2][1=\alpha, 2=\psi]{\genvDREP[#1][#2]_2}

\newcommandx{\lbd}[2][1=\theta,2=\phi]{\lambda_{#1, #2}^{(\alpha)}}
\newcommand{\lrav}[1]{\left|#1 \right|}
\newcommand{\lrn}[1]{\left\|#1 \right\|}
\newcommand{\lr}[1]{\left(#1 \right)}
\newcommand{\lrb}[1]{\left[#1 \right]}

\newcommand{\lrcb}[1]{\left\{#1 \right\}}
\newcommand{\lrder}[2]{\left.\lrb{#1}\right|_{#2}}

\newcommandx{\argcond}[2]{\lrb{\left.#1 \,\right|\, #2}}

\newcommandx{\liwae}[1][1=N]{\ell^{(\mathrm{IWAE})}_{#1}}
\newcommandx{\liren}[2][1=\alpha, 2=N]{\ell^{(#1)}_{#2}}
\newcommandx{\lirenL}[3][1=\alpha, 2=N, 3=L]{\ell^{(#1)}_{#2, #3}}
\newcommandx{\lirenBB}[2][1=\alpha, 2=N]{\tilde{\ell}^{(#1)}_{#2}}
\newcommand{\notationb}{\tilde{b}}
\newcommand{\PE}{\mathbb E}
\newcommand{\PP}{\mathbb P}

\newcommandx{\renorm}[2][1=\ell, 2=\alpha]{\varphi_{N, #1}^{(#2)}}
\newcommand{\rmd}{\mathrm d}
\newcommand{\rme}{\mathrm e}
\newcommand{\rset}{\mathbb{R}}

\newcommandx{\RalphaN}[1][1=\alpha]{R_{#1,N}}
\newcommandx{\RalphaNb}[2][1=\alpha,2=N]{R_{#1,#2}}
\newcommandx{\Rtalpha}[2][1=\alpha, 2=\phi]{\overline{w}_{\theta, #2}^{(#1)}}
\newcommand{\set}[2]{\lrcb{#1\;: \; #2}}

\newcommandx{\sumW}[2][1=j,2=\beta]{\overline{W}_{#1,N}(\beta)}

\newcommand{\tq}{\tilde{q}}
\newcommand{\tS}{S}
\newcommandx{\tw}[4][1=\phi, 2=\varepsilon,3=\phi',4=\theta]{\tilde{w}_{#4, #1, #3}(#2; x)}

\newcommandx{\tvDREP}[1][1=\alpha]{\tilde{v}^{(#1, \mathrm{DREP})}_{\theta, \phi}}
\newcommandx{\w}[2][1=\theta, 2=\phi]{w_{#1, #2}}
\newcommandx{\wb}[2][1=\theta, 2=\Phi]{w_{#1, \boldsymbol{#2}}}

\newcommandx{\Zalpha}[1][1=\phi']{Z_{\alpha}(#1)}
\newcommandx{\ZalphaN}[1][1=\phi']{\hat{Z}_{N,\alpha}(#1)}
\newcommand{\Y}{Y}

\begin{document}
\title{Learning with Importance Weighted
Variational Inference: Asymptotics for Gradient Estimators of the VR-IWAE Bound}

\author{Kamélia Daudel \\ ESSEC Business School \\ daudel@essec.edu \And François Roueff \\ Télécom Paris, Institut Polytechnique de Paris \\ francois.roueff@telecom-paris.fr}
\maketitle
{\small{}{}{}}{\small\par}

\begin{abstract}
Several variational bounds involving importance weighting ideas generalize the Evidence Lower BOund (ELBO) for marginal likelihood optimization, such as the Importance-weighted Auto-Encoder (IWAE), Variational Rényi (VR) and VR-IWAE bounds. Yet, it remains unclear how the joint choice of bound and gradient estimator impacts the behavior of the resulting variational inference (VI) algorithms. This paper provides a unified theoretical comparison of reparameterized (REP) and doubly-reparameterized (DREP) gradient estimators tied to the IWAE, VR and VR-IWAE bounds. Through asymptotic analyses of the Signal-to-Noise Ratio as the number of Monter Carlo samples $N$ goes to infinity, we identify a bias-variance tradeoff in these gradient estimators and we formally justify the superiority of DREP over REP in importance-weighted VI. An additional asymptotic analysis for challenging regimes, where both $N$ and the Kullback-Leibler divergence between the variational and posterior densities go to infinity, indicates that importance-weighted VI gradient estimators point in a well-founded direction even when the variational approximation deteriorates. Together, these complementary results characterize the optimization trajectory in importance-weighted VI from poor initialization to final convergence. Importantly, our proof techniques establish general theoretical tools for the study of sample means ratios whose scope extend beyond VI and constitute an independent contribution to the field of Monte Carlo methods.
\end{abstract}


\setcounter{tocdepth}{2}

\section{Introduction}

Variational inference (VI) methods seek to find the best approximation to an unknown target posterior density within a more tractable family of probability densities $\mathcal{Q}$ \citep{Neal1998,jordan1999,blei2017variational}. In the standard Bayesian setting, the model is typically assumed fixed and the aim is to identify the member of $\mathcal{Q}$ that minimizes a certain measure of dissimilarity to the posterior. More generally, VI is applicable to models depending on a parameter $\theta$, in which case the primary goal is to optimize the marginal log-likelihood. In both scenarios, the intractability of the marginal likelihood necessitates the maximization of variational bounds, which involve the variational family $\mathcal{Q}$ and are constructed as surrogate objective functions to the marginal log-likelihood that are amenable to optimization. \looseness=-1

While the most traditional variational bound is the Evidence Lower BOund (ELBO), an active line of research in VI is interested in deriving alternative bounds that can improve on the outcome of VI. In particular, popular generalizations of the ELBO relying on importance weighting ideas have been proposed, such as the importance-weighted auto-encoder (IWAE) bound in \cite{burda2015importance} that is parametrized by the number of Monte Carlo samples $N$ and the Variational Rényi (VR) bound in \cite{li2016renyi} that is parametrized by the Rényi alpha-divergence parameter $\alpha$. \cite{daudel2022Alpha} also introduced the VR-IWAE bound, a variational bound depending on $(N,\alpha) \in \mathbb{N} \times [0,1)$ that unifies the ELBO, IWAE and VR bounds. A shared feature of these variational bounds is that they can readily be optimized using reparametrized gradient estimators \citep{kingma2014autoencoding,pmlr-v32-rezende14,burda2015importance,li2016renyi,Tucker2019DoublyRG,daudel2022Alpha}. \looseness=-1

Yet, the theoretical understanding of how the joint choice of importance-weighted variational bound and gradient estimator governs the optimization trajectory remains limited. The properties of the variational bounds themselves are relatively well-established: \cite{maddison2017,nowozin2018debiasing,domke2018} study the asymptotic behavior of the IWAE bound as $N \to \infty$, with \cite{daudel2022Alpha} exploring the more general VR-IWAE case.
Furthermore, recent works investigate the desirable characteristics of the variational approximation at optimality for the VR bound \citep{margossian2024variationalinferenceuncertaintyquantification, margossian2025generalized} and the IWAE bound \citep{cherief2025asymptotics}. Despite these insights into importance-weighted variational bounds and their minima, a comprehensive analysis of their gradient estimators is currently missing from the literature. \looseness=-1

More precisely, current analyses of reparameterized (REP) and doubly-reparameterized (DREP) gradient estimators, which are the primary gradient estimators in importance-weighted VI, are fragmented and incomplete. Although \cite{daudel2022Alpha} capture the asymptotics of the REP gradient estimator of the VR-IWAE bound as $N \to \infty$, the role of $\alpha$ is completely overlooked outside of the distinction between the pathological case $\alpha = 0$ (IWAE) identified in \cite{rainforth2018tighter} and the case $\alpha \in (0,1)$. As for the DREP gradient estimator of the VR-IWAE bound, existing results are restricted to an informal argument specific to the IWAE setting \citep{Tucker2019DoublyRG}.
Beyond these theoretical gaps, there is also a general lack of understanding regarding how REP and DREP gradient estimators perform as the quality of the variational approximation deteriorates. \looseness=-1

As a result, comparisons between REP and DREP gradient estimators have been confined to empirical studies suggesting that tuning $\alpha$ improves performance and that the DREP gradient estimator outperforms the REP one \citep{li2016renyi,daudel2022Alpha,GeffnerDomke2020Biased,dhaka2021challenges}. Our work addresses these limitations through two main contributions to importance-weighted VI, both of which are grounded in broader theoretical results for the Monte Carlo estimation of sample mean ratios we establish. Namely, our contributions are as follows: \looseness=-1
\begin{enumerateList}
  \item \textbf{First unified comparison between REP and DREP gradient estimators in importance-weighted VI.}
  We establish an asymptotic characterization of the expectation and variance for Monte Carlo estimators taking the form of a ratio of sample means. By applying this theory to importance-weighted VI, we obtain asymptotic analyses as $N \to \infty$ that hold under mild conditions and that allow the systematic comparison of REP and DREP gradient estimators tied to the IWAE, VR and VR-IWAE bounds. Our analyses are based on the Signal-to-Noise Ratio (SNR), a metric commonly used in VI that quantifies the quality of an estimator via the ratio of its expected value to its standard deviation. We show that the hyperparameter $\alpha$ is responsible for a bias-variance tradeoff in the REP and DREP gradient estimators of the VR-IWAE bound. Moreover, we provide the first theoretical justification for the empirical superiority of DREP over REP in importance-weighted VI. \looseness=-1
  
  \item \textbf{First study of importance-weighted gradient estimators in the challenging regime where the variational approximation deteriorates.} We analyze the case where the quality of the variational density deteriorates in the gradient estimators. This approach is motivated by observing that the asymptotic regimes in (i) may only kick in for a large $N$ in practice if the variational density chosen within $\mathcal{Q}$ does not match the posterior density well enough. Specifically, we derive a second analysis at the SNR level when both $N$ and the Kullback-Leibler divergence between the variational and posterior densities go to infinity. Our insights indicate that importance-weighted VI gradient estimators point in a well-founded direction even in such settings. As with our first contribution, the technical tools used in (ii) are valid outside of importance-weighted VI, which signals that their utility might extend beyond the context considered here. 
\end{enumerateList}
Crucially, our theoretical analyses in (i) and (ii) are complementary analyses that reflect different stages in importance-weighted VI algorithms. While (ii) captures the behavior in challenging settings such as poor parameter initialization, (i) corresponds to settings where the variational density is an accurate approximation of the posterior density, such as during training. In addition, provided that the variational family is sufficiently expressive, we obtain that optimization under the challenging regime of (ii) drives the trajectory toward the well-behaved regime of (i). We validate our theory over three experiments of increasing complexity, before outlining directions for future research. Deferred results, proofs and numerical experiments can be found in the appendix.

\section{Background}
\label{sec:background}

Consider a model with joint distribution $p_\theta(x, z)$ parameterized by $\theta \in \Theta \subseteq \rset^a$, where $x$ denotes an observation and $z$ is a latent variable valued in $\rset^d$. In Bayesian inference, one seeks to sample from the posterior density $p_\theta(z|x)$. However, in many important cases sampling directly from the posterior density is impossible. This then hinders usual inference tasks such as marginal likelihood optimization, where the goal is to find the optimal $\theta \in \Theta$ which maximizes the marginal log likelihood 
\begin{align}\label{eq:loglik}
\ell(\theta; x) = \log \lr{\int p_\theta(x, z) \rmd z}.
\end{align}  
To tackle this challenge, VI methods introduce a probability density $q_\phi(z|x)$ parameterized by $\phi \in \Phi$, whose distribution is easy to sample from and where typically $ \Phi \subseteq \rset^b$. Specifically, in the context of marginal likelihood maximization, VI methods solve an optimization problem involving a variational bound, that is a lower bound on the marginal log likelihood $\ell(\theta; x)$. While the most common example of variational bound is the ELBO \looseness=-1
\begin{align} \label{eq:ELBO-def}
  &\mathrm{ELBO} (\theta, \phi; x) = \int q_{\phi}(z|x) \log \w(z;x) ~ \rmd z, \quad \w(z;x) = \frac{p_\theta(x,z)}{q_{\phi}(z|x)} ~,~ z \in \rset^d,
\end{align}
several alternatives to the ELBO based on importance-weighting ideas have been introduced. Two notable instances of importance-weighted variational bounds are the IWAE bound \citep{burda2015importance} and 
the VR-IWAE bound \citep{daudel2022Alpha}, where the latter is defined for all $N \in \mathbb{N}^\star$ and all $\alpha \in [0,1)$ by
\begin{align} \label{eq:lalpha-def} 
  \liren(\theta, \phi; x) = \frac{1}{1-\alpha} \int \prod_{i = 1}^N q_{\phi}(z_i|x) \log \lr{ \frac{1}{N} \sum_{j = 1}^N \w(z_j;x)^{1-\alpha} } \rmd z_{1:N}.
  \end{align} 
The VR-IWAE bound serves as a unifying objective that bridges several existing VI frameworks, inheriting appealing theoretical and practical properties from these connections. It recovers the ELBO when $N=1$ (or as $\alpha \to 1$) and the IWAE bound \citep{burda2015importance} when $\alpha = 0$. \cite{daudel2022Alpha} also show that it is a principled approximation of the VR bound ($N \to \infty$), a flexible variational bound which seeks to improve on the ELBO by employing Rényi's alpha-divergences  \citep{li2016renyi}. We expand below on the desirable properties of the VR-IWAE bound originating from its ties to the ELBO, IWAE and VR bounds. \looseness=-1

\begin{enumerateList}
  \item \textit{Tightness.} Increasing the sample size $N$ and/or decreasing $\alpha$ in the VR-IWAE bound yields a tighter lower bound on the marginal log likelihood provided that the support of $p_\theta(\cdot|x)$ is contained within the support of $q_\phi(\cdot|x)$: for all $\alpha, \alpha' \in [0,1)$ with $\alpha \leq \alpha'$ and all $N \in \mathbb{N}^*$, \looseness=-1
\begin{align}
 & \mathrm{ELBO} (\theta, \phi; x) \leq  \liren(\theta, \phi; x) \leq  \liren[\alpha][N+1](\theta, \phi; x) \leq \mathrm{VR}^{(\alpha)}(\theta, \phi; x) \leq \ell(\theta;x) \\
 & \mathrm{ELBO} (\theta, \phi; x) \leq \liren[\alpha'](\theta, \phi; x) \leq \liren(\theta, \phi; x) \leq \liren[0](\theta, \phi; x)  \leq \ell(\theta; x), \label{eq:mono:alpha}
\end{align} 
where $\mathrm{VR}^{(\alpha)}(\theta, \phi; x)$ denotes the VR bound from \citep{li2016renyi} and is given by
\begin{align}
  & \mathrm{VR}^{(\alpha)}(\theta, \phi; x) = \frac{1}{1-\alpha} \log \lr{\int q_{\phi}(z|x)~\w(z; x)^{1-\alpha}~\rmd z}, \quad \alpha \in [0,1)\cup  (1, \infty). \label{eq:defVRbound} 
\end{align} 
More precisely, \cite[Theorem 3]{daudel2022Alpha} establishes that: as $N \to \infty$
  \begin{align} \label{eq:OneOverNGenDomke}
    & \liren  (\theta,\phi;x) = \mathrm{VR}^{(\alpha)}(\theta, \phi; x) - \frac{1}{2 N} \gammaA(\theta, \phi;x)^2+o\left(\frac{1}{N}\right) \\
    & \mbox{where} \quad \gammaA(\theta, \phi;x)^2 = \frac{1}{1-\alpha}\mathbb{V}_{Z \sim q_{\phi}(\cdot|x)}\lr{\frac{\w(Z; x)^{1-\alpha}}{\PE_{Z' \sim q_{\phi}(\cdot|x)}(\w(Z'; x)^{1-\alpha})}} \label{eq:defGammaAlphaFirst},
\end{align}
meaning that the VR-IWAE bound is a finite sample approximation of the VR bound which gets tighter and converges to the VR bound at a fast $1/N$ rate. 

\item \textit{Ease of optimization.} The VR-IWAE bound can be straightforwardly optimized via stochastic gradient ascent \citep[SGA;][]{Robbins1951ASA,bottou2018optimization} paired up with the reparameterization trick \citep{kingma2014autoencoding,pmlr-v32-rezende14,burda2015importance,daudel2022Alpha}. The reparameterization trick is a common tool used in VI which has led to empirical advantages when estimating gradients with respect to $\phi$ and which relies on the assumption that $Z = f(\varepsilon,\phi;x) \sim q_\phi (\cdot|x)$ with
$\varepsilon \sim \REPq$ \citep{kucukelbir2017automatic,xu2019variance,mohamed2020monte}. Under the assumption that $Z$ can be reparameterized, the gradient of the VR-IWAE
bound $\nabla_{\theta, \phi} \liren(\theta, \phi; x)$ reads  \looseness=-1
\begin{align}\label{eq:gradient-vriwae}
   \int \prod_{i = 1}^N \REPq(\varepsilon_{i}) \lr{\sum_{j=1}^N \frac{\w(f(\varepsilon_j, \phi; x);x)^{1-\alpha}}{\sum_{\ell = 1}^N ~\w(f(\varepsilon_{\ell}, \phi; x);x)^{1-\alpha}} ~ \nabla_{\theta, \phi} \log \w(f(\varepsilon_{j}, \phi; x);x)} \rmd \varepsilon_{1:N} 
\end{align}
and since this gradient can be estimated by the unbiased Monte Carlo estimator
\begin{align} \label{eq:estimVRIWAEVRgradREP}
\frac{1}{M} \sum_{m = 1}^M \sum_{j=1}^N  \frac{\w(f(\varepsilon_{m,j}, \phi; x);x)^{1-\alpha}}{\sum_{\ell = 1}^N \w(f(\varepsilon_{m, \ell}, \phi; x); x)^{1-\alpha}} ~ \nabla_{\theta, \phi} \log \w(f(\varepsilon_{m,j}, \phi; x);x)
\end{align}
with $\varepsilon_{m,1}, \ldots, \varepsilon_{m,N}$ being i.i.d. samples generated from $\REPq$ for all $m = 1\ldots M$, the resulting procedure enjoys the typical convergence guarantees of unbiased SGA \citep{hoffman2013stochastic}. In particular, \eqref{eq:estimVRIWAEVRgradREP} recovers the unbiased gradient estimator used for ELBO optimization in \cite{kingma2014autoencoding} when $N=1$ (or $\alpha \to 1$) and the one used for IWAE bound optimization in \cite{burda2015importance} when $\alpha = 0$. Moreover, \eqref{eq:estimVRIWAEVRgradREP} is used in \cite{li2016renyi} as an approximation of $\nabla_{\theta, \phi} \mathrm{VR}^{(\alpha)}(\theta, \phi; x)$ to derive a practical SGA scheme based on the VR bound as
\begin{align} \label{eq:gradVRintegralform}
  & \nabla_{\theta, \phi} \mathrm{VR}^{(\alpha)}(\theta, \phi; x) = \frac{
    \int \REPq(\varepsilon)\, w_{\theta, \phi}(f(\varepsilon, \phi; x); x)^{1 - \alpha}
    \nabla_{\theta, \phi} \log w_{\theta, \phi}(f(\varepsilon, \phi; x); x)\, d\varepsilon}{    \int \REPq(\varepsilon)\, w_{\theta, \phi}(f(\varepsilon, \phi; x); x)^{1 - \alpha} d\varepsilon}.
\end{align} 
Consequently, on top of encompassing the typical ELBO and IWAE optimization frameworks, the VR-IWAE bound provides theoretical grounding for SGA schemes employing \eqref{eq:estimVRIWAEVRgradREP} with $\alpha \in (0,1)$: while these schemes are biased when viewed through the angle of the VR bound, they are unbiased when the objective function is the VR-IWAE bound \citep{daudel2022Alpha}. \looseness=-1
\end{enumerateList}
By unifying the ELBO, IWAE and VR bounds, the VR-IWAE bound allows one to treat the choice of the variational bound as a problem of hyperparameter tuning within a shared framework. This motivates a closer examination of how $(N, \alpha)$ influence the optimization trajectory, which we will achieve in this paper via a thorough characterization of key importance-weighted gradient estimators associated to the VR-IWAE bound. 

\section{Analyzing Importance-weighted VI Gradient Estimators}
\label{sub:GradN}

We begin by defining the notation and the gradient estimators that form the basis of our study. \looseness=-1

\subsection{Notation and Gradient Estimators Definitions}
\label{subsec:notaAndDefin}

In the following, we let $x$ be a fixed observation, $\alpha \in [0,1)$, $\Theta\subseteq\rset^a$ and $\Phi\subseteq\rset^b$ be two open subsets, and we assume that for all
$\theta\in\Theta,\phi\in\Phi$, the following holds:
\begin{hyp}{A} 
\item \label{hyp:VRIWAEwell-defined}   $\int p_\theta(x, z) \;\nu(\rmd z) <
  \infty$, $\int q_{\phi}(z|x)\;\nu(\rmd z)=1$ and
  $q_{\phi}(z|x),p_\theta(z|x)>0$ for $\nu$-a.e. $z$.
\end{hyp}
Here $\nu$ is a $\sigma$-finite measure on $\mathbb{R}^d$, typically
the Lebesgue measure on $\rset^d$ as used
in~(\ref{eq:loglik})--(\ref{eq:gradVRintegralform}), or the Lebesgue measure
on a subset of $\rset^d$, or any other convenient dominating measure
on the Borel sets of $\rset^d$. To avoid specifying the dominating
measure, we will use the notation $\PE$ throughout the paper. If the
probability distribution of the involved random variables has not
been specified beforehand, we use a subscript to the symbol $\PE$ to
indicate their densities with respect to $\nu$. Using this convention,
(\ref{eq:lalpha-def}) and~(\ref{eq:defVRbound}) are then respectively
written as 
\begin{align}\label{eq:lalpha-def-E}
&\liren(\theta, \phi; x)= \frac{1}{1-\alpha}
\PE_{Z_j\overset{\text{\tiny i.i.d.}}{\sim}q_{\phi}(\cdot|x)}\lr{\log \lr{
    \frac{1}{N} \sum_{j = 1}^N \w(Z_j;x)^{1-\alpha} }}  \\
  \label{eq:defVRbound-E}
  &\mathrm{VR}^{(\alpha)}(\theta, \phi; x) = \frac{1}{1-\alpha} \log
\PE_{Z\sim q_{\phi}(z|x)}  \lr{\w(Z; x)^{1-\alpha}},
\end{align}
with \ref{hyp:VRIWAEwell-defined} ensuring
$\w(\cdot; x)$ is well-defined and positive a.s. under the expectation
$\PE$. 

On top of \ref{hyp:VRIWAEwell-defined} and in order to compute gradient estimators, we assume that the mappings
$\theta\mapsto p_{\theta}(x)$, $(z,\theta)\mapsto p_{\theta}(z|x)$,
and $(z,\phi)\mapsto q_{\phi}(z|x)$ are differentiable on $\Theta$,
$\rset^d\times\Theta$ and $\rset^d\times\Phi$, respectively.  We also assume that integration with respect to $q_{\phi} (\cdot|x)$ can always be reparameterized as an expectation
independent of $\phi$, and that the involved functions of $\theta$ and
$\phi$ inside this expectation are differentiable, meaning \ref{hyp:reparameterized} below holds:
\begin{hyp}{A}  
\item \label{hyp:reparameterized} There exist a function $f$ and a
  density $\REPq$ such that $f(\varepsilon,\phi;x) \sim q_{\phi} (\cdot|x)$
  with $\varepsilon \sim \REPq$. In addition, for   $\varepsilon \sim \REPq$, we
  have a.s. that the mapping
  $\phi\mapsto f(\varepsilon,\phi;x)$ is differentiable on $\Phi$. \looseness=-1
\end{hyp} 
For a differentiable function $g$ on $\rset^D$ with
$D \in \mathbb{N}^\star$, $\partial_{y_k} g(y)$ denotes the partial
derivative of $g(y_1, \ldots, y_D)$ with respect to $y_k$ evaluated at
$y = (y_1, \ldots, y_D) \in \rset^D$ as in \eqref{estimREP} and
$\nabla_y g(y)$ denotes the gradient vector
$\lr{\partial_{y_k} g(y)}_{k=1,\dots,D}$ as
in~(\ref{eq:gradient-vriwae}).  When differentiating, it can be
convenient to evaluate the resulting function at a different variable:
we write $\lrder{h(y,z)}{y= y'}$ to indicate that the multivariate
function $(y,z)\mapsto h(y,z)$ is evaluated at $(y',z)$. In the
following, we let $M, N \in \mathbb{N}^\star$ and
$\varepsilon_{m,1}, \ldots, \varepsilon_{m,N}$ be i.i.d. samples
generated from $\REPq$ for all $m = 1 \ldots M$, we set
$Z_{m,j} = f(\varepsilon_{m,j}, \phi;x)$ for all $j = 1 \ldots N$ and
all $m =1 \ldots M$ and we let $\psi$ denote a component of the
$\mathbb{R}^{a+b}$-valued variable
$(\theta, \phi)=(\theta_1, \ldots, \theta_a,\phi_1,\dots,\phi_b)$. We next define below the two gradient estimators of the VR-IWAE bound that are central to our study. \looseness=-1

\begin{defi}[REP and DREP gradient estimators of the VR-IWAE bound] The reparameterized (REP) gradient estimator of the VR-IWAE bound w.r.t. $\psi$ is given by
  \begin{align} \label{estimREP}
    \gradREP = \frac{1}{M} \sum_{m = 1}^M \sum_{j=1}^N  \frac{\w(Z_{m,j};x)^{1-\alpha}}{\sum_{\ell = 1}^N ~\w(Z_{m,\ell}; x)^{1-\alpha}} ~ \partial_{\psi} \log \w(f(\varepsilon_{m,j}, \phi; x);x). 
  \end{align}
Futher assuming $\psi$ is a component of the
    $\mathbb{R}^{b}$-valued variable
    $\phi=(\phi_1,\dots,\phi_b)$, the doubly-reparameterized (DREP) gradient estimator of the VR-IWAE bound w.r.t. $\psi$ is given by \looseness=-1
  \begin{align}
      \gradDREP & = \frac{1}{M} \sum_{m = 1}^M \sum_{j=1}^N \lr{h_{m,j}^{(\alpha)}(\theta, \phi; x) \lrder{\partial_{\psi'} \log \w(f(\varepsilon_{m,j}, \phi'; x); x)}{\phi' = \phi}}, \label{estimDREP}
  \end{align}
where, for all $m =1 \ldots M$ and all $j = 1 \ldots N$,
      \begin{align}
      h_{m,j}^{(\alpha)}(\theta, \phi; x) & = \alpha ~ \frac{\w(Z_{m,j};x)^{1-\alpha} }{\sum_{\ell = 1}^N \w(Z_{m,\ell}; x)^{1-\alpha}} + (1-\alpha)~\lr{\frac{\w(Z_{m,j};x)^{1-\alpha} }{\sum_{\ell = 1}^N \w(Z_{m,\ell};x)^{1-\alpha}}}^2\;. \label{eq:defHmj}
      \end{align} 
  \end{defi}
The REP and DREP gradient estimators are both unbiased estimators of the gradient of the VR-IWAE bound
\citep{daudel2022Alpha}. Establishing such properties rigorously requires assumptions to interchange derivatives and expectations, which are typically taken for granted and omitted in VI \cite[see, for instance,][]{daudel2022Alpha}. For the sake of completeness, we provide precise sufficient conditions ensuring that the interchange is valid in \Cref{app:differentiabilityCondition}: \ref{hyp:hypZeroREP} for the REP estimator and \ref{hyp:hypZeroDREP} for the DREP estimator. We now briefly comment on the specificities of the REP and DREP gradient estimators of the VR-IWAE bound. \looseness=-1

As previously mentioned, the REP gradient estimator of the VR-IWAE bound acts as a bridge between the REP gradient estimator of the ELBO ($N=1$ or $\alpha \to 1$), the REP gradient estimator of the IWAE bound ($\alpha = 0$) and the (biased) REP gradient estimator of the VR bound ($\alpha \in (0,1)$). The DREP gradient estimator of the VR-IWAE bound \eqref{estimDREP} was introduced in \cite{daudel2022Alpha} as a generalization of the DREP gradient estimator of the ELBO \citep{roeder2017} and of the IWAE bound  \citep{Tucker2019DoublyRG}. The main motivation behind DREP gradient estimators comes from the observation made in \cite{roeder2017} that:\looseness=-1
    \begin{align*}
    \partial_{\psi} \log \w(f(\varepsilon, \phi; x);x) = \lrder{\partial_{\psi'} \log \w(f(\varepsilon, \phi'; x); x) - \partial_{\psi'} \log q_{\phi'}(f(\varepsilon, \phi;x)|x)}{\phi' = \phi}.
    \end{align*}
This identity reveals that the REP gradient estimator of the ELBO ($N = 1$ in \eqref{estimREP}) has a nonzero variance at optimality, since the left-hand side above has nonzero variance when $q_{\phi}(\cdot|x)$ matches the exact posterior density $p_\theta(\cdot|x)$ everywhere due to the score function term $\lrder{\partial_{\psi'} \log q_{\phi'}(f(\varepsilon, \phi;x)|x)}{\phi' = \phi}$. By contrast, the DREP gradient estimator of the ELBO ($N = 1$ in \eqref{estimDREP}) enjoys a variance which goes to zero as the variational density approaches the posterior one, and is shown in \cite{roeder2017} to lead to improved empirical performance. \looseness=-1

In the spirit of \cite{roeder2017}, the DREP gradient estimator of the IWAE bound \citep{Tucker2019DoublyRG} and of the VR-IWAE bound \citep{daudel2022Alpha} maintain the zero-variance property at optimality. However, the behavior of these gradient estimators during the optimization procedure compared to their REP counterparts is much less straightforward. To clarify these dynamics and avance the field of importance-weighted VI, we seek to provide the first informative asymptotic analyses for the REP and DREP gradient estimators of the VR-IWAE bound enabling their comparison. To that end, we will start by proving a new general result capturing the asymptotic behavior in expectation and variance of sample mean ratios as the number of Monte Carlo samples increases. \looseness=-1

\subsection{Key Technical Result on the Asymptotic Moments of Sample Mean Ratios}
\label{sec:key-result-ratios}

We present below a key technical result which captures the asymptotic
behavior in expectation and variance for ratios of sample means as the
number of Monte Carlo samples increases. \looseness=-1
\begin{thm}
  \label{thm:ratios-limits-optimal-cond}
Let $(X_1,W_1),(X_2,W_2),\dots$ be
i.i.d. with the same distribution as a generic pair $(X,W)$ valued in
$\rset\times\rset_+$. Denote, for all $N\geq1$, $\overline{X}_N = N^{-1} \sum_{i=1}^N X_i$ and $\overline{W}_N = N^{-1} \sum_{i=1}^N W_i$. Suppose that there exists $\eta>0$ such that
\begin{align}  \label{eq:simple-cond-ratios}
    \sup_{u>0}\lr{u^{-\eta}\;\PP\lr{W\leq u} }<\infty.
\end{align}
Then the following assertions hold.
\begin{enumerateList}
\item\label{item:ratios-limits-optimal-cond1} If $\PE(|X|)<\infty$, then
\begin{align}
  \lim_{N\to\infty}\PE\lr{\frac{\overline{X}_N}{\lr{\overline{W}_N}^\mu}}=\frac{\PE(X)}{\lr{\PE(W)}^\mu} \;. \label{eq:limExp}
\end{align}
\item\label{item:ratios-limits-optimal-cond2bis} If $\PE(|X|)<\infty$,
  $\PE(W^2)<\infty$ and $\PE(|X|\,W)<\infty$, then
  \begin{equation}
    \label{eq:ratios-limits-optimal-cond2bis}
\lim_{N\to\infty}N\,\lr{\PE\lr{\frac{\overline{X}_{N}}{\lr{\overline{W}_{N}}^{\mu}}}
  -\frac{\PE(X)}{\lr{\PE(W)}^\mu}}=
\frac{(\mu+1)\mu\,\PE(X)\,\mathbb{V}\lr{W}}{2\,\lr{\PE(W)}^{\mu+2}}-\frac{\mu\,\mathbb{C}\mathrm{ov}\lr{X,W}}{\lr{\PE(W)}^{\mu+1}}\;.
  \end{equation}
  Moreover, if instead 
$\PE(|X|)<\infty$, $\PE(W)<\infty$ and $\PE(\lrav{\PE(W)X-\PE(X)W}W)<\infty$,
then the convergence in~(\ref{eq:ratios-limits-optimal-cond2bis})
still holds when $\mu=1$ with a limit expressed by \looseness=-1
\begin{equation}
  \label{eq:ratios-limits-optimal-cond2bis:alter}
\lim_{N\to\infty}N\,\lr{\PE\lr{\frac{\overline{X}_{N}}{\overline{W}_{N}}}
  -\frac{\PE(X)}{\PE(W)}}=  \frac{\PE\lr{W\lrb{\PE(X)W-\PE(W)X}}}{\lr{\PE(W)}^{3}}.
\end{equation}
\item\label{item:ratios-limits-optimal-cond3} If  $\PE(|X|^2)<\infty$, then
\begin{align} \label{eq:limVar}
 \lim_{N\to\infty}
 \mathbb{V}\lr{\frac{\overline{X}_N}{\lr{\overline{W}_{N}}^{\mu}}} =0 \;.
\end{align}
\item\label{item:ratios-limits-optimal-cond4} If $\PE(|X|^2)<\infty$ and $\PE(W^2)<\infty$, then
\begin{align} \label{eq:limVarFirstOrder}
    \lim_{N\to\infty}N\,
    \mathbb{V}\lr{\frac{\overline{X}_N}{\lr{\overline{W}_{N}}^{\mu}}} =
    \frac{\mathbb{V}\lr{\PE(W) X - \mu\PE(X) W}}{\lr{\PE(W)}^{2 \mu+2}}\;.
\end{align}
Moreover, if instead $\PE(|X|)<\infty$,  $\PE(W)<\infty$ and
$\PE(\lrav{\PE(W)X-\PE(X)W}^2)<\infty$, then the convergence in
\eqref{eq:limVarFirstOrder} still holds when $\mu = 1$ thus giving
\begin{align} \label{eq:limVarFirstOrderMuOne}
      \lim_{N\to\infty}N\,
    \mathbb{V}\lr{\frac{\overline{X}_N}{\overline{W}_{N}}} =
    \frac{\mathbb{V}\lr{\PE(W) X - \PE(X) W}}{\lr{\PE(W)}^{4}}\;.
\end{align}
\end{enumerateList}
\end{thm}
The proof of \Cref{thm:ratios-limits-optimal-cond} is deferred to
\Cref{sec:proof-s-crefsec:key}. Observe that \ref{item:ratios-limits-optimal-cond1} and \ref{item:ratios-limits-optimal-cond3} establish the intuitive limiting behavior for the expectation and variance of $\overline{X}_N/(\overline{W}_N)^\mu$, while \ref{item:ratios-limits-optimal-cond2bis} and \ref{item:ratios-limits-optimal-cond4} capture the behavior of the corresponding first-order terms. Together, \ref{item:ratios-limits-optimal-cond1}-\ref{item:ratios-limits-optimal-cond4} will provide a general framework to derive the asymptotics for the REP \eqref{estimREP} and DREP \eqref{estimDREP} gradient estimators of the VR-IWAE bound via selected choices of $\overline{X}_N, \overline{W}_N$ and $\mu$. Before proceeding to do so in \Cref{sec:behavior-gradient-vr-Nasymp}, we comment on the assumptions we make in \Cref{thm:ratios-limits-optimal-cond} and we discuss the implications of this result for VI and more broadly Monte Carlo methods. \looseness=-1

\paragraph*{Assumptions.} Essentially, the assumptions made in
\Cref{thm:ratios-limits-optimal-cond} ensure that the quantities of
interest appearing in this theorem are well-defined. This is clear
from the specific conditions associated with assertions
\ref{item:ratios-limits-optimal-cond1}-\ref{item:ratios-limits-optimal-cond4}
(e.g., the integrability condition $\PE(|X|)<\infty$ in
\ref{item:ratios-limits-optimal-cond1} that guarantees the existence
of $\PE(X)$) and we now discuss the condition
\eqref{eq:simple-cond-ratios}, which is common to all the assertions
\ref{item:ratios-limits-optimal-cond1}-\ref{item:ratios-limits-optimal-cond4}. \looseness=-1

Assuming that there exists $\eta > 0$ such that
\eqref{eq:simple-cond-ratios} holds ensures that the distribution of
$W$ does not place excessive mass near $0$, since  $\PP\lr{W \leq
  u}=O(u^{\eta})$ as $u\downarrow0$. Hence this
condition slightly strengthens the right-continuity at zero of the
distribution function $u\mapsto\PP\lr{W \leq u}$ by imposing a
power decay of arbitrary low positive exponent. Note that without
right-continuity, the
quantities studied in \Cref{thm:ratios-limits-optimal-cond} would be
ill-defined due to the nonzero probability of a vanishing denominator
in $\overline{X}_N/(\overline{W}_N)^\mu$.
As shown in the lemma below, \eqref{eq:simple-cond-ratios} admits
equivalent formulations frequently encountered in the related
literature.
\begin{lem} 
  \label{lem:ass-neg-moment-for-all-mu}
  Set $W_1, \ldots, W_N$ be  i.i.d. with
  the same distribution as a generic random variable $W$ valued in
$\rset_+$. Then the following three assertions are
  equivalent.
  \begin{enumerateList}
  \item \label{item:condition-equiv-ours} There exists $\eta>0$ such that~(\ref{eq:simple-cond-ratios}) holds.
  \item \label{item:condition-equiv-N} There exists $\mu>0$ and $N\geq1$ such that
    $\PE((\overline{W}_N)^{-\mu})<\infty$. 
  \item \label{item:jacob-condition} There exists $\mu>0$ such that
    $\PE(W^{-\mu})<\infty$. 
  \end{enumerateList}
\end{lem}
\Cref{lem:ass-neg-moment-for-all-mu} is a direct consequence of the
more precise statement \Cref{lem:equivlimsupconditionGen}, which can
be found in \Cref{app:usefulEq} alongside its proof. Conditions akin to \ref{item:condition-equiv-N}, although usually stated for a specific value of $\mu$, are typically used when investigating importance-weighted VI bounds or their gradients \cite[see, e.g.,][]{maddison2017,domke2018,daudel2022Alpha}. As for \ref{item:jacob-condition}, it appears in \cite{deligiannidis25}, which studies the asymptotic bias of self-normalized importance sampling. The assertion \ref{item:condition-equiv-N} can in fact be equivalently stated with an arbitrary $\mu>0$ or an arbitrary $N\geq1$ (see \Cref{lem:equivlimsupconditionGen}), which yields \ref{item:jacob-condition} as the special case where $N$ is arbitrarily fixed to $N=1$. \looseness=-1

\paragraph*{Related work and broader impact.} \Cref{thm:ratios-limits-optimal-cond} subsumes several results studying importance-weighted estimators. Specifically, \Cref{thm:ratios-limits-optimal-cond} contributes to the Monte Carlo literature by broadening the applicability of existing importance-weigthing results while operating under strictly less restrictive assumptions. 

Indeed, the most closely related existing studies have focused on deriving \eqref{eq:ratios-limits-optimal-cond2bis:alter} and \eqref{eq:limVarFirstOrderMuOne}, or specific instantiations of these, under stronger assumptions and different proof techniques compared to ours \citep{rainforth2018tighter,daudel2022Alpha,deligiannidis25}; see \Cref{tab} for a quick overview and more details are also available in \Cref{app:relWThm1}. Furthermore, to the best of our knowledge, existing studies are entirely restricted to setting where $\mu = 1$ while \Cref{thm:ratios-limits-optimal-cond} accommodates the general case $\mu > 0$. As we shall see, this extension is particularly relevant in the context of VI and will emerge as a theoretical prerequisite for analyzing the DREP gradient estimator of the VR-IWAE bound. This will permit us to bridge the gap left by \cite{Tucker2019DoublyRG} and \cite{daudel2022Alpha}, who proposed the DREP gradients for the IWAE and VR-IWAE bounds, respectively, but did not provide asymptotic guarantees. \looseness=-1

Ultimately, our framework provides a theoretical foundation for analyzing importance-weighted estimators, both within and well beyond VI. Applications outside of VI include any setting where an importance-weighted estimator is employed to approximate a quantity of interest, e.g., the fine-tuning of Large Language Models \citep{liao2026iris}. \looseness=-1

\begin{table}[b!]
    \centering
    \caption{Closely-related results subsumed by \Cref{thm:ratios-limits-optimal-cond} under less stringent conditions}
    \label{tab}
    \renewcommand{\arraystretch}{1.2} 
    \begin{tabular}{@{} lll @{}} 
        \toprule
        \cite{rainforth2018tighter} & \eqref{eq:ratios-limits-optimal-cond2bis:alter} and \eqref{eq:limVarFirstOrderMuOne} with stronger moment conditions on $(X,W)$   \\
        \& \cite{daudel2022Alpha} & and problem-specific choices of $(X,W)$ \\
        \midrule
        \cite{deligiannidis25} & \eqref{eq:ratios-limits-optimal-cond2bis:alter} with stronger moment condition on $| \mathbb{E}(W)X - \mathbb{E}(X)W | W$   \\
        \bottomrule
    \end{tabular}
\end{table}

\subsection{Asymptotics of gradient estimators of the VR-IWAE bound as $N \to \infty$}
\label{sec:behavior-gradient-vr-Nasymp}

Let us now build a unified framework that enables the comparison between the REP \eqref{estimREP} and DREP \eqref{estimDREP} gradient estimators of the VR-IWAE bound. The structure of the REP and DREP gradient estimators suggests that the asymptotic behavior in expectation and variance of these estimators can be derived by applying \Cref{thm:ratios-limits-optimal-cond} for selected choices of $\overline{X}_N, \overline{W}_N$ and $\mu$. 
Letting $\alpha \in [0,1)$, $\varepsilon \sim \REPq$,
$\varepsilon_1,\varepsilon_2, \ldots$ be i.i.d. copies of
$\varepsilon$ and denoting \looseness=-1
\begin{align} \label{eq:tw}
\tw = \w[\theta][\phi](f(\varepsilon, \phi'; x); x), \quad \theta\in\Theta\,,\;\phi,\phi' \in\Phi,
\end{align}
we notably see that setting
$X = (1-\alpha)^{-1} \partial_\psi
\lrb{\tw[\phi][\varepsilon][\phi]^{1-\alpha}}$,
$W = \tw[\phi][\varepsilon][\phi]^{1-\alpha}$ and $\mu = 1$ in
\Cref{thm:ratios-limits-optimal-cond} permits to capture the
asymptotics of the REP gradient estimator in expectation and variance. Starting for instance with \eqref{eq:limExp}, we directly get that, under the assumptions of \ref{item:ratios-limits-optimal-cond1} in
\Cref{thm:ratios-limits-optimal-cond}, \looseness=-1
\begin{align} \label{eq:limExpREP}
\lim_{N\to\infty}\PE\lr{\gradREP}= \frac{1}{1-\alpha} \frac{\PE( \partial_\psi \lrb{\tw[\phi][\varepsilon][\phi]^{1-\alpha}})}{\PE( \tw[\phi][\varepsilon][\phi]^{1-\alpha})}.
\end{align}
While \eqref{eq:limExpREP} provides an asymptotic limit, its practical implications are not immediately transparent. The same conclusion holds for results obtained by substituting the above choice for $(X, W, \mu)$ into \Cref{thm:ratios-limits-optimal-cond}. The remainder of this subsection is then devoted to deriving meaningful insights on the behavior of the REP and DREP gradient estimators of the VR-IWAE bound by leveraging \Cref{thm:ratios-limits-optimal-cond} as a theoretical foundation. 
To that end, we will rely on the following assumption, which corresponds to a rewriting of the condition \eqref{eq:simple-cond-ratios} from \Cref{thm:ratios-limits-optimal-cond} for $W = \tw[\phi][\varepsilon][\phi]^{1-\alpha}$ that does not depend on $\alpha$. 
\begin{hyp}{A} 
\item \label{hyp:inverseGrad} 
There exists $\delta>0$ such that $\sup_{t>0} \lr{t^{-\delta} \PP(\tw[\phi][\varepsilon][\phi]\leq
    t)} < \infty$.
\end{hyp}
The theorem below then analyses the behavior in expectation of the REP gradient estimator \eqref{estimREP} as $N \to \infty$. \looseness=-1
\begin{thm} \label{thm:GradNStudyAllalpha}
Assume~\ref{hyp:inverseGrad} and \ref{hyp:hypZeroREP}. Then, as $N\to\infty$,
\begin{align}
\PE(\gradREP) = \partial_{\psi} \mathrm{VR}^{(\alpha)}(\theta, \phi; x) + o \lr{1}. \label{eq:asymptoticalGradientThetaOneFormer}
\end{align}
Further assuming that $\mathbb{V}\lr{\tw[\phi][\varepsilon][\phi]^{1-\alpha}} < \infty$ yields: as $N \to \infty$,
\begin{align}
\PE(\gradREP) = \partial_{\psi} \mathrm{VR}^{(\alpha)}(\theta, \phi; x) - \frac{1}{2 N} \partial_{\psi} [\gammaA(\theta,\phi; x)^2] + o \lr{\frac{1}{N}}.
\label{eq:asymptoticalGradientThetaOne}
\end{align}
\end{thm}
The proof of \Cref{thm:GradNStudyAllalpha} is deferred to
\Cref{app:proofGradNAllalpha}. By \Cref{thm:GradNStudyAllalpha}, $\PE(\gradREP)$ thus converges to $\partial_\psi \mathrm{VR}^{(\alpha)} (\theta, \phi;x)$ at a fast $1/N$ rate, where $\mathrm{VR}^{(\alpha)}(\theta, \phi; x)$ is the VR bound \eqref{eq:defVRbound}. We now discuss the methodological implications of \Cref{thm:GradNStudyAllalpha}. \looseness=-1

\Cref{thm:GradNStudyAllalpha} predicts that the REP gradient estimator of the VR-IWAE bound points on average in the direction which maximizes the VR bound when $\partial_\psi \mathrm{VR}^{(\alpha)} (\theta, \phi;x) \neq 0$ and which minimizes $\gammaA(\phi;x)^2$ when $\partial_\psi \mathrm{VR}^{(\alpha)} (\theta, \phi;x) = 0$. In both cases, the REP gradient estimator of the VR-IWAE bound targets a well-founded direction that is expected to improve on ELBO maximization. 

Indeed, the VR bound with $\alpha \in (0,1)$ is known to penalize under-dispersion, driving $q_\phi(\cdot|x)$ to capture regions of posterior probability mass that the standard ELBO typically ignores \citep{li2016renyi,margossian2024variationalinferenceuncertaintyquantification,minka2005divergence,Bui2016BlackboxF,margossian2025generalized}. A similar conclusion applies when $\alpha = 0$, with the particularity that $\mathrm{VR}^{(0)} (\theta, \phi;x) = \ell(\theta; x)$ so that the REP gradient estimator of the VR-IWAE bound points on average in the direction which maximizes the marginal log-likelihood for $\psi \in \{ \theta_1, \ldots, \theta_b \}$ and minimizes $\gammaA[0](\theta, \phi;x)^2 = \mathbb{V} \lr{ {p_\theta(z|x)}/{q_\phi(z|x)}}$ for $\psi \in \{ \phi_1, \ldots, \phi_b \}$ \citep{rainforth2018tighter}. \looseness=-1

Since $\PE(\gradREP) = \partial_{\psi} \liren(\theta, \phi;x)$ under \ref{hyp:hypZeroREP}, \Cref{thm:GradNStudyAllalpha} provides
an asymptotic result for the gradient of the VR-IWAE bound which is coherent with the one established in \cite[Theorem 3]{daudel2022Alpha} and reviewed in \eqref{eq:OneOverNGenDomke}.  
Note also that, assumption-wise, \ref{hyp:inverseGrad} is inherited from \Cref{thm:ratios-limits-optimal-cond},  \ref{hyp:hypZeroREP} permits to leverage the reparameterization trick assumption and $\mathbb{V}\lr{\tw[\phi][\varepsilon][\phi]^{1-\alpha}} < \infty$ ensures that $\gammaA[0](\theta, \phi;x)^2$ is well-defined. \looseness=-1

Our next step to further comprehend the behavior of the REP gradient
estimator \eqref{estimREP}  is to capture the behavior of its Signal-to-Noise ratio
(SNR) as $N \to \infty$. Recalling that for a random variable $X$ the
SNR is given by
$\mathrm{SNR}[X] = {|\PE(X)|}/{\sqrt{\mathbb{V}(X)}}$, an
unbiased gradient estimator is expected to be accurate if the SNR is
high (so that the target dominates the additive stochastic error) and
noisy otherwise. Since the summands of the $\sum_m$ summations
in~(\ref{estimREP}) are i.i.d., we have
$\mathrm{SNR}[\gradREP]=\sqrt{M}\,\mathrm{SNR}[\gradREP[1,N]]$ and we
can interpret the $\sum_m$ summations as a way to improve the SNR by
the factor $\sqrt{M}$ through averaging, while multiplying the
computational cost by $M$. As such, it suffices to consider the
setting $M=1$ in the following and to focus on the role of
$(N,\alpha)$ in $\mathrm{SNR}[\gradREP[1,N]]$. We then have the
theorem below. \looseness=-1 

\begin{thm} \label{thm:SNR-REP} Assume \ref{hyp:inverseGrad} and \ref{hyp:hypZeroREP}. Further assume that $\mathbb{V}\lr{\tw[\phi][\varepsilon][\phi]^{1-\alpha}} < \infty$ and that $\mathbb{V}\lr{{\partial_{\psi} \lr{\tw[\phi][\varepsilon][\phi]^{1-\alpha}}}} < \infty$. Then, as $N\to\infty$, \looseness=-1
\begin{align} \label{eq:SNR_REP_N_Two}
& \mathrm{SNR}[\gradREP[1,N]] =
\sqrt{N} \, \dfrac{\lrav{\partial_{\psi}
\mathrm{VR}^{(\alpha)}(\theta, \phi; x) - \frac{1}{2N}
\partial_{\psi} [\gammaA(\theta,\phi;
x)^2]+o\lr{\frac1N}}}{\sqrt{\vREP(\theta, \phi;x)}+o\lr{1}}\;, \\
& \label{eq:SNR_REP_N_TwovREP-def} \text{where} \quad \vREP(\theta, \phi;x) =
\frac1{(1-\alpha)^{2}} \mathbb{V}\lr{\partial_{\psi}
\lr{\frac{\tw[\phi][\varepsilon][\phi]^{1-\alpha}}{\PE(\tw[\phi][\varepsilon][\phi]^{1-\alpha})}}}\;.
\end{align}
\end{thm}
The proof of \Cref{thm:SNR-REP} is deferred to
\Cref{app:thm:SNR-REP}. Besides $\partial_{\psi} \mathrm{VR}^{(\alpha)}(\theta, \phi; x)$ and $\partial_{\psi} [\gammaA(\theta,\phi; x)^2]$, which we had already identified as key quantities in \Cref{thm:GradNStudyAllalpha} to analyze the REP gradient estimator \eqref{estimREP}, it is now clear that $\vREP(\theta, \phi; x)$ also plays an important role in the success of this estimator. Note that compared to \Cref{thm:GradNStudyAllalpha}, \Cref{thm:SNR-REP} further requires $\mathbb{V}\lr{{\partial_{\psi} \lr{\tw[\phi][\varepsilon][\phi]^{1-\alpha}}}} < \infty$ so that $\vREP(\theta, \phi; x)$ is well-defined. \looseness=-1

Asymptotic results related to \eqref{eq:SNR_REP_N_Two} appear in \cite{rainforth2018tighter} and in the proof of \cite[Theorem~1]{daudel2022Alpha}, albeit under more restrictive assumptions (see \Cref{app:relWThm1}). For the specific case $\alpha = 0$,
\cite{rainforth2018tighter} establish that the SNR scales as $\mathcal{O}(\sqrt{N})$ when learning $\theta$ and decays at a rate of $\mathcal{O}(1/\sqrt{N})$ when learning $\phi$. Furthermore, they identify the quantities $\partial_{\psi} \mathrm{VR}^{(0)}(\theta, \phi; x)$, $\partial_{\psi} [\gammaA[0](\theta,\phi; x)^2]$ and $\vREP[0](\theta, \phi; x)$ in their SNR result. \cite{daudel2022Alpha} show that the SNR decay for $\phi$ can be mitigated by setting $\alpha \in (0,1)$, but the quantities $\partial_{\psi} \mathrm{VR}^{(\alpha)}(\theta, \phi; x)$, $\partial_{\psi} [\gammaA(\theta,\phi; x)^2]$ and $\vREP(\theta, \phi; x)$ do not appear explicitly in their SNR result nor are they interpreted. \Cref{thm:SNR-REP} thus generalizes the approach of \cite{rainforth2018tighter} to the entire range $[0, 1)$, and operates under strictly weaker assumptions than those required by the two aforementioned previous works. This in turn yields important implications regarding the role of $\alpha$ across the entire range $[0,1)$. \looseness=-1

More precisely, when $\partial_{\psi} \mathrm{VR}^{(\alpha)}(\theta, \phi; x) \neq 0$ (which is a reasonable assumption to make unless  $\alpha = 0$ and $\psi$ is among the components of $\phi=(\phi_1, \ldots, \phi_b)$ or we are at a local optimum for $(\theta, \phi)$), \eqref{eq:SNR_REP_N_Two} simplifies to 
\begin{align*} 
\mathrm{SNR}[\gradREP[1,N]] =
\sqrt{N} \, \dfrac{\lrav{\partial_{\psi}
     \mathrm{VR}^{(\alpha)}(\theta, \phi; x)+o\lr{1}}}{\sqrt{\vREP(\theta, \phi;x)}+o\lr{1}}\;,
   \end{align*}
which highlights a bias-variance tradeoff at play in the REP gradient estimator of the VR-IWAE bound via the choice of the hyperparameter $\alpha$. Here, the gradient of the VR bound represents the bias-controlling component, while $\vREP(\theta, \phi; x)$ is the variance term. As $\alpha$ decreases from $1$ to $0$, the VR bound interpolates monotonically from the ELBO ($\alpha \to 1$) toward the marginal log-likelihood ($\alpha = 0$) (a monotonicity property recovered, for instance, by taking $N \to \infty$ in \eqref{eq:mono:alpha}). While this monotonicity property of the VR bound suggests a preference for values of $\alpha$ near zero in marginal likelihood optimization, the practical choice of $\alpha$ is in fact constrained by the simultaneous evolution of the variance term $\vREP(\theta, \phi; x)$. Consequently, \Cref{thm:SNR-REP} formally justifies how the tuning of $\alpha$ has led to improved empirical performance in tasks seeking to optimize the marginal log-likelihood using the REP gradient estimator \eqref{estimREP} \citep[see, e.g.,][]{li2016renyi,daudel2022Alpha}. \looseness=-1

Crucially, $\vREP(\theta, \phi; x)$ can be further interpreted in light of \cite{roeder2017,Tucker2019DoublyRG}: the intuition that the score function adds variance to REP gradient estimators of the IWAE bound ($\alpha = 0$) is now supported theoretically for the general REP gradient estimator \eqref{estimREP} ($\alpha \in [0,1)$) by looking at the asymptotics of its SNR as $N \to \infty$. The score function adds variance via $ \lrder{\partial_{\psi'} \log q_{\phi'}(f(\varepsilon, \phi;x)|x)}{\phi'
    = \phi}$, which can notably be observed by considering the case $p_\theta(\cdot|x) = q_{\phi}(\cdot|x)$ in \eqref{eq:SNR_REP_N_TwovREP-def} leading to (see \eqref{rem:VarREPequalapp} of \Cref{app:thm:SNR-REP} for details): \looseness=-1
\begin{align} \label{eq:VREPoptim}
  \vREP(\theta, \phi;x) = \mathbb{V}\lr{  \lrder{\partial_{\psi'} \log q_{\phi'}(f(\varepsilon, \phi;x)|x)}{\phi' = \phi}}.
  \end{align}
Although the DREP gradient estimator \eqref{estimDREP} was designed to reduce the variance of the REP one \eqref{estimREP}, a theoretical justification of this property is currently lacking. We therefore next derive the asymptotic SNR of the DREP estimator to formally compare the REP and DREP approaches. Since the DREP gradient estimator \eqref{estimDREP} only impacts the learning of $\phi$, our analysis focuses on its asymptotic behavior w.r.t. $\phi$. This leads us to the theorem below.

\begin{thm} \label{lem:SNR} Assume \ref{hyp:inverseGrad} and
  \ref{hyp:hypZeroDREP}. Futher assume that
  $\mathbb{V}\lr{\tw[\phi][\varepsilon][\phi]^{1-\alpha}} < \infty$
  and $\mathbb{V}\lr{{[\partial_{\psi'} (\tw^{2(1-\alpha)})]|_{\phi' =
        \phi}}} < \infty$, where $\psi$ denotes a component of the $\mathbb{R}^{b}$-valued variable $\phi=(\phi_1,\dots,\phi_b)$. 
The following assertions hold.
\begin{enumerateList}
\item\label{item:drep-snr-alpha1} If $\alpha \in (0,1)$ and $\mathbb{V}\lr{{\lrder{\partial_{\psi'} \lr{\tw^{1-\alpha}}}{\phi' = \phi}}}<\infty $, then, as $N\to\infty$,
      \begin{align}\label{eq:lem:SNR1}
&        \mathrm{SNR}[\gradDREP[1,N]] = \sqrt{N}~\frac{\lrav{\partial_{\psi} \mathrm{VR}^{(\alpha)}(\theta, \phi; x)+o\lr{1}}}{\sqrt{\genvDREP(\theta, \phi;x)}                                         +o(1)} \\ 
      \label{eq:lem:SNR1:var}& \mbox{with }\genvDREP(\theta, \phi;x) = \frac{\alpha^2}{(1-\alpha)^2} \mathbb{V}
\lr{\lrder{\partial_{\psi'} \lr{\frac{\tw^{1-\alpha}}{\PE(\tw^{1-\alpha})}}}{\phi' = \phi}}\;.
      \end{align}
    \item\label{item:drep-snr-alpha0} If $\alpha = 0$, then, as
      $N\to\infty$, 
       \begin{align}\label{eq:lem:SNR0}
&     \mathrm{SNR}\lrb{\gradDREP[1,N][\psi][0]} = \sqrt{N}~\frac{\frac{1}{2} \partial_{\psi} [\gammaA[0](\theta,\phi;
x)^2]+o\lr{1}}{\sqrt{\genvDREP[0](\theta, \phi;x)}+o(1)}
                                          , \\
& \mbox{with } \genvDREP[0](\theta, \phi;x) =  \mathbb{V} \left(\frac{
  \lrder{\partial_{\psi'} \lr{\tw^{2}}}{\phi' =
    \phi}}{2\,\PE(\tw[\phi][\varepsilon][\phi])^{2}} \right. \label{eq:lem:SNR0:var} \nonumber \\
& \qquad \qquad \qquad \qquad \qquad \qquad \left. - \frac{\tw[\phi][\varepsilon][\phi]\,\PE\lr{\lrder{ \partial_{\psi'} \lr{\tw^{2}}}{\phi' =
                         \phi}}}{\PE(\tw[\phi][\varepsilon][\phi])^3}\right). 
       \end{align}
\end{enumerateList}
\end{thm} 
The proof of \Cref{lem:SNR} is deferred to \Cref{app:lem:SNR}. Similarly to \Cref{thm:SNR-REP} for the REP gradient estimator \eqref{estimREP}, \Cref{lem:SNR} highlights a bias-variance tradeoff in the DREP gradient estimator \eqref{estimDREP} through the choice of the hyperparameter $\alpha \in [0,1)$. 

However, a critical distinction arises when comparing the two quantities $\vREP(\theta, \phi;x)$ and $\genvDREP(\theta, \phi;x)$. Unlike $\vREP(\theta, \phi;x)$, which remains non-zero even when the variational approximation matches the true posterior ($p_\theta(\cdot|x) = q_{\phi}(\cdot|x)$), we get that $\genvDREP(\theta, \phi;x) = 0$ in that case. \Cref{thm:SNR-REP,lem:SNR} thus suggest choosing the DREP gradient estimator over the REP one, aligning with the original intuitions of \cite{Tucker2019DoublyRG,roeder2017}. From an optimization perspective, the property that $\genvDREP$ vanishes near the optimum, hence driving the asymptotic variance of the DREP gradient estimator \eqref{estimDREP} to zero, is referred to as the interpolation condition. When this condition holds, the convergence rate of SGA is known to improve significantly \citep{vaswani2019fast}.

To the best of our knowledge, \Cref{lem:SNR} is the first asymptotic result in importance-weighted VI which supports the use of DREP gradient estimators over REP gradient estimators. While the asymptotic $\sqrt{N}$ rate for the SNR of the DREP gradient estimator of the IWAE bound ($\alpha = 0$) was in fact informally derived in \cite[Section 8.2]{Tucker2019DoublyRG}, the focus of \cite{Tucker2019DoublyRG} was on how the $\sqrt{N}$ rate of the DREP gradient estimator improved on the $\sqrt{1/N}$ rate of the REP gradient estimator when learning $\phi$ thanks to a variance reduction phenomenon. In addition, no conditions ensuring that the informal asymptotic analysis written in \cite{Tucker2019DoublyRG} held were stated. 
  
  In contrast, our approach permits us to show under mild conditions
  that the SNR rate scales like $\sqrt{N}$ for both the REP and DREP
  gradient estimators when $\alpha \in [0,1)$ (except when $\alpha =
  0$ and $\psi \in \{ \phi_1, \ldots, \phi_b \}$ in the REP gradient
  estimator). Importantly, we go beyond the variance reduction
  phenomenon that originates from increasing $N$ by taking a closer
  look at the constants appearing asymptotically in the SNR of the REP
  and DREP gradient estimators. We obtain that within a same class of
  gradient estimators (REP or DREP), the tuning of $\alpha$ leads to
  further variance reduction in parallel to increasing $N$, at the
  cost of potentially increasing the bias. Furthermore, $\vREP(\theta,
  \phi;x)$ and $\genvDREP(\theta, \phi;x)$ are quantities that enable
  the comparison between the REP and DREP gradient estimators, with
  theoretical evidence suggesting that the latter is expected to
  outperform the former. As an aside, observe that all the
  results we obtained extend naturally to the case $\alpha = 1$
  corresponding to the ELBO, albeit in a simpler non-asymptotic way
  (see \Cref{rem:alpha1REP,rem:alpha1DREP} in the appendix for more details). \looseness=-1

We now present two insightful examples where \Cref{thm:GradNStudyAllalpha,thm:SNR-REP,lem:SNR} apply and in which all the quantities of interest appearing in these results are analytically tractable. \looseness=-1

\begin{ex} \label{ex:Gaussian} Let $\theta, \phi \in \rset^d$. Set
  $p_\theta(z|x) = \mathcal{N}(z;\theta, \boldsymbol{I}_d)$ and
  $q_{\phi}(z|x) = \mathcal{N}(z; \phi, \boldsymbol{I}_d)$, where
  $\boldsymbol{I}_d$ is the $d$-dimensional identity matrix. Consider
  the reparameterization given by
  $Z = f(\varepsilon, \phi;x) = \varepsilon + \phi$ where
  $\varepsilon \sim \mathcal{N}(0, \boldsymbol{I}_d)$ and the goal is
  to learn~$\phi$. Then, we can apply
  \Cref{thm:GradNStudyAllalpha,thm:SNR-REP,lem:SNR} and all the terms
  appearing in these results are analytically tractable. In
  particular, to emphasize the impact
  of the latent space dimension $d$, we set
  $\theta-\phi = \epsilon \cdot \boldsymbol{u}_d$ where
  $\boldsymbol{u}_d$ denotes the $d$-dimensional vector whose
  coordinates are all equal to $1$ and $\epsilon > 0$. Then, we have that:
  for all $\alpha \in [0,1)$ and $k = 1 \ldots d$,
\begin{align*}
    & \PE(\gradREP[1,N][\phi_k]) = \epsilon \alpha +\frac{\epsilon(1-\alpha)\rme^{{(1-\alpha)^2 d \epsilon^2}}}{N}  + o\lr{\frac{1}{N}} \nonumber  \quad \mbox{as $N \to \infty$}, \\
   & \vREP[\alpha][\phi_k](\theta, \phi; x) = \rme^{ (1-\alpha)^2 d \epsilon^2} \lr{1 + (1-\alpha)^2 \epsilon^2}, \\
  & \genvDREP[\alpha][\phi_k](\theta, \phi; x) =
  \begin{cases}
  0  & \mbox{ if } \alpha \in (0,1), \\ \epsilon^2 \rme^{2 d \epsilon^2} \;\lr{\rme^{4 d \epsilon^2} - 4 \rme^{2 d \epsilon^2} + 4 \rme^{d \epsilon^2} - 1}  & \mbox{ if } \alpha = 0. 
  \end{cases}
\end{align*}
The derivation details for this example are deferred to \Cref{app:proof:ex:Gaussian}.
\end{ex}
\Cref{ex:Gaussian} showcases how $\alpha$ enables a bias-variance tradeoff for the REP and DREP gradient estimators of the VR-IWAE bound. For the REP gradient estimator, the variance term $\vREP[\alpha][\phi_k](\theta, \phi; x)$ is reduced by increasing $\alpha$ at the cost of increasing the bias in the gradient of the VR bound $\alpha \epsilon$. Viewing this result in terms of SNR, \Cref{thm:SNR-REP} yields: as $N \to \infty$, \looseness=-1
\begin{align*}
\mathrm{SNR}[\gradREP[1,N][\phi_k]] = \begin{cases}  
  \dfrac{ \sqrt{N} \epsilon \lr{\alpha \exp\lr{\frac{-(1-\alpha)^2 d \epsilon^2}{2}}}}{\sqrt{1+(1-\alpha)^2\epsilon^2}} (1+o(1)) & \mbox{if $\alpha \in (0,1)$}\\
    O \lr{1/\sqrt{N}} & \mbox{if $\alpha = 0$}.
\end{cases}
\end{align*}
Hence setting $\alpha \in (0,1)$ not only improves the asymptotic SNR rate in $N$, but increasing $\alpha$ within that range also improves the leading order constant. For the DREP gradient estimator, it is worth highlighting that $\genvDREP[\alpha][\phi_k](\theta, \phi;x)$ is $0$ for all $\alpha \in (0,1)$ as opposed to the case $\alpha = 0$, meaning that letting $\alpha \in (0,1)$ in this estimator improves the SNR rate in $N$ (note that the setting where $\genvDREP(\theta, \phi;x) = 0$ for certain values of $\alpha$ seems unlikely enough in practice to warrant a specific theoretical study of the leading order term in the SNR). \looseness=-1 

\begin{ex} \label{ex:LinGauss} Let $\theta \in \rset^d$,
  $\phi=(\tilde{a},\notationb) \in \rset^d \times \rset^d$ and
  $A=\mathrm{diag}(\tilde{a})$. Set
  $p_{\theta}(z)=\mathcal{N}(z;\theta, \boldsymbol{I}_d)$,
  $p_\theta(x|z)=\mathcal{N}(x;z, \boldsymbol{I}_d)$ and
  $q_{\phi}(z|x)=\mathcal{N}(z;Ax+\notationb, 2/3~\boldsymbol{I}_d)$
  as in \cite{rainforth2018tighter}. Consider the reparameterization
  given by $Z = \sqrt{\frac{2}{3}} \varepsilon + Ax +\notationb$,
  where $\varepsilon \sim \mathcal{N}(0, \boldsymbol{I}_d)$. Then, the
  assumptions from \Cref{thm:GradNStudyAllalpha,thm:SNR-REP,lem:SNR}
  are met and all the terms appearing in these results are
  analytically tractable. In particular, denoting
  $x = (x_1, \ldots, x_d)$, $\theta = (\theta_1, \ldots, \theta_d)$,
  $\notationb= (\notationb_1, \ldots, \notationb_d)$, letting
  $Ax+\notationb = (\theta+x)/2 + \epsilon \boldsymbol{u}_d$ with
  $\epsilon>0$, we have that: for all
  $\alpha \in [0,1)$  and $k = 1 \ldots d$, \looseness=-1
\begin{align*}
  & \PE(\gradREP[1,N][\notationb_k]) = -\frac{6\epsilon\alpha}{4-\alpha} - \frac{{24\epsilon(1-\alpha)(4-\alpha)^{d-1}} \rme^{\frac{24(1-\alpha)^2d\epsilon^2}{(5-2\alpha)(4-\alpha)}}}{N 3^{d/2}(5-2\alpha)^{\frac{d}{2} + 1}} + o\lr{\frac{1}{N}}, \quad \mbox{as $N \to \infty$} \\
  & \vREP[\alpha][\notationb_k](\theta, \phi; x) =  \frac{4(4-\alpha)^{d/2}}{(15-6\alpha)^{d/4}} \rme^{\frac{12(1-\alpha)^2}{(4-\alpha)(5-2\alpha)} d \epsilon^2} \sqrt{{\frac{2}{5-2\alpha} + \lr{\frac{12(1-\alpha) \epsilon}{(5-2\alpha)(4-\alpha)}}^2}} \\
  & \genvDREP[\alpha][\notationb_k](\theta, \phi;x) = \frac{\alpha^{2}}{16}~\vREP[\alpha][\notationb_k](\theta, \phi; x), \quad \alpha \in (0,1)
\end{align*}
and $\genvDREP[0][\notationb_k](\theta, \phi;x)$ is given by \eqref{eq:bkvDREPtwo} in \Cref{app:subsec:proof:ex:LinGauss}. The derivation details for this example are deferred to \Cref{app:subsec:proof:ex:LinGauss}.
\end{ex}
\Cref{ex:LinGauss} illustrates again the bias-variance tradeoff occurring in the REP and DREP gradient estimators as $\alpha$ varies. Furthermore, the reduction in asymptotic variance obtained by using the DREP gradient estimator over the REP one is remarkably simple to visualize when $\alpha \in (0,1)$: it is captured by a multiplicative factor ${\alpha^{2}}/{16}$, which itself translates into a direct improvement in terms of asymptotic SNR: as $N \to \infty$,
\begin{align*}
\mathrm{SNR}[\gradDREP[1,N][\notationb_k]] = {4}{\alpha}^{-1}~\mathrm{SNR}[\gradREP[1,N][\notationb_k]], \quad \mbox{$\alpha \in (0,1)$}.
\end{align*}

\subsection{Gradient estimators in deteriorating variational approximation regimes}
\label{sub:Collapse}

The asymptotic results of \Cref{sec:behavior-gradient-vr-Nasymp} are
obtained as $N\to\infty$ for a given model and fixed parameters
$\theta$ and $\phi$. However, how large $N$ needs to be in practice for the provided asymptotic behaviors to kick in depends on how big the mismatch between the target density and the variational density is. 

One way to see this is to keep $\alpha\in[0,1)$ fixed in
\Cref{ex:Gaussian,ex:LinGauss} and to focus on the impact of $d$ and
$\epsilon$ in those examples. Increasing $d$ or $\epsilon$ can be interpreted as worsening the mismatch between the posterior density $p_\theta(\cdot|x)$ and its variational approximation $q_\phi(\cdot|x)$: as $d$ or $\epsilon$ increase, the main terms in the predicted asymptotic behaviors of \Cref{ex:Gaussian,ex:LinGauss} vanish, and the variances will blow up. Consequently, as the mismatch between the posterior and variational densities grows, the sample size $N$ required to reach the asymptotic regimes predicted by \Cref{thm:GradNStudyAllalpha,thm:SNR-REP,lem:SNR} may become prohibitively large. This calls for a specific study of importance-weighted VI gradient estimators in settings where the variational approximation of the posterior density deteriorates. \looseness=-1

To that end, we consider in the following an asymptotic regime where $N\to\infty$ in such a way that the mismatch between the target density and the variational one increases in the meantime. Since this mismatch can be the result of various causes, we start by picking a criterion reflecting this mismatch, before making and motivating assumptions that enable us to carry out the analysis. Specifically, we propose to use the Kullback-Leibler (KL) divergence \looseness=-1
\begin{align}
\KL: = \ell(\theta;x) - \mathrm{ELBO}(\theta, \phi;x) \label{eq:decompELBO}
\end{align}
as a natural measure of mismatch and we seek to capture the behavior of importance-weighted VI gradient estimators as $N,\KL\to\infty$. Here, it is understood that the considered target and/or variational density depend on $N$ in any possible way, such as through $\theta$, $\phi$, or $x$ (for example via their dimensions) but with the constraint that both $N\to\infty$ and $\KL \to \infty$. This regime is highly challenging since the gap between the marginal log-likelihood and the ELBO (that is, the VR-IWAE bound with $N=1$) goes to infinity and the variational approximation thus severely deteriorates.

We now introduce two assumptions that enable a rigorous treatment of this regime for the REP gradient estimator \eqref{estimREP}. The first assumption is on the log-weight $\log \tw[\phi][\varepsilon][\phi]$: \looseness=-1
\begin{hyp}{B} 
\item \label{hyp:reparamHighDimGaussian} For $\varepsilon \sim \REPq$, the
  random field $\lr{\log
    \tw[\phi][\varepsilon][\phi]}_{(\theta,\phi)\in\Theta\times\Phi}$ is a Gaussian
  random field.
\end{hyp}
The second assumption, denoted \ref{hyp:hypZeroREPhighDim}, guarantees that we can interchange derivative and expectation signs as needed when \ref{hyp:reparamHighDimGaussian} also holds. It is postponed to \Cref{sec:gaussian-case-df} to not overload the paper with technical assumptions that can be overlooked in a first read and we concentrate instead on \ref{hyp:reparamHighDimGaussian}. Assumption \ref{hyp:reparamHighDimGaussian} permits us to use the convenient representation \looseness=-1 
\begin{align} \label{eq:lognormalReparm}
&  \log \tw[\phi][\varepsilon][\phi] = \PE_{\varepsilon \sim \REPq}(\log \tw[\phi][\varepsilon][\phi]) - \sqrt{\mathbb{V}_{\varepsilon \sim \REPq}\lr{\log \tw[\phi][\varepsilon][\phi]}} S(\varepsilon, \theta, \phi;x),
\end{align}
where $\lr{S(\varepsilon, \theta, \phi;x)}_{(\theta,\phi)\in\Theta\times\Phi}$ is a centered Gaussian random field with unit variance. Intuitively, as importance weights often asymptotically approach a log-normal distribution \citep[see, e.g.,][for an explanation via the Central Limit Theorem]{daudel2022Alpha}, this assumption is meant to capture the limiting behavior where the weights are exactly log-normal. Based on \eqref{eq:lognormalReparm}, we can also link our measure of mismatch to the variance of the log-weights as \looseness=-1
\begin{equation}
  \label{eq:gaussian-exp-moment-rep-proof}
 \frac12 \mathbb{V}_{\varepsilon \sim \REPq}\lr{\log \tw[\phi][\varepsilon][\phi]} = \ell(\theta;x) - \PE_{\varepsilon \sim \REPq}\lr{\log\tw[\phi][\varepsilon][\phi]} = \KL,
\end{equation}
which follows from \ref{hyp:reparameterized}, \eqref{eq:decompELBO} and the formula for exponential moments of Gaussian variables
$$
\PE_{\varepsilon \sim \REPq}\lr{\tw[\phi][\varepsilon][\phi]}=\rme^{\PE_{\varepsilon \sim \REPq}\lr{\log \tw[\phi][\varepsilon][\phi]}+\frac12 \mathbb{V}_{\varepsilon \sim \REPq}\lr{\log \tw[\phi][\varepsilon][\phi]}}.
$$
The representation \eqref{eq:lognormalReparm} will turn out to be key in our proofs since the REP gradient estimator involves differentiating $\log \tw[\phi][\varepsilon][\phi]$. We now investigate the asymptotic behavior of the REP gradient estimator \eqref{estimREP} under \ref{hyp:reparamHighDimGaussian}-\ref{hyp:hypZeroREPhighDim} as $N, \KL \to \infty$. This leads us to the next proposition.
\begin{prop} \label{thm:CollapseSNRGaussian}
  Assume \ref{hyp:reparamHighDimGaussian} and
  \ref{hyp:hypZeroREPhighDim}. Let $N,\KL\to\infty$ with
  \begin{align}
  & \label{eq:gaussian-high-dim-growing-cond}
  \log N= o\lr{\KL} \\
& \label{eq:SNRcondCollapse-corr}
                                                        \limsup \lrav{\corrREP} <1,
\end{align}
where \begin{align}
  \label{eq:def-corr-parameter-rep}
&\corrREP :=  \mathbb{C}\mathrm{orr}_{\varepsilon \sim \REPq}\lr{\log \tw[\phi][\varepsilon][\phi],\partial_{\psi}\log \tw[\phi][\varepsilon][\phi]} .
\end{align}
Then, it holds that
\begin{align} \label{eq:SNRcollapseGaussian-new-UB}
  \USNRREP=O\lr{\sqrt{\log N}\,\lrav{\corrREP}},
\end{align}
with
\begin{align}
\label{eq:def-useful-snr-rep}
\USNRREP:=\mathrm{SNR}\lrb{\gradREP[1,N]-\PE\lr{\gradREP[1,1]}}.
\end{align}
\end{prop}
The proof of \Cref{thm:CollapseSNRGaussian} is deferred to
\Cref{app:thm:gaussian-high-dim-gradient-snr}. \Cref{thm:CollapseSNRGaussian}
provides an upper bound on $\USNRREP$ as $N,\KL \to \infty$ under two
conditions: the first condition
\eqref{eq:gaussian-high-dim-growing-cond} ties $N$ to $\KL$ in such a
way that $N$ should not grow quicker than $\rme^{\KL}$, while the
second condition \eqref{eq:SNRcondCollapse-corr} is a mild condition
to avoid the degenerate case where the log likelihood and its
derivative tend to be colinear (by requiring these correlations to
stay away from $1$ and $-1$). 

Here, $\USNRREP$ captures whether the REP gradient estimator \eqref{estimREP} with $N>1$ behaves significantly differently from the case $N=1$. More precisely, it  evaluates how the difference $\PE(\gradREP[1,N])-\PE(\gradREP[1,1])$ compares to the stochastic error in the REP gradient estimator, that is, if the shift obtained by using $N > 1$ instead of $N = 1$ is large enough to dominate the noise of the REP gradient estimator. In the theorem below, we provide more precise results characterizing its asymptotic behavior.

\begin{thm} \label{thm:CollapseSNRGaussian-contd}
Assume \ref{hyp:reparamHighDimGaussian} and \ref{hyp:hypZeroREPhighDim}. Let $N,\KL\to\infty$ with~(\ref{eq:gaussian-high-dim-growing-cond}) and
\begin{equation}
  \label{eq:SNRcondCollapse-corr-stronger}
\sqrt{\log N}\,  \corrREP = O(1).
\end{equation}
Then, it holds that
\begin{align}
\label{eq:SNRcollapseGaussian-new}
  \USNRREP & =\sqrt{2\log N}\; |\corrREP|\lr{1+o(1)}.
\end{align}
If moreover, 
\begin{equation}
  \label{eq:SNRcondCollapse-corr-even-stronger}
\sqrt{\log N}\,  \corrREP = o(1),
\end{equation}
then we have
  \begin{equation}
    \label{eq:SNR-Equiv-as-old-times}
    \mathrm{SNR}\lrb{\gradREP[1,N]}=    \mathrm{SNR}\lrb{\gradREP[1,1]}\,\lr{1+o(1)}+o(1)\;.
  \end{equation}
\end{thm}
The proof of \Cref{thm:CollapseSNRGaussian-contd} can be found in \Cref{app:thm:CollapseSNRGaussian-contd}.  \Cref{thm:CollapseSNRGaussian-contd} shows that, as $N, \KL \to \infty$ with \eqref{eq:gaussian-high-dim-growing-cond} and under the condition \eqref{eq:SNRcondCollapse-corr-stronger} (which ensures \eqref{eq:SNRcondCollapse-corr} holds), the upper bound given in \eqref{eq:SNRcollapseGaussian-new-UB} is tight. In addition, it captures the asymptotic behavior of $\mathrm{SNR}[\gradREP[1,N]]$ under the stronger condition \eqref{eq:SNRcondCollapse-corr-even-stronger}. We now further delve into the interpretation of \Cref{thm:CollapseSNRGaussian-contd}. \looseness=-1

\paragraph{Assumptions made in \Cref{thm:CollapseSNRGaussian-contd}.} As previously mentioned, \eqref{eq:gaussian-high-dim-growing-cond} assumes that we are in a regime where $N$ does not grow faster than $\rme^{\KL}$. We next elucidate the conditions \eqref{eq:SNRcondCollapse-corr-stronger} and \eqref{eq:SNRcondCollapse-corr-even-stronger}. By appealing to the equivalent formulation of $|\corrREP|$ proved in \eqref{eq:rewritingCorr} of \Cref{subsec:addNotation}, we get that \eqref{eq:SNRcondCollapse-corr-stronger} and \eqref{eq:SNRcondCollapse-corr-even-stronger} are assumptions on the behavior of \looseness=-1
\begin{align*}
\sqrt{\log N} \lrav{\corrREP} = \sqrt{\frac{\log N}{ 2\KL}} \mathrm{SNR}\lrb{{\partial_{\psi} \log \lr{\frac{q_\phi(f(\varepsilon, \phi;x)|x)}{p_\theta(f(\varepsilon, \phi;x)|x)}}}} 
\end{align*}
Hence, \eqref{eq:SNRcondCollapse-corr-stronger} and \eqref{eq:SNRcondCollapse-corr-even-stronger} are conditions on ${\partial_{\psi} \log \lr{{q_\phi(f(\varepsilon, \phi;x|x))}/{p_\theta(f(\varepsilon, \phi;x)|x)}}}$, which is an estimator of $\partial_\psi \KL$. As a result, the assumptions made in \Cref{thm:CollapseSNRGaussian-contd} not only tie $N$ to $\KL$, but also to an estimator of $\partial_\psi \KL$: they quantify how the mismatch between the target density and the variational one evolves as a function of $N$. \looseness=-1

\paragraph{Interpreting \eqref{eq:SNRcollapseGaussian-new} and \eqref{eq:SNR-Equiv-as-old-times}.} The asymptotic result \eqref{eq:SNRcollapseGaussian-new} demonstrates that the optimization trajectory does not change significantly by increasing $N$ beyond $N =1$ in the REP gradient estimator \eqref{estimREP}, since $\USNRREP$ is bounded as $N$ grows (see \Cref{rem:practicalHighDim} in the appendix for a finer analysis). Furthermore, under the stronger condition \eqref{eq:SNRcondCollapse-corr-even-stronger}, $\mathrm{SNR}[\gradREP[1,N]]$ reverts to $\mathrm{SNR}[\gradREP[1,1]]$ by \eqref{eq:SNR-Equiv-as-old-times}. \looseness=-1

Hence, once the variational approximation deteriorates beyond the threshold of \eqref{eq:gaussian-high-dim-growing-cond}, \Cref{thm:CollapseSNRGaussian-contd} suggests that utilizing $N>1$ in importance-weighted VI methods imposes a computational burden without improving over the $N=1$ ELBO baseline. However, a key takeaway message from \Cref{thm:CollapseSNRGaussian-contd} is also that importance-weighted VI methods continue to yield valid gradient estimators for learning the parameters of interest $(\theta, \phi)$ even when the variational approximation deteriorates significantly. Remarkably, reverting to the $N = 1$ ELBO baseline in fact ensures that the optimization trajectory refines the variational approximation by minimizing $\KL$ w.r.t. $\phi$. Consequently, if the variational family is sufficiently expressive, $\KL$ is expected to decrease during training, thus lowering the exponential threshold $\rme^{\KL}$ and allowing a practical sample budget $N>1$ to become algorithmically viable again. We illustrate \Cref{thm:CollapseSNRGaussian-contd} by revisiting \Cref{ex:Gaussian}. \looseness=-1
\begin{ex}\label{ex:GaussianHighDim} Consider the setting of
  \Cref{ex:Gaussian} where $\theta-\phi = \epsilon \cdot
  \boldsymbol{u}_d$ with $\epsilon > 0$. Then, $\KL =
  \frac12\epsilon^2 d$, $\corrREP = 1/\sqrt{d}$ and,
  \eqref{eq:gaussian-high-dim-growing-cond} and
  \eqref{eq:SNRcondCollapse-corr-even-stronger} respectively become \looseness=-1
\begin{align} \label{eq:Nexpd}
  \log N\ll \epsilon^2 d\quad\text{and}\quad   \log N\ll d\;.
\end{align}
Now applying \Cref{thm:CollapseSNRGaussian-contd}: as $N,\epsilon^2d
\to \infty$, we have under \eqref{eq:Nexpd} that for all $\alpha \in [0,1)$ and $\psi \in \{\phi_1, \ldots, \phi_d\}$, 
\begin{align}  \label{eq:Nexpd-thm-applied}
  & \mathrm{SNR}[\gradREP[1,N]]= \epsilon \; (1+o(1)) + o(1). 
  \end{align}
We also have that $\mathbb{V}(\gradDREP[1,1]) = 0$. The derivation details for this example are deferred to
 \Cref{app:ex:GaussianHighDim}.
\end{ex}
As $N, \epsilon^2d \to \infty$, \Cref{ex:GaussianHighDim} states that
the condition \eqref{eq:Nexpd} amounts to assuming that the number of
samples $N$ does not grow quicker than exponentially with
$\rme^{\epsilon^2 d}$, in which case the REP gradient estimator with
$N>1$ reverts to the REP gradient estimator with $N=1$
(ELBO). Bypassing \eqref{eq:Nexpd} imposes a heavy computational
budget as the latent space dimension $d$ increases unless $\epsilon$
decreases to compensate, that is, the variational approximation
closely matches the posterior density.

Note that $\mathbb{V}(\gradDREP[1,1]) = 0$ for this example. As it
turns out, this property is also true if we consider a general
Gaussian model with known positive definite covariance matrix (see
\Cref{subsec:genGaussModel} for details). Such models represent the
main models of interest satisfying \ref{hyp:reparamHighDimGaussian}
and we argue that, although it is possible to study the DREP gradient
estimator as $N,\KL \to \infty$ for these models, it is better to skip
this study for the sake of conciseness in order to avoid lenghtly
derivations that are only applied to the corner case where
$\mathbb{V}(\gradDREP[1,1]) = 0$. Instead, we wrap up this subsection by putting our results into perspective with the existing literature and by discussing extensions of the assumption \ref{hyp:reparamHighDimGaussian}. \looseness=-1

\paragraph{Related work and broader impact.} The degradation of importance weights in high-mismatch regimes is a known phenomenon in the existing literature \citep{bengtsson2008curse}. Notably, \cite{agapiou2017importance,chatterjee2018sample} analyze the sample size $N$ leading to accurate estimation of self-normalized importance sampling (SNIS) estimators for bounded test functions. \cite{agapiou2017importance} obtain bounds on the bias and mean squared error of SNIS estimators suggesting $N$ should be exponential with the KL divergence between the proposal and target distributions. Provided that the log-weight is concentrated around its expected value, \cite{chatterjee2018sample} demonstrate that having $N$ of this order is necessary and sufficient to control the L1 error of SNIS estimators. \looseness=-1

While \cite{agapiou2017importance,chatterjee2018sample} establish the threshold below which SNIS estimators fail, these works provide no information on how this failure impacts the optimization trajectory of importance-weighted VI algorithms. Our work shows that the REP gradient estimator \eqref{estimREP} does not break down when $N = o(\rme^{\KL})$ but rather reverts to the ELBO case $N = 1$ and in particular minimizes $\KL$ w.r.t. $\phi$. Hence, $\KL$ is expected to decrease during training, thus reducing the sample budget $N>1$ needed to escape the weight collapse regime. In addition, our results are obtained without assuming that $\alpha = 0$ (SNIS setting) and that the test function (in our case $\partial_\psi \tw[\phi][\varepsilon][\phi]$) is bounded. In doing so, we fill a known theoretical gap in the VI literature \citep{GeffnerDomke2020Biased,vaitl2022gradients}. \looseness=-1

Relatedly, \cite{daudel2022Alpha} shows that the VR-IWAE bound reverts to the ELBO in the exact Gaussian random field case as well as in an instance of approximately Gaussian random field case. Our work departs from \cite{daudel2022Alpha} in two fundamental ways as (i) it shifts the focus from variational bounds to gradient estimators (ii) it considers the setting in which the variational approximation deteriorates and its deterioration manifests through the KL divergence $\KL$. In particular, moving from variational bounds to their gradient estimators has strong consequences in terms of proof techniques since we need to differentiate \eqref{eq:lognormalReparm} in the proofs of \Cref{thm:CollapseSNRGaussian} and \Cref{thm:CollapseSNRGaussian-contd}. To achive this, we rely on a novel bound which is key to show the collapse of self-normalized weighted averages (see \Cref{prop:new:collapse-behavior-general} in \Cref{app:prelim:collapse}). Due to the generality of \Cref{prop:new:collapse-behavior-general}, we anticipate that this bound will be of use beyond our current work. 

\paragraph*{Beyond the assumption \ref{hyp:reparamHighDimGaussian}.} A logical extension of \ref{hyp:reparamHighDimGaussian} is to assume that the first two moments of $\log \tw$ exist. We can then write without any loss of generality that
\begin{align} \label{eq:furtureWork}
\log \tw = \mathbb{E}_{\varepsilon \sim \REPq}(\log \tw) - \sqrt{\mathbb{V}_{\varepsilon \sim \REPq}(\log \tw)} \tilde{S}(\varepsilon, \theta, \phi, \phi';x),
\end{align}
where $(\tilde{S}(\varepsilon, \theta, \phi, \phi';x))_{(\theta,\phi,\phi') \in \Theta \times \Phi^2}$ is a centered and normalized random field. Under adequate conditions, the centered and normalized log likelihood ratio
$\tilde{S}(\varepsilon, \theta, \phi, \phi';x)$ is expected to be well approximated by a Gaussian random field. We conjecture that the gradient estimators $\gradREP[1,N]$ and $\gradDREP[1,N]$ will revert to the case $N=1$ under conditions akin to (and possibly stronger than) the ones from \Cref{thm:CollapseSNRGaussian-contd}. 

Adapting our results to this more general case in order to capture the behavior of both the REP and DREP gradient estimators requires further intricate derivations, which are left for future work. Indeed, our proofs build on the property that $S(\varepsilon, \theta, \phi; x)$ and its derivative $\partial_\psi S(\varepsilon, \theta, \phi; x)$ are two independent and centered Gaussian processes under \ref{hyp:reparamHighDimGaussian} (see \ref{item:propREPest-gaussian} in \Cref{subsec:addNotation}); the property \ref{item:propREPest-gaussian} does not hold anymore when we move away from \ref{hyp:reparamHighDimGaussian} to consider settings of the form \eqref{eq:furtureWork} where $\tilde{S}(\varepsilon, \theta, \phi, \phi';x)$ is well approximated by a Gaussian random field under adequate conditions. \looseness=-1

\subsection{Discussion and practical guidance}

Our approaches in \Cref{sec:behavior-gradient-vr-Nasymp,sub:Collapse} are complementary approaches that characterize different behaviors/phases occuring importance-weighted VI algorithms. 

On the one hand, \Cref{sec:behavior-gradient-vr-Nasymp} captures the setting where the variational approximation matches the posterior density reasonably well, in which case the benefits of importance-weighted VI methods are fully leveraged. For this setting, increasing $N$ leads to a variance reduction phenomenon and SNRs that scale with $\sqrt{N}$ (except for edge cases), with the tuning of $\alpha$ leading to further variance reduction. Furthermore, the gradient estimators $\gradREP[1,N]$ and $\gradDREP[1,N]$ point on average in the direction that is expected to improve on ELBO maximization.

On the other hand, \Cref{sub:Collapse} sheds light on the behavior of importance-weighted VI algorithms in the challenging regime where the
variational density is not suitably close to the posterior density in terms of KL divergence. While a weight collapse occurs when $N = o(\rme^{\KL})$ where $\KL$ can be large, this regime actively seeks to reduce $\KL$ by improving the quality of the variational approximation. \looseness=-1

These behaviors align with the two training phases hypothesized by \cite{vaitl2022gradients}: \Cref{sub:Collapse} describes the initial training phrase of importance-weighted VI algorithms and \Cref{sec:behavior-gradient-vr-Nasymp} corresponds to the final training phase near convergence. While \cite{vaitl2022gradients} treats these phases as decoupled, our analysis ties them together. Provided that the variational family is sufficiently expressive, the initial training phase indeed naturally drives the optimization trajectory toward the final one by lowering the exponential threshold $\rme^{\KL}$. From there, we can derive the following practical considerations:

\paragraph{Budget allocation.} Given a fixed total computational budget of $MN$ (where $M$ is the number of gradient steps), our results lead to a two-phase allocation strategy:

\begin{itemize}
  \item Initial training phase (high mismatch): Setting $N > 1$ imposes a strict computational penalty without altering the optimization trajectory. Therefore, the practitioner should set $N = 1$ to maximize $M$, allowing the model to traverse the parameter space as rapidly as possible to reduce $\KL$.
  
  \item Final training phrase (low mismatch): When the variational approximation is sufficiently close to the target in terms of KL divergence, the practitioner should increase the budget to $N > 1$. \looseness=-1
\end{itemize}

\paragraph{On-the-fly tuning of $N$ and $\alpha$.} Although $\KL$ cannot be calculated on-the-fly, practitioners can cheaply monitor the Effective Sample Size ($\mathrm{ESS}$) or the variance of the log-weights as a proxy. For instance, if the monitored $\mathrm{ESS}$ collapses toward $1$, it signals that the gradient estimators are suffering from a weight collapse. In this case, the practitioner can either automatically drop $N$ to $1$ to save computational effort, or dynamically tune $\alpha$ toward a value that balances the bias-variance tradeoff we identified in this paper. A detailed investigation of these tuning mechanisms is left for future work.

\paragraph{High dimensions.} Although high dimensionality typically initializes the variational approximation into the initial training phrase, the fundamental quantity in importance-weighted VI is the mismatch between the variational approximation and the posterior density, rather than the latent dimension $d$ itself. Our analysis suggests that, should the variational family be flexible enough, the variational approximation is improved during the optimisation procedure until it enters the first regime even in high dimensions. As we shall see now in our numerical experiments, one can perfectly emulate this high-dimensional behavior even within a moderate latent dimension $d$, so long as the variational approximation is sufficiently poor. \looseness=-1

\section{Numerical experiments}
\label{sec:numExp}

In this section, we provide empirical evidence supporting our theoretical claims. 

\subsection{Gaussian example}
\label{num:GaussEx}

We consider the Gaussian example from \Cref{ex:Gaussian}, for which \Cref{thm:SNR-REP} predicts the following asymptotic behavior of $\mathrm{SNR}[\gradREP[1,N][\phi_k]]$ for all $\alpha \in [0,1)$: as $N \to \infty$,
\begin{align*}
  \mathrm{SNR}[\gradREP[1,N][\phi_k]] =
   \dfrac{ \sqrt{N} \epsilon \lr{\alpha \exp\lr{\frac{-(1-\alpha)^2 d \epsilon^2}{2}} + \frac{1-\alpha}{N} \exp\lr{ \frac{(1-\alpha)^2d \epsilon^2}{2}}}}{\sqrt{1+(1-\alpha)^2\epsilon^2}} (1+o(1)). 
\end{align*}
To check the validity of this asymptotic result, we first study the case $\alpha \in (0,1)$. Consequently, we investigate whether the asymptotic behavior as $N \to \infty$ 
\begin{align}
  & \mathrm{SNR}[\gradREP[1,N][\phi_k]] =
   \dfrac{ \sqrt{N} \epsilon \alpha \exp\lr{\frac{-(1-\alpha)^2 d \epsilon^2}{2}}}{\sqrt{1+(1-\alpha)^2\epsilon^2}} (1+o(1)), \quad \mbox{$\alpha \in (0,1)$} \label{eq:predictThmtwo}
\end{align}   
matches with the observed $\mathrm{SNR}[\gradREP[1,N][\phi_k]]$ as $N$ increases (note that only the leading order term remains in \eqref{eq:predictThmtwo}). To this end, we let $\epsilon \in \{0.2, 1., 2\}$, $\alpha \in \{0.1, 0.3, 0.5, 0.7, 0.9\}$, $d \in \{10, 100\}$, $N \in \{2^j, j = 1 \ldots 14\}$ and $\phi_k$ be a random coordinate in $(\phi_1\ldots \phi_d)$. The results are shown in Figure \ref{fig:GaussianExample}.

\begin{figure}[t]
  
  \begin{tabular}{ccc} 
    \includegraphics[scale=0.28]{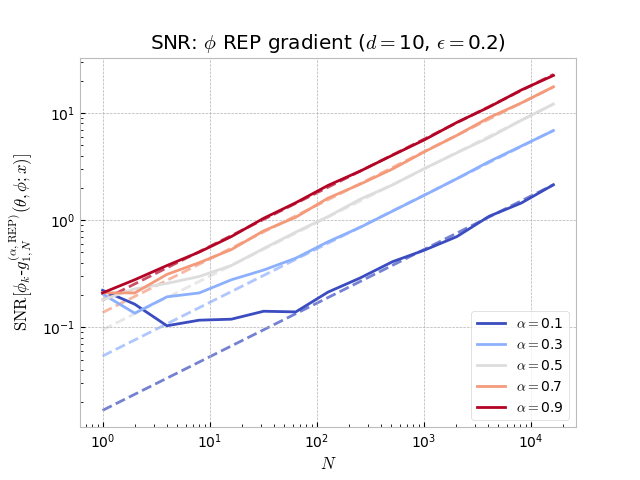} & \includegraphics[scale=0.28]{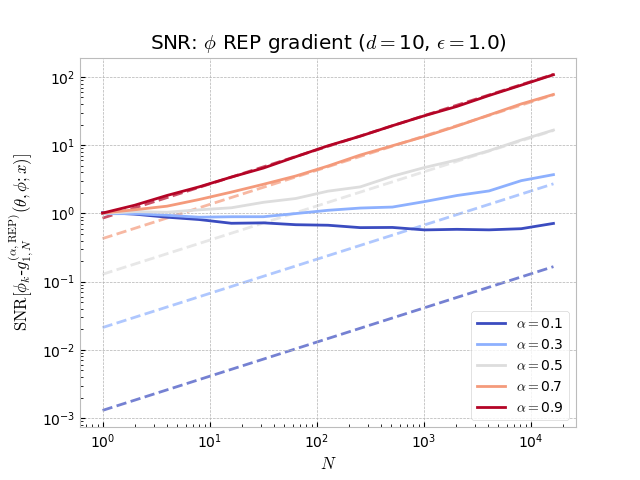}  & \includegraphics[scale=0.28]{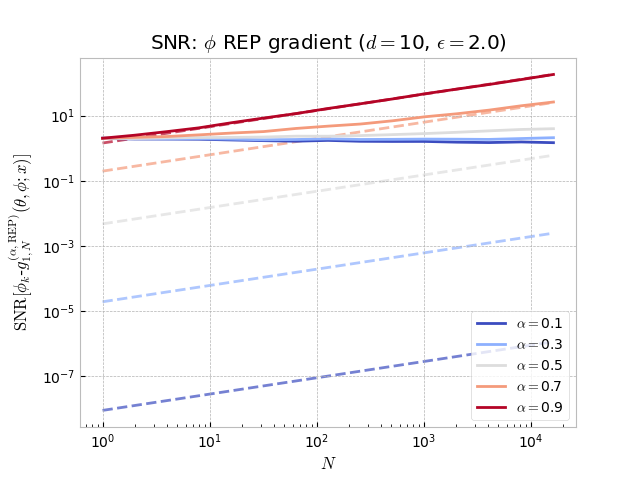}
\\
\includegraphics[scale=0.28]{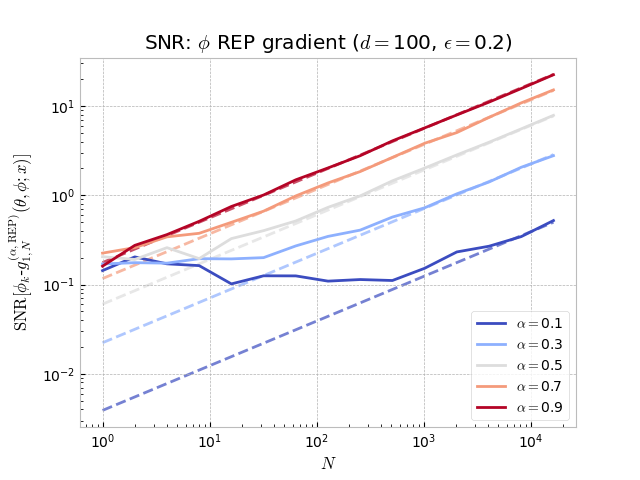} &
\includegraphics[scale=0.28]{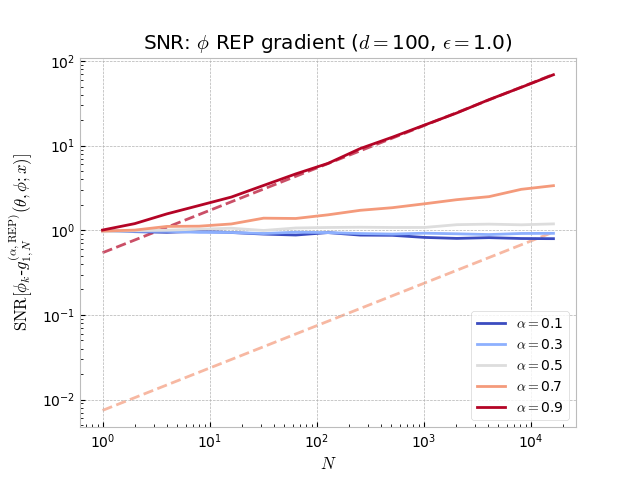} &
\includegraphics[scale=0.28]{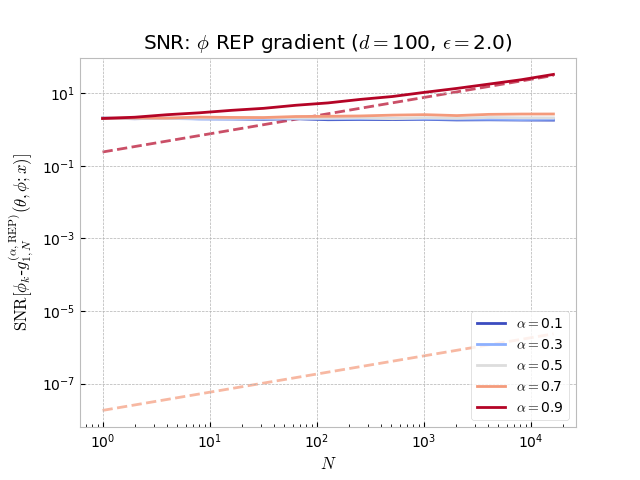}
  \end{tabular}
  \caption{Plotted is $\mathrm{SNR}[\gradREP[1,N][\phi_k]]$ computed over 2000 Monte Carlo samples for the Gaussian example described
  in \Cref{num:GaussEx} as a function of $N$, for varying values of $(\alpha,d, \epsilon)$ and a random coordinate $\phi_k$. The solid lines correspond to $\mathrm{SNR}[\gradREP[1,N][\phi_k]]$, the dashed lines correspond to predictions of the form \eqref{eq:predictThmtwo}. \label{fig:GaussianExample}}
\end{figure}

Observe first that, in the favourable setting of a low dimension $d$ with a small perturbation $\epsilon$ near the optimum (that is $(d,\epsilon) = (10,0.2)$), the asymptotic behavior predicted by \eqref{eq:predictThmtwo} for each $\alpha$ ends up matching the corresponding observed SNR as $N$ increases. In particular, the bias-variance tradeoff behavior highlighted in our theoretical analysis of \Cref{ex:Gaussian} is empirically confirmed for $N$ large enough, with the bias and variance of the REP gradient estimator increasing and decreasing with $\alpha$, respectively. In addition, the highest the value of $\alpha$ (and the lowest the value of $\epsilon$), the lowest $N$ needs to be before the observed SNR and the prediction made by \eqref{eq:predictThmtwo} coincide. This follows from the fact that the second-order term in the numerator of the SNR when $\alpha \in (0,1)$ and $\epsilon >0$ is expected to vanish quicker with $N$ as $\alpha$ increases and $\epsilon$ decreases.

As $d$ increases, $\mathrm{SNR}[\gradREP[1,N][\phi_k]]$ should revert to $\mathrm{SNR}[\gradREP[1,1][\phi_k]]$ according to \Cref{ex:GaussianHighDim} unless the number of samples $N$ is very large or $\epsilon$ is small, meaning that the asymptotic behavior \eqref{eq:predictThmtwo} will not reflect the observed SNRs anymore. This is indeed what we observe in Figure \ref{fig:GaussianExample} ($N = 10^0$ corresponds to $\mathrm{SNR}[\gradREP[1,1][\phi_k]] = \epsilon$), with the weight collapse occurring quicker the higher the value of $\epsilon$, that is the further the variational approximation is from the target posterior density. Interpreting the different values of $\epsilon$ as different stages of the training procedure, \Cref{ex:GaussianHighDim} sheds light on the role the flexibility of the variational family plays in the success of the REP gradient estimator: the advantages of using the REP gradient estimator with $N>1$ instead of with $N=1$ arise after an initial training period ensuring that the quality of the variational approximation increases. \looseness=-1

Additional results for the DREP gradient estimator and for the IWAE case $(\alpha = 0)$ leading to similar conclusions are provided in \Cref{app:num:GaussEx} for the sake of completeness. \looseness=-1 
\subsection{Linear Gaussian example}
\label{num:LinGaussEx}
We are next interested in the linear Gaussian example from \Cref{ex:LinGauss}. The data set $\mathcal{D} = \{x_1,...,x_T\}$ is generated by sampling $T = 1024$ datapoints from $\mathcal{N}(0,2 \boldsymbol{I}_d)$. We let $\theta = \theta^\star + 2 \epsilon$, $A = A^\star$ and $\notationb = \notationb^\star + 2 \epsilon$ with $\epsilon>0$, $\theta^\star = T^{-1} \sum_{t=1}^T x_t$, $A^\star = \frac{1}{2} \boldsymbol{I}_d$ and $\notationb^\star = \frac{1}{2} \theta^\star$, so that $Ax + \notationb = \frac{1}{2}(x+\theta) + \epsilon$. For all $\alpha \in [0,1)$ and all $k = 1 \ldots d$, \Cref{ex:LinGauss} thus states that \Cref{thm:SNR-REP,lem:SNR} predict the asymptotic behavior as $N \to \infty$ of $\mathrm{SNR}[\gradREP[1,N][\notationb_k]]$ and $\mathrm{SNR}[\gradDREP[1,N][\notationb_k]]$ respectively. \looseness=-1 

Let us now check the validity of the asymptotic results obtained in \Cref{ex:LinGauss}. Focusing first on the learning of $\notationb$ when $\alpha \in (0,1)$, this boils down to checking whether as $N$ increases the observed $\mathrm{SNR}[\gradREP[1,N][\notationb_k]]$ and $\mathrm{SNR}[\gradDREP[1,N][\notationb_k]]$ match with the asymptotic behaviors: as $N \to \infty$,
\begin{align}
     & \mathrm{SNR}[\gradREP[1,N][\notationb_k]] = ~\frac{\sqrt{N} \frac{3\epsilon\alpha}{4-\alpha} (1+o(1))}{\frac{(4-\alpha)^{d/2}}{(15-6\alpha)^{d/4}} \exp\lr{\frac{12(1-\alpha)^2}{(4-\alpha)(5-2\alpha)} d \epsilon^2} \sqrt{{\frac{2}{5-2\alpha} + \lr{\frac{12(1-\alpha) \epsilon}{(5-2\alpha)(4-\alpha)}}^2}}} \label{eq:LeadOrderSNRone} \\
     & \mathrm{SNR}[\gradDREP[1,N][\notationb_k]] = {4}{\alpha}^{-1}~\mathrm{SNR}[\gradREP[1,N][\notationb_k]] \label{eq:LeadOrderSNRtwo}
\end{align}
(note that once more only the leading order term remains in \eqref{eq:LeadOrderSNRone}). Now let $\epsilon \in \{0.2, 1\}$, $\alpha \in \{ 0.1, 0.3, 0.5, 0.7, 0.9 \}$, $d \in \{10, 100\}$, $N \in \{2^j, j = 1 \ldots 14 \}$ and $\notationb_k$ be a random coordinate in $(\notationb_1, \ldots, \notationb_d)$. The results are shown in Figures \ref{fig:LinGaussExamplePhibOne} and \ref{fig:LinGaussExamplePhibTwo}.

\begin{figure}[ht!]
\begin{center}\begin{tabular}{ccc} 
\includegraphics[scale=0.28]{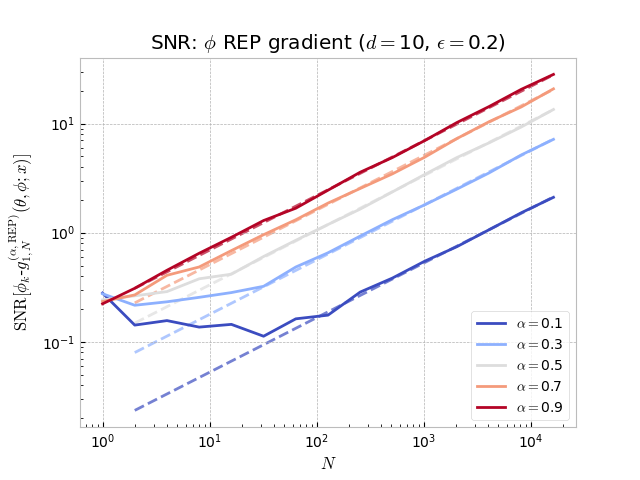} & 
\includegraphics[scale=0.28]{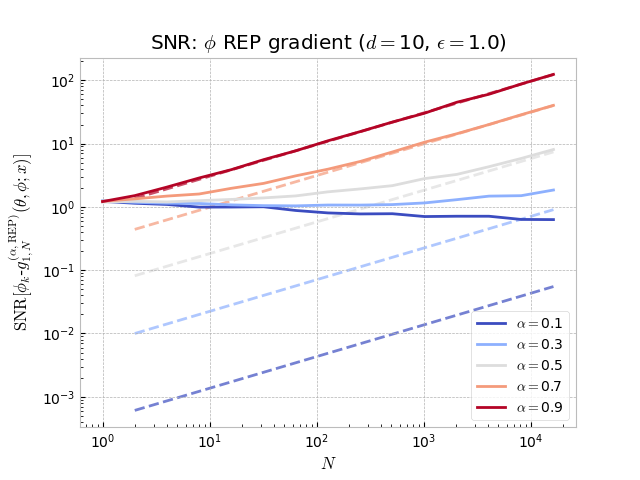} 
\\
\includegraphics[scale=0.28]{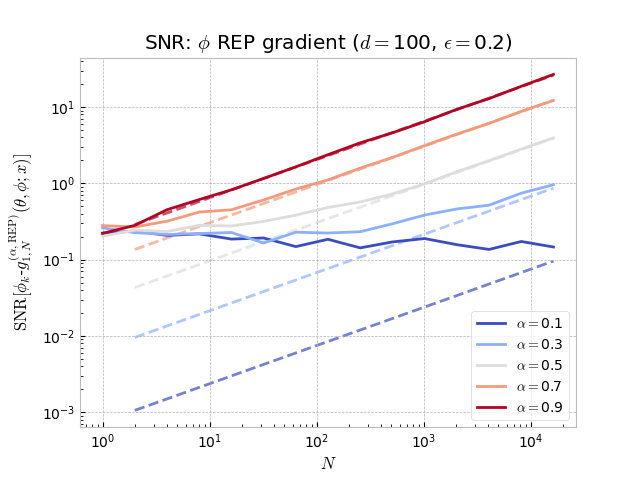}
& \includegraphics[scale=0.28]{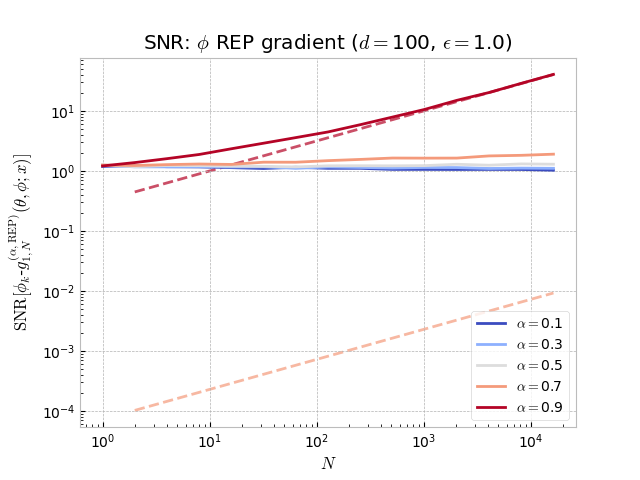}
\end{tabular}
\end{center}
  \caption{Plotted is $\mathrm{SNR}[\gradREP[1,N][\notationb_k]]$ computed over 2000 Monte Carlo samples for the Linear Gaussian example described
  in \Cref{num:LinGaussEx} as a function of $N$, for varying values of $(\alpha,d, \epsilon)$ and a randomly selected
  datapoint $x$. The solid lines correspond to $\mathrm{SNR}[\gradREP[1,N][\notationb_k]]$, the dashed lines correspond to predictions of the form \eqref{eq:LeadOrderSNRone}. \label{fig:LinGaussExamplePhibOne}}
\end{figure}

\begin{figure}[ht!]
  \begin{center}
  \begin{tabular}{cc} 
    \includegraphics[scale=0.28]{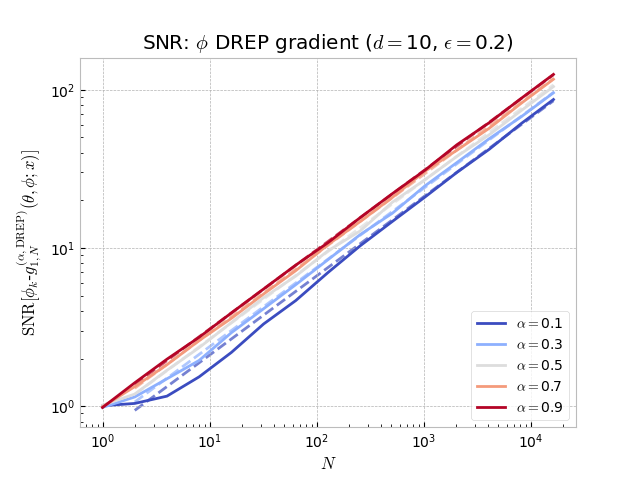} &
    \includegraphics[scale=0.28]{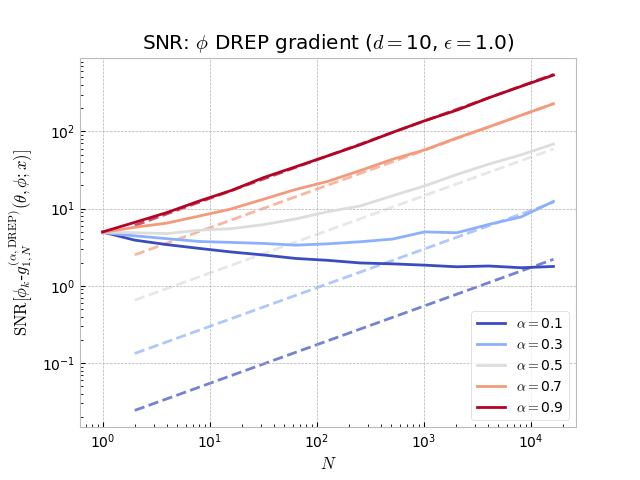}  \\
     \includegraphics[scale=0.28]{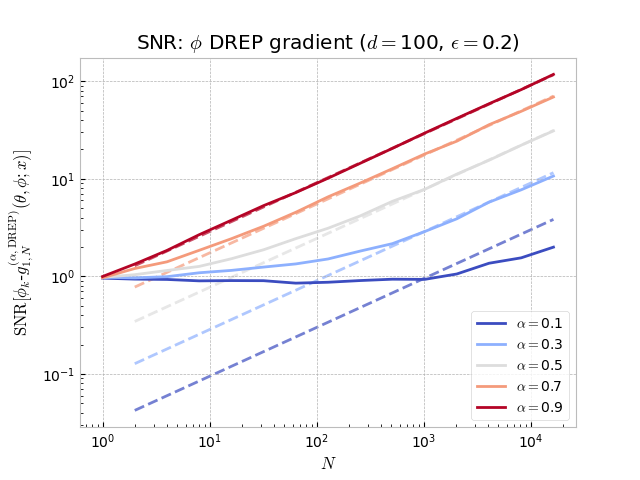} & \includegraphics[scale=0.28]{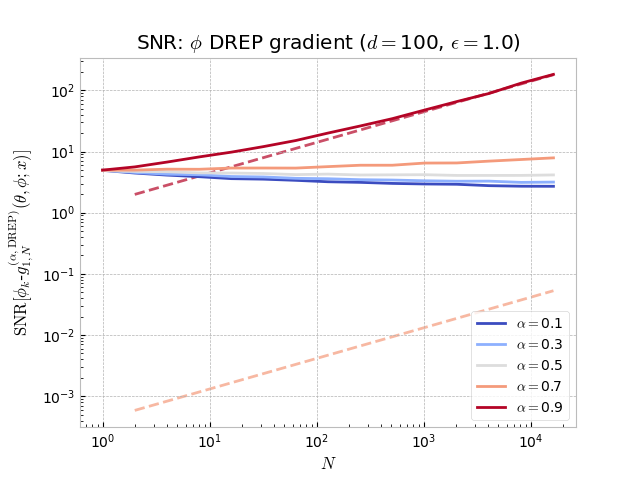}
  \end{tabular}
\end{center}
  \caption{Plotted is $\mathrm{SNR}[\gradDREP[1,N][\notationb_k]]$ computed over 2000 Monte Carlo samples for the Linear Gaussian example described in \Cref{num:LinGaussEx} as a function of $N$, for varying values of $(\alpha,d, \epsilon)$ and a randomly selected
  datapoint $x$. The solid lines correspond to $\mathrm{SNR}[\gradDREP[1,N][\notationb_k]]$, the dashed lines correspond to predictions of the form \eqref{eq:LeadOrderSNRone}. \label{fig:LinGaussExamplePhibTwo}}
\end{figure}

Starting with Figure \ref{fig:LinGaussExamplePhibOne}, we see that the conclusions drawn in \Cref{num:GaussEx} for the REP gradient estimator regarding the validity of the asymptotic behaviors predicted by \Cref{thm:SNR-REP,lem:SNR} when $\alpha \in (0,1)$ apply here too. Moving on to Figure \ref{fig:LinGaussExamplePhibTwo}, we have that these conclusions are also valid for the DREP gradient estimator, which notably confirms empirically that the DREP gradient estimator outperforms the REP one by a $4\alpha^{-1}$ factor in terms of SNR for the setting considered here. Additional plots in which 
we replace $\notationb_k$ by $\theta_k$ and/or we set $\alpha = 0$ are available in \Cref{app:num:LinGaussEx}. They further validate our conclusions regarding the empirical validity of \Cref{thm:SNR-REP,lem:SNR}.

As $d$ and $\epsilon$ increase, $\mathrm{SNR}[\gradDREP[1,N][\notationb_k]]$ reverts to $\mathrm{SNR}[\gradDREP[1,1][\notationb_k]]$ in Figure \ref{fig:LinGaussExamplePhibOne} unless $N$ is increasingly large and we observe a similar behavior for the DREP gradient estimator in Figure \ref{fig:LinGaussExamplePhibTwo}, which is coherent with our conjecture that an increasing mismatch between the posterior and variational densities takes precedence over the choice of the importance-weighted variational bound. This effect can to some extent be mitigated by increasing $\alpha$ or $N$, although we argue that improving on the quality of the variational approximation through the learning of $\phi$ is a more direct approach to fully leverage the advantages of importance-weighted VI gradient estimators. 

\subsection{Variational auto-encoder}
\label{num:exp:vae}

We consider the case of a variational auto-encoder (VAE) model designed to generate MNIST digits with a $d$-dimensional latent space, where $p_\theta(z)$ is a fixed standard Gaussian distribution, $p_\theta(x | z)$ is a product over the output dimensions of independent Bernoulli random variables with logits $\pi_\theta(z)$, $q_\phi(z | x) = \mathcal{N}(z; \mu_\phi(x), \sigma_\phi(x))$ and the functions $\pi_\theta(z)$ and $(\mu_\phi(x), \sigma_\phi(x))$ are parameterized by neural networks. More precisely, both the encoding and decoding networks are MLPs with two hidden layers of size 200 and $\tanh$ nonlinearities.

Let us then check the validity of our asymptotic results for this more involved real-data example. Focusing on the case $\alpha \in (0,1)$, \Cref{thm:SNR-REP} predicts that: as $N \to \infty$,
  \begin{align} \label{eq:vae}
    & \mathrm{SNR}[\gradREP[1,N]] =
     \sqrt{N} \, \dfrac{\lrav{\partial_{\psi}
     \mathrm{VR}^{(\alpha)}(\theta, \phi; x)+o(1)}}{\sqrt{\vREP(\theta, \phi;x)}+o\lr{1}}.
  \end{align}
  Unlike the two previous examples, we do not have a closed form expression for $\partial_{\psi} \mathrm{VR}^{(\alpha)}(\theta, \phi; x)$ and $\vREP(\theta, \phi;x)$ here and we will resort to using approximations instead. Letting $N' \in \mathbb{N}^\star$, we approximate $\partial_{\psi} \mathrm{VR}^{(\alpha)}(\theta, \phi; x)$ by $\gradREP[1,N']$ following \Cref{thm:GradNStudyAllalpha} and we  approximate $\vREP(\theta, \phi;x)$ by 
$$
{N'}~\sum_{i=1}^{N'} \lr{ \frac{\tw[\phi][\varepsilon_{1,i}][\phi]^{1-\alpha}}{\sum_{\ell = 1}^{N'} ~\tw[\phi][\varepsilon_{1,\ell}][\phi]^{1-\alpha}} \lrb{\partial_{\psi} \lr{\log \tw[\phi][\varepsilon_{1,i}][\phi]}-  \gradREP[1,N'] } }^2
$$
with $\varepsilon_{1,1}, \ldots, \varepsilon_{1,N'}$ being i.i.d. samples generated from $\REPq$. 
Letting $\alpha \in \{0.1, 0.3, 0.5, 0.7, 0.9 \}$, $d \in \{10, 100\}$, $N \in \{2^j, j = 1 \ldots 10\}$, $N'=2^{10}$, we then obtain Figure \ref{fig:VAEtheta} in which $\psi \in \{\theta_k, \phi_{k'}\}$ with $k$ and $k'$ that were randomly selected.

\begin{figure}[t] 
  \begin{center}
  \begin{tabular}{cc}
  \includegraphics[scale=0.28]{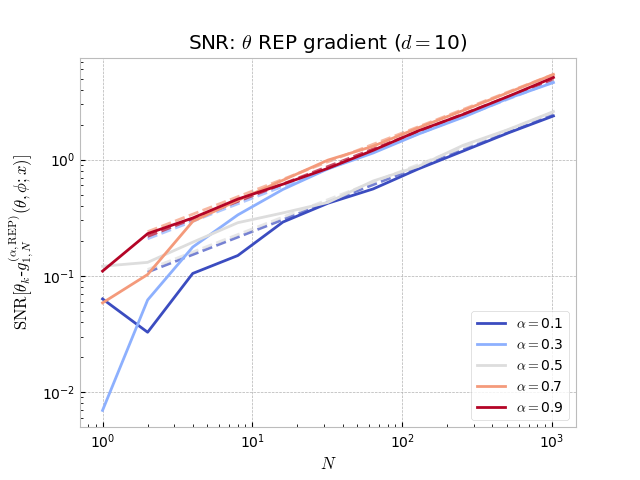} & \includegraphics[scale=0.28]{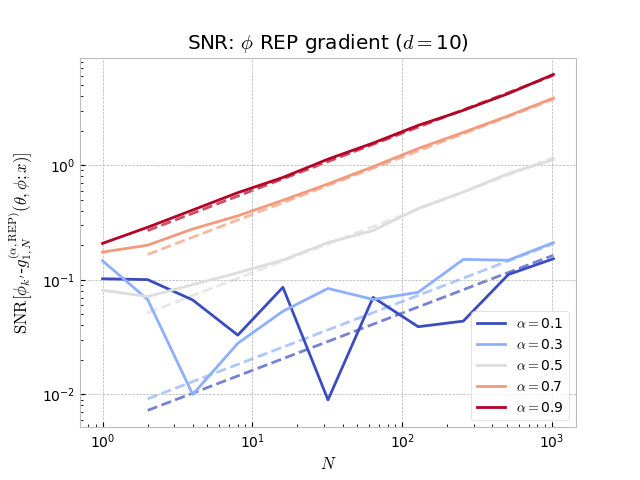} \\
    \includegraphics[scale=0.28]{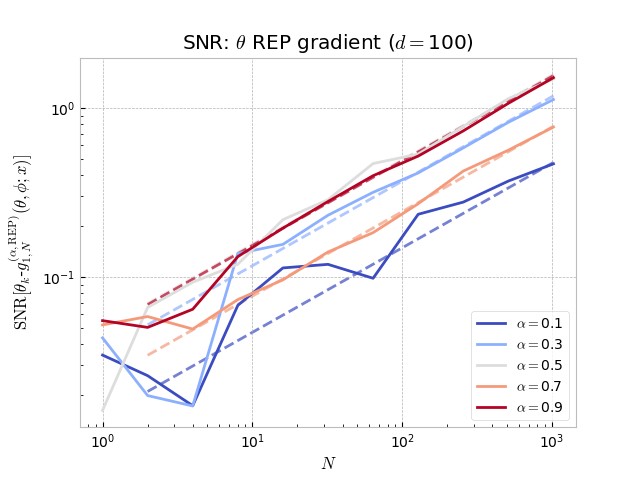} &   \includegraphics[scale=0.28]{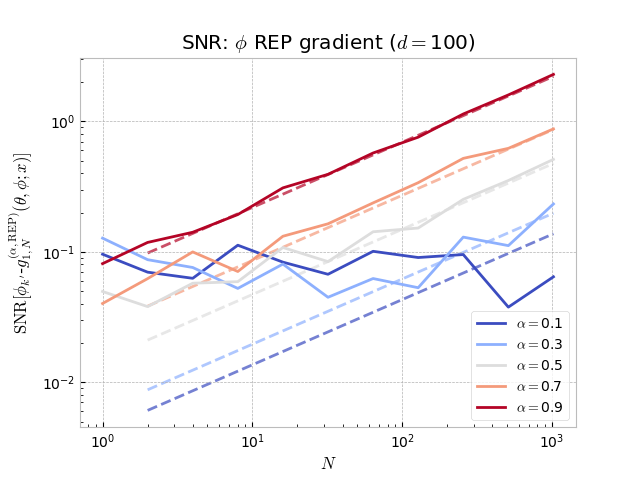}
  \end{tabular}
\end{center}
    \caption{Plotted is $\mathrm{SNR}[\gradREP[1,N]]$ computed over 2000 Monte Carlo samples for the VAE example described in \Cref{num:exp:vae} as a function of $N$, for varying values of $(\alpha,d)$, $\psi \in \{\theta_k, \phi_{k'} \}$ and a randomly selected datapoint $x$. The solid lines correspond to $\mathrm{SNR}[\gradREP[1,N]]$, the dashed lines correspond to predictions of the form \eqref{eq:vae} in which $\partial_{\psi} \mathrm{VR}^{(\alpha)}(\theta, \phi; x)$ and $\vREP(\theta, \phi;x)$ have also been approximated. \label{fig:VAEtheta}}
\end{figure}
Even though we had to resort to an additional layer of approximation here, our conclusions remain similar to the ones obtained in \Cref{num:GaussEx,num:LinGaussEx}, in the sense that the behavior predicted in \eqref{eq:vae} matches the observed rates for $N$ sufficiently large and we can observe the impact of increasing the mismatch between the variational and posterior densities via $d$. The theoretical studies we have carried out in this paper thus provides a useful framework to understand importance-weighted VI methods in real-data scenarios. \looseness=-1

\section{Conclusion}
\label{sec:ccl}

We proposed two complementary SNR analyses which advance the understanding of importance-weighted VI methods by casting a new light on the innerworkings of the ELBO, IWAE, VR, and VR-IWAE bound methodologies at the gradient level. 
Our proof techniques rely on novel theoretical tools that apply more broadly within importance weighting and are of independent interest. The validity of our analyses is demonstrated over several toy and real-data examples and we now mention three potential directions of research to extend our work. \looseness=-1

A first direction of research, also outlined at the end of \Cref{sub:Collapse}, is to investigate the extension of our results for the Gaussian random field case to the approximate Gaussian random field one in order to encompass more general settings of interest. Another possibility is to see how our results can be further exploited to tune $\alpha$ and/or to choose the variational family. Finally, our work showcases that there is much to gain theoretically and empirically from investigating the connections between sampling and importance-weighted VI. One may seek to delve deeper into those connections with importance weighting and related Monte Carlo approximations.


\section*{Acknowledgements and Disclosure of Funding}

We are grateful to the Associate Editor and to the reviewers for their insightful comments which helped us improve the manuscript. We also want to thank Pierre E. Jacob and El Mahdi Khribch for stimulating discussions that encouraged us to pursue weaker positive moment conditions for \eqref{eq:ratios-limits-optimal-cond2bis:alter} in \Cref{thm:ratios-limits-optimal-cond}. These exchanges prompted a re-examination of our proof strategy, which ultimately allowed us to obtain the refined results presented in \Cref{thm:ratios-limits-optimal-cond}.

\appendix

\section{Interchanging Derivative and Expectation Signs}
\label{app:differentiabilityCondition}

Throughout the paper, we often come across expectations of the form
$\PE_{\varepsilon \sim \REPq}\lr{ g(\psi, \varepsilon)}$ for some real valued
parameter $\psi$ in a parameter space $\Psi$,
which we want to differentiate w.r.t $\psi$. More specifically, using
that $\REPq$ does not depend on $\psi$, we want compute the derivative by
interchanging the derivative and the expectation signs, that is
\begin{align}
 \partial_{\psi}\PE_{\varepsilon \sim \REPq}\lr{ g(\psi, \varepsilon)}=\PE_{\varepsilon \sim \REPq}\lr{ \partial_{\psi} g(\psi, \varepsilon)} , \quad \psi \in \Psi\;.\label{eq:interchange}
\end{align}
General conditions to make the identity~(\ref{eq:interchange}) valid are
well known (see, e.g., \cite{ecuyer}) and specifying them can be overlooked at first reading as it sometimes burdens the technical
content with lengthy assumptions whose sole aim is to make interchanges
of derivatives and expectations well justified. Nevertheless, we
specify hereafter sufficient conditions for all the interchanges of derivatives and expectations to be valid in our main results.

To this end, we will repeatedly use the notation $\tw$ defined in~(\ref{eq:tw})
which, for every $\varepsilon$ and $x$, is differentiable with respect to
$(\theta,\phi,\phi')$ on $\Theta\times\Phi^2$ following the
differentiability conditions already assumed on
$\theta\mapsto p_{\theta}(x)$, $(z,\theta)\mapsto p_{\theta}(z|x)$,
$(z,\phi)\mapsto q_{\phi}(z|x)$ and
$\phi\mapsto f(\varepsilon,\phi;x)$ in the beginning of
\Cref{sub:GradN}. Furthermore, $\psi$ will denote a component of the $\mathbb{R}^{a+b}$-valued variable $(\theta,\phi)=(\theta_1, \ldots, \theta_a,\phi_1,\dots,\phi_b)$ and we will say that $\mathcal{V}$ is a $\psi$-neighborhood of 
$(\theta, \phi)\in\mathbb{R}^{a+b}$ if there exists $r>0$ such that for
all $\psi'\in(\psi-r,\psi+r)$, the vector obtained by replacing $\psi$
by $\psi'$ in $(\theta, \phi)$ belongs to $\mathcal{V}$.

The remaining of \Cref{app:differentiabilityCondition} is then 
concerned with providing sufficient conditions to ensure that the interchanges of derivatives and expectations necessary to establish the results of \Cref{sec:behavior-gradient-vr-Nasymp} and  \Cref{sub:Collapse} are valid.
\subsection{Interchanging Derivative and Expectation Signs in \Cref{sec:behavior-gradient-vr-Nasymp}}

The first
assumption below is concerned with the REP gradient estimator.
\begin{hypZeroREP}{A}
\item\label{hyp:hypZeroREP} All $(\theta,\phi)\in\Theta\times\Phi$ admit a
  $\psi$-neighborhood $\mathcal{V}\subset\Theta\times\Phi$ such that,
  for $\ell=0,1,2$,
  \begin{equation}
    \label{eq:hypZeroREP}
\PE_{\varepsilon \sim \REPq} \lr{\sup_{(\theta',\phi')\in\mathcal{V}}
  \lrav{\tw[\phi'][\varepsilon][\phi'][\theta']^{\ell(1-\alpha)}\;\partial_{\psi'} \lr{\log \tw[\phi'][\varepsilon][\phi'][\theta']}}}<\infty\;.
  \end{equation}
\end{hypZeroREP}
We have the following result which will cover all the interchanges of
derivatives and expectations that are necessary for our study of the
REP gradient estimator in \Cref{sec:behavior-gradient-vr-Nasymp}.
\begin{prop}
  \label{prop:interch-deriv-expect-rep}
  Under \ref{hyp:hypZeroREP}, 
we have for all $(\theta,\phi)\in\Theta\times\Phi$,
  \begin{align}\label{eq:interch-deriv-expect-rep1}
&\partial_{\psi}  \PE_{\varepsilon \sim \REPq}\lr{\tw[\phi][\varepsilon][\phi]^{1-\alpha}}=  \PE\lr{\partial_{\psi} \lr{\tw[\phi][\varepsilon][\phi]^{1-\alpha}}}\;,\\
    \label{eq:interch-deriv-expect-rep2}
&\partial_{\psi}  \PE_{\varepsilon \sim \REPq}\lr{\tw[\phi][\varepsilon][\phi]^{2(1-\alpha)}}=  \PE\lr{\partial_{\psi} \lr{\tw[\phi][\varepsilon][\phi]^{2(1-\alpha)}}}\;.
  \end{align}
Moreover, for all $N\geq1$,
  \begin{equation}
    \label{eq:interchange-liren}
  \partial_{\psi}  \PE_{\varepsilon_j\overset{\text{\tiny
        i.i.d.}}{\sim}\REPq}\lr{\log \lr{ \sum_{j=1}^N\tw[\phi][\varepsilon_j][\phi]^{1-\alpha}}
  }=
  \PE_{\varepsilon_j\overset{\text{\tiny
          i.i.d.}}{\sim}\REPq}\lr{    \partial_{\psi} \log
      \lr{\sum_{j=1}^N\tw[\phi][\varepsilon_j][\phi]^{1-\alpha} }} \;.    
  \end{equation}  
\end{prop}
\begin{proof}
Suppose that  \ref{hyp:hypZeroREP} holds. Observe that, for
$\ell=1,2$, by~(\ref{eq:tw}), we have
$$
\partial_{\psi} \lr{\tw[\phi][\varepsilon][\phi]^{\ell(1-\alpha)}}=\ell(1-\alpha)\;\w[\theta][\phi](f(\varepsilon, \phi; x);
      x)^{\ell(1-\alpha)}\;\partial_{\psi} \lr{\log \w[\theta][\phi](f(\varepsilon, \phi; x);
      x)}\;.
$$
Thus, by the usual dominated
convergence argument for interchanging the expectation and derivative
signs, the condition in~(\ref{eq:hypZeroREP}) for $\ell=1,2$
leads to~(\ref{eq:interch-deriv-expect-rep1})
and~(\ref{eq:interch-deriv-expect-rep2}), respectively.

Let now $N\geq1$. To get~(\ref{eq:interchange-liren}), by the same
dominated convergence argument, it is
sufficient to show that all $(\theta,\phi)\in\Theta\times\Phi$ admit a
$\psi$-neighborhood $\mathcal{V}\subset\Theta\times\Phi$ such that
$$
  \PE_{\varepsilon_j\overset{\text{\tiny
          i.i.d.}}{\sim}\REPq}\lr{ \sup_{(\theta',\phi')\in\mathcal{V}}   \lrav{\partial_{\psi'} \log
      \lr{\sum_{j=1}^N\tw[\phi'][\varepsilon_j][\phi'][\theta']^{1-\alpha} }}}<\infty\;.
  $$
Differentiating and using that the sup of a sum is bounded by the sum
of the sup from above, the latter condition is implied by having that for all $j=1,\dots,N$,
$$
  \PE_{\varepsilon_j\overset{\text{\tiny
          i.i.d.}}{\sim}\REPq}\lr{ \sup_{(\theta',\phi')\in\mathcal{V}}   \frac{\lrav{\partial_{\psi'}\lr{\tw[\phi'][\varepsilon_j][\phi'][\theta']^{1-\alpha} } }}
      {\sum_{k=1}^N\tw[\phi'][\varepsilon_k][\phi'][\theta']^{1-\alpha} }}<\infty\;.
  $$
  Using that $\tw[\phi'][\varepsilon_k][\phi'][\theta']\geq0$ 
  for all $k\neq j$, this ratio is bounded
  from above by
  $$
  \frac{\lrav{\partial_{\psi}\lr{\tw[\phi'][\varepsilon_j][\phi'][\theta']^{1-\alpha} } }}
      {\tw[\phi'][\varepsilon_j][\phi'][\theta']^{1-\alpha} }=(1-\alpha)\;\lrav{\partial_{\psi}\lr{\log\tw[\phi'][\varepsilon_j][\phi'][\theta'] }}\;.
      $$
      Hence assuming~(\ref{eq:hypZeroREP}) with $\ell=0$ ensures that the
      previous condition holds and we have concluded the proof of~(\ref{eq:interchange-liren}).
\end{proof}
The next assumption is concerned
with the DREP gradient estimator and is only required for a component $\psi$ of
the $\mathbb{R}^{b}$-valued variable $\phi=(\phi_1,\dots,\phi_b)$,
\begin{hypZeroDREP}{A}
\item\label{hyp:hypZeroDREP} All $(\theta,\phi)$ in $\Theta\times\Phi$ admit a
  $\psi$-neighborhood $\mathcal{V}\subset\Theta\times\Phi$ such that:
\begin{enumerateList}
\item  If $\alpha \in (0,1)$, we have
  \begin{align}
    \label{eq:interAsumpDREPall-alpha-neq-zero}
&\PE_{\varepsilon \sim \REPq} \lr{\sup_{(\theta,\phi')\in\mathcal{V}} \lrav{\partial_{\psi'} \lr{\tw[\phi'][\varepsilon][\phi']^{1-\alpha}}}}<\infty\;.    \\
    \label{eq:interAsumpDREP1-alpha-neq-zero}
    &\PE_{\varepsilon \sim \REPq} \lr{\sup_{(\theta,\phi')\in\mathcal{V}} \lrav{\partial_{\psi'} \lr{\tw[\phi][\varepsilon][\phi']^{1-\alpha}}}}<\infty\;.    \\
    \label{eq:interAsumpDREP2-alpha-neq-zero}
&\PE_{Z\sim q_{\phi}(\cdot|x)} \lr{\w(Z;
                                                x)^{1-\alpha}\,\sup_{(\theta,\phi')\in\mathcal{V}}
                                                \lrav{\partial_{\psi'}\lr{\log
                                                q_{\phi'}(Z |x)}\;\frac{q_{\phi'}(Z|x)}{q_{\phi}(Z|x)}}}<\infty\;.    
  \end{align}
\item If $\alpha=0$, we have
  \begin{align}
    \label{eq:interAsumpDREPall-alpha-zero}
&\PE_{\varepsilon \sim \REPq} \lr{\sup_{(\theta,\phi')\in\mathcal{V}}
                                                           \lrav{\partial_{\psi'}
                                                           \lr{\tw[\phi'][\varepsilon][\phi']^{2}}}}<\infty\;,
    \\
    \label{eq:interAsumpDREP1-alpha-zero}
&\PE_{\varepsilon \sim \REPq} \lr{\sup_{(\theta,\phi')\in\mathcal{V}}
                                                           \lrav{\partial_{\psi'}
                                                           \lr{\tw[\phi][\varepsilon][\phi']^{2}}}}<\infty\;,
    \\
    \label{eq:interAsumpDREP2-alpha-zero}
&\PE_{Z\sim q_{\phi}(\cdot|x)} \lr{\w[\theta][\phi](Z; x)^{2} \sup_{(\theta,\phi')\in\mathcal{V}} \lrav{\partial_{\psi'}\lr{\log
                                                q_{\phi'}(Z |x)}\;\frac{q_{\phi'}(Z|x)}{q_{\phi}(Z|x)}}}<\infty\;.    
  \end{align}
\end{enumerateList}
\end{hypZeroDREP} 
We have the following result which will cover all the interchanges of
derivatives and expectations that are necessary for our study of the
DREP gradient estimator in \Cref{sec:behavior-gradient-vr-Nasymp}.
\begin{prop}
  \label{prop:interch-deriv-expect-drep}
Let $\psi$ denote a component of the
$\mathbb{R}^{b}$-valued variable
$\phi=(\phi_1,\dots,\phi_b)$.   Under \ref{hyp:hypZeroDREP}, we have the following
  assertions for all  $(\theta,\phi)$ in $\Theta\times\Phi$.
\begin{enumerateList}
\item\label{item:drep-exp-deriv-alpha-neq-zero}  If $\alpha \in (0,1)$, we have
  \begin{align}
    \label{eq:drep-exp-deriv-alpha-neq-zero1}
    \lrder{\partial_{\psi'}\PE_{\varepsilon \sim \REPq}\lr{\tw^{1-\alpha}}}{\phi'=\phi}=\PE_{\varepsilon \sim \REPq}\lr{\lrder{\partial_{\psi'}\lr{\tw^{1-\alpha}}}{\phi'=\phi}},\\
    \label{eq:drep-exp-deriv-alpha-neq-zero2}
    \partial_{\psi}\PE_{\varepsilon \sim \REPq}\lr{\tw[\phi][\varepsilon][\phi]^{1-\alpha}}=\alpha\,\PE_{\varepsilon \sim \REPq}\lr{\lrder{\partial_{\psi'}\lr{\tw^{1-\alpha}}}{\phi'=\phi}}.
  \end{align}
\item\label{item:drep-exp-deriv-alpha-zero} If $\alpha=0$, we have
  \begin{align}
    \label{eq:drep-exp-deriv-alpha-zero}
    \partial_{\psi}\PE_{\varepsilon \sim \REPq}\lr{\tw[\phi][\varepsilon][\phi]^{2}} =
    -\PE_{\varepsilon \sim \REPq}\lr{\lrder{\partial_{\psi'}\lr{\tw^{2}}}{\phi'=\phi}}.
  \end{align}
\end{enumerateList}
\end{prop}
\begin{proof}
  First observe that~(\ref{eq:drep-exp-deriv-alpha-zero}) in
  Assertion~\ref{item:drep-exp-deriv-alpha-zero} corresponds
  to~(\ref{eq:drep-exp-deriv-alpha-neq-zero2}) in
  Assertion~\ref{item:drep-exp-deriv-alpha-neq-zero} with $\alpha$
  replaced by $-1$. In addition, the case $\alpha=0$ in
  Assumption~\ref{hyp:hypZeroDREP} corresponds to the case
  $\alpha\in(0,1)$ with $\alpha$ also replaced by $-1$.
  We thus only have to prove that for any
  $\alpha\in(0,1)\cup\lrcb{-1}$,
  Conditions~(\ref{eq:interAsumpDREPall-alpha-neq-zero})--(\ref{eq:interAsumpDREP2-alpha-neq-zero})
  imply the identities~(\ref{eq:drep-exp-deriv-alpha-neq-zero1})
  and~(\ref{eq:drep-exp-deriv-alpha-neq-zero2}).

  We will in fact prove that more generally: for all $\alpha\in\mathbb{R}$, Conditions~(\ref{eq:interAsumpDREPall-alpha-neq-zero})--(\ref{eq:interAsumpDREP2-alpha-neq-zero})
  imply the identities~(\ref{eq:drep-exp-deriv-alpha-neq-zero1})
  and~(\ref{eq:drep-exp-deriv-alpha-neq-zero2}). The case $\alpha=1$ is immediate so let us pick $\alpha\in\mathbb{R}\setminus\lrcb1$ in the following and assume that~(\ref{eq:interAsumpDREPall-alpha-neq-zero})--(\ref{eq:interAsumpDREP2-alpha-neq-zero}) hold. 
  Note first that (\ref{eq:drep-exp-deriv-alpha-neq-zero1}) is a simple interchange of the derivative and expectation signs which holds by dominated convergence
  using~(\ref{eq:interAsumpDREP1-alpha-neq-zero}). To then prove~(\ref{eq:drep-exp-deriv-alpha-neq-zero2}),
  observe that by~(\ref{eq:interAsumpDREPall-alpha-neq-zero}), we can interchange the derivative and expectation
  signs on the left-hand side
  of~(\ref{eq:drep-exp-deriv-alpha-neq-zero2}), that is: for all $\alpha \in (0,1)$, 
  \begin{align*}
    \partial_{\psi}\PE_{\varepsilon \sim \REPq}\lr{\tw[\phi][\varepsilon][\phi]^{1-\alpha}}=    \PE_{\varepsilon \sim \REPq}\lr{\partial_{\psi}\lr{\tw[\phi][\varepsilon][\phi]^{1-\alpha}}}.
  \end{align*}
  Since $\partial_{\psi}\lr{\tw[\phi][\varepsilon][\phi]^{1-\alpha}}=\lrder{\partial_{\psi'}\lr{\tw^{1-\alpha}}+    \partial_{\psi'}\lr{\tw[\phi'][\varepsilon][\phi]^{1-\alpha}}}{\phi'=\phi}$, to obtain (\ref{eq:drep-exp-deriv-alpha-neq-zero2}) it remains to prove that 
  \begin{align}
    \label{eq:drep-exp-deriv-alpha-neq-zero2-lhs}
\frac{1}{\alpha-1}\,   \PE_{\varepsilon \sim \REPq}\lr{\lrder{\partial_{\psi'}\lr{\tw[\phi'][\varepsilon][\phi]^{1-\alpha}}}{\phi'=\phi}}=\PE_{\varepsilon \sim \REPq}\lr{ \lrder{\partial_{\psi'}\lr{\tw^{1-\alpha}}}{\phi'=\phi}}\;.
  \end{align}
Observe then that, using \eqref{eq:tw}, it holds that
  $$
 \frac{1}{\alpha-1} \lrder{\partial_{\psi'}\lr{\tw[\phi'][\varepsilon][\phi]^{1-\alpha}}}{\phi'=\phi}=\,\lrder{\w(z;
    x)^{1-\alpha}\frac{\partial_{\psi}
      q_{\phi}(z)}{q_{\phi}(z|x)}}{z=f(\varepsilon, \phi; x)}\;.
  $$
  Hence the left-hand side of~(\ref{eq:drep-exp-deriv-alpha-neq-zero2-lhs}) reads
  \begin{align}\nonumber
  \PE_{\varepsilon \sim \REPq}\lr{\lrder{\w(z;
    x)^{1-\alpha}\frac{\partial_{\psi}
      q_{\phi}(z|x)}{q_{\phi}(z|x)}}{z=f(\varepsilon, \phi; x)}}&=\PE_{Z\sim q_{\phi}(\cdot|x)}\lr{\w(Z;
    x)^{1-\alpha}\frac{\partial_{\psi}
                                                              q_{\phi}(Z|x)}{q_{\phi}(Z|x)}}\\
\nonumber    &=\int \w(Z;
    x)^{1-\alpha}\,\partial_{\psi}
                                                              q_{\phi}(Z|x)\;
      \nu(\rmd z)  
\end{align}
and the proof of~(\ref{eq:drep-exp-deriv-alpha-neq-zero2-lhs}) will be concluded if we can prove that 
\begin{align}
  \PE_{\varepsilon \sim \REPq}\lr{ \lrder{\partial_{\psi'}\lr{\tw^{1-\alpha}}}{\phi'=\phi}}
  & \stackrel{\text{(a)}}{=} \lrder{\partial_{\psi'}\PE_{\varepsilon \sim \REPq}{ \lr{\tw^{1-\alpha}}}}{\phi'=\phi} \nonumber \\
  &\stackrel{\text{(b)}}{=}\lrder{\partial_{\psi'}\PE_{Z\sim  q_{\phi'}(\cdot|x)}\lr{\w[\theta][\phi](Z;
  x)^{1-\alpha}}}{\phi'=\phi} \nonumber \\
 & \stackrel{\text{(c)}}{=} \int \w(Z;
  x)^{1-\alpha}\,\partial_{\psi}
                                                            q_{\phi}(Z|x)\;
    \nu(\rmd z). \nonumber
  \end{align}
  The equality (a) follows from interchanging the integral and the
  derivative signs again thanks to
  (\ref{eq:interAsumpDREP1-alpha-neq-zero}) and the equality (b)
  follows from the reparameterization trick. As for the equality (c),
  it is obtained by interchanging the integral and the derivative
  thanks to the fact that for a $\psi$-neighborhood $\mathcal{V}$ of
  $(\theta,\phi)$,
   \begin{align*}
  \int \w(z;
    x)^{1-\alpha}\,\sup_{(\theta,\phi')\in\mathcal{V}}\lrav{\partial_{\psi'}
                                                              q_{\phi'}(z|x)}\;
      \nu(\rmd z)<\infty\;,
   \end{align*}
  which holds using (\ref{eq:interAsumpDREP2-alpha-neq-zero}) and that
   $$
   \sup_{(\theta,\phi')\in\mathcal{V}}\lrav{\partial_{\psi'}
                                                              q_{\phi'}(Z|x)}
   =  \sup_{(\theta,\phi')\in\mathcal{V}}\lrav{\lr{\partial_{\psi'} \log 
                                                              q_{\phi'}(Z|x)}\,\frac{{q_{\phi'}(Z|x)}}{q_{\phi}(Z|x)}}
      \,q_{\phi}(Z|x)\;.
  $$
  The proof of \Cref{prop:interch-deriv-expect-drep} is thus
concluded.
\end{proof}
The conclusions from \Cref{prop:interch-deriv-expect-drep} will be
used to express the constants in the asymptotic behavior of the SNR of
the DREP estimator, see the proof of \Cref{lem:SNR}.

\subsection{Interchanging Derivative and Expectation Signs in \Cref{sub:Collapse}}
\label{sec:gaussian-case-df}
We now examine the specific case where the Gaussian
assumption~\ref{hyp:reparamHighDimGaussian} holds. In this special
case, we rely on the following assumption for the REP
gradient estimator.

\begin{hypZeroREP}{B}
\item\label{hyp:hypZeroREPhighDim} All $(\theta,\phi)$ in
$\Theta\times\Phi$ admit a $\psi$-neighborhood $\mathcal{V} \subset \Theta\times\Phi$ such that
  \begin{align}
    \label{eq:df-gaussian-rep}
&\PE_{\varepsilon \sim \REPq}\lr{
  \sup_{(\theta',\phi')\in\mathcal{V}}\lr{\partial_{\psi'} \log
                                 \tw[\phi'][\varepsilon][\phi'][\theta']}^2} <\infty\;,\\
    \label{eq:df-gaussian-rep-B}
    &\PE_{\varepsilon \sim \REPq}\lr{
  \sup_{(\theta',\phi')\in\mathcal{V}}\lr{ \log
    \tw[\phi'][\varepsilon][\phi'][\theta']}^2} <\infty\;.
  \end{align}
\end{hypZeroREP}
We have the following result which will cover all the interchanges of derivatives and expectations that are necessary for our study of the REP gradient estimator in \Cref{thm:CollapseSNRGaussian}.
\begin{prop}\label{prop:gaussian-case-REP-df}
Let $\varepsilon \sim \REPq$. Under \ref{hyp:reparamHighDimGaussian}
  and~\ref{hyp:hypZeroREPhighDim}, we have that:
for all
  $(\theta,\phi)\in\Theta\times\Phi$, 
  \begin{align}
    \label{eq:gaussian-case-REP-df1-one}
&    \partial_\psi\PE_{\varepsilon \sim \REPq}\lr{\log\tw[\phi][\varepsilon][\phi]}=    \PE_{\varepsilon \sim \REPq}\lr{\partial_\psi\log\tw[\phi][\varepsilon][\phi]} \\
    \label{eq:gaussian-case-REP-df2-one}
&    \partial_\psi\PE_{\varepsilon \sim \REPq}\lr{\lr{\log\tw[\phi][\varepsilon][\phi]}^2}=
                                       \PE_{\varepsilon \sim \REPq}\lr{\partial_\psi\lr{\lr{\log\tw[\phi][\varepsilon][\phi]}^2}}
  \end{align}
and $\lr{\lr{\log\tw[\phi][\varepsilon][\phi],\partial_\psi\log\tw[\phi][\varepsilon][\phi]}}_{(\theta,\phi)\in\Theta\times\Phi}$ is a Gaussian process.
\end{prop}
\begin{proof}
  Using the dominated convergence theorem,~(\ref{eq:df-gaussian-rep}) in~\ref{hyp:hypZeroREPhighDim}
  implies~(\ref{eq:gaussian-case-REP-df1-one}) but also that
  $\partial_\psi\log\tw[\phi][\varepsilon][\phi]$ can be seen as the
  $L^2$ limit of a linear combination of the process
  $\lr{\log\tw[\phi][\varepsilon][\phi]}_{(\theta,\phi)\in\Theta\times\Phi}$. 
  
  It
  follows by (\ref{eq:lognormalReparm}) that
  $\lr{\lr{\log\tw[\phi][\varepsilon][\phi],\partial_\psi\log\tw[\phi][\varepsilon][\phi]}}_{(\theta,\phi)\in\Theta\times\Phi}$
  is a Gaussian process. As for~(\ref{eq:gaussian-case-REP-df2-one}),
  it follows from the dominated convergence theorem and 
  \begin{multline*}
    \PE_{\varepsilon \sim \REPq}\lr{
  \sup_{(\theta',\phi')\in\mathcal{V}} {\partial_{\psi'}\lr{ \lr{\log
    \tw[\phi'][\varepsilon][\phi'][\theta']}^2}}}\\
\leq 2\lr{
  \PE_{\varepsilon \sim \REPq}\lr{
  \sup_{(\theta',\phi')\in\mathcal{V}}\lr{\partial_{\psi'} \log
    \tw[\phi'][\varepsilon][\phi'][\theta']}^2}  \PE_{\varepsilon \sim \REPq}\lr{
  \sup_{(\theta',\phi')\in\mathcal{V}}\lr{ \log
    \tw[\phi'][\varepsilon][\phi'][\theta']}^2}}^{\frac12}<\infty
\;,
\end{multline*}
where the first inequality uses the Cauchy-Schwarz inequality and the
second is a consequence of~(\ref{eq:df-gaussian-rep})
and~(\ref{eq:df-gaussian-rep-B}). 
\end{proof}

\section{Deferred Proofs of \Cref{sec:key-result-ratios}}
\label{sec:proof-s-crefsec:key}

In this section, we let $(X_1,W_1),\dots,(X_N,W_N)$ be i.i.d. with the same distribution as a generic pair $(X,W)$ valued in $\rset\times\rset_+$ and we set: for
all $v\geq0$,
\begin{align}  
&\phi(v)=\PE\lr{\rme^{-v\,W}}\quad\text{and}\quad
  \phi_1(v)=\PE\lr{X\,\rme^{-v\,W}}, \label{eq:defphiAndPhiOne}\\
&\overline{W}_N:=\frac1N\sum_{i=1}^NW_i  
\quad\text{and}\quad\overline{X}_N:=\frac1N\sum_{i=1}^NX_i,\\
  \label{eq:double-produit}
& \overline{X^2}_N:=\frac1N\sum_{i=1}^NX_i^2\quad\text{and}\quad
\overline{X}^{(i\neq j)}_N:=\frac2{N(N-1)}\displaystyle\sum_{1\leq
  i<j\leq N}X_iX_j,\\  
  \label{eq:centered-empirical}
& \overline{X^c}_N:=\frac1N\sum_{i=1}^NX_i^c\quad\text{and}\quad
\overline{X^c}^{(i\neq j)}_N:=\frac2{N(N-1)}\displaystyle\sum_{1\leq
  i<j\leq N}X_i^cX_j^c,\\
  \nonumber
&  \text{where we set $X_i^c:=X_i-\PE(X)$}.
\end{align}

\subsection{Preliminary Lemmas}

In the following, we provide two preliminary lemmas that gather useful properties of the functions $\phi$ and $\phi_1$ defined in \eqref{eq:defphiAndPhiOne}.

\begin{lem}[Properties of the Laplace transform $\phi$]
  \label{lem:tech2}
  The following assertions
  hold. 
  \begin{enumerate}[label=(\roman*)]
  \item \label{item:lem:conslimsupconditionGen7prereq1} 
  The function $\phi$ is continuous and bounded on
  $\rset_+$. 
\item \label{item:lem:conslimsupconditionGen7prereq1bis0} Assume that $\PE\lr{W}=1$. Then, for all
  $N\in\mathbb{N}^*$ and all
  $u\geq0$,
  $\lr{\phi\lr{\frac uN}}^N-\rme^{-u}\geq 0  \;.$
\item \label{item:lem:conslimsupconditionGen7prereq1bis} Assume that $\PE\lr{W}=1$ and $\PE\lr{W^2}<\infty$. Then, for all
  $N\in\mathbb{N}^*$ and all 
  $u\in\lrb{0,\frac{N}{\PE\lr{W^2}}}$,
  \begin{align}
  0\leq N \lr{\lr{\phi\lr{\frac uN}}^N-\rme^{-u}}\leq \rme^{-u/2} \;. \label{eq:lem:conslimsupconditionGen7prereq1biseq}
  \end{align}
  Furthermore, 
  for all
  $u\geq0$,
  \begin{align}
  \lim_{N\to\infty} N \lr{\lr{\phi\lr{\frac uN}}^N-\rme^{-u}}=\frac{\mathbb{V}\lr{W}}2\, u^2\;\rme^{-u}\;. \label{item:lem:conslimsupconditionGen7prereq1ter} 
  \end{align}
\item \label{item:lem:conslimsupconditionGen7prereq0} Assume that there exists $\eta>0$ such that
  \begin{align}  \label{eq:simple-cond-ratio-equiv}
    \sup_{t>0} \lr{t^\eta\; \PE\lr{\rme^{-t\,W}}}<\infty \;.
  \end{align}
Then, for any exponent $m>0$, there exists $c,p>0$ such that: for all $q\geq p$ and $v\geq0$,
  $$
  \lr{\phi\lr{\frac vq}}^{q}\leq 1\wedge \lr{c\,v^{-m}}\;.
  $$
  \end{enumerate}
\end{lem}

\begin{proof} The assertion \ref{item:lem:conslimsupconditionGen7prereq1} is immediate since $\phi$ is valued in $[0,1]$ and is continuous by the dominated convergence theorem. In addition, \ref{item:lem:conslimsupconditionGen7prereq1bis0} follows from Jensen's inequality applied to the convex function $W \mapsto \rme^{-{uW}/{N}}$ and paired up with the fact that $\PE(W) = 1$ by assumption. We now move on to the proof of \ref{item:lem:conslimsupconditionGen7prereq1bis} and \ref{item:lem:conslimsupconditionGen7prereq0}. 

  \begin{itemize}[wide=0pt, labelindent=\parindent]
\item \textbf{Proof of \ref{item:lem:conslimsupconditionGen7prereq1bis}.}
  First note that for all $x\geq0$, 
  \begin{align}
  \label{eq:lem:tech1-3}
     - \frac{x^2}2 \lr{1-\rme^{-x}} & \leq \rme^{-x}-1+x-\frac{x^2}2\leq 0
  \end{align}
since for all $x\geq0$: $1-\rme^{-t}\leq1-\rme^{-x}$ for $0\leq t\leq x$ and 
\begin{align*}
&    \rme^{-x}-1+x-\frac{x^2}2=-\int_0^x(x-t)(1-\rme^{-t})\rmd t.
  \end{align*}
Applying the second inequality in~(\ref{eq:lem:tech1-3}) with $x=v W$ and taking expectation yields: \looseness=-1
  $$
  \phi(v)\leq 1-v+\frac{v^2}2\PE(W^2), \quad v\geq0.
  $$
  We deduce from the inequality above and \ref{item:lem:conslimsupconditionGen7prereq1bis0} that: for all $u\geq0$
  and $N\in\mathbb{N}^*$,
  $$
  0\leq \lr{\phi\lr{\frac
      uN}}^N-\rme^{-u} \leq\rme^{-u}\lr{\rme^{u+N\ln\lr{1-\frac uN+\frac{u^2}{2N^2}\PE\lr{W^2}}}-1}\;.
  $$
  Since $\ln (1+ x)\leq x$, we then get that for all $u\geq0$ and $N\in\mathbb{N}^*$,
  $$
  0 \leq \lr{\phi\lr{\frac
      uN}}^N-\rme^{-u}\leq\rme^{-u}\lr{\rme^{\frac{u^2}{2N}\PE\lr{W^2}}-1}\leq\rme^{-u\lr{1-\frac{u}{2N}\PE\lr{W^2}}}\;.
  $$
  Restricting $u$ to the set $\lrb{0,\frac{N}{\PE\lr{W^2}}}$ yields \eqref{eq:lem:conslimsupconditionGen7prereq1biseq}. 
  
  Now let $u\geq0$ be given. Applying~(\ref{eq:lem:tech1-3}) with
  $x=u\,W/N$ and taking expectation, we get that for all $N\in\mathbb{N}^*$,
  $$
  \lrav{\phi\lr{\frac
      uN}- 1+\frac
      uN-\frac{\PE\lr{W^2}}2\,\frac{u^2}{N^2}}\leq\frac{u^2}{2N^2}\;\PE\lr{W^2\lr{1-\rme^{-W\,u/N}}}\;.
  $$
  By dominated convergence, the expectation in the upper bound
  converges to 0 as $N\to\infty$. Hence we have, as $N\to\infty$,
  $$
  \ln \phi\lr{\frac
      uN}= \ln\lr{1-\frac
      uN+\frac{\PE\lr{W^2}}2\,\frac{u^2}{N^2} +o\lr{\frac1{N^2}}}=
    -\frac
    uN+\frac{\mathbb{V}\lr{W}}2\,\frac{u^2}{N^2} +o\lr{\frac1{N^2}}\;,
    $$
    where used that $\ln(1+x)=x-x^2/2+o(x^2)$ as $x\to0$ and $\PE\lr{W^2}-1=\mathbb{V}\lr{W}$.
  It follows that, as $N\to\infty$,
  $$
  N \lr{\lr{\phi\lr{\frac uN}}^N-\rme^{-u}}=N\,\rme^{-u}\,\lr{\rme^{N\ln \phi\lr{\frac
      uN}+u}-1}=\rme^{-u}\,N\,\lr{\rme^{\frac{\mathbb{V}\lr{W}}2\,\frac{u^2}{N} +o\lr{\frac1{N}}}-1}\;,
$$
which yields the given limit \eqref{item:lem:conslimsupconditionGen7prereq1ter} using that $\rme^x=1+x+o(x)$ as $x\to0$. 

    \item \textbf{Proof of \ref{item:lem:conslimsupconditionGen7prereq0}.} First note that for any $v\geq0$, $p\mapsto \lr{\phi\lr{\frac vp}}^{p}$ is non-increasing as $p$ increases (since for all $q\geq p$, $\lr{\phi\lr{\frac vp}}^{p}= \lr{\phi\lr{\frac v q \frac{q}{p}}}^{p} \geq \lr{\phi\lr{\frac vq}}^{q}$  by Jensen's inequality). Furthermore, by \eqref{eq:simple-cond-ratio-equiv}, we have: for any $p>0$,
$$
c_1(p):=\sup_{t>0} \lr{(t/p)^\eta\; \phi(t/p)}<\infty \;.
$$
Let $m>0$ and take $p=m/\eta$. Then we have for all $q\geq p$ and $v \geq 0$,
$$
\lr{\phi\lr{\frac vq}}^{q}\leq \lr{\phi\lr{\frac vp}}^{p}\leq (v/p)^{-p\,\eta}\;\lr{\sup_{t>0}(t/p)^\eta\;\phi\lr{\frac tp}}^{p}=c\;v^{-m}\;,
$$
where we set $c=p^{p\eta}\;\lr{c_1(p)}^p$, hence the proof of \ref{item:lem:conslimsupconditionGen7prereq0} is completed.
\end{itemize} 
\end{proof}

\begin{lem}[Properties of $\phi_1$] \label{lem:tech3} The following assertions
  hold. 
  \begin{enumerateList}
\item \label{item:lem:conslimsupconditionGen7prereq3} Assume that
  $\PE(|X|)<\infty$. Then, $\phi_1$ is continuous and bounded on
  $\rset_+$.
\item \label{item:lem:conslimsupconditionGen7prereq4} Assume that $\PE(X)=0$
  and $\PE(|X|\, W)<\infty$. Then, for all $u\geq0$,
  \begin{align*}
    & \sup_{N \geq1}   \lr{ N \lrav{\phi_1\lr{\frac uN}+\PE\lr{X W}\;\frac
      uN }}\leq \PE(|X| W)u\;,\\
    & \lim_{N \to\infty} N \lrav{\phi_1\lr{\frac uN}+\PE\lr{X W}\frac uN}=0\;.
  \end{align*}
\item \label{item:lem:conslimsupconditionGen7prereq5} Assume that
  $\PE(X)=0$ and $\PE(|X|\,W^\delta) <\infty$ with  $0<\delta<1$. Then, for all $u \geq 0$,
  \begin{align*}
    & \sup_{N \geq1} \lr{N^{\delta} \lrav{\phi_1\lr{\frac uN}}}\leq
      \PE\lr{ |X| \,W^\delta } \; u^\delta\;, \\
    & \lim_{N \to\infty} N^{\delta} \lrav{\phi_1\lr{\frac uN}}=0\;.
  \end{align*}
\end{enumerateList}
\end{lem}

\begin{proof} The assertion \ref{item:lem:conslimsupconditionGen7prereq3} is immediate since $\phi_1$ is valued in $[-\PE(X_-),\PE(X_+)]$ and it is
continuous by the dominated convergence theorem when $\PE(|X|)<\infty$. As for \ref{item:lem:conslimsupconditionGen7prereq4}, first note that for all $x\geq0$,
  \begin{align}
  \label{eq:lem:tech1-1}
    0 & \leq \rme^{-x}-1+x\leq x\lr{1-\rme^{-x}},
  \end{align}
since for all $x\geq0$, $1-\rme^{-t}\leq1-\rme^{-x}$ for $0\leq t\leq x$ and
\begin{align*}
&    \rme^{-x}-1+x=\int_0^x\lr{1-\rme^{-t}}\rmd t.
  \end{align*}   
Using $\PE(X)=0$ and $\PE(|X|\,W)<\infty$, we have, for all $u\geq0$ and $N\geq1$,
\begin{align*}
  \lrav{\phi_1\lr{\frac u{N}}+\frac {u\,\PE\lr{X\,W}}{N}}&=
                                                             \lrav{\PE\lr{X\lr{\rme^{-u\,W/N}-1+u\,W/N}}}\\
&\leq
                                                             \PE\lr{\lrav{X}\lr{\rme^{-u\,W/N}-1+u\,W/N}}\\                                                             
& \leq \frac u{N}\;\PE\lr{\,|X|\,W\;\lr{1-\rme^{-u\,W/N}}}\;,
\end{align*}
where in the last inequality we applied~(\ref{eq:lem:tech1-1}) with
$x=u\,W/N$. The two claims of the assertion \ref{item:lem:conslimsupconditionGen7prereq4} easily follow.

Moving on to \ref{item:lem:conslimsupconditionGen7prereq5}, we use
that, for all $x\geq0$, $0\leq 1-\rme^{-x}\leq x$, which implies
$$
0\leq 1-\rme^{-x}=\lr{1-\rme^{-x}}^\delta\;\lr{1-\rme^{-x}}^{1-\delta}\leq x^\delta\;\lr{1-\rme^{-x}}^{1-\delta}\;.
$$
We get that, since $\PE(X)=0$, for all $u \geq 0$ and $N \geq 1$,
\begin{align*}
 \lrav{\phi_1\lr{\frac u{N}}} & =  \lrav{\PE\lr{ X \lr{\rme^{-uW/N}  -1 }}} \\
 & \leq \PE\lr{ |X| \lr{1 - \rme^{-uW/N} } } \\
 & \leq \lr{\frac{u}{N}}^\delta \PE\lr{ |X| \,{W}^\delta\,\lr{1-\rme^{-uW/N}}^{1-\delta}} 
\end{align*}
and the desired asssertions in
\ref{item:lem:conslimsupconditionGen7prereq5} follow under $\PE\lr{ |X| \,{W}^\delta}<\infty$ using
that $\lr{1-\rme^{-uW/N}}^{1-\delta}$ is bounded by 1 and vanishes as
$N\to\infty$ for $u\,W>0$.
\end{proof}
Next, we provide useful equivalences for the condition \eqref{eq:simple-cond-ratio-equiv}.

\subsection{Useful Equivalences} 
\label{app:usefulEq}
  \begin{prop}
    \label{lem:equivlimsupconditionGen}
    Set $\overline{W}_N = N^{-1} \sum_{i=1}^N W_i$ for all
    $N \in \mathbb{N}^\star$, where $W_1, \ldots, W_N$ are positive
    i.i.d. random variables distributed as a generic random variable $W$. Then the following assertions hold.
    \begin{enumerate}[label=(\roman*),series=W]
  \item\label{item:lem:equivlimsupconditionGen2} For all $\mu>0$, we have
  \begin{align} \label{ass:limsupWequiv}
    \limsup_{N\to\infty  } \PE\lr{\lr{\overline{W}_N}^{-\mu}} <
    \infty\Longleftrightarrow\exists N\geq1\;,\,\PE\lr{\lr{\overline{W}_N}^{-\mu}} < \infty.
  \end{align}
  Furthermore, if the assertions of the equivalence~(\ref{ass:limsupWequiv}) hold for some $\mu>0$ and the distribution of $W$ does not reduce to a Dirac measure, there exists $N_0\geq1$ such that $\PE\lr{\lr{1/\overline{W}_N}^\mu}=\infty$ for $1\leq N < N_0$ and the sequence $\lr{\PE\lr{\lr{1/\overline{W}_N}^\mu}}_{N \geq N_0}$ is strictly decreasing in $(0,\infty)$. 
\item \label{item:lem:equivlimsupconditionGen1} For all $\eta>0$,
    the two conditions~(\ref{eq:simple-cond-ratios})
    and~(\ref{eq:simple-cond-ratio-equiv}) are equivalent.
\item\label{item:lem:equivlimsupconditionGen4}   If
  condition~(\ref{eq:simple-cond-ratio-equiv}) holds for $\eta>0$,
  then $\PE\lr{(\overline{W}_N)^{-\mu}}<\infty$ for all $\mu>0$
  and $N\geq1$ such that $N>\mu/\eta$.
\item \label{item:implicationcond-jacob} If
  condition~(\ref{eq:simple-cond-ratio-equiv}) holds for $\eta>0$, then
  $\PE(W^{-\mu})<\infty$ for  $0<\mu<\eta$.
\item\label{item:reverse_implicationcond-jacob} If
  $\PE(W^{-\mu})<\infty$ for $\mu>0$, then
  condition~(\ref{eq:simple-cond-ratios}) holds for $0<\eta\leq\mu$.
\item\label{item:lem:equivlimsupconditionGen3improved} If
  $\PE\lr{\lr{\overline{W}_N}^{-\mu}} < \infty$ for some $\mu>0$ and
  $N\geq1$, then~(\ref{eq:simple-cond-ratio-equiv}) hold with $\eta=\mu/N$.
\item \label{item:lem:equivlimsupconditionGen6} Suppose 
that~(\ref{eq:simple-cond-ratios}) holds for some $\eta>0$.  
Then, for all $\mu>0$, we have
  \begin{align} \label{ass:limWnegpower}
    \lim_{N\to\infty  } \PE\lr{(1/\overline{W}_N)^\mu} =
    \begin{cases}
      0 &\text{ if $\PE(W)=\infty$,}\\
    \lr{\PE(W)}^{-\mu} &\text{ otherwise.}
    \end{cases}
  \end{align}
\end{enumerate}
\end{prop}
  \begin{proof}
    We prove Assertions~\ref{item:lem:equivlimsupconditionGen2}--\ref{item:lem:equivlimsupconditionGen6} successively.
\begin{enumerateList}
\item \textbf{Proof of Assertion~\ref{item:lem:equivlimsupconditionGen2}.}
Let $\Gamma$ denote the Euler Gamma function, so that $\Gamma(\mu)=\int_0^\infty
t^{\mu-1}\rme^{-t}\;\rmd t$ for all $\mu > 0$. Following an idea of \cite{cressie81}, we have: for all $N\geq1$, \looseness=-1
\begin{align}
  \nonumber
\PE\lr{\lr{\overline{W}_N}^{-\mu}}& = (\Gamma(\mu))^{-1}\,\int_0^\infty
                                 t^{\mu-1}\PE\lr{\rme^{-\overline{W}_N\,t}}\;\rmd t\\
\label{eq:lem:equivlimsupconditionGen1}  &= (\Gamma(\mu))^{-1}\,\int_0^\infty
t^{\mu-1}\lr{\PE\lr{\rme^{-W\,t/N}}}^N\;\rmd t,
\end{align}
where we have used that, for all $x>0$,
$x^{-\mu}=(\Gamma(\mu))^{-1}\,\int_0^\infty
t^{\mu-1}\rme^{-x\,t}\;\rmd t$ and Tonelli's theorem. For all $t\geq0$ and $N\geq1$, by Jensen's inequality and strict convexity of
$x\mapsto x^{(N+1)/N}$ on $(0,\infty)$, we have, for all $t>0$,
$$
\lr{\PE\lr{\rme^{-W\,t/N}}}^{N}=\lr{\PE\lr{\lr{\rme^{-W\,t/(N+1)}}^{(N+1)/N}}}^{N}\geq\lr{\PE\lr{\rme^{-W\,t/(N+1)}}}^{N+1}\;,
$$
with equality if and only if both sides of the inequality are
infinite or $\rme^{-W\,t/(N+1)}$ is equal to its
mean a.s. and we can thus deduce Assertion~\ref{item:lem:equivlimsupconditionGen2}.

\item \textbf{Proof of
  Assertion~\ref{item:lem:equivlimsupconditionGen1}.} Let $\eta >0$. We first assume that 
\begin{equation}
  \label{eq:simple-cond-negative-moments}
  \sup_{t>0} \lr{t^\eta\; \PE\lr{\rme^{-t\,W}}}<\infty
\end{equation}
holds and
show
\begin{equation}
    \label{eq:simple-cond-negative-moments-equiv}
    \sup_{u>0}\lr{u^{-\eta}\;\PP\lr{W\leq u} }<\infty.
  \end{equation}
We have, for all
$u>0$, by Markov's inequality,
$$
\PP\lr{W\leq u}=\PP\lr{\rme^{\eta(1-W/u)}\geq
  1}\leq\rme^{\eta}\;\PE\lr{\rme^{-\eta\,W/u}}\leq \rme^{\eta}\;C\,\eta^{-\eta}\,u^\eta,
$$
where $C$ is the sup in the left-hand side
of~(\ref{eq:simple-cond-negative-moments}). Hence (\ref{eq:simple-cond-negative-moments-equiv}) follows. Let us now
assume~(\ref{eq:simple-cond-negative-moments-equiv}) and show
that~(\ref{eq:simple-cond-negative-moments}) holds. 
We have, for all $t>0$,
  $$
  \PE\lr{\rme^{-t\,W}}=t\int_0^\infty \rme^{-t\,u} \; \PP\lr{W\leq
    u}\;\rmd u = \int_0^\infty \rme^{-v} \; \PP\lr{W\leq
    v/t}\;\rmd v \;.
  $$
  Let $C$ now denote the (finite) sup in the left-hand side
  of~(\ref{eq:simple-cond-negative-moments-equiv}), so that
  $\PP\lr{W\leq u}\leq C\,u^{\eta}$ for all $u>0$.  Using this bound
  in the previous integral, we get that, for all $t>0$,
  $$
  \PE\lr{\rme^{-t\,W}}\leq C\,t^{-\eta} \int_0^\infty \rme^{-v} \; v^{\eta}\;\rmd v\;.
  $$
  Since the latter integral is a finite constant, this concludes the
  proof of \ref{item:lem:equivlimsupconditionGen1}. 
\item \textbf{Proof of
    Assertion~\ref{item:lem:equivlimsupconditionGen4}.}  Suppose
  that~(\ref{eq:simple-cond-ratio-equiv}) holds for some
  $\eta>0$. Then there exists $c>0$ such that, for all $t\geq0$,
  $\PE\lr{\rme^{-W\,t}}\leq1\wedge\lr{c\,t^{-\eta}}$. 
Hence for $\mu>0$ and $N\geq1$ such that  $N>\mu/\eta$, the right-hand side
of~(\ref{eq:lem:equivlimsupconditionGen1}) is finite and
$\PE\lr{(1/\overline{W}_N)^\mu}$ as well.
\item \textbf{Proof of Assertion~\ref{item:implicationcond-jacob}.}
  This is a special case of \ref{item:lem:equivlimsupconditionGen4}
  imposing $N=1$.
\item  \textbf{Proof of
    Assertion~\ref{item:reverse_implicationcond-jacob}.} Suppose that
  $\PE(W^{-\mu})<\infty$ for some $\mu>0$. Then $\PP(W=0)=0$ and, for all $u>0$,
  $$
  \PP(W\leq u)=\PP(W^{-\mu}u^{\mu}\geq 1)\leq u^{\mu}\PE(W^{-\mu})\;,
  $$
  where we used the Markov inequality.

\item \textbf{Proof of
    Assertion~\ref{item:lem:equivlimsupconditionGen3improved}.}
  Applying Assertions~\ref{item:reverse_implicationcond-jacob} and~\ref{item:lem:equivlimsupconditionGen1}
  successively with $W$
  replaced by $\overline{W}_N$, the
  given assumption implies
  $$
  \sup_{t>0}\lr{t^\mu\PE\lr{\rme^{-t\overline{W}_N}}}<\infty\;.
  $$
  Since
  $\PE\lr{\rme^{-t\overline{W}_N}}=\lr{\PE\lr{\rme^{-t\,W/N}}}^N$, we
  get, setting $u=t/N$,
  $$
  \sup_{u>0}\lr{u^\mu\,\lr{\PE\lr{\rme^{-u\,W}}}^N}<\infty\;.
  $$
  The claimed implication follows.

\item \textbf{Proof of
    Assertion~\ref{item:lem:equivlimsupconditionGen6}.}  By the
  strong law of large numbers we have $\overline{W}_N\to\PE\lr{W}$
  as $N\to\infty$ a.s., including in the case where
  $\PE\lr{W}=\infty$. Hence the sequence
  $\lr{(1/\overline{W}_N)^\mu}_{N\geq 1}$ converges to the right-hand
  side of~(\ref{ass:limWnegpower}) a.s. To conclude the proof of
  Assertion~\ref{item:lem:equivlimsupconditionGen6}, it suffices to
  show that, for all $\mu>0$, the sequence
  $\lr{(1/\overline{W}_N)^\mu}_{N\geq N_0}$ is uniformly integrable
  for some $N_0\geq1$ \citep[see for example Theorem 3.5 of][]{BillingsleyBook}. By the Markov inequality, the uniform integrability then follows from the fact
  that: for some $\epsilon>0$, 
  $$
  \limsup_{N\to\infty  } \PE\lr{\lr{(1/\overline{W}_N)^{\mu}}^{1+\epsilon}}<\infty.
  $$
  In other words, the uniform integrability will follow from
  $$
  \limsup_{N\to\infty  } \PE\lr{(1/\overline{W}_N)^{\mu'}}<\infty,
  $$
  with $\mu' = \mu(1+\epsilon)> \mu$. By Assertion~\ref{item:lem:equivlimsupconditionGen4}, the assertions of the equivalence~(\ref{ass:limsupWequiv}) hold for
  all $\mu>0$. The proof is concluded by taking $\mu' = \mu(1+\epsilon)$ in lieu of $\mu$ in the assertions of the equivalence~(\ref{ass:limsupWequiv}).
\end{enumerateList}
\end{proof}

\begin{rem}
  Assertion~\ref{item:lem:equivlimsupconditionGen1} in
  \Cref{lem:equivlimsupconditionGen} is related to the famous
  Abelian and Tauberian theorems for the Laplace transform of a
  probability measure on $\rset_+$, see \cite[Section~XIII.5]{feller1991introduction}.
\end{rem}

\begin{rem}
  Assertion~\ref{item:lem:equivlimsupconditionGen2} in \Cref{lem:equivlimsupconditionGen} can be deduced from the proof of \cite[Lemma~4]{daudel2022Alpha}. Here we provided a simpler
  proof of this assertion and, more importantly, we obtained the novel assertions~\ref{item:lem:equivlimsupconditionGen1}--\ref{item:lem:equivlimsupconditionGen6}. 
\end{rem}

\subsection{Asymptotic Behavior of Sample Means Ratios: Key Special Cases}

In this section, we investigate the asymptotic behavior of sample means ratios in some special cases that are pivotal to prove
\Cref{thm:ratios-limits-optimal-cond}. We start with a preliminary lemma.

\begin{lem} \label{lem:ReexpressingExpectations}
 Suppose that~(\ref{eq:simple-cond-ratio-equiv}) holds for some $\eta>0$ and let $\mu>0$. Then the following assertions hold.
 \begin{enumerate}[label=(\roman*)]
 \item \label{item:lem:conslimsupconditionGen7prereq2}
 If $X$ is valued in $\rset_+$, then the two following identities
 hold for $N\geq1$ and $N\geq2$ respectively:
  \begin{align}
    \label{eq:lem:equivlimsupconditionGen2-N1}
&\PE\lr{\frac{\overline{X}_N}{\lr{\overline{W}_N}^\mu}} = \int_0^\infty
\frac{t^{\mu-1}}{\Gamma(\mu)}\,\phi_1\lr{\frac t{N}}\;\lr{\phi\lr{\frac t{N}}}^{N-1}\rmd t,\\
    \label{eq:lem:equivlimsupconditionGen2bis-N2}
&\PE\lr{\frac{\overline{X}^{(i\neq j)}_N}{\lr{\overline{W}_N}^\mu}} = \int_0^\infty
\frac{t^{\mu-1}}{\Gamma(\mu)}\lr{\phi_1\lr{\frac t{N}}}^2\lr{\phi\lr{\frac t{N}}}^{N-2}\rmd t.
  \end{align}
\item \label{item:lem:conslimsupconditionGen7} If $\PE(|X|)<\infty$,
  then there exists $N_0\geq 3$ such that for all $N\geq N_0$, the
  identities~(\ref{eq:lem:equivlimsupconditionGen2-N1})
  and~(\ref{eq:lem:equivlimsupconditionGen2bis-N2}) hold with well
  defined and finite integrals. Moreover the following identity holds for all
  $N\geq N_0$:
  \begin{equation}
    \label{eq:identity-doubleproduit-centered}
    \PE\lr{\frac{\overline{X^c}^{(i\neq
          j)}_N}{\lr{\overline{W}_N}^\mu}}=\PE\lr{\frac{\overline{X}^{(i\neq
          j)}_N}{\lr{\overline{W}_N}^\mu}}-2\PE\lr{\frac{\PE(X)\,\overline{X}_N}{\lr{\overline{W}_N}^\mu}}
    +\PE\lr{\frac{\lr{\PE(X)}^2}{\lr{\overline{W}_N}^\mu}}\;.
  \end{equation}
\end{enumerate}
\end{lem}

\begin{proof} We prove \ref{item:lem:conslimsupconditionGen7prereq2} and \ref{item:lem:conslimsupconditionGen7} sequentially.

\begin{enumerateList}
\item \textbf{Proof of \ref{item:lem:conslimsupconditionGen7prereq2}.} 
Proceeding as in~(\ref{eq:lem:equivlimsupconditionGen1}) and using that
$x^{-\mu}=(\Gamma(\mu))^{-1}\,\int_0^\infty
t^{\mu-1}\rme^{-x\,t}\;\rmd t$ for all $x>0$, we have by
Tonelli's theorem (since $\overline{X}_N$ and $\overline{W}_N$ are
nonnegative): for all $N\geq1$, \looseness=-1
\begin{align}
\nonumber
\PE\lr{\frac{\overline{X}_N}{\lr{\overline{W}_N}^\mu}} & =  (\Gamma(\mu))^{-1}\,\PE\lr{\overline{X}_N \int_0^\infty
t^{\mu-1}\;\rme^{-\overline{W}_N\,t}\rmd
                                                         t}\\
  \nonumber 
  & = (\Gamma(\mu))^{-1}\,\PE\lr{X_1 \int_0^\infty
t^{\mu-1}\;\rme^{-\overline{W}_N\,t}\rmd
                                                         t} \\ \nonumber
& =  (\Gamma(\mu))^{-1}\,\int_0^\infty
t^{\mu-1}\lr{\PE\lr{X_1\rme^{-\overline{W}_N\,t}}} \rmd
                                                         t\\
  \nonumber
  &= (\Gamma(\mu))^{-1}\,\int_0^\infty
t^{\mu-1}\,\phi_1(t/N)\,\lr{\phi(t/N)}^{N-1}\rmd
                                           t.
\end{align}
As for (\ref{eq:lem:equivlimsupconditionGen2bis-N2}), it is obtained similarly by using that for all $t>0$ and $N\geq2$,
$$
\frac2{N(N-1)}\sum_{1\leq i<j\leq N}\PE\lr{X_iX_j\rme^{-\overline{W}_N\,t}}=\lr{\phi_1\lr{\frac t{N}}}^2\lr{\phi\lr{\frac t{N}}}^{N-2}\;,
$$
so that 
\begin{align*}
\PE\lr{\frac{\overline{X}^{(i\neq j)}_N}{\lr{\overline{W}_N}^\mu}} =  (\Gamma(\mu))^{-1}\,\int_0^\infty
t^{\mu-1}\,\phi_1(t/N)^2\,\lr{\phi(t/N)}^{N-2}\rmd t.
\end{align*}

\item \textbf{Proof of \ref{item:lem:conslimsupconditionGen7}.}
Setting $u=(N-1) t/N$ for $N\geq2$ and $u=(N-2)t/N$ for $N\geq3$
respectively yield that~(\ref{eq:lem:equivlimsupconditionGen2-N1})
and~(\ref{eq:lem:equivlimsupconditionGen2bis-N2}) can be equivalently
written as
  \begin{align}
    \label{eq:lem:equivlimsupconditionGen2}
    &\PE\lr{\frac{\overline{X}_N}{\lr{\overline{W}_N}^\mu}} = \lr{\frac N{N-1}}^\mu\int_0^\infty
\frac{u^{\mu-1}}{\Gamma(\mu)}\,\phi_1\lr{\frac u{N-1}}\;\lr{\phi\lr{\frac u{N-1}}}^{N-1}\rmd u,\\
    \label{eq:lem:equivlimsupconditionGen2bis}
&    \PE\lr{\frac{\overline{X}^{(i\neq j)}_N}{\lr{\overline{W}_N}^\mu}} = \lr{\frac N{N-2}}^\mu\int_0^\infty
\frac{u^{\mu-1}}{\Gamma(\mu)}\lr{\phi_1\lr{\frac u{N-2}}}^2\lr{\phi\lr{\frac u{N-2}}}^{N-2}\rmd u.
  \end{align}
  Based on the proof \ref{item:lem:conslimsupconditionGen7prereq2}, ~(\ref{eq:lem:equivlimsupconditionGen2}) will be valid when $X$ is not valued in $\rset_+$ provided that Fubini's theorem applies. In other words, (\ref{eq:lem:equivlimsupconditionGen2}) will hold if $\PE\lr{{\overline{\lrav{X}}_N}/{\lr{\overline{W}_N}^\mu}}<\infty$ where $\overline{\lrav{X}}_N=N^{-1}\sum_{i=1}^N\lrav{X_i}$. Denoting $\tilde{\phi}_1(v)=\PE\lr{|X|\rme^{-v\,W}}$, we have by applying (\ref{eq:lem:equivlimsupconditionGen2}) that 
  $$
  \PE\lr{\frac{\overline{|X|}_N}{\lr{\overline{W}_N}^\mu}} = \lr{\frac N{N-1}}^\mu\frac1{\Gamma(\mu)}\int_0^\infty
  u^{\mu-1}\,\tilde{\phi}_1\lr{\frac u{N-1}}\;\lr{\phi\lr{\frac u{N-1}}}^{N-1} \rmd u
  \;.
  $$
  Since $\PE(|X|)<\infty$ by assumption, the function $\tilde{\phi}_1$
  is bounded on $\rset_+$. Furthermore,
  \ref{item:lem:conslimsupconditionGen7prereq0} of \Cref{lem:tech2}
  with $m=\mu+1$
  yields that: there exist $c>0$ and $N_0 \geq 2$
  such that, for all $N\geq N_0$, \looseness=-1
  \begin{align}\label{eq:boundedFubini}
\left|u^{\mu-1}\,\tilde{\phi}_1\lr{\frac u{N-1}}    \lr{\phi\lr{\frac u{N-1}}}^{N-1} \right| \leq  \PE(|X|)~u^{\mu-1} \lrb{1\wedge \lr{c\,u^{-\mu-1}}}, \quad u>0.
  \end{align}
Consequently
$\PE\lr{{\overline{\lrav{X}}_N}/{\lr{\overline{W}_N}^\mu}}<\infty$ for
all $N\geq N_0$. Hence, for all $N\geq N_0$ Fubini's theorem applies
and (\ref{eq:lem:equivlimsupconditionGen2}) holds, and so does its
equivalent form~(\ref{eq:lem:equivlimsupconditionGen2-N1}). The proof is concluded by applying the same line of reasoning to (\ref{eq:lem:equivlimsupconditionGen2bis}) and using that there exist $c>0$ and $N_0 \geq 3$ such that, for all $N\geq N_0$, \looseness=-1
  \begin{align}\label{eq:boundedFubiniTwo}
\left|u^{\mu-1}\,\tilde{\phi}_1\lr{\frac u{N-2}}^2    \lr{\phi\lr{\frac u{N-2}}}^{N-2} \right| \leq \lr{ \PE(|X|)}^2~u^{\mu-1} \lrb{1\wedge \lr{c\,u^{-\mu-1}}}, \quad u>0.
  \end{align}
  Hence we get~(\ref{eq:lem:equivlimsupconditionGen2bis-N2}). We
  conclude by observing that the
  identity~(\ref{eq:identity-doubleproduit-centered}) is obtained by
  applying~(\ref{eq:lem:equivlimsupconditionGen2-N1})
  and~(\ref{eq:lem:equivlimsupconditionGen2bis-N2}) to $X$ and $X_i$
  replaced by $X^c=X-\PE(X)$ and $X_i^c$ by $X_i-\PE(X)$ (with
  corresponding similar upper bounds in~(\ref{eq:boundedFubini}) and~(\ref{eq:boundedFubiniTwo})) and noting that
  $\phi_1^c(u)=\PE\lr{X^c\rme^{-u\,W}}$ satisfies
    $\phi_1^c=\phi_1-\PE(X)\,\phi$.  
\end{enumerateList}
\end{proof}
\begin{prop}\label{prop:conslimsupconditionGen}
  Suppose that~(\ref{eq:simple-cond-ratio-equiv}) holds for some $\eta>0$ and let $\mu>0$. Then the following assertions hold.
\begin{enumerate}[label=(\roman*)]
\item \label{item:lem:conslimsupconditionGen0} If $\PE(W)=1$ and $\PE(W^2)<\infty$, then
  $$
  \lim_{N\to\infty}N\,\lr{\PE\lr{\lr{\overline{W}_N}^{-\mu}}-1}= \frac{(\mu+1)\mu}2\,\mathbb{V}\lr{W}\;.
  $$
\item \label{item:lem:conslimsupconditionGen7ter} If $\PE(|X|)<\infty$, then
  \begin{align*}
  \lim_{N\to\infty}\PE\lr{\frac{\overline{X}_N}{\lr{\overline{W}_N}^\mu}}=\frac{\PE(X)}{\lr{\PE(W)}^\mu}
\quad\text{and}\quad\lim_{N\to\infty}\PE\lr{\frac{\overline{X}^{(i\neq
      j)}_N}{\lr{\overline{W}_N}^\mu}}=
  \frac{\lr{\PE(X)}^2}{\lr{\PE(W)}^\mu} \;.
  \end{align*}
\item \label{item:lem:conslimsupconditionGen8} If $\PE(X)=0$ and
  $\PE(|X|\,W)<\infty$, then
  $$
  \lim_{N\to\infty}N\,\PE\lr{\frac{\overline{X}_N}{\lr{\overline{W}_N}^\mu}}=-\frac{\mu\;\mathbb{C}\mathrm{ov}\lr{X,W}}{\lr{\PE(W)}^{\mu+1}} \;.
  $$

\item \label{item:possibleImprovbis} If  $\PE(X) = 0$ and $\PE(|X|^2)<\infty$, then
$$
\lim_{N\to\infty} \PE\lr{\frac{\lr{\overline{X}_N}^2}{\lr{\overline{W}_{N}}^{\mu}}}= 0 \;.
$$

\item \label{item:possibleImprovTwo} If $\PE(X) = 0$ and
  $\PE\lr{|X|\,W^\delta})<\infty$ with $0<\delta<1$, then
\begin{align*}
  & \lim_{N\to\infty}N^{\delta}\,\PE\lr{\frac{\overline{X}_N}{\lr{\overline{W}_N}^\mu}}=0.
\end{align*}
Further assuming that $\delta=1/2$ and $\PE(|X|^2)<\infty$, it holds that
\begin{align}
  \label{eq:item:possibleImprovTwo2}
  & \lim_{N\to\infty}N\,\PE\lr{\frac{\lr{\overline{X}_N}^2}{\lr{\overline{W}_{N}}^{\mu}}}=\frac{\mathbb{V}\lr{X}}{\lr{\PE(W)}^{\mu}}\;.
\end{align}
\end{enumerate}  
\end{prop}
\begin{proof} We prove Assertions
  \ref{item:lem:conslimsupconditionGen0}--\ref{item:possibleImprovTwo}
  successively.
\begin{enumerateList}
\item \textbf{Proof of \ref{item:lem:conslimsupconditionGen0}.} 
Using \eqref{eq:lem:equivlimsupconditionGen1} and the definition of the Euler Gamma function we obtain that, for all $N\in\mathbb{N}^*$, 
  \begin{align*}
  N\,\lr{\PE\lr{\lr{\overline{W}_N}^{-\mu}}-1}&=
  (\Gamma(\mu))^{-1}\,\int_0^\infty
                                                u^{\mu-1}\,N\,\lr{\lr{\phi\lr{\frac
                                                uN}}^N-\rme^{-u}}\;\rmd u\\
    &=:  (\Gamma(\mu))^{-1}\,\lr{A_N+B_N}\;,
  \end{align*}
  where $A_N$ and $B_N$ are defined by separating the integral in two
  parts and setting 
  \begin{align*}
    & A_N = \int_0^\infty u^{\mu-1}\,N\,\lr{\lr{\phi\lr{\frac uN}}^N-\rme^{-u}} \mathds{1}_{u \in \lrb{0,\frac{N}{\PE\lr{W^2}}}} \;\rmd u \\
    & B_N =  \int_0^\infty u^{\mu-1}\,N\,\lr{\lr{\phi\lr{\frac uN}}^N-\rme^{-u}}  \mathds{1}_{u \in \lr{\frac{N}{\PE\lr{W^2}},\infty}} \;\rmd u.
  \end{align*}
Using next \eqref{eq:lem:conslimsupconditionGen7prereq1biseq} and~\eqref{item:lem:conslimsupconditionGen7prereq1ter} of \Cref{lem:tech2}, we obtain that the dominated convergence theorem applies and leads to
  $$
  \lim_{N\to\infty}A_N=\frac{\mathbb{V}\lr{W}}2\,\int_0^\infty u^{\mu+1}\;\rme^{-u}\;\rmd u=\frac{\mathbb{V}\lr{W}\,\Gamma(\mu+2)}2.
  $$
  As for $B_N$, \ref{item:lem:conslimsupconditionGen7prereq0} of \Cref{lem:tech2} with $m=\mu+2$ yields that: there exist
  $c>0$ and $N_0\geq1$ such that, for all $N\geq N_0$,
  $$
  \lr{\phi\lr{\frac uN}}^N-\rme^{-u}\leq \lr{\phi\lr{\frac
      uN}}^N\leq c\; u^{-\mu-2}\;.
  $$
  Combining this with \ref{item:lem:conslimsupconditionGen7prereq1bis0} of
  \Cref{lem:tech2}, gives us that, for all $N\geq N_0$,
  $$
  0\leq B_N\leq c\,N\, \int_{\frac{N}{\PE\lr{W^2}}}^\infty u^{-3}\;\rmd
  u \leq \frac c2\;\lr{\PE\lr{W^2}}^2\;N^{-1}\;.
  $$
  Thus $B_N$ converges to zero and the given limit is proved using the
  well known identity $\Gamma(\mu+2)=(\mu+1)\,\mu\,\Gamma(\mu)$.

\item \textbf{Proof of \ref{item:lem:conslimsupconditionGen7ter}.} Let
  us first consider the case where $\PE(X)=0$, that is, we show that
  \begin{align}
  \lim_{N\to\infty}\PE\lr{\frac{\overline{X}_N}{\lr{\overline{W}_N}^\mu}}=
   \lim_{N\to\infty}\PE\lr{\frac{\overline{X}^{(i\neq j)}_N}{\lr{\overline{W}_N}^\mu}}=0 \;. \label{item:lem:conslimsupconditionGen7bis}
\end{align}
By \ref{item:lem:conslimsupconditionGen7} of
\Cref{lem:ReexpressingExpectations}, we can use the
identities~(\ref{eq:lem:equivlimsupconditionGen2-N1})
and~(\ref{eq:lem:equivlimsupconditionGen2bis-N2}) for $N$ large
enough, or their equivalent versions~(\ref{eq:lem:equivlimsupconditionGen2})
and~(\ref{eq:lem:equivlimsupconditionGen2bis}).
The desired result
\eqref{item:lem:conslimsupconditionGen7bis} will follow if we can
apply the dominated convergence theorem in the right-hand sides
of~(\ref{eq:lem:equivlimsupconditionGen2})
and~(\ref{eq:lem:equivlimsupconditionGen2bis}). Using
\ref{item:lem:conslimsupconditionGen7prereq0} of \Cref{lem:tech2} and
since $\PE(|X|) < \infty$ ($\PE(X) = 0$ by assumption), we have that:
for all $m > \mu$, there exist $c>0$ and $N_0\geq 3$ such that, for
all $N\geq N_0$, \eqref{eq:boundedFubini} and
\eqref{eq:boundedFubiniTwo} hold. Consequently we can apply the
dominated convergence theorem: using
\ref{item:lem:conslimsupconditionGen7prereq3} of \Cref{lem:tech3} with
$\phi_1(0)=\PE(X)=0$ and the fact that $\phi$ is bounded by $1$ yields
\eqref{item:lem:conslimsupconditionGen7bis}. Now considering the case
where $\PE(X)\neq0$, we get from the first limit in
\eqref{item:lem:conslimsupconditionGen7bis} that \looseness=-1
\begin{align*}
    \lim_{N\to\infty}\PE\lr{\frac{\overline{X}_N - \PE(X)}{\lr{\overline{W}_N}^\mu}}=0
\end{align*}
and the first desired result follows using \ref{item:lem:equivlimsupconditionGen6} of \Cref{lem:equivlimsupconditionGen}.
As for the second one, applying the second limit in
\eqref{item:lem:conslimsupconditionGen7bis} with $X$ and $X_i$
replaced by $X^c=X-\PE(X)$ and $X_i^c=X_i -\PE(X)$ for all $i = 1
\ldots N$ yields
$$
\lim_{N \to \infty} \PE\lr{\frac{\overline{X^c}^{(i\neq j)}_N}{\lr{\overline{W}_{N}}^{\mu}}} = 0.
$$
Under the present assumptions, we can apply
\ref{item:lem:conslimsupconditionGen7} of
\Cref{lem:ReexpressingExpectations} and, in particular
Identity~(\ref{eq:identity-doubleproduit-centered}), which, inserted in the
previous limit, gives
\begin{align*}
\lim_{N \to \infty} \lrb{ \PE\lr{\frac{\overline{X}^{(i\neq j)}_N}{\lr{\overline{W}_{N}}^{\mu}}} - 2 \PE(X) \PE \lr{\frac{\overline{X}_N}{\lr{\overline{W}_{N}}^{\mu}}} + \PE(X)^2 \PE\lr{\frac{1}{\lr{\overline{W}_{N}}^{\mu}}}} = 0
\end{align*}
The limits of $\PE
\lr{{\overline{X}_N}/{\lr{\overline{W}_{N}}^{\mu}}}$ and
  $\PE\lr{{\lr{\overline{W}_{N}}^{-\mu}}}$ as $N\to\infty$ have been
established previously to be ${\PE(X)}/{\lr{\PE(W)}^\mu}$ and
$\lr{\PE(W)}^{-\mu}$, and this leads to the second desired limit.

\item \textbf{Proof of \ref{item:lem:conslimsupconditionGen8}.} Let us define
\begin{equation}
  \label{eq:defIN}
I_N:=N\,\PE\lr{\frac{\overline{X}_N}{\lr{\overline{W}_N}^\mu}}+\PE\lr{X\,W}\lr{\frac N{N-1}}^{\mu+1}\,\frac{\Gamma(\mu+1)}{\Gamma(\mu)}\PE\lr{\frac1{\lr{\overline{W}_{N-1}}^{\mu+1}}}\;.  
\end{equation}
By \ref{item:lem:conslimsupconditionGen7ter} with $X = 1$, the second term in the right-hand side of~(\ref{eq:defIN}) converges to \looseness=-1
$$
\PE\lr{X\,W}\frac{\Gamma(\mu+1)}{\Gamma(\mu)}\frac1{\lr{\PE(W)}^{\mu+1}}=\frac{\mu\;\PE\lr{X\,W}}{\lr{\PE(W)}^{\mu+1}}\;,
$$
where we used the well known identity of the Euler Gamma
function. Now using \eqref{eq:lem:equivlimsupconditionGen1} with $N$ replaced by $N-1$ and $\mu$ replaced by $\mu+1$, we have: for all $N\geq2$,
$$
\PE\lr{\frac1{\lr{\overline{W}_{N-1}}^{\mu+1}}} = \frac1{\Gamma(\mu+1)}\,\int_0^\infty
u^{\mu}\lr{\phi\lr{\frac u{N-1}}}^{N-1}\;\rmd u
$$
and by pairing it up with \ref{item:lem:conslimsupconditionGen7} of \Cref{lem:ReexpressingExpectations} ($\PE(|X|)< \infty$ since $\PE(X) = 0$), we get that for $N$ large enough 
$$
I_N=N \,\lr{\frac N{N-1}}^{\mu}\,\int_0^\infty
\frac{u^{\mu-1}}{\Gamma(\mu)}\,\lr{\phi_1\lr{\frac u{N-1}}+\frac {u\,\PE\lr{X\,W}}{N-1}}\;\lr{\phi\lr{\frac u{N-1}}}^{N-1}\;\rmd u\;.
$$
Therefore, to conclude the proof of~\ref{item:lem:conslimsupconditionGen8}, it remains to show that $I_N$ converges to zero as $N\to\infty$, which is implied by
\begin{equation}
  \label{eq:lem:conslimsupconditionGen8remains}
  \lim_{N\to\infty}\int_0^\infty
\frac{u^{\mu-1}}{\Gamma(\mu)}\,\lrcb{N\,\lr{\phi_1\lr{\frac u{N-1}}+\frac {u\,\PE\lr{X\,W}}{N-1}}}\;\lr{\phi\lr{\frac u{N-1}}}^{N-1}\;\rmd u=0\;.
\end{equation}
Using~\ref{item:lem:conslimsupconditionGen7prereq0}  of~\Cref{lem:tech2} and the first claim in \ref{item:lem:conslimsupconditionGen7prereq4} of \Cref{lem:tech3}, we get that the absolute value of the integrand
in~(\ref{eq:lem:conslimsupconditionGen8remains}) is bounded by an
integrable function for $N$ large enough, therefore the dominated convergence theorem
applies. Furthermore, using~\ref{item:lem:conslimsupconditionGen7prereq1} of \Cref{lem:tech2} and the second claim in~\ref{item:lem:conslimsupconditionGen7prereq4}
 of~\Cref{lem:tech3}, the point-wise limit of the integrand is zero.
Thus~(\ref{eq:lem:conslimsupconditionGen8remains}) follows from the
dominated convergence convergence theorem and the proof of \ref{item:lem:conslimsupconditionGen8} is concluded. \looseness=-1

\item Proof of \ref{item:possibleImprovbis}. Using  the
  definitions introduced in~(\ref{eq:double-produit}), we have
  $$
   \PE\lr{\frac{\lr{\overline{X}_N}^2}{\lr{\overline{W}_{N}}^{\mu}}}=
   \frac1N
   \PE\lr{\frac{\lr{\overline{X^2}_N}}{\lr{\overline{W}_{N}}^{\mu}}} +
   \frac{N(N-1)}{N^2}\PE\lr{\frac{\overline{X}^{(i\neq j)}_N}{\lr{\overline{W}_{N}}^{\mu}}}\;.
   $$
   Under the assumption $\PE(|X|^2) < \infty$,~\ref{item:lem:conslimsupconditionGen7ter} applies to both
expectations in the right-hand side (with $X$ and $X_i$ replaced by their squares for the first
one) and yields the result.

\item \textbf{Proof of \ref{item:possibleImprovTwo}.} Let
  $0<\delta<1$. Again we can
  apply the identities of \ref{item:lem:conslimsupconditionGen7} in
  \Cref{lem:ReexpressingExpectations} for $N$ large enough. Here we
  pick the equivalent form~(\ref{eq:lem:equivlimsupconditionGen2bis})
  to write, for $N$ large enough, 
\begin{align*}
N^{\delta}\PE\lr{\frac{\overline{X}_N}{\lr{\overline{W}_N}^\mu}} = \lr{\frac N{N-1}}^{\mu+\delta}\int_0^\infty
\frac{u^{\mu-1}}{\Gamma(\mu)} (N-1)^{\delta}\,\phi_1\lr{\frac u{N-1}}\;\lr{\phi\lr{\frac u{N-1}}}^{N-1}\rmd u.
\end{align*}
The first claim then follows from the dominated convergence theorem
since \ref{item:lem:conslimsupconditionGen7prereq0} of
\Cref{lem:tech2} paired up with
\ref{item:lem:conslimsupconditionGen7prereq5} of \Cref{lem:tech3}
yield: there exist $c>0$ and $N_0\geq3$ such that, for all $N\geq N_0$,
\begin{align*}
& \left|{u^{\mu-1}} (N-1)^{\delta}\,\phi_1\lr{\frac u{N-1}}\;\lr{\phi\lr{\frac u{N-1}}}^{N-1}\right| \leq u^{\mu+\delta-1}\PE(|X|{W}^\delta) \lrb{ 1\wedge \lr{c\,u^{-\mu-\delta-1}}} \\
&\text{and}\quad  \lim_{N \to \infty} \left|{u^{\mu-1}} (N-1)^{\delta}\,\phi_1\lr{\frac u{N-1}}\;\lr{\phi\lr{\frac u{N-1}}}^{N-1}\right| = 0,
\end{align*}
where we have also used the fact that $\phi$ is bounded by $1$.

Moving on to the second claim where  $\delta=1/2$ and $\PE(|X|^2)<\infty$, we observe that
$$
N\,\PE\lr{\frac{\lr{\overline{X}_N}^2}{\lr{\overline{W}_{N}}^{\mu}}}=\PE\lr{\frac{\overline{X^2}_N}{\lr{\overline{W}_{N}}^{\mu}}}+
  (N-1)\,\PE\lr{\frac{\overline{X}^{(i\neq j)}_N}{\lr{\overline{W}_{N}}^{\mu}}}\;,
$$
where $\overline{X^2}_N$ and $\overline{X}^{(i\neq j)}_N$ are defined in~(\ref{eq:double-produit}).
By applying~\ref{item:lem:conslimsupconditionGen7ter} with $X$ and
$X_i$ replaced by their squares the first expectation in the
right-hand side converges to the right-hand side
of~(\ref{eq:item:possibleImprovTwo2}). Thus the proof of~(\ref{eq:item:possibleImprovTwo2}) will be complete if we can show that
\begin{align*}
  \lim_{N \to\infty}N\,\PE\lr{\frac{\overline{X}^{(i\neq j)}_N}{\lr{\overline{W}_N}^\mu}}=0\;. 
\end{align*}
Under our assumptions, we can apply
\ref{item:lem:conslimsupconditionGen7} of
\Cref{lem:ReexpressingExpectations}, notably giving that, the previous
display is equivalent to
\begin{align*}
& \lim_{N \to\infty} \lr{\frac N{N-2}}^{\mu+1}\int_0^\infty
\frac{u^{\mu-1}}{\Gamma(\mu)}\lr{(N-2)^{1/2} \phi_1\lr{\frac u{N-2}}}^2\lr{\phi\lr{\frac u{N-2}}}^{N-2}\rmd u=0\;.
\end{align*}
Again  \ref{item:lem:conslimsupconditionGen7prereq0} of
\Cref{lem:tech2} paired up with
\ref{item:lem:conslimsupconditionGen7prereq5} of \Cref{lem:tech3}
(with $\delta=1/2$)
allow us to apply the
dominated convergence theorem, which yields the claimed limit.
\end{enumerateList}

\end{proof}
\subsection{Proof of \Cref{thm:ratios-limits-optimal-cond}}

By \ref{item:lem:equivlimsupconditionGen1} of \Cref{lem:equivlimsupconditionGen}, the condition (\ref{eq:simple-cond-ratios}) is equivalent to the condition~(\ref{eq:simple-cond-ratio-equiv}). Consequntly, we will use the facts established in \Cref{prop:conslimsupconditionGen} (which relies
on~(\ref{eq:simple-cond-ratio-equiv})) to prove
\Cref{thm:ratios-limits-optimal-cond} (which assumes (\ref{eq:simple-cond-ratios})).

\begin{enumerateList}
\item \textbf{Proof of~\ref{item:ratios-limits-optimal-cond1}.} This is exactly~\ref{item:lem:conslimsupconditionGen7ter} of
\Cref{prop:conslimsupconditionGen}.
\item \textbf{Proof of~\ref{item:ratios-limits-optimal-cond2bis}.} We first start with the general case $\mu > 0$. Multiplying both sides of~(\ref{eq:ratios-limits-optimal-cond2bis}) by
  $\lr{\PE(W)}^\mu$ and replacing $W/\PE(W)$, $W_i/\PE(W)$ by $W$ and
  $W_i$ respectively, we see that we can assume
  $\PE(W)=1$ without loss of generality, in which
  case~(\ref{eq:ratios-limits-optimal-cond2bis}) reads
  \begin{align}
    \label{eq:ratios-limits-optimal-cond2ter}
\lim_{N\to\infty}N\,\lr{\PE\lr{\frac{\overline{X}_{N}}{\lr{\overline{W}_{N}}^\mu}}
        -\PE(X)}=\frac{(\mu+1)\mu\,\PE(X)}2\,\mathbb{V}\lr{W}-\mu\;\mathbb{C}\mathrm{ov}\lr{X,W}\;.
\end{align}
Under our assumptions and since $\PE(W)=1$, we can apply \ref{item:lem:conslimsupconditionGen0} and \ref{item:lem:conslimsupconditionGen8} of \Cref{prop:conslimsupconditionGen}, so that \looseness=-1
\begin{align}
    \label{eq:ratios-limits-optimal-cond2ter2}
&\lim_{N\to\infty}N\,\lr{\frac{\PE(X)}{\lr{\overline{W}_{N}}^\mu}-\PE(X)}=\frac{(\mu+1)\mu\,\PE(X)}2\,\mathbb{V}\lr{W}\;, \\
    \label{eq:ratios-limits-optimal-cond2ter1}
&\lim_{N\to\infty}N\,\PE\lr{\frac{\overline{X^c}_{N}}{\lr{\overline{W}_{N}}^\mu}}=-\mu\;\mathbb{C}\mathrm{ov}\lr{X,W}\;,
\end{align}
where $\overline{X^c}_N$ is defined in~(\ref{eq:centered-empirical}). Observing that
$$
\frac{\overline{X^c}_{N}}{\lr{\overline{W}_{N}}^\mu}=\frac{\overline{X}_{N}}{\lr{\overline{W}_{N}}^\mu}-\frac{\PE(X)}{\lr{\overline{W}_{N}}^\mu}\;,
$$
we see that~(\ref{eq:ratios-limits-optimal-cond2ter}) is obtained by
summing (\ref{eq:ratios-limits-optimal-cond2ter2}) and (\ref{eq:ratios-limits-optimal-cond2ter1}).

Moving on to the case $\mu = 1$, we can assume once again that $\PE(W)=1$ without loss of generality, in which
  case~(\ref{eq:ratios-limits-optimal-cond2bis:alter}) becomes
  \begin{equation}
    \label{eq:ratios-limits-optimal-cond2}
\lim_{N\to\infty}N\,\lr{\PE\lr{\frac{\overline{X}_{N}}{\overline{W}_{N}}}
        -\PE(X)}=\PE\lr{W\lrb{\PE(X)\,W-X}}\;.
  \end{equation}
Now replacing $X$ and $X_i$ by $\lr{X-\PE(X)W}$ and
$\lr{X_i-\PE(X)W_i}$, respectively, we can further assume that
$\PE(X)=0$ and~(\ref{eq:ratios-limits-optimal-cond2}) now reads
\begin{align*}
  \lim_{N\to\infty} N\, \PE\lr{\frac{\overline{X}_{N}}{\overline{W}_{N}}}
    = -\PE\lr{X\,W}\;,
\end{align*}
under the assumptions $\PE(W)=1$, $\PE(X)=0$ and $\PE(|X|\,W)<\infty$, which follows directly from an application of \ref{item:lem:conslimsupconditionGen8} in \Cref{prop:conslimsupconditionGen} with $\mu=1$.

\item \textbf{Proof of \ref{item:ratios-limits-optimal-cond3}.} 
By Minkowski's inequality, it holds that
   \begin{align*}
     \lr{\mathbb{V}\lr{\frac{\overline{X}_N}{(\overline{W}_{N})^{\mu}}}}^{\frac{1}{2}}& 
     \leq {\lr{\mathbb{V}\lr{\frac{\overline{X}_N-\PE(X)}{(\overline{W}_{N})^{\mu}}}}^{\frac12}+\lrav{\PE(X)}\,\lr{\mathbb{V}\lr{\frac{1}{(\overline{W}_{N})^{\mu}}}}^{\frac12}}\\
&    \leq {\lr{\PE\lr{\frac{\lr{\overline{X^c}_N}^2}{(\overline{W}_{N})^{2\mu}}}}^{\frac12}+\lrav{\PE(X)}\,\lr{\mathbb{V}\lr{\frac{1}{(\overline{W}_{N})^{\mu}}}}^{\frac12}}\;.
   \end{align*}
The variance in the second term of the sum in the last right-hand side converges to $0$ by applying \ref{item:lem:equivlimsupconditionGen6} of
 \Cref{lem:equivlimsupconditionGen} twice. As for the first term, it
 converges to zero by applying~\ref{item:possibleImprovbis} of
 \Cref{prop:conslimsupconditionGen}  (with $\mu$, $X_i$ and $X$
 replaced by  $2\mu$, $X^c$ and $X_i^c$), which concludes the proof of  \ref{item:ratios-limits-optimal-cond3}. 

\item \textbf{Proof of \ref{item:ratios-limits-optimal-cond4}.} We start with the general case $\mu > 0$. First observe that we can assume that $\PE(W)=1$ again without loss of generality. For all $N\in\mathbb{N}^*$,
\begin{equation}
  \label{eq:variance-decomp-3terms}
\mathbb{V}\lr{\frac{\overline{X}_N}{(\overline{W}_{N})^{\mu}}}=
V_N^{(1)}+\lr{\PE(X)}^2\,V_N^{(2)}+2\,\PE(X)\,C_N\;,  
\end{equation}
where
\begin{align*}
V_N^{(1)}:=
  \mathbb{V}\lr{\frac{\overline{X^c}_N}{(\overline{W}_{N})^{\mu}}}\;,\quad
  V_N^{(2)}:=\mathbb{V}\lr{\lr{\overline{W}_{N}}^{-\mu}}
  \quad\text{and}\quad C_N:=\mathbb{C}\mathrm{ov}\lr{\frac{\overline{X^c}_N}{(\overline{W}_{N})^{\mu}},\lr{\overline{W}_{N}}^{-\mu}}\;.  
\end{align*}
We treat each term separately.
\begin{itemize}
  \item First variance term: $V_N^{(1)}$. Note that $\PE(W^2)<\infty$ and
    $\PE(|X|^2)<\infty$ imply that $\PE(|X|\,W^{1/2})$ and
    $\PE(|X|\,W)$ are finite. We can thus
    apply~(\ref{eq:item:possibleImprovTwo2})
    and~\ref{item:lem:conslimsupconditionGen8} 
    of \Cref{prop:conslimsupconditionGen}  (with $\mu$, $X_i$  and $X$
 replaced by $2\mu$, $X_i^c$ and $X^c$), which gives us the asymptotic
behavior of the two expectations in the expression
$$
V_N^{(1)}=\PE\lr{\frac{\lr{\overline{X^c}_N}^2}{(\overline{W}_{N})^{2\mu}}}-\lr{\PE\lr{\frac{\overline{X^c}_N}{(\overline{W}_{N})^{\mu}}}}^2\;.
$$
We obtain
\begin{align} \label{eq:V2N-prereq0}
  \lim_{N\to\infty}N\,V_N^{(1)} &=\mathbb{V}\lr{X}\;.
\end{align}

\item Second variance term: $V_N^{(2)}$. For all $N\in\mathbb{N}^*$, 
  $$
  \mathbb{V}\lr{\lr{\overline{W}_{N}}^{-\mu}}=\lrcb{\PE\lr{\lr{\overline{W}_{N}}^{-2\mu}}-1}-\lrcb{\lr{1+\lrb{\PE\lr{\lr{\overline{W}_{N}}^{-\mu}}-1}}^2-1}\;,
$$
Applying \ref{item:lem:conslimsupconditionGen0} of
\Cref{prop:conslimsupconditionGen} to get the asymptotic behavior of
the expectations in the right-hand side, we get that
\begin{align}\nonumber
  \lim_{N\to\infty}N\,\mathbb{V}\lr{\lr{\overline{W}_{N}}^{-\mu}}&=(2\mu+1)\,\mu\,\mathbb{V}\lr{W}-(\mu+1)\,\mu\,\mathbb{V}\lr{W}\\
\label{eq:V2N-prereq}  &=\mu^2\,\mathbb{V}\lr{W}\;.
\end{align}

\item Covariance term: $C_N$.
$$
C_N=\PE\lr{\frac{\overline{X^c}_N}{(\overline{W}_{N})^{2\mu}}}-\PE\lr{\frac{\overline{X^c}_N}{(\overline{W}_{N})^{\mu}}}\,\PE\lr{\lr{\overline{W}_{N}}^{-\mu}}\;.
$$
Applying \ref{item:lem:conslimsupconditionGen8}
of \Cref{prop:conslimsupconditionGen} twice, once
with $\mu$ replaced by $2\mu$ and once again without changing $\mu$, along with
 \ref{item:lem:equivlimsupconditionGen6} of \Cref{lem:equivlimsupconditionGen}, we get that 
\begin{align*}
\lim_{N\to\infty}N\,C_N&=
                         -2\mu\;\mathbb{C}\mathrm{ov}\lr{X,W}+\mu\;\mathbb{C}\mathrm{ov}\lr{X,W}=
    -\mu\;\mathbb{C}\mathrm{ov}\lr{X,W}
    \;.
\end{align*}
\end{itemize}
Combining (\ref{eq:variance-decomp-3terms}) with \eqref{eq:V2N-prereq0}, \eqref{eq:V2N-prereq} and the above limit, we obtain that
\begin{align*}
  \lim_{N\to\infty}N\,\mathbb{V}\lr{\frac{\overline{X}_N}{(\overline{W}_{N})^{\mu}}}=
                                                                                       \mathbb{V}\lr{X}+\lr{\PE(X)}^2\,\mu^2\,\mathbb{V}\lr{W}-2\,\PE(X)\,\mu\;\mathbb{C}\mathrm{ov}\lr{X,W}  \;.
\end{align*}
This is exactly the desired result \ref{item:ratios-limits-optimal-cond4} in the case $\PE(W)=1$, and the proof is completed for a general $\mu > 0$.

Moving on to the case $\mu = 1$, we can assume once more that $\PE(W)= 1$ without loss of generality, in which case~(\ref{eq:limVarFirstOrderMuOne}) reads
\begin{align} \label{eq:NrateInterVar}
    \lim_{N\to\infty}N\,
    \mathbb{V}\lr{\frac{\overline{X}_N}{{\overline{W}_{N}}}} =
    \mathbb{V}\lr{X -\PE(X) W}\;.
\end{align}
Replacing $X$ and $X_i$ by $(X-\PE(X)W)$ and $(X_i - \PE(X) W_i)$,
respectively, we can further assume that $\PE(X) = 0$,
$\PE(X^2)<\infty$ (transposing the original
$\PE(\lrav{\PE(W)X-\PE(X)W}^2)<\infty$ to the new definition of $X$
when $\PE(W)=1$) and \eqref{eq:NrateInterVar} reads 
\begin{align}
\label{eq:item:ratios-limits-optimal-cond4-last}      \lim_{N\to\infty}N\,
    \mathbb{V}\lr{\frac{\overline{X}_N}{{\overline{W}_{N}}}} =
    \mathbb{V}\lr{X}\;.
\end{align}
Now using that  $\PE(|X|^2)<\infty$ and $\PE(W)=1$, which implies
$\PE(|X|\sqrt{W})<\infty$, we can apply the two assertions in 
\ref{item:possibleImprovTwo} of \Cref{prop:conslimsupconditionGen}
with $\mu=1,2$ and $\delta=1/2$ so that
\begin{align*}
  \lim_{N\to\infty} N \lr{\,\PE\lr{\frac{\overline{X}_N}{\lr{\overline{W}_N}}}}^2= 0 \quad\text{and}\quad \lim_{N\to\infty}N\,\PE\lr{\lr{\frac{\overline{X}_N}{\overline{W}_{N}}}^2}={\mathbb{V}\lr{X}}.
\end{align*}
This yields~(\ref{eq:item:ratios-limits-optimal-cond4-last}) and the
proof is completed.
\end{enumerateList}

\subsection{Related works for \Cref{thm:ratios-limits-optimal-cond}}
\label{app:relWThm1}

\begin{enumerateList}
  \item \cite{deligiannidis25} concurrently established a result that is comparable to \eqref{eq:ratios-limits-optimal-cond2bis:alter}: their analysis in \cite[Theorem~2.1]{deligiannidis25} focused on the self-normalized importance sampling estimator, which corresponds to taking $(X,W) = (f(Z) \omega(Z), \omega(Z))$ with $Z \sim q$, where $\omega = \pi/q$ is the importance weight, $\pi$ is the target distribution, $q$ the proposal distribution and $f$ is a function of interest. Letting $Z_1, \ldots, Z_N$ be i.i.d. random variables with the same distribution as $Z$, \eqref{eq:ratios-limits-optimal-cond2bis:alter} then becomes 
\begin{align*}
  \lim_{N \to \infty} N \lr{\PE \lr{\frac{\sum_{i=1}^N w(Z_i) f(Z_i)}{\sum_{i=1}^N  w(Z_i)}} - \PE(w(Z)f(Z)) } = \PE\lr{w(Z)^2\lrb{f(Z) - \PE(w(Z)f(Z))}},
\end{align*}
which is exactly the result obtained in \cite{deligiannidis25}. However, their derivation relies on a distinct proof strategy that necessitates a stronger positive moment condition on $| \mathbb{E}(W)X - \mathbb{E}(X)W | W$.

\item \cite{rainforth2018tighter} and \cite{daudel2022Alpha} directed their analyses exclusively towards the REP gradient estimator of the IWAE and VR-IWAE bound respectively. They showed that  \eqref{eq:ratios-limits-optimal-cond2bis:alter} and \eqref{eq:limVarFirstOrderMuOne} hold for specific choices of $(X,W)$ under much stringent conditions than that of \Cref{thm:ratios-limits-optimal-cond}. More precisely, letting $\varepsilon \sim \tq$, they take 
\begin{align*}
  & X = \w(f(\varepsilon,\phi;x);x)^{1-\alpha} \partial_\psi \log \w(f(\varepsilon,\phi;x);x) \\
  & W = \w(f(\varepsilon,\phi;x);x)^{1-\alpha},
\end{align*}
with \cite{rainforth2018tighter} considering the case $\alpha = 0$ and assuming that $X$ and $W$ have moments of order $6$, while \cite{daudel2022Alpha} considers the case $\alpha \in (0,1)$ and assumes moments of order $8$ for $X$ and $W$. The derivations in \cite{rainforth2018tighter} and \cite{daudel2022Alpha} rely on distinct proof strategies that necessitates stronger conditions compared to those used in \Cref{thm:ratios-limits-optimal-cond}.
\end{enumerateList}

\section{Deferred proofs and results for \Cref{sec:behavior-gradient-vr-Nasymp}}
\label{sec:deferr-proofs-result-2}
\subsection{Proofs of
  \Cref{thm:GradNStudyAllalpha,thm:SNR-REP,lem:SNR}}

\subsubsection{Notation}
\label{sec:preliminaries}

Let us first introduce some notation and resulting identities that will be used for the proofs of \Cref{thm:GradNStudyAllalpha,thm:SNR-REP,lem:SNR}. Letting $\varepsilon \sim \REPq$ and with $\tw[\phi][\varepsilon][\phi]$ as in~(\ref{eq:tw}), we define
$$
X = (1-\alpha)^{-1} \partial_{\psi} \lr{\tw[\phi][\varepsilon][\phi]^{1-\alpha}}\quad\text{and}\quad
W = \tw[\phi][\varepsilon][\phi]^{1-\alpha}\;,
$$
Furthermore, let
$\varepsilon_1, \ldots, \varepsilon_N$ be i.i.d. copies of
$\varepsilon$, and define
$X_i$ and
$W_i$ for all $i =1,2,\ldots$ as $X$ and $W$ but with $\varepsilon_i$
replacing $\varepsilon$. Further
denote $\overline{X}_N = \frac{1}{N} \sum_{i=1}^N X_i$ and
$\overline{W}_N = \frac{1}{N} \sum_{i=1}^N W_i$ for all $N\in\mathbb{N}^*$.
Then, using the above
definitions  in~\eqref{estimREP}, we have
  \begin{align}
    \label{eq:expREP_Y}
&  \PE\lr{ \gradREP[1,N]} = \PE\lr{\frac{\overline{X}_N}{\overline{W}_N}}\;,\\
    \label{eq:estimREP_Y}
&        \mathbb{V}(\gradREP[1,N]) = \mathbb{V}\lr{\frac{\overline{X}_N}{\overline{W}_N}}\;,
  \end{align}
  \begin{rem}\label{rem:cond-notationREP}
    We note that with these definitions,~\ref{hyp:inverseGrad} means
    that~(\ref{eq:simple-cond-ratios}) holds with
    $\eta=\delta/(1-\alpha)$, $\mathbb{V}\lr{\tw[\phi][\varepsilon][\phi]^{1-\alpha}} < \infty$ that $\PE\lr{W^2}<\infty$, $\mathbb{V}\lr{{\partial_{\psi} \lr{\tw[\phi][\varepsilon][\phi]^{1-\alpha}}}} < \infty$ that $\PE\lr{\lrav{X}^{2}}<\infty$ and \ref{hyp:hypZeroREP} implies $\PE\lr{\lrav{X}\,W^{\ell-1}}<\infty$ for $\ell=0,1,2$.    
  \end{rem}
  \begin{rem}\label{rem:alpha1REP}
    The above quantities are defined for $0\leq\alpha<1$. However, the
    definition of $W$ boils down for $\alpha=1$ to $W=1$, hence
    $\overline{W}_N=1$ for all $N\geq1$. As for $X$, since it holds
    for all $\alpha\in[0,1)$ that
    $X=W\partial_\psi\lr{\log\tw[\phi][\varepsilon][\phi]}$, it
    becomes $X=\partial_\psi\lr{\log\tw[\phi][\varepsilon][\phi]}$ for
    $\alpha=1$. In this case, explicit formulas hold
    of~(\ref{eq:expREP_Y}) and~(\ref{eq:estimREP_Y}), namely,
    \begin{align*}
    &\PE\lr{ \gradREP[1,N]}=\PE\lr{\partial_\psi\lr{\log\tw[\phi][\varepsilon][\phi]}}=\partial_\psi\mathrm{ELBO} (\theta, \phi; x)\;,\\
     & \mathbb{V}(\gradREP[1,N]) =\frac1N\mathbb{V}\lr{\partial_\psi\lr{\log\tw[\phi][\varepsilon][\phi]}}\;.
    \end{align*}
    where the second equality in the first displayed line follows from
    \ref{hyp:hypZeroREP}, which for $\alpha=1$ takes the form
    $$
    \PE_{\varepsilon \sim \REPq} \lr{\sup_{(\theta',\phi')\in\mathcal{V}}
  \lrav{\partial_{\psi'} \lr{\log \tw[\phi'][\varepsilon][\phi'][\theta']}}}<\infty\;.
$$
Thus the results of \Cref{thm:GradNStudyAllalpha,thm:SNR-REP} hold for $\alpha=1$ in a non-asymptotic way,
namely, with all $o$-terms replaced by zeros,
$\mathrm{VR}^{(1)}(\theta, \phi; x)=\mathrm{ELBO} (\theta, \phi; x)$,
$\vREP[1](\theta,
\phi;x)=\mathbb{V}\lr{\partial_\psi\lr{\log\tw[\phi][\varepsilon][\phi]}}$
and $\partial_{\psi} [\gammaA[1](\theta,\phi; x)^2]=0$.
  \end{rem}
  \subsubsection{Proof of \Cref{thm:GradNStudyAllalpha}}
  \label{app:proofGradNAllalpha}

Under \ref{hyp:hypZeroREP}, it holds that 
\begin{align}
  &  \partial_{\psi} \liren(\theta, \phi; x) \label{eq:expectationREP} =\PE(\gradREP) \\
  \label{eq:GradNStudyAllalpha-ident2}
  &  \partial_{\psi} \mathrm{VR}^{(\alpha)}(\theta, \phi; x) = \frac{\PE(X)}{\PE(W)} \\
  \label{eq:GradNStudyAllalpha-ident3}
  &\partial_{\psi} [\gammaA(\theta,\phi;x)^2]= 2\,\frac{\PE\lr{W\lrb{\PE(W)\,X-\PE(X)\,W}}}{\lr{\PE(W)}^3},
\end{align}
which follows from  \Cref{prop:interch-deriv-expect-rep}, the definitions of \Cref{sec:preliminaries},
(\ref{eq:lalpha-def-E}), (\ref{eq:defVRbound-E}), \eqref{eq:defGammaAlphaFirst} and~(\ref{eq:interch-deriv-expect-rep1})--(\ref{eq:interchange-liren}), as they imply \looseness=-1
\begin{align}
\nonumber
&   \liren (\theta,\phi; x)= \frac1{1-\alpha}\,\PE\lr{\log \overline{W}_N }
                            \;,\\
\nonumber
& \mathrm{VR}^{(\alpha)}(\theta, \phi; x) =    \frac1{1-\alpha}\,
                                    \log\PE\lr{W} \;,\\
\nonumber
& \gammaA(\theta, \phi;x)^2 =
                 \frac1{1-\alpha}\,\lr{\frac{\PE(W^2)}{\PE(W)^2}-1}\;,
  \nonumber \\
  &  \partial_{\psi}\PE(W)=(1-\alpha)\,\PE(X) \; ,\\
  \nonumber
  &  \partial_{\psi}\PE\lr{W^2}=2\,(1-\alpha)\,\PE\lr{X\,W} \; ,\\
  \nonumber
  &  \partial_{\psi}\PE\lr{\log \overline{W}_n} =(1-\alpha)\,\PE\lr{\frac{\overline{X}_n}{\overline{W}_n}}\;. 
\end{align}
We can conclude by using \Cref{rem:cond-notationREP} and the interchanges of the derivative and expectation signs obtained above under \ref{hyp:hypZeroREP} to rewrite the asymptotic results \ref{item:ratios-limits-optimal-cond1} and the second part of \ref{item:ratios-limits-optimal-cond2bis} from \Cref{thm:ratios-limits-optimal-cond} into the convenient formulations~\eqref{eq:asymptoticalGradientThetaOneFormer} and \eqref{eq:asymptoticalGradientThetaOne} respectively.

\subsubsection{Proof of \Cref{thm:SNR-REP}}\label{app:thm:SNR-REP}

Note that the conditions of \Cref{thm:SNR-REP} are stronger
than those required in \Cref{thm:GradNStudyAllalpha}. We thus already
have that \eqref{eq:asymptoticalGradientThetaOne} holds as $N\to\infty$,
which gives that $|\PE(\gradREP)|$ behaves as the numerator
of the r.h.s of~\eqref{eq:SNR_REP_N_Two}. To conclude the proof, it remains to study the asymptotic behavior of $\mathbb{V}(\gradREP)$ as
$N\to\infty$. As in the proof of \Cref{thm:GradNStudyAllalpha}, we
first note that,  following
\Cref{rem:cond-notationREP}, the assumptions of \Cref{thm:SNR-REP}
imply
\begin{enumerate}[label=(W\arabic*),series=W]
\item\label{item:wrapping-ass-snr-rep-step1} There exists $\eta>0$ such that (\ref{eq:simple-cond-ratios})
  holds and   $\PE(|X|)$, $\PE(W)$ and
  $\PE\lr{\lrav{\PE(W)X-\PE(X)W}^2}$ are all finite.
\end{enumerate}
Then the claimed asymptotic behavior of the variance is obtained in
two steps: (i) we show a first asymptotic behavior for
$\mathbb{V}(\gradREP[1,N])$ as $N\to\infty$ which holds
under~\ref{item:wrapping-ass-snr-rep-step1} and (ii) we show how the
conclusions of \Cref{prop:interch-deriv-expect-rep} allow us to
express this asymptotic result under the convenient formulation
\eqref{eq:SNR_REP_N_TwovREP-def}.

\noindent\textbf{First step.}
Under~\ref{item:wrapping-ass-snr-rep-step1}, we can apply the second
part of Assertion~\ref{item:ratios-limits-optimal-cond4} of
\Cref{thm:ratios-limits-optimal-cond}. Using~(\ref{eq:limVarFirstOrderMuOne})
and (\ref{eq:estimREP_Y}) we have that, as $N\to\infty$,
   \begin{align*}
     \mathbb{V}(\gradREP[1,N]) & = \frac{1}{N}\frac{\mathbb{V}\lr{\PE(W) X - \PE(X) W}}{\lr{\PE(W)}^{4}} +o\lr{\frac{1}{N}} \\
     & = \frac{1}{N} \mathbb{V}\lr{{\frac{X}{\PE(W)} - W \frac{\PE(X)}{\PE(W)^2}}} +o\lr{\frac{1}{N}}. 
   \end{align*}

\noindent\textbf{Second step.} We prove that, under \ref{hyp:hypZeroREP}, the constant defined
in~(\ref{eq:SNR_REP_N_TwovREP-def}) reads
\begin{align}
\label{eq:newdef:vREP}
\vREP(\theta, \phi;x) =
                 \mathbb{V}\lr{{\frac{X}{\PE(W)} - W
                 \frac{\PE(X)}{\PE(W)^2}}} \;.
\end{align}
By using the definitions of $X$ and $W$ in
\Cref{sec:preliminaries} and~(\ref{eq:interch-deriv-expect-rep1}) of \Cref{prop:interch-deriv-expect-rep}, we have 
\begin{align*}
\partial_{\psi} \lr{  \frac{W}{\PE(W)}} = \frac{\partial_{\psi} W}{\PE(W)} - W \frac{\partial_{\psi} \PE(W)}{(\PE(W))^2} = (1-\alpha) \lr{{\frac{X}{\PE(W)} - W
\frac{\PE(X)}{\PE(W)^2}}} 
\end{align*}
meaning that the two variables in the variances of the right-hand sides of~(\ref{eq:SNR_REP_N_TwovREP-def})
and~(\ref{eq:newdef:vREP}) coincide up to the $(1-\alpha)$ multiplicative term, which is squared when put in front of the variance. The
identity~(\ref{eq:newdef:vREP}) follows.  

\begin{rem} \label{rem:VarREPequal}
As a byproduct of the proof of \Cref{app:thm:SNR-REP}, we have established \eqref{eq:newdef:vREP}, that is using the definition of $X$ and $W$ in \Cref{sec:preliminaries}, 
\begin{align*}
& \vREP(\theta, \phi;x) \\ 
& = \frac{1}{(1-\alpha)^2}
                   \mathbb{V}\lr{{\frac{\partial_{\psi}(\tw[\phi][\varepsilon][\phi]^{1-\alpha})}{\PE(\tw[\phi][\varepsilon][\phi]^{1-\alpha})} - 
                   \frac{\tw[\phi][\varepsilon][\phi]^{1-\alpha} \PE(\partial_{\psi} (\tw[\phi][\varepsilon][\phi]^{1-\alpha}))}{\PE(\tw[\phi][\varepsilon][\phi]^{1-\alpha})^2}}} \\
& = \mathbb{V}\lr{{\frac{\tw[\phi][\varepsilon][\phi]^{1-\alpha}}{\PE(\tw[\phi][\varepsilon][\phi]^{1-\alpha})} \lr{\partial_{\psi}(\log \tw[\phi][\varepsilon][\phi]) -
                   \PE \lr{\frac{\tw[\phi][\varepsilon][\phi]^{1-\alpha}}{\PE(\tw[\phi][\varepsilon][\phi]^{1-\alpha})} \partial_{\psi} (\log \tw[\phi][\varepsilon][\phi])}}}}.
\end{align*}
Now considering the case $p_\theta(\cdot|x) = q_{\phi}(\cdot|x)$ in the above and using 
    \begin{align} \label{eq:Roeder}
    \partial_{\psi} \log \w(f(\varepsilon, \phi; x);x) = \lrder{\partial_{\psi'} \log \w(f(\varepsilon, \phi'; x); x) - \partial_{\psi'} \log q_{\phi'}(f(\varepsilon, \phi;x)|x)}{\phi' = \phi},
    \end{align} it follows that 
\begin{align} \label{rem:VarREPequalapp}
  \vREP(\theta, \phi;x) = \mathbb{V}\lr{  \lrder{\partial_{\psi'} \log q_{\phi'}(f(\varepsilon, \phi;x)|x)}{\phi' = \phi}} \quad \mbox{when $p_\theta(\cdot|x) = q_{\phi}(\cdot|x)$}.
  \end{align}
\end{rem}

\subsubsection{Proof of \Cref{lem:SNR}}\label{app:lem:SNR}

We first introduce some additional notation for the DREP
estimators. Namely, we set
\begin{align*}
&X^{(1)} =  \frac1{1-\alpha}\lrder{\partial_{\psi'} \lr{\tw^{1-\alpha}}}{\phi' =
                 \phi} = \tw[\phi][\varepsilon][\phi]^{1-\alpha} \lrder{\partial_{\psi} \log \tw}{\phi' =
                 \phi} \\
  &X^{(2)}=\frac1{2(1-\alpha)}\lrder{\partial_{\psi'} \lr{\tw^{2(1-\alpha)}}}{\phi' =
                 \phi}   = \tw[\phi][\varepsilon][\phi]^{1-\alpha} X^{(1)}=W\,X^{(1)}\;.
\end{align*}
Then, $\mathbb{V}\lr{\tw[\phi][\varepsilon][\phi]^{1-\alpha}} < \infty$ means
that $\PE\lr{\lrav{W}^2} < \infty$, $\mathbb{V}\lr{{\lrder{\partial_{\psi'} \lr{\tw^{1-\alpha}}}{\phi' = \phi}}}<\infty $ that
$\PE\lr{|X^{(1)}|^{2}}<\infty$ and $\mathbb{V}\lr{{[\partial_{\psi'} (\tw^{2(1-\alpha)})]|_{\phi' = \phi}}} < \infty$ that
$\PE\lr{|X^{(2)}|^{2}} <\infty$.
With \Cref{rem:cond-notationREP}, this leads to the following two lists of conditions valid
respectively in Part~\ref{item:drep-snr-alpha1} 
and~\ref{item:drep-snr-alpha0} of \Cref{lem:SNR}:
\begin{enumerate}[resume*=W]
\item\label{item:wrapping-ass-snr-drep-step1} There exists $\eta>0$ such that (\ref{eq:simple-cond-ratios})
  holds and $\PE\lr{\lrav{X^{(1)}}}$, $\PE\lr{|X^{(2)}|^2}$,
  $\PE(W)$ and $\PE\lr{|\PE(W)X^{(1)}-\PE\lr{X^{(1)}}W|^2}$ are all finite.
\item \label{item:wrapping-ass-snr-drep-alpha0-step1}  There exists $\eta>0$ such that (\ref{eq:simple-cond-ratios})
  holds and $\PE\lr{\lrav{X^{(2)}}^2}$
and $\PE(W^2)$ are finite.
\end{enumerate}
As in \Cref{sec:preliminaries}  for $X_i$ and $W_i$, the random
variables $X_i^{(1)}$ and
$X^{(2)}_i$ are defined for all $i\geq1$ as $X^{(1)}$ and $X^{(2)}$ respectively but with 
$\varepsilon_i$ replacing $\varepsilon$, and we set
$\overline{X}_N^{(\ell)} = \frac{1}{N} \sum_{i=1}^N X^{(\ell)}_i$ for
$\ell=1,2$. Then, using the above
definitions  in~(\ref{estimDREP}), we have
\begin{align}
  \label{eq:Exp-gradDREP-XY}
&  \PE\lr{\gradDREP[1,N]}= \PE\lr{\alpha\,\frac{\overline{X}_N^{(1)}}{\overline{W}_N}+\frac{1-\alpha}{N}\,\frac{\overline{X}_N^{(2)}}{(\overline{W}_N)^2}}\\
  \label{eq:gradDREP-XY}
&  \mathbb{V}(\gradDREP[1,N])= \mathbb{V}\lr{\alpha\,\frac{\overline{X}_N^{(1)}}{\overline{W}_N}+\frac{1-\alpha}{N}\,\frac{\overline{X}_N^{(2)}}{(\overline{W}_N)^2}}
\end{align}
and we further define
\begin{align}
  \label{eq:def-vrdrep}
  \VRDREP= &\alpha \,\frac{\PE\lr{X^{(1)}}}{\PE(W)} \;, \\
    \label{eq:def-gammadrep}
  \gammaDREP[0]=&
 \frac{\PE\lr{X^{(2)}}}{\lr{\PE(W)}^2}
\end{align}
as well as
\begin{align}
    \label{eq:newdef:vDREP1}
  \vDREPun(\theta, \phi;x) =  &                  \alpha^2\,
                                                \mathbb{V}\lr{{\frac{X^{(1)}}{\PE(W)} - W
                                                    \frac{\PE(X^{(1)})}{\PE(W)^2}}} \\
    \label{eq:newdef:vDREP2}
  \text{and}\quad
    \vDREPdeux(\theta, \phi;x) =  &(1-\alpha)^2\,
                                                \mathbb{V}\lr{{\frac{X^{(2)}}{\PE(W)^2}
                                                    - 2 W
                                                    \frac{\PE(X^{(2)})}{\PE(W)^3}}} \;.
\end{align}
As for the proofs of \Cref{thm:GradNStudyAllalpha,thm:SNR-REP}, the
proof of \Cref{lem:SNR} is made of two steps. In the first step, we
establish that under~\ref{item:wrapping-ass-snr-drep-step1} with
$\alpha\in(0,1)$: as $N\to\infty$,
\begin{equation}
  \label{eq:snrDREP-straightconstants-alpha1}
  \mathrm{SNR}[\gradDREP[1,N]] = \sqrt{N}~\frac{\lrav{\VRDREP} +o(1)}{\sqrt{\vDREPun(\theta, \phi;x)}
+o(1)} \;,
\end{equation}
while under~\ref{item:wrapping-ass-snr-drep-alpha0-step1} with $\alpha=0$: as $N\to\infty$,                                  
\begin{equation}
  \label{eq:snrDREP-straightconstants-alpha0}
  \mathrm{SNR}\lrb{\gradDREP[1,N][\psi][0]} = \sqrt{N}~\frac{\lrav{\gammaDREP[0]}+o\lr{1}}{\sqrt{\vDREPdeux[0](\theta, \phi;x)}+o(1)}\;.
\end{equation}
In the second step we use the conclusions of
\Cref{prop:interch-deriv-expect-drep} to rewrite 
\eqref{eq:snrDREP-straightconstants-alpha1} and \eqref{eq:snrDREP-straightconstants-alpha0} under the convenient formulations \eqref{eq:lem:SNR1} and \eqref{eq:lem:SNR0}. 

\noindent\textbf{First step.} Applying~\ref{item:ratios-limits-optimal-cond1}
and~\ref{item:ratios-limits-optimal-cond3}--\ref{item:ratios-limits-optimal-cond4}
from \Cref{thm:ratios-limits-optimal-cond} with $X^{(1)}$ or $X^{(2)}$
replacing $X$, we obtain the following implications:
\begin{align}
 \label{eq:handleRNsnr1plus1NX1}
  & \PE\lr{\lrav{X^{(1)}}}<\infty\Longrightarrow
  \lim_{N\to\infty}  \PE\lr{\frac{\overline{X}_N^{(1)}}{\overline{W}_N}} =\frac{\PE\lr{X^{(1)}}}{\PE(W)}\;,\\
 \label{eq:handleRNsnr1plus1NX2}
&  \PE\lr{\lrav{X^{(2)}}}<\infty\Longrightarrow
\lim_{N\to\infty}   \PE\lr{\frac{\overline{X}_N^{(2)}}{\lr{\overline{W}_N}^2}} =\frac{\PE\lr{X^{(2)}}}{\lr{\PE(W)}^2}\;,\\
  \label{eq:handleRNsnrVarAnalysisX2term1}
  & \PE\lr{|X^{(2)}|^2}<\infty\Longrightarrow
\lim_{N\to\infty}      \mathbb{V}\lr{\frac{\overline{X}^{(2)}_N}{\lr{\overline{W}_{N}}^2}} = 0\;,\\
  \label{eq:handleRNsnrVarAnalysisX1term1}
& \left.\substack{\PE\lr{\lrav{X^{(1)}}}<\infty,\; \PE(W)<\infty,\\
  \mathbb{V}\lr{\PE(W)X^{(1)}-\PE(X^{(1)})W}<\infty
}\right\}
  \Longrightarrow\lim_{N\to\infty}
N\,
  \mathbb{V}\lr{\frac{\overline{X}^{(1)}_N}{\overline{W}_{N}}}=
  \frac{\mathbb{V}\lr{\PE(W)X^{(1)} - W \PE(X^{(1)})}}{\lr{\PE(W)}^4} \;,
\\
  \label{eq:handleRNsnrVarAnalysisX2term0}
  &  \left.\substack{\PE\lr{|X^{(2)}|^2}<\infty\\
  \PE(W^2)<\infty}\right\}\Longrightarrow
\lim_{N\to\infty}N\,    \mathbb{V}\lr{\frac{\overline{X}^{(2)}_N}{\lr{\overline{W}_{N}}^2}} = \frac{\mathbb{V}\lr{\PE(W)\,X^{(2)} - 2\,W \,\PE(X^{(2)})}}{\lr{\PE(W)}^{6}}\;.
\end{align}
From~(\ref{eq:Exp-gradDREP-XY}),~(\ref{eq:handleRNsnr1plus1NX1})
and~(\ref{eq:handleRNsnr1plus1NX2}), using the quantities defined
in~(\ref{eq:def-vrdrep}) and~(\ref{eq:def-gammadrep}) we get that if
$\alpha\in(0,1)$ and 
$\PE\lr{\lrav{X^{(i)}}}<\infty$ for $i=1,2$: as $N\to\infty$,
\begin{equation} 
  \label{eq:Drep-grad-exp-asymp}
  \PE\lr{\gradDREP}=
    \VRDREP +o(1) \;.
\end{equation}
Pairing (\ref{eq:handleRNsnrVarAnalysisX2term1}) and
(\ref{eq:handleRNsnrVarAnalysisX1term1}) up, we get that  if $\alpha\in(0,1)$ and~\ref{item:wrapping-ass-snr-drep-step1} holds: as $N\to\infty$,
$$
\mathbb{V}\lr{\alpha\,\frac{\overline{X}_N^{(1)}}{\overline{W}_N}+\frac{1-\alpha}{N}\,\frac{\overline{X}_N^{(2)}}{\overline{W}_N^2}}=\frac{\alpha^2}{N}
\frac{\mathbb{V}\lr{X^{(1)} - W
    \frac{\PE(X^{(1)})}{\PE(W)}}}{\lr{\PE(W)}^2} +o\lr{\frac1N}\;.
$$
Combining this with (\ref{eq:gradDREP-XY})
and~(\ref{eq:Drep-grad-exp-asymp})
yields~(\ref{eq:snrDREP-straightconstants-alpha1}).

From~(\ref{eq:Exp-gradDREP-XY})
and~(\ref{eq:handleRNsnr1plus1NX2}), we get that if $\alpha=0$ and
$\PE\lr{\lrav{X^{(2)}}}<\infty$: as $N\to\infty$,
\begin{equation} 
  \label{eq:Drep-grad-exp-asymp-alpha0}
  \PE\lr{\gradDREP[M,N][\psi][0]}=
    \frac1N\,\gammaDREP[0]+o\lr{\frac1N} \;.
\end{equation}
Now, finally,~(\ref{eq:gradDREP-XY}),~(\ref{eq:newdef:vDREP2}) and~(\ref{eq:handleRNsnrVarAnalysisX2term0}) in
the case $\alpha=0$ imply that,
under~\ref{item:wrapping-ass-snr-drep-alpha0-step1},  as $N\to\infty$,
$$
\mathbb{V}\lr{\gradDREP[1,N][\psi][0]}=  \frac{\vDREPdeux[0](\theta, \phi;x)+o(1)}{N^3},
$$
which, combined with~(\ref{eq:Drep-grad-exp-asymp}) yields~(\ref{eq:snrDREP-straightconstants-alpha0}).

\noindent\textbf{Second step.} Under \ref{hyp:hypZeroDREP}, the conclusions of \Cref{prop:interch-deriv-expect-drep} hold and we can show that:
  \begin{enumerateList}
  \item If $\alpha\in(0,1)$,
    (\ref{eq:snrDREP-straightconstants-alpha1}) is the same
    as~(\ref{eq:lem:SNR1}), that is, we have:
    \begin{align}
      \label{eq:SNR-drep-step2-1}
  & \partial_{\psi}
      \mathrm{VR}^{(\alpha)}(\theta, \phi; x)= \VRDREP \;, \\
      \label{eq:SNR-drep-step2-2}
      &\genvDREP(\theta, \phi;x)=\vDREPun(\theta, \phi;x) \;.
    \end{align}
Indeed by
definitions~(\ref{eq:defVRbound-E}) and~(\ref{eq:def-vrdrep}), (\ref{eq:SNR-drep-step2-1}) reads
$$
\frac1{1-\alpha} \frac{\partial_{\psi}\PE\lr{\tw[\phi][\varepsilon][\phi]^{1-\alpha}}}{\PE(W)}=\frac\alpha{1-\alpha}\frac{\PE\lr{\lrder{\partial_{\psi'}\tw^{1-\alpha}}{\phi'=\phi}}}{\PE(W)}\;,
$$
which holds by~(\ref{eq:drep-exp-deriv-alpha-neq-zero2}). As
for~(\ref{eq:SNR-drep-step2-2}), by~(\ref{eq:lem:SNR1:var})
and~(\ref{eq:newdef:vDREP1}), it easily follows by differentiating
inside the variance in~(\ref{eq:lem:SNR1:var}) and
using~(\ref{eq:drep-exp-deriv-alpha-neq-zero1}).  

\item If $\alpha=0$,
(\ref{eq:snrDREP-straightconstants-alpha0}) is the same
as~(\ref{eq:lem:SNR0}), that is, we have:
\begin{align}
  \label{eq:SNR-drep-step2-3}
  &  \partial_{\psi}\lr{\gammaA[0](\theta,\phi;
    x)^2}=-2\;\gammaDREP[0] \;, \\ 
  \label{eq:SNR-drep-step2-4}
  & \genvDREP[0](\theta, \phi;x)=
    \vDREPdeux[0](\theta, \phi;x)\;.    
\end{align}
Indeed, by~(\ref{eq:defGammaAlphaFirst}), with $\alpha=0$, we have
$\gammaA[0](\theta,\phi;
    x)^2=\lr{\frac{\PE(W^2)}{\lr{\PE(W)}^2}-1}$.
Observe moreover that since $\alpha=0$, we have $\PE(W)=\int
p_{\theta}(x,z)\;\nu(\rmd z)$ and thus $\PE(W)$ does not depend on
$\phi$ and therefore not on $\psi$. We thus get that
$\partial_{\psi}\lr{\gammaA[0](\theta,\phi;
    x)^2}=  \partial_{\psi}\PE(W^2)/\lr{\PE(W)}^2$.
By definition of $X^{(2)}$ and~(\ref{eq:def-gammadrep}),
(\ref{eq:SNR-drep-step2-3}) then follows from~(\ref{eq:drep-exp-deriv-alpha-zero}).
As for~(\ref{eq:SNR-drep-step2-4}), it simply follows by
identifying~(\ref{eq:lem:SNR0:var}) with~(\ref{eq:newdef:vDREP2})
using $\alpha=0$ and the definitions of $W$ and   $X^{(2)}$.
\end{enumerateList}

\begin{rem}\label{rem:alpha1DREP}
  As in \Cref{rem:alpha1REP} for the REP estimator, let us briefly examine the
  case $\alpha=1$. In this case we get
  $\gradDREP[1,N][\psi][1]=\overline{X}^{(1)}_N$ with $X^{(1)}=\lrder{\partial_{\psi} \log \tw}{\phi' =
    \phi}$. When the DREP estimator is unbiased
  (under~\ref{hyp:hypZeroDREP} adapted to $\alpha=1$), we further have
  that $\PE\lr{X^{(1)}}=\partial_{\psi}\mathrm{ELBO} (\theta,
\phi; x)$, so that, for all $N\geq1$,
  $$
  \mathrm{SNR}[\gradDREP[1,N]] = \sqrt{N}~\frac{\lrav{\partial_{\psi}\mathrm{ELBO} (\theta,
\phi; x)}}{\sqrt{\genvDREP[1](\theta, \phi;x)}}\;,
$$
with $\genvDREP[1](\theta, \phi;x)=\mathbb{V}\lr{\lrder{\partial_{\psi} \log \tw}{\phi' =
    \phi}}$. This extends the asymptotic
formula~(\ref{eq:lem:SNR1}) obtained for $\alpha\in(0,1)$ to the case $\alpha=1$. 
  \end{rem}

\subsection{Proof of examples}

Before providing the proofs for \Cref{ex:Gaussian,ex:LinGauss}, we first recall below the two settings considered in these two examples alongside with some known results taken from \cite{daudel2022Alpha}. 

\begin{ex}[Known results from \cite{daudel2022Alpha} associated to \Cref{ex:Gaussian}] \label{ex:GaussPrev} Let $\theta, \phi \in \rset^d$. Set $p_\theta(z|x) = \mathcal{N}(z;\theta, \boldsymbol{I}_d)$ and $q_{\phi}(z|x) = \mathcal{N}(z; \phi, \boldsymbol{I}_d)$, where $\boldsymbol{I}_d$ is the $d$-dimensional identity matrix. Then, denoting the Euclidean norm of a finite dimensional vector $x$ with real
entries by
$\lrn{x}$ and its associated inner product by $\langle \cdot,\cdot \rangle$, \looseness=-1
\begin{align*}
\mathrm{VR}^{(\alpha)}(\theta, \phi; x) = \ell(\theta; x)- \frac{{\alpha} \|\theta-\phi \|^2}{2} \quad \mbox{and} \quad \gammaA(\theta, \phi;x)^2 = \frac{\exp\lr{(1-\alpha)^2 \|\theta-\phi \|^2} - 1}{1-\alpha}. 
\end{align*}
In addition, it holds that
  \begin{align*}
    \log \lr{\frac{p_\theta(z|x)}{q_{\phi}(z|x)}} =
    -\frac{\|\theta-\phi\|^2}{2} -\|\theta-\phi \| S, \quad S =
    \frac{\langle z - \phi, \phi- \theta \rangle}{\|\theta-\phi \|},
    \quad z \in \rset^d.
  \end{align*}
For $\phi = \theta + \boldsymbol{u}_d$, we have 
  \begin{equation*}
    \log \lr{\frac{p_\theta(Z|x)}{q_{\phi}(Z|x)}} = -\frac{d}{2}- \sqrt{d} {\tS},\quad \tS \sim \mathcal{N}(0,1).
\end{equation*}

\end{ex}

\begin{ex}[Known results from \cite{daudel2022Alpha} associated to \Cref{ex:LinGauss}] \label{ex:LinGaussPrev}
Let $\theta \in \rset^d$, $\phi=(\tilde{a},\notationb) \in \rset^d \times \rset^d$ and $A=\mathrm{diag}(\tilde{a})$. Set $p_{\theta}(z)=\mathcal{N}(z;\theta, \boldsymbol{I}_d)$, $p_\theta(x|z)=\mathcal{N}(x;z, \boldsymbol{I}_d)$ and $q_{\phi}(z|x)=\mathcal{N}(z;Ax+\notationb, 2/3~\boldsymbol{I}_d)$ \cite[as in][]{rainforth2018tighter}. Then, \looseness=-1
\begin{align*}
  &\mathrm{VR}^{(\alpha)}(\theta, \phi; x) = \ell(\theta; x) +  \frac{d}{2} \lr{\log \lr{\frac{4}{3}} + \frac{1}   {1-\alpha} \log \lr{\frac{3}{4-\alpha}}} - \frac{3\alpha}{4-\alpha}\Big\|Ax+\notationb-\frac{\theta+x}{2}\Big\|^2 \\
  &\gammaA(\theta, \phi;x)^2 = \frac{1}{1-\alpha}\lr{(4-\alpha)^{d}(15-6\alpha)^{-\frac{d}{2}}\exp\left(\frac{24(1-\alpha)^2}{(5-2\alpha)(4-\alpha)}\Big\|Ax+\notationb-\frac{\theta+x}{2}\Big\|^2\right)-1}.
\end{align*}
In addition, it holds that 
\begin{align*} 
\log \lr{\frac{p_\theta(Z|x)}{q_{\phi}(Z|x)}} = - \log \PE(\exp(- \sigma \sqrt{d} \tS)) - \sigma \sqrt{d} \tS, \quad \tS = \sum_{j = 1}^d \frac{\xi_{j}}{\sigma \sqrt{d}} 
\end{align*}
with $\sigma^2 = \mathbb{V}(\xi_{1}) = {1}/{18} + {8\lambda^2}/{3} < \infty$ and $\xi_j = y_j^2/4 - 2 \lambda y_j - 1/6$ for all $j = 1 \ldots d$, where $\lambda = \big\| \frac{\theta + x}{2} - Ax - \notationb \big\| / \sqrt{d}$, $y = (y_1, \ldots, y_d) \sim \mathcal{N}(0,2/3 \boldsymbol{I}_d)$ and $\xi_1, \ldots, \xi_d$ are centered i.i.d. random variables which are absolutely continuous with respect to the Lebesgue measure. 
\end{ex}
We next review the assumptions that need to be checked to obtain \Cref{ex:Gaussian,ex:LinGauss}. 

\subsubsection{Checking assumptions}
\label{sec:checking-assumptions}

First note that (i) the assumption \ref{hyp:VRIWAEwell-defined} holds with $\nu$ as the Lebesgue measure on $\rset^d$ for both \Cref{ex:GaussPrev,ex:LinGaussPrev} and (ii) $\theta\mapsto p_{\theta}(x)$,
$(z,\theta)\mapsto p_{\theta}(z|x)$, and
$(z,\phi)\mapsto q_{\phi}(z|x)$ are differentiable on $\Theta$,
$\rset^d\times\Theta$ and $\rset^d\times\Phi$, with, $\Theta=\Phi=\rset^d$ for \Cref{ex:GaussPrev} ($p_{\theta}(x)$
is arbitrary in \Cref{ex:GaussPrev} since only $p_\theta(z|x)$ is given) and $\Theta=\rset^d$ and $\Phi=\rset^{2d}$ for \Cref{ex:LinGaussPrev} (here $\lr{\phi_1,\dots,\phi_{2d}}=\lr{\tilde{a}_1,\dots,\tilde{a}_d, \notationb_1,\dots,\notationb_d}$ since $A$ is diagonal). \Cref{ex:Gaussian,ex:LinGauss} build on the settings described in \Cref{ex:GaussPrev,ex:LinGaussPrev} respectively, hence these assumptions also hold for \Cref{ex:Gaussian,ex:LinGauss}. In addition, \ref{hyp:reparameterized}
holds with the reparametrization proposed in \Cref{ex:Gaussian} (that is, $q= \mathcal{N}(0, \boldsymbol{I}_d)$ and
$f(\varepsilon, \phi;x) = \varepsilon + \phi$) and the one proposed in \Cref{ex:LinGauss} (that is, $q= \mathcal{N}(0, \boldsymbol{I}_d)$ and $f(\varepsilon, \phi;x)=\sqrt{\frac{2}{3}} \varepsilon + Ax +\notationb$). The assumptions that remain to be checked are then \ref{hyp:inverseGrad}, $\mathbb{V}\lr{\tw[\phi][\varepsilon][\phi]^{1-\alpha}} < \infty$, $\mathbb{V}\lr{{\partial_{\psi} \lr{\tw[\phi][\varepsilon][\phi]^{1-\alpha}}}} < \infty$, \ref{hyp:hypZeroREP}, $\mathbb{V}\lr{{[\partial_{\psi'} (\tw^{2(1-\alpha)})]|_{\phi' = \phi}}} < \infty$, $\mathbb{V}\lr{{\lrder{\partial_{\psi'} \lr{\tw^{1-\alpha}}}{\phi' = \phi}}}<\infty $  and~\ref{hyp:hypZeroDREP}.

  \begin{itemize}[wide=0pt, labelindent=\parindent]
  \item \textbf{Checking \ref{hyp:inverseGrad}.}
    By~\ref{item:reverse_implicationcond-jacob} in
    \Cref{lem:equivlimsupconditionGen}, it is sufficient to prove that
    \begin{equation}
      \label{eq:A3suff}
    \text{there exists $\eta>0$ such that}\quad  \PE_{Z\sim q_{\phi}(\cdot|x)}\lr{\w[\theta][\phi](Z; x)^{-\eta}}<\infty.
    \end{equation}

    \item \textbf{Checking $\mathbb{V}\lr{\tw[\phi][\varepsilon][\phi]^{1-\alpha}} < \infty$, $\mathbb{V}\lr{{\partial_{\psi} \lr{\tw[\phi][\varepsilon][\phi]^{1-\alpha}}}} < \infty$ and \ref{hyp:hypZeroREP}.} Using the Hölder inequality, $\mathbb{V}\lr{\tw[\phi][\varepsilon][\phi]^{1-\alpha}} < \infty$, $\mathbb{V}\lr{{\partial_{\psi} \lr{\tw[\phi][\varepsilon][\phi]^{1-\alpha}}}} < \infty$ and \ref{hyp:hypZeroREP}
    hold if we can show that, for all
    $(\theta,\phi)\in\Theta\times\Phi$ and all real value $p>0$, there exists a
    $\psi$-neighborhood $\mathcal{V}$ of $(\theta,\phi)$ such that
    \begin{align}
      \label{eq:ass_suffcond1}
      &\PE_{\varepsilon \sim \REPq}\lr{\sup_{(\theta',\phi')\in\mathcal{V}}\tw[\phi'][\varepsilon][\phi'][\theta']^p}<\infty \\
      \label{eq:ass_suffcond2}
      &\PE_{\varepsilon \sim \REPq}\lr{\sup_{(\theta',\phi')\in\mathcal{V}}\lrav{\partial_{\psi'}\log\tw[\phi'][\varepsilon][\phi'][\theta']}^p}<\infty,
    \end{align}
    where $\psi$ is a given component of $(\theta,\phi)$.

    \item \textbf{Checking $\mathbb{V}\lr{{[\partial_{\psi'} (\tw^{2(1-\alpha)})]|_{\phi' = \phi}}} < \infty$,
    $\mathbb{V}\lr{{\lrder{\partial_{\psi'} \lr{\tw^{1-\alpha}}}{\phi' = \phi}}}<\infty $  and~\ref{hyp:hypZeroDREP}.} Using the Hölder inequality once more, $\mathbb{V}\lr{{[\partial_{\psi'} (\tw^{2(1-\alpha)})]|_{\phi' = \phi}}} < \infty$,
    $\mathbb{V}\lr{{\lrder{\partial_{\psi'} \lr{\tw^{1-\alpha}}}{\phi' = \phi}}}<\infty $  and~\ref{hyp:hypZeroDREP} hold if we can show that, for all
    $(\theta,\phi)\in\Theta\times\Phi$ and all real value $p>0$, there exists a
    $\psi$-neighborhood $\mathcal{V}$ of $(\theta,\phi)$ such
    that~(\ref{eq:ass_suffcond1}) and~(\ref{eq:ass_suffcond2}) hold
    and, in addition,
    \begin{align}
      \label{eq:ass_suffcond3}
      &\PE_{\varepsilon \sim \REPq}\lr{\sup_{(\theta,\phi')\in\mathcal{V}}\tw[\phi][\varepsilon][\phi']^p}<\infty\\
      \label{eq:ass_suffcond4}
      &\PE_{\varepsilon \sim \REPq}\lr{\sup_{(\theta,\phi')\in\mathcal{V}}\lrav{\partial_{\psi'}\log\tw[\phi][\varepsilon][\phi']}^p}<\infty\\
      \label{eq:ass_suffcond5}
      &\PE_{Z\sim q_\phi(\cdot|x)}\lr{\sup_{(\theta,\phi')\in\mathcal{V}}\w[\theta][\phi'](Z; x)^p}<\infty\\
        \label{eq:ass_suffcond6}
      &\PE_{Z\sim q_\phi(\cdot|x)}\lr{\sup_{(\theta,\phi')\in\mathcal{V}}\lr{\frac{q_{\phi'}(Z|x)}{q_{\phi}(Z|x)}}^p}<\infty\\
      \label{eq:ass_suffcond7}
      &\PE_{Z\sim q_\phi(\cdot|x)}\lr{\sup_{(\theta,\phi')\in\mathcal{V}}\lrav{\partial_{\psi}\log q_{\phi'}(Z|x)}^p}<\infty,
    \end{align}
    where $\psi$ is a given component of
    $\phi$.
  \end{itemize} 
When proving \Cref{ex:Gaussian,ex:LinGauss}, we will then notably show that the conditions (\ref{eq:A3suff})-(\ref{eq:ass_suffcond7}) are satisfied, so that all
the theorems of \Cref{sec:behavior-gradient-vr-Nasymp} apply to the settings described in \Cref{ex:GaussPrev,ex:LinGaussPrev} with the
reparameterizations of \Cref{ex:Gaussian,ex:LinGauss}.

\subsubsection{Proof of \Cref{ex:Gaussian}}
\label{app:proof:ex:Gaussian}
Recall that in this setting $\psi$ is a component of $\phi$, say $\phi_k$ for some all $k = 1 \ldots d$, $p_\theta(z|x) = \mathcal{N}(z;\theta, \boldsymbol{I}_d)$ and
$q_{\phi}(z|x) = \mathcal{N}(z; \phi, \boldsymbol{I}_d)$, where
$\boldsymbol{I}_d$ denotes the $d$-dimensional identity matrix. \looseness=-1

\begin{itemize}[wide=0pt, labelindent=\parindent]
  \item \textbf{Checking \ref{hyp:inverseGrad}, $\mathbb{V}\lr{\tw[\phi][\varepsilon][\phi]^{1-\alpha}} < \infty$, $\mathbb{V}\lr{{\partial_{\psi} \lr{\tw[\phi][\varepsilon][\phi]^{1-\alpha}}}} < \infty$ \ref{hyp:hypZeroREP}, $\mathbb{V}\lr{{[\partial_{\psi'} (\tw^{2(1-\alpha)})]|_{\phi' = \phi}}} < \infty$, $\mathbb{V}\lr{{\lrder{\partial_{\psi'} \lr{\tw^{1-\alpha}}}{\phi' = \phi}}}<\infty $ and~\ref{hyp:hypZeroDREP}.} Let us prove~(\ref{eq:A3suff})-(\ref{eq:ass_suffcond7}). To this end, observe that
  \begin{align}
    \log \lr{\frac{p_\theta(z|x)}{q_{\phi}(z|x)}} = -\frac{\|\theta-\phi\|^2}{2} - \langle z - \phi, \phi- \theta \rangle, \quad z \in \rset^d. \label{eq:normWeightsGaussianEx}
  \end{align}
  Since $\langle Z - \phi, \phi- \theta \rangle$ is a Gaussian
  variable if $Z\sim q_\phi(\cdot|x)$, it admits exponential moments of
  any exponent and we have that~(\ref{eq:A3suff}) holds. We
  further get from~(\ref{eq:normWeightsGaussianEx}) that
  $\log \tw = c_0(\theta,\phi,\phi')+\langle \varepsilon, \phi- \theta
  \rangle$ for some continuous function $c_0$ defined on
  $\Theta\times\Phi^2$. Conditions~(\ref{eq:ass_suffcond1})--(\ref{eq:ass_suffcond7}) follow using that $\varepsilon$ is a
  Gaussian variable.
  
  \item \textbf{Closed-form expressions for the quantities
  appearing in \Cref{thm:GradNStudyAllalpha,thm:SNR-REP,lem:SNR}.} First recall from \Cref{ex:GaussPrev} that
    \begin{align}
    \mathrm{VR}^{(\alpha)}(\theta, \phi; x) = \ell(\theta; x)- \frac{{\alpha} \|\theta-\phi \|^2}{2}, \quad  \gammaA(\theta, \phi;x)^2 = \frac{\exp\lr{(1-\alpha)^2 \|\theta-\phi \|^2} - 1}{1-\alpha}\;. \label{eq:GammaAlphaLogNormalEx}
    \end{align}
  Computing $\partial_{\psi} \mathrm{VR}^{(\alpha)}(\theta, \phi; x)$
  and $\partial_{\psi} [\gammaA(\theta, \phi; x)^2]$ follows directly
  from differentiating in \eqref{eq:GammaAlphaLogNormalEx} w.r.t. $\phi_k$:
  \begin{align}
  & \partial_{\phi_k} \mathrm{VR}^{(\alpha)}(\theta, \phi; x) = - \alpha(\phi_k -\theta_k) \label{eq:GaussianExSNROne} \\
  & \partial_{\phi_k} [\gammaA(\theta, \phi; x)^2] = 2 {(1-\alpha)(\phi_k- \theta_k) \exp \lr{(1-\alpha)^2 \|\phi-\theta\|^2}}. \label{eq:GaussianExSNRTwo}
  \end{align}
  Furthermore, using \eqref{eq:normWeightsGaussianEx}: for all $y>0$ and all $\phi' \in \rset^d$, 
    \begin{align}
        \log \lr{\lr{\frac{\tw}{\PE(\tw[\phi][\varepsilon][\phi])}}^y} & = - \frac{y \|\theta-\phi\|^2}{2} -y \langle \varepsilon + \phi' - \phi, \phi- \theta \rangle \nonumber \\
        & = y g(\theta, \phi, \phi') - y \langle \varepsilon, \phi- \theta \rangle \label{eq:logWbarGaussEx},
    \end{align}
    where $g(\theta, \phi, \phi') = - \frac{\|\theta-\phi\|^2}{2} -y \langle \phi' - \phi, \phi- \theta \rangle$. Setting $y = 1-\alpha$ in \eqref{eq:logWbarGaussEx}, we deduce that:
    \begin{align}
    \frac{\tw^{1-\alpha}}{\PE(\tw^{1-\alpha})} &  = \frac{\exp\lr{-(1-\alpha) \langle \varepsilon, \phi- \theta \rangle}}{\PE(\exp\lr{-(1-\alpha) \langle \varepsilon, \phi- \theta \rangle})} = \frac{\exp\lr{-(1-\alpha) \|\theta-\phi \| \tilde{S}}}{ \exp \lr{ \frac{1}{2} (1-\alpha)^2\|\theta-\phi \|^2}} \label{eq:interGaussEx_N}
    \end{align}
  with $\tilde{S} = \frac{\langle \varepsilon, \phi- \theta \rangle}{\|\theta-\phi \|} \sim \mathcal{N}(0,1)$. Hence, 
  \begin{align*}
    & \partial_{\phi_k} \lr{\frac{\tw[\phi][\varepsilon][\phi]^{1-\alpha}}{\PE(\tw[\phi][\varepsilon][\phi]^{1-\alpha})}}  = -(1-\alpha) \lr{ \varepsilon^{(k)} + (1-\alpha)(\phi_k - \theta_k)} \frac{\tw[\phi][\varepsilon][\phi]^{1-\alpha}}{\PE(\tw[\phi][\varepsilon][\phi]^{1-\alpha})} \\
    & \lrder{\partial_{\phi'_k} \lr{\frac{\tw^{1-\alpha}}{\PE(\tw^{1-\alpha})}}}{\phi' = \phi} = 0,
  \end{align*}
where $\varepsilon^{(k)}$ denotes the $k$-th element of $\varepsilon \in \rset^d$. Consequently, 
\begin{align}
  & \vREP[\alpha][\phi_k](\theta, \phi; x) = \PE\lr{{\lr{ \varepsilon^{(k)} + (1-\alpha)(\phi_k - \theta_k)}^2} \lr{\frac{\tw[\phi][\varepsilon][\phi]^{1-\alpha}}{\PE(\tw[\phi][\varepsilon][\phi]^{1-\alpha})} }^2} \nonumber \\
  & \genvDREP[\alpha][\phi_k](\theta, \phi; x) = 0, \quad \alpha \in (0,1). \label{eq:GaussianExSNRfour} 
\end{align}
Thus, using \eqref{eq:interGaussEx_N}, we have that: for all $\psi \in \{\theta_1, \ldots, \theta_d, \phi_1, \ldots, \phi_d \}$,
\begin{align}
  \vREP[\alpha][\phi_k](\theta, \phi; x) & = \rme^{ - (1-\alpha)^2\|\theta-\phi \|^2} \PE\lr{(\varepsilon^{(k)}+ (1-\alpha)(\phi_k - \theta_k))^2 \rme^{-2(1-\alpha) \langle \varepsilon, \phi -\theta \rangle}} \nonumber \\
  & = \rme^{ - (1-\alpha)^2\|\theta-\phi \|^2} \rme^{ 2 (1-\alpha)^2\|\theta-\phi \|^2} \lr{1 + (1-\alpha)^2(\phi_k - \theta_k)^2} \nonumber \\
  & = \rme^{ (1-\alpha)^2\|\theta-\phi \|^2} \lr{1 + (1-\alpha)^2(\phi_k - \theta_k)^2}. \label{eq:GaussianExSNRthree}
\end{align}
Now considering the special case $\alpha = 0$, it holds that using \eqref{eq:Roeder}: 
\begin{align} 
 & \genvDREP[0][\phi_k](\theta, \phi;x)  = \mathbb{V} \left( A_{\theta, \phi}(\varepsilon; x) - 2 \frac{\tw[\phi][\varepsilon][\phi]}{\PE(\tw[\phi][\varepsilon][\phi])} \PE \lr{A_{\theta, \phi}(\varepsilon; x)} \right) \nonumber \\
  & = \PE \lr{\left( A_{\theta, \phi}(\varepsilon; x) - 2 \frac{\tw[\phi][\varepsilon][\phi]}{\PE(\tw[\phi][\varepsilon][\phi])} \PE \lr{A_{\theta, \phi}(\varepsilon; x)} \right)^2} - \PE(A_{\theta, \phi}(\varepsilon; x))^2 \label{eq:ComputingDREP0}
\end{align}
with 
$$
A_{\theta, \phi}(\varepsilon; x) = \frac{\tw[\phi][\varepsilon][\phi]^{2}}{\PE(\tw[\phi][\varepsilon][\phi])^{2}} \lr{\partial_{\phi_k} \log \tw[\phi][\varepsilon][\phi] + \lrder{\partial_{\phi'_k} \log q_{\phi'}(f(\varepsilon, \phi;x)|x)}{\phi' = \phi}}.
$$
Since $\PE(\tw[\phi][\varepsilon][\phi]) = \exp(\ell(\theta;x))$ and thus does not depend on $\phi_k$, we deduce that  
\begin{align} \label{eq:ComputingDREP0Contd} A_{\theta,
    \phi}(\varepsilon; x) = \frac{1}{2} \partial_{\phi_k} \lr{
    \lr{\frac{\tw[\phi][\varepsilon][\phi]}{\PE(\tw[\phi][\varepsilon][\phi])}}^2}
  +
  \frac{\tw[\phi][\varepsilon][\phi]^{2}}{\PE(\tw[\phi][\varepsilon][\phi])^{2}}
  \lrder{\partial_{\phi'_k} \log q_{\phi'}(f(\varepsilon,
    \phi;x)|x)}{\phi' = \phi}.
\end{align}
Using \eqref{eq:interGaussEx_N} with $\alpha = 0$, taking the square and differentiating, we have that
\begin{align*}
  \frac{1}{2} \partial_{\phi_k} \lr{ \lr{\frac{\tw[\phi][\varepsilon][\phi]}{\PE(\tw[\phi][\varepsilon][\phi])}}^2} = - \lr{ \varepsilon^{(k)} + \phi_k - \theta_k} \lr{\frac{\tw[\phi][\varepsilon][\phi]}{\PE(\tw[\phi][\varepsilon][\phi])}}^2.
\end{align*}
Furthermore, $\lrder{\partial_{\phi'_k} \log q_{\phi'}(f(\varepsilon, \phi;x)|x)}{\phi' = \phi} = \lrder{\varepsilon^{(k)} + \phi - \phi'}{\phi' = \phi} = \varepsilon^{(k)}$ and thus 
\begin{align*}
  A_{\theta, \phi}(\varepsilon; x) = (\theta_k - \phi_k) \lr{\frac{\tw[\phi][\varepsilon][\phi]}{\PE(\tw[\phi][\varepsilon][\phi])}}^2.
\end{align*}
Consequently, since for all $\ell \in \mathbb{N}^\star$,
\begin{align} \label{eq:checkHypExGauss}
  \PE \lr{  \lr{\frac{\tw[\phi][\varepsilon][\phi]}{\PE(\tw[\phi][\varepsilon][\phi])}}^\ell} = \rme^{\frac{\ell(\ell-1)}{2}\|\theta-\phi\|^2},
\end{align}
we obtain that
\begin{align}
  \genvDREP[0][\phi_k](\theta, \phi;x) & = (\theta_k - \phi_k)^2 \rme^{2\|\theta-\phi\|^2} \;\lr{\rme^{4\|\theta-\phi\|^2} - 4 \rme^{2 \|\theta-\phi\|^2} + 4 \rme^{\|\theta-\phi\|^2} - 1}.
\end{align}
Lastly, we obtain the desired expressions in \Cref{ex:Gaussian} by
taking $\theta-\phi = \epsilon \cdot \boldsymbol{u}_d$ with
$\epsilon > 0$ in \eqref{eq:GaussianExSNROne},
\eqref{eq:GaussianExSNRTwo}, \eqref{eq:GaussianExSNRfour},
\eqref{eq:GaussianExSNRthree} and applying
\Cref{thm:GradNStudyAllalpha,thm:SNR-REP,lem:SNR}. \looseness=-1
\end{itemize}

\subsubsection{Proof of \Cref{ex:LinGauss}}  
\label{app:subsec:proof:ex:LinGauss}
Recall that in this setting $\psi$ is a component of $(\theta, \notationb)$, say $\theta_k$ or $\notationb_k$ for some $k = 1 \ldots d$, $p_{\theta}(z)=\mathcal{N}(z;\theta, \boldsymbol{I}_d)$,
$p_\theta(x|z)=\mathcal{N}(x;z, \boldsymbol{I}_d)$ and
$q_{\phi}(z|x)=\mathcal{N}(z;Ax+\notationb, 2/3~\boldsymbol{I}_d)$ with
$A=\mathrm{diag}(\tilde{a})$ and $\tilde{a} = (\tilde{a}_1, \ldots, \tilde{a}_d)$.
We thus have that
\begin{align}
  & p_\theta(x,z) = \frac{1}{(2 \pi)^d} \exp \lr{-\frac{1}{2} \lr{\|z - \theta \|^2 + \|z - x \|^2}} \nonumber \\ 
\label{eq:LinGaussPtheta}  & p_\theta(x)=\int  p_\theta(x,z) \rmd z = \mathcal{N}(x; \theta, 2 \boldsymbol{I}_d) ,\\
  & p_\theta(z|x) = \mathcal{N}(z;(\theta+x)/2, 1/2~\boldsymbol{I}_d)
    \;. \label{eq:JointRain}
\end{align}

\begin{itemize}[wide=0pt, labelindent=\parindent]
  \item \textbf{Checking \ref{hyp:inverseGrad}, $\mathbb{V}\lr{\tw[\phi][\varepsilon][\phi]^{1-\alpha}} < \infty$, $\mathbb{V}\lr{{\partial_{\psi} \lr{\tw[\phi][\varepsilon][\phi]^{1-\alpha}}}} < \infty$ \ref{hyp:hypZeroREP}, $\mathbb{V}\lr{{[\partial_{\psi'} (\tw^{2(1-\alpha)})]|_{\phi' = \phi}}} < \infty$, $\mathbb{V}\lr{{\lrder{\partial_{\psi'} \lr{\tw^{1-\alpha}}}{\phi' = \phi}}}<\infty $ and~\ref{hyp:hypZeroDREP}.} Let us prove~(\ref{eq:A3suff})-(\ref{eq:ass_suffcond7}). To this end, using the densities and conditional densities above, we obtain that
$$
\log \w(z;x)= C_1(\theta,\phi;x)-\frac14\|z\|^2+\langle z,\theta+x-\frac32\lr{Ax+\notationb}\rangle\;,
$$
where $C_1(\cdot;x)$ is a $\mathcal{C}^\infty$ 
function from $\Theta\times\Phi$
to $\rset$. For $Z\sim q_{\phi}(\cdot|x)$, it is a Gaussian vector with
scalar covariance matrix equal to $2/3$ on its diagonal, and we get
that~(\ref{eq:A3suff}) holds for any $\eta <3$. Next, using that
$f(\varepsilon, \phi';x)=\sqrt{\frac{2}{3}} \varepsilon + A'x +\notationb'$. we have
$$
\log\tw=C_2(\theta,\phi,\phi';x)-\frac16\|\varepsilon\|^2-\frac{1}{\sqrt 6}\left\langle
\varepsilon,A'x
  +\notationb'-2\lr{\theta+x}+3\lr{Ax+\notationb}\right\rangle\;,
$$
where $C_2(\cdot;x)$ is a $\mathcal{C}^\infty$ 
function from $\Theta\times\Phi^2$
to $\rset$. Since
$\tq = \mathcal{N}(0, \boldsymbol{I}_d)$, we easily
get~(\ref{eq:ass_suffcond1})--(\ref{eq:ass_suffcond4}). Conditions~(\ref{eq:ass_suffcond5})--(\ref{eq:ass_suffcond7})
are obtained similarly. 

\item \textbf{Closed-form expressions for the quantities appearing in \Cref{thm:GradNStudyAllalpha,thm:SNR-REP,lem:SNR}.} 
First recall from \Cref{ex:LinGaussPrev} that
\begin{align*}
  &\mathrm{VR}^{(\alpha)}(\theta, \phi; x) = \ell(\theta; x) +  \frac{d}{2} \lr{\log \lr{\frac{4}{3}} + \frac{1}   {1-\alpha} \log \lr{\frac{3}{4-\alpha}}} - \frac{3\alpha}{4-\alpha}\Big\|Ax+\notationb-\frac{\theta+x}{2}\Big\|^2 \\
  &\gammaA(\theta, \phi;x)^2 = \frac{1}{1-\alpha}\left[(4-\alpha)^{d}(15-6\alpha)^{-\frac{d}{2}}\exp\left(\frac{24(1-\alpha)^2}{(5-2\alpha)(4-\alpha)}\Big\|Ax+\notationb-\frac{\theta+x}{2}\Big\|^2\right)-1\right].
\end{align*}
Computing $\partial_{\psi} \mathrm{VR}^{(\alpha)}(\theta, \phi; x) $
and $\partial_{\psi} [\gammaA(\theta, \phi;x)^2]$ follows directly
from differentiating the equations above w.r.t. $\theta$/$\phi$ paired
up with~(\ref{eq:LinGaussPtheta}). so that
\begin{align*}
 & \partial_{\theta_k} \mathrm{VR}^{(\alpha)}(\theta, \phi; x) = \frac{x_k - \theta_k}{2} + \frac{3\alpha}{4-\alpha} \lr{\tilde{a}_k x_k+\notationb_k-\frac{\theta_k+x_k}{2}} \\
 & \partial_{\notationb_k} \mathrm{VR}^{(\alpha)}(\theta, \phi; x) = - 2 \cdot \frac{3\alpha}{4-\alpha} \lr{\tilde{a}_k x_k+\notationb_k-\frac{\theta_k+x_k}{2}}
\end{align*}
and 
\begin{multline*}
 \partial_{\notationb_k} [\gammaA(\theta, \phi;x)^2] = 2 \cdot \frac{24(1-\alpha)(4-\alpha)^{d-1}}{3^{\frac{d}{2}}(5-2\alpha)^{\frac{d}{2} +1}} \exp\left(\frac{24(1-\alpha)^2}{(5-2\alpha)(4-\alpha)}\Big\|Ax+\notationb-\frac{\theta+x}{2}\Big\|^2\right) \\ \times \lr{\tilde{a}_k x_k+\notationb_k-\frac{\theta_k+x_k}{2}}.
\end{multline*}
As for computing $\tw$ above, we have that, for all $\phi,\phi' \in \rset^{2d}$, 
\begin{multline} \label{eq:checkHyps}
\frac{p_\theta(f(\varepsilon, \phi';x)|x)}{q_{\phi}(f(\varepsilon, \phi';x)|x)} = C \exp\lr{\frac{3}{4} \|A'x+\notationb'-Ax-\notationb\|^2} \\
\times \exp\lr{-\frac{1}{2 \cdot 3}\lr{\|\varepsilon\|^2 + 2
    \bigg\langle \varepsilon, \sqrt{\frac{3}{2}}\lr{A'x+\notationb' + 3(Ax + \notationb)
      -2(\theta + x)} \bigg\rangle}        } \;,
\end{multline}
where $C$ does not depend on $A$ nor $\notationb$. Hence, 
\begin{multline*}
  \PE\lr{\lr{\frac{p_\theta(f(\varepsilon, \phi';x)|x)}{q_{\phi}(f(\varepsilon, \phi';x)|x)}}^{1-\alpha}} = C^{1-\alpha} \exp\lr{\frac{3(1-\alpha)}{4} \|A'x+\notationb'-Ax-\notationb\|^2} \\ \times \PE\lr{\exp\lr{-\frac{1-\alpha}{2 \cdot 3}\lr{\|\varepsilon\|^2 + 2 \bigg\langle \varepsilon, \sqrt{\frac{3}{2}}\lr{A'x+\notationb' + 3(Ax + \notationb) -2(\theta + x)} \bigg\rangle}        }}\;,
\end{multline*}
with 
\begin{multline*}
  \PE\lr{\exp\lr{-\frac{1-\alpha}{2 \cdot 3}\lr{\|\varepsilon\|^2 + 2 \bigg\langle \varepsilon, \sqrt{\frac{3}{2}}\lr{A'x+\notationb' + 3(Ax + \notationb) -2(\theta + x)} \bigg\rangle}        }} \\ = \lr{\frac{3}{4-\alpha}}^{\frac{d}{2}} \exp \lr{\frac{(1-\alpha)^2}{4(4-\alpha)} \bigg\| A'x+\notationb' + 3(Ax + \notationb) -2(\theta + x) \bigg\|^2 }.
\end{multline*}
Consequently: for all $\phi,\phi' \in \rset^{2d}$, 
\begin{multline}
  \frac{\tw^{1-\alpha}}{\PE(\tw^{1-\alpha})} =  \exp\lr{-\frac{1-\alpha}{2\cdot 3}\lr{\|\varepsilon\|^2 + 2 \bigg\langle \varepsilon, \sqrt{\frac{3}{2}}\lr{A'x+\notationb' + 3(Ax + \notationb) -2(\theta + x)} \bigg\rangle}        } \\
  \times \lr{\frac{4-\alpha}{3}}^{d/2} \exp \lr{-\frac{(1-\alpha)^2}{4(4-\alpha)} \bigg\| A'x+\notationb' + 3(Ax + \notationb) -2(\theta + x) \bigg\|^2}. \label{eq:ratioLinGaussN}
\end{multline}
Hence, for all $k = 1, \ldots, d$,
\begin{multline}
\partial_{\theta_k} \lr{\frac{\tw[\phi][\varepsilon][\phi]^{1-\alpha}}{\PE(\tw[\phi][\varepsilon][\phi]^{1-\alpha})} } \\ = (1-\alpha)\lr{\sqrt{\frac{2}{3}}\varepsilon^{(k)} + \frac{1-\alpha}{4-\alpha} [4 (\tilde{a}_k x_k + \notationb_k) - 2(\theta_k + x_k)]   }  \frac{\tw[\phi][\varepsilon][\phi]^{1-\alpha}}{\PE(\tw[\phi][\varepsilon][\phi]^{1-\alpha})} \label{eq:ratioLinGaussNDiff}
\end{multline}
and we also have
\begin{align}
& \partial_{\notationb_k} \frac{\tw[\phi][\varepsilon][\phi]^{1-\alpha}}{\PE(\tw[\phi][\varepsilon][\phi]^{1-\alpha})}  = -2 \partial_{\theta_k} \frac{\tw[\phi][\varepsilon][\phi]^{1-\alpha}}{\PE(\tw[\phi][\varepsilon][\phi]^{1-\alpha})}  \label{eq:LinGaussREPone}, 
\end{align}
where $\varepsilon^{(k)}$ denotes the k-th element of $\varepsilon \in \rset^d$. Furthermore,
\begin{align}
& \lrder{\partial_{\notationb'_k} \frac{\tw^{1-\alpha}}{\PE(\tw^{1-\alpha})} }{\phi' = \phi} = -\frac{1}{2} \partial_{\theta_k} \frac{\tw[\phi][\varepsilon][\phi]^{1-\alpha}}{\PE(\tw[\phi][\varepsilon][\phi]^{1-\alpha})} . \label{eq:LinGaussDREPone} 
\end{align}
In addition, 
\begin{multline}
\vREP[\alpha][\theta_k](\theta, \phi; x) \\ = \PE\lr{  \lr{\sqrt{\frac{2}{3}}\varepsilon^{(k)} + \frac{1-\alpha}{4-\alpha} [4 (\tilde{a}_k x_k + \notationb_k) - 2(\theta_k + x_k)]   }^2  \lr{\frac{\tw[\phi][\varepsilon][\phi]^{1-\alpha}}{\PE(\tw[\phi][\varepsilon][\phi]^{1-\alpha})} }^2  } \label{eq:vREPGaussEx}
\end{multline}
and now using \eqref{eq:ratioLinGaussN}, we have that: for all $\ell \in \mathbb{N}^\star$,
\begin{multline*}
 \REPq(\varepsilon) \lr{\frac{\tw[\phi][\varepsilon][\phi]^{1-\alpha}}{\PE(\tw[\phi][\varepsilon][\phi]^{1-\alpha})} }^\ell =
  \lr{\frac{4-\alpha}{3}}^{d\ell/2} \exp \lr{-\frac{(1-\alpha)^2\ell}{4(4-\alpha)} \|\mu\|^2} \\
  (2\pi)^{-d/2} \exp\lr{-\frac{1}{2} \|\varepsilon\|^2} \exp\lr{-\frac{(1-\alpha)\ell}{2\cdot 3}\lr{\|\varepsilon\|^2 + 2 \bigg\langle \varepsilon, \sqrt{\frac{3}{2}}\mu \bigg\rangle}        }
\end{multline*}
where for convenience we have denoted $\mu_k = 4(\tilde{a}_k x_k+\notationb_k) -2(\theta_k + x_k)$ for all $k=1\ldots d$ and $\mu = (\mu_1, \ldots, \mu_d)$. Thus, 
\begin{multline} \label{eq:GaussianExGenl}
  \REPq(\varepsilon) \lr{\frac{\tw[\phi][\varepsilon][\phi]^{1-\alpha}}{\PE(\tw[\phi][\varepsilon][\phi]^{1-\alpha})} }^\ell \\ = h(\ell) \lr{\frac{2\pi \cdot 3}{3+(1-\alpha)\ell}}^{-d/2} \exp\lr{-\frac{3+(1-\alpha)\ell}{2\cdot 3}\bigg\|\varepsilon +\frac{(1-\alpha)\ell}{3+(1-\alpha)\ell} \sqrt{\frac{3}{2}} \mu \bigg\|^2}  
\end{multline}
with 
\begin{align*}
h(\ell) & = \lr{\frac{4-\alpha}{3}}^{d\ell/2} \exp \lr{-\frac{(1-\alpha)^2\ell}{4(4-\alpha)} \|\mu\|^2} \lr{\frac{3}{3+(1-\alpha)\ell}}^{d/2} \exp\lr{\frac{(1-\alpha)^2\ell^2}{4(3+(1-\alpha)\ell)} \|\mu\|^2} \\
& = \lr{\frac{4-\alpha}{3}}^{d\ell/2} \exp \lr{\frac{3(1-\alpha)^2\ell(\ell-1)}{4(4-\alpha)(3 + (1-\alpha)\ell)} \|\mu\|^2} \lr{\frac{3}{3+(1-\alpha)\ell}}^{d/2}.
\end{align*}
Consequently, taking $\ell =2$ in \eqref{eq:GaussianExGenl} and plugging this in \eqref{eq:vREPGaussEx}, 
 we get that 
\begin{align*}
  \vREP[\alpha][\theta_k](\theta, \phi; x) = \frac{(4-\alpha)^d}{(15-6\alpha)^{d/2}} \rme^{\frac{(1-\alpha)^2}{(4-\alpha)(5-2\alpha)} \frac{3}{2} \|\mu \|^2} \lr{\frac{2}{5-2\alpha} + \frac{(1-\alpha)^2 3^2}{(5-2\alpha)^2(4-\alpha)^2} \mu_k^2},
\end{align*}
from which we can also deduce $\vREP[\alpha][\notationb_k](\theta, \phi; x)$ using \eqref{eq:LinGaussREPone}. Similarly, using  \eqref{eq:LinGaussDREPone}, we can deduce from the above $\genvDREP[\alpha][\notationb_k](\theta, \phi;x)$ for all $\alpha \in (0,1)$. Now considering the case $\alpha = 0$, we deduce using \eqref{eq:ratioLinGaussNDiff} and \eqref{eq:LinGaussDREPone} with $\alpha = 0$ that:
\begin{align*}
  \frac{1}{2} \partial_{\notationb_k} \lr{\frac{\tw[\phi][\varepsilon][\phi]}{\PE(\tw[\phi][\varepsilon][\phi])}}^2
   = -2 \lr{\sqrt{\frac{2}{3}} \varepsilon^{(k)} + \frac{1}{4}\mu_k}
  \, \lr{\frac{\tw[\phi][\varepsilon][\phi]}{\PE(\tw[\phi][\varepsilon][\phi])}}^2.
\end{align*}
Hence, $A_{\theta, \phi}(\varepsilon; x)$ as defined in \eqref{eq:ComputingDREP0Contd} with $\phi_k = \notationb_k$ becomes 
\begin{align*}
  A_{\theta, \phi}(\varepsilon; x) =
  - \frac{1}{2} \lr{ \sqrt{\frac{2}{3}} \varepsilon^{(k)} + \mu_k} \lr{\frac{\tw[\phi][\varepsilon][\phi]}{\PE(\tw[\phi][\varepsilon][\phi])}}^2\;,
  \end{align*}
where we have used that $\lrder{\partial_{\notationb'_k} \log q_{\phi'}(f(\varepsilon, \phi;x)|x)}{\phi' = \phi} = \sqrt{3/2}~\varepsilon^{(k)}$. Now using \eqref{eq:GaussianExGenl} with $\alpha = 0$ and $\ell = 2$, we obtain that 
\begin{align*}
  \PE(A_{\theta, \phi}(\varepsilon; x)) = - \frac{3}{2\times 5} \frac{4^d}{15^{d/2}} \rme^{\frac{3}{8\times 5} \|\mu\|^2} \mu_k.
  \end{align*}
Next, using the above as well as \eqref{eq:GaussianExGenl} with $\alpha = 0$ and $\ell = 2$ and \eqref{eq:ComputingDREP0} with $\phi_k = \notationb_k$, we can compute all the terms appearing in $\genvDREP[0][\notationb_k](\theta, \phi;x)$ since
\begin{align*} 
\genvDREP[0][\notationb_k](\theta, \phi;x) 
& = \PE (A_{\theta, \phi}(\varepsilon; x)^2) - 4 \PE\lr{A_{\theta, \phi}(\varepsilon; x) \frac{\tw[\phi][\varepsilon][\phi]}{\PE(\tw[\phi][\varepsilon][\phi])}} \PE \lr{A_{\theta, \phi}(\varepsilon; x)} \\
& \qquad + \lr{4 \PE\lr{\lr{\frac{\tw[\phi][\varepsilon][\phi]}{\PE(\tw[\phi][\varepsilon][\phi])}}^2} - 1} \PE(A_{\theta, \phi}(\varepsilon; x))^2  
\end{align*}
with
\begin{align*}
\PE (A_{\theta, \phi}(\varepsilon; x)^2) & = \frac{1}{4} \PE\lr{\lr{ \sqrt{\frac{2}{3}} \varepsilon^{(k)} + \mu_k}^2 \lr{\frac{\tw[\phi][\varepsilon][\phi]}{\PE(\tw[\phi][\varepsilon][\phi])}}^4} \\ & = \frac{1}{4} \lr{\frac{4}{3}}^{2d} \lr{\frac{3}{7}}^{d/2}\rme^{\frac{9}{4 \cdot 7}\|\mu\|^2}\,\lr{\frac{2}{7} + \frac{3^2}{7^2} \mu_k^2}
\end{align*}
and
\begin{align*}
\PE\lr{A_{\theta, \phi}(\varepsilon; x) \frac{\tw[\phi][\varepsilon][\phi]}{\PE(\tw[\phi][\varepsilon][\phi])}} & = - \frac{1}{2} \PE\lr{\lr{ \sqrt{\frac{2}{3}} \varepsilon^{(k)} + \mu_k} \lr{\frac{\tw[\phi][\varepsilon][\phi]}{\PE(\tw[\phi][\varepsilon][\phi])}}^3}  = \frac{1}{4} h(3) \mu_k. 
\end{align*}
Thus, we finally get
\begin{align}
  \genvDREP[0][\notationb_k](\theta, \phi;x) & = \frac{1}{4} \lr{\frac{4}{3}}^{2d} \lr{\frac{3}{7}}^{d/2}\rme^{\frac{9}{4 \cdot 7}\|\mu\|^2}\lr{\frac{2}{7} + \frac{3^2}{7^2} \mu_k^2} \label{eq:bkvDREP} \\
  & \quad - \lr{\frac{4}{3}}^{3d/2} \lr{\frac{1}{2}}^{d/2} e ^{\frac{3}{4\cdot 4} \|\mu\|^2}  \frac{3}{2\cdot 5} \frac{4^d}{15^{d/2}} \rme^{\frac{3}{8\times 5} \|\mu\|^2} \mu_k^2 \nonumber  \\
  & \quad + \lr{4 \frac{4^d}{15^{d/2}} \rme^{\frac{3}{8 \cdot 5} \|\mu\|^2}-1} \frac{3^2}{2^2\times 5^2} \frac{4^{2d}}{15^{d}} \rme^{\frac{3}{4\times 5} \|\mu\|^2} \mu_k^2.  \nonumber
\end{align} 
Lastly, we obtain the desired expressions in \Cref{ex:LinGauss} by taking $Ax+\notationb = (\theta+x)/2 + \epsilon \boldsymbol{u}_d$ with $\epsilon>0$ so that $\mu_k = 4\epsilon$ and $\|\mu\|^2 = 16 d \epsilon^2$. Notice in particular that \eqref{eq:bkvDREP} becomes: 
\begin{multline}
  \genvDREP[0][\notationb_k](\theta, \phi;x) = \frac{1}{4} \lr{\frac{4}{3}}^{2d} \lr{\frac{3}{7}}^{d/2}\rme^{\frac{36}{7}d \epsilon^2}\,\lr{\frac{2}{7} + \lr{\frac{12 \epsilon}{7}}^2} \label{eq:bkvDREPtwo} \\
   - \lr{\frac{4}{3}}^{3d/2} \lr{\frac{1}{2}}^{d/2} e ^{3 d \epsilon^2}  \frac{3}{2\cdot 5} \frac{4^d}{15^{d/2}} \rme^{\frac{6}{5} d \epsilon^2} \mu_k^2  + \lr{4 \frac{4^d}{15^{d/2}} \rme^{\frac{6}{5} d\epsilon^2} -1}\frac{3^2}{2^2\times 5^2} \frac{4^{2d}}{15^{d}} \rme^{\frac{12}{5} d\epsilon^2} \mu_k^2.
\end{multline} 
\end{itemize}

\section{Deferred proofs and results for Section 3.2}
\label{sec:deferr-proofs-result}

\subsection{$L^p$ norms}
For a real valued random variable $X$ we denote by $\|X\|_p$ its $L^p$
norm, that is, $\lrn{X}_\infty$ denotes the smallest a.s. upper bound
of $|X|$ and, for all $1\leq p<\infty$,
$$
\lrn{X}_p = \lr{\PE\lr{\lrav{X}^p}}^{\frac{1}{p}}\;.
$$
Recall that the $L^p$ space of random variables with finite
$L^p$-norms is a Banach space when endowed with this $L^p$-norm and
that, by Jensen's inequality we have $\lrn{X}_p\leq \lrn{X}_{p'}$ for
all $p\leq p'$.
We will also repeatedly use the Hölder inequality
$$
\|XY\|_1\leq\|X\|_p\,\|Y\|_{p'}\;,
$$
which holds for all $X\in L^p$ and
$Y\in L^{p'}$ with $p,p'\in[1,\infty]$ such that $1/p+1/p'=1$. 

\subsection{Preliminary results for the proofs of \Cref{thm:CollapseSNRGaussian}}
\label{app:prelim:collapse}

In the following we will use
the asymptotic equivalence
\begin{equation}
  \label{eq:asymp-BarPhi}
1-\Phi\lr u= u^{-1}\varphi(u) (1+o(1))\quad\text{as $u\to\infty$,}  
\end{equation}
where $\varphi$ denotes the standard Normal density and $\Phi$ denotes the cumulative distribution function of the standard Normal distribution. In fact it is easy to show that the $o$-term in this asymptotic
equivalence is non-negative as soon as $u>0$, since, in this case,
\begin{equation}
  \label{eq:asymp-BarPhi-bound}
1-\Phi\lr u=\int_u^\infty \varphi(y)\;\rmd y \leq u^{-1}\int_u^\infty y\,\varphi(y)\;\rmd y =
u^{-1}\varphi(u).
\end{equation}
We then have the following two lemmas.
\begin{lem}\label{lem:MaxProbaBounds}
  Let $\Y_1, \ldots, \Y_N$ be i.i.d. normal random variables and set
  $M_N= \max_{1 \leq i \leq n} \Y_i$. Then for all
  $c>0$, we have, as $N\to\infty$,
  \begin{align}\label{eq:MaxProbaBoundLarge}
&    \PP\lr{M_N> \sqrt{2\,(1+c)\,\log N}} = O\lr{\lr{\log
                                               N}^{-1/2}\,N^{-c}},\\
    \label{eq:MaxProbaBoundSmall}
&    \PP\lr{M_N< \sqrt{2\,(1+c)^{-1}\,\log
                                    N}}=O\lr{\rme^{-N^{c/(1+c)}/\sqrt{4\pi\log N}}}.
  \end{align} 
\end{lem}
\begin{proof}
  For any $u\in\rset$ and $N\geq1$, we have
  $$
  \PP\lr{M_N> u}=1-\exp\lr{N\log\lr{1-\lr{1-\Phi\lr{u}}}}\;.
  $$
  Then using~(\ref{eq:asymp-BarPhi}) with $u=\sqrt{2x\log
    N}$ and $x>0$, we have, as $N\to\infty$,
  $$
  1-\Phi\lr{u}=\frac{N^{-x}}{\sqrt{4\pi\,x\log
    N}}\,(1+o(1)).
  $$
 We successively get the following assertions depending on where $x$
 lies w.r.t. 1
 \begin{align*}
   &\text{if $x>1$, as  $N\to\infty$,}\quad  \PP\lr{M_N> \sqrt{2x\log N}}=\frac{N^{1-x}}{\sqrt{4\pi\,x\,\log N}}\;\lr{1+o(1)},\\
   &\text{if  $x\in(0,1)$, as $N \to \infty$,}\quad  \PP\lr{M_N
     \leq \sqrt{2x\log N}} = O \lr{
     \exp\lr{-N^{1-x}/\sqrt{4\pi\log N}}},
 \end{align*}
The two bounds~(\ref{eq:MaxProbaBoundLarge})
  and~(\ref{eq:MaxProbaBoundSmall}) easily follow.
 \end{proof}
The next lemma extends
\cite[Lemma 1]{daudel2022Alpha}, which follows from the case $m=1$.
\begin{lem}\label{lem:MaxMoment}
  Let $\Y_1, \ldots, \Y_N$ be i.i.d. normal random variables and set
  $M_N= \max_{1 \leq i \leq n} \Y_i$. Let $m\in[1,\infty)$. Then, as $N \to \infty$,
  \begin{align}
   & \label{eq:MaxAbsMomentn01}
  \PE\lr{\lrav{M_N}^m}=\lr{2\log N}^{m/2}\lr{1+O\lr{\frac{\log\log N}{\log N}}}\\    
   & \label{eq:MaxNegMomentn01}
  \PE\lr{\lr{M_N}_-^m}=O\lr{\rho^N}\quad\text{for any $\rho\in(1/2,1)$},
  \end{align}
  where $(x)_-=\max(-x,0)$ denotes the negative part of $x$.
\end{lem}
\begin{proof}
  The Gaussian distribution belongs to the
  maximum domain of attraction of the Gumbel distribution. More
  precisely, we have \citep[see][]{de2007extreme}
  $$
  \lim_{N \to \infty } \PP\lr{ a_N^{-1} \left(M_N - b_N\right) \leq x} = \exp(-\rme^{-x}),
  $$
  with $a_N = 1/\sqrt{2\log N}$, $b_N = \sqrt{2\log
    N}-\frac{1}{2}(\log\log N+\log4\pi)/(\sqrt{2\log N})$. Since the
  Gaussian distribution has finite moments of all orders and so does
  the Gumbel distribution, \cite[Theorem
  2.1]{pickands1968momentconvergence} yields
  $$
\lim_{N\to\infty}  \PE\lr{\lrav{L_N}^m} = \int_{-\infty}^\infty
\lrav{x}^{m}\; \exp(-x-\rme^{-x})\;\rmd x,
$$
where we have set $L_N=a_N^{-1}(M_N - b_N)$. Now writing $b_N^{-1}M_N=1+\frac{a_N}{b_N}L_N$ and since $m \in [1, \infty)$ by assumption, the Minkowski
inequality yields
$$
\lrav{\lr{\PE\lr{\lrav{b_N^{-1}M_N}^m}}^{1/m}-1}
\leq\lr{\PE\lr{\lrav{b_N^{-1}M_N-1}^m}}^{1/m}=\frac{a_N}{b_N}\;\lr{\PE\lr{\lrav{L_N}^m}}^{1/m}.
$$
Hence $\PE\lr{\lrav{M_N}^m}=b_N^m\lr{1+O\lr{\frac{a_N}{b_N}}}^m$, which
leads to~(\ref{eq:MaxAbsMomentn01}). To
get~(\ref{eq:MaxNegMomentn01}), we use that
 $\PE\lr{\lr{M_N}_-^m} = \PE\lr{\lrav{M_N}^m\mathbbm{1}_{\lrcb{M_N<0}}}$,
and, for any $p>1$, the H{\"o}lder inequality leads to
$$
\PE\lr{\lr{M_N}_-^m}\leq\lr{\PE\lr{\lrav{M_N}^{m\,p}}}^{1/p}\,\lr{\PP\lr{M_N<0}}^{(p-1)/p}=\lr{\PE\lr{\lrav{M_N}^{m\,p}}}^{1/p}\,\lr{\frac12}^{N\,(p-1)/p}\;.
$$
Since~(\ref{eq:MaxAbsMomentn01}) also holds with $m$ replaced
by $m\,p$, we get the claimed bound by choosing $p$ large enough to have $\lr{\frac{1}{2}}^{(p-1)/p} \in (\frac{1}{2}, \rho)$ with $\rho \in (\frac{1}{2}, 1)$.
\end{proof}
We next present a key proposition, which provides a general bound that will be used to show the collapse of self-normalized weighted averages (in the sense that one weight dominates over all others in the self-normalized weighted averages). 
\begin{prop} \label{prop:new:collapse-behavior-general} Let
  $\lr{\lr{W_{i},\tilde{\Y}_{i}}}_{i\geq1}$ be a sequence of i.i.d. 
  $\rset_+^*\times\rset$-valued random variables. Let $F$ denote the distribution function of $ \log W_1$, and its generalized inverse by
  $F^{\leftarrow}$. In addition, let $Z\sim\mathcal{N}(0,1)$ and
  let us define, for all $s\in\rset$,
  \begin{equation}
    \label{eq:new:def-zeta}
  \zeta(s)=\frac{\PE\argcond{\rme^{F^{\leftarrow}\circ\Phi(Z)-F^{\leftarrow}\circ\Phi(s)}}{Z\leq
        s}}{1-\Phi\lr{s}}.
\end{equation}
Let us further define, for all $N\geq1$, $\tilde{\Y}_N^*=\max\lr{\lrav{\tilde{\Y}_{1}},\dots,\lrav{\tilde{\Y}_{N}}}$ and for
all $u\in\rset_+$, $s \in \rset$, $m\in[1,\infty)$ and $p\in(1,\infty)$, 
\begin{equation}
  \label{eq:new:def-zeta-tilde}
  \tilde{\zeta}_{N}^{(m,p)}(u,s) =
  N\,\lr{\lrn{\tilde{\Y}_{1}}_{p\,m}^{1/m}\;s^{(p-1)/p}
    +\lrn{\tilde{\Y}_{1}\,\mathbbm{1}_{\lrcb{\tilde{\Y}_N^*>u}}}_m}\;.  
\end{equation}
Then, there exists a non-decreasing random sequence
$\lr{J_N}_{N\geq1}$ such that, for all $N\geq1$, $J_N$ is valued in
$\lrcb{1,\dots,N}$ with $W_{J_N}=\max\lr{W_{1},\dots,W_{N}}$
and, for all $m,\delta\in[1,\infty)$,
$\underline{b}< 2< \overline{b}$, $p\in(1,\infty)$, $u>0$ and
$N\in\mathbb{N}^*$, we have
  \begin{align}\label{eq:new:collapse-behavior-general}
    \lrn{\sum_{i=1}^N
    \overline{W}_{i,N}^{\,\delta}\,\tilde{\Y}_{i}
    - \tilde{\Y}_{J_N}}_m
&\leq
    C\,\lr{u\,\sup_{\stackrel{b\in\lrb{\underline{b},\overline{b}}}{q=1,m}}\lr{\zeta(b\,\sqrt{\log N})}^{\frac
                            qm}
      +\tilde{\zeta}_{N}^{(m,p)}\lr{u,\lr{\log
                      N}^{-\frac12}\,N^{1-\frac{\overline{b}}2}}}\;,
  \end{align}
  where $C>0$ is a constant only depending on $p$, $\underline{b}$,
  $\overline{b}$, $\delta$ and $m$, and
  $\overline{W}_{i,N}$ denotes the self-normalized weight defined by
  $$
\overline{W}_{i,N} :=  \lr{\sum_{1\leq k\leq
    N}W_k}^{-1}\,W_i\;,\qquad 1\leq i\leq N.
  $$
\end{prop}

\begin{proof}
  Let us denote by $Q$ the probability kernel of a regular version of
  the conditional distribution of $\tilde{\Y}_{1}$ given $W_{1}$ (which exists since the real line is a Polish space, see
  \cite[Thorem~B.3.11]{douc_etal_18:_markov_chains}), that is
  $\PP\argcond{\tilde{\Y}_{1}\leq
    \tilde{y}}{W_{1}}=Q(W_{1},(-\infty,\tilde{y}])$ $\PP$-a.s. for all
  $\tilde{y}\in\rset$. Further denote the associated conditional
  generalized inverse by $Q^{\leftarrow}$, that is
  $$
  \mbox{$\forall w\in\rset_+^*$, $\forall u\in(0,1)$}, \quad
  Q^{\leftarrow}(w,u)=\inf\set{\tilde{y}\in\rset}{Q(w,(-\infty,\tilde{y}])\geq
    u}.
  $$
  Recall that $Z\sim\mathcal{N}(0,1)$ and let $U$ be independent of
  $Z$ with uniform distribution on $[0,1]$. By definition of $F$
  and $Q$, $\lr{W_{1},\tilde{\Y}_{1}}$ have the same distribution as
  $\lr{\rme^{F^{\leftarrow}\circ\Phi(Z)},Q^{\leftarrow}(\rme^{F^{\leftarrow}\circ\Phi(Z)},U)}$.
  Therefore, from now on, we 
  let $\lr{Z_i}_{i\geq1}$ and  $\lr{U_i}_{i\geq1}$ be two independent
  i.i.d. sequence distributed as $Z$ and $U$, respectively and set,
  without loss of generality,
  \begin{equation}
    \label{eq:new:YinZrelation}
  \text{$\forall i\geq 1$}, \quad
  W_{i}=\rme^{F^{\leftarrow}\circ\Phi(Z_i)}\quad\text{and}\quad
  \tilde{\Y}_{i}=Q^{\leftarrow}(W_{i},U_i)\;.
  \end{equation}
  Now denoting by $I_N$ the (random) index in $\lrcb{1,\dots,N}$
  such that $Z_{I_N}= \max(Z_1,\dots,Z_N)$, we have that $I_N$ is
  a.s. uniquely defined since $Z_{1}, \ldots, Z_{N}$ are independent and $\Phi$ is continuous. Furthermore, as required in the proposition, $\lr{I_N}_{N\geq1}$ is a non-decreasing random sequence
  such that: for all $N\geq1$, $I_N$ is valued in $\lrcb{1,\dots,N}$ with $W_{I_N}=\max\lr{W_{1},\dots,W_{N}}$.

  Next let $m,\delta\in[1,\infty)$, $\underline{b}< 1< \overline{b}$ and $p\in(1,\infty)$. It remains to show that there exists a
  constant $C>0$ only depending on these constants such that the bound~(\ref{eq:new:collapse-behavior-general}) holds for all $u>0$ and $N\in\mathbb{N}^*$. To this end, we first write
  $$
  \sum_{i=1}^N \lr{\overline{W}_{i,N}}^\delta\,\tilde{\Y}_{i} = \tilde{\Y}_{I_N} - A_N+B_N\;,
  $$
  where, denoting $J_{N}:=\lrcb{1,\dots,N}\setminus\lrcb{I_N}$, we set
  \begin{align*}
    A_N := \lr{1-\lr{\overline{W}_{I_N,N}}^\delta} \tilde{\Y}_{I_N} \quad\text{and}\quad
    B_N:=\sum_{i\in J_N}\lr{\overline{W}_{i,N}}^\delta\,\tilde{\Y}_{i} \;.
  \end{align*}
  To evaluate the left-hand side
  of~(\ref{eq:new:collapse-behavior-general}), that is $\lrn{\sum_{i=1}^N
  \overline{W}_{i,N}^{\,\delta}\,\tilde{\Y}_{i}
  - \tilde{\Y}_{I_N}}_m$, the idea is to separate the integration domain in $\lrcb{\underline{b}\,\leq \frac{Z_{I_N}}{\sqrt{\log N}}\leq  \overline{b}}\cap\lrcb{\tilde{\Y}_N^*\leq u}$ and its complementary set. More specifically, to prove~(\ref{eq:new:collapse-behavior-general}) it
  suffices to show that
  \begin{align}
    \label{eq:new:collapse-behavior-general2}
&    \PE\lr{\lrav{A_N}^m\,\mathbbm{1}_{\lrcb{\underline{b}
        \leq \frac{Z_{I_N}}{\sqrt{\log N}}\leq \overline{b}}\cap\lrcb{\tilde{\Y}_N^*\leq u}}}
                                                             \leq
                                                \delta\,C_2\;u^m\;
\sup_{{b\in\lrb{\underline{b},\overline{b}}},{q=1,m}}\lr{\zeta(b\,\sqrt{\log N})}^q\\
        \label{eq:new:collapse-behavior-general2bis}
& \PE\lr{\lrav{B_N}^m\,\mathbbm{1}_{\lrcb{\underline{b}
        \leq \frac{Z_{I_N}}{\sqrt{\log N}}\leq \overline{b}}\cap\lrcb{\tilde{\Y}_N^*\leq u}}}
                                                             \leq
                     C_2\;u^m\;\sup_{{b\in\lrb{\underline{b},\overline{b}}},{q=1,m}}\lr{\zeta(b\,\sqrt{\log N})}^q \\
                     \label{eq:new:collapse-behavior-general1}
 & \PE\lr{\lr{\sum_{i=1}^N\lrav{\tilde{\Y}_{i}}}^{m}\,\mathbbm{1}_{\lrcb{\frac{Z_{I_N}}{\sqrt{\log N}}\notin\lrb{\underline{b}
        , \overline{b}}}\cup\lrcb{\tilde{\Y}_N^*>u}}} \leq C_1\;\lr{\tilde{\zeta}_N^{(m,p)}(u,\lr{\log
                      N}^{-\frac12}\,N^{1-\overline{b}/2})}^m,
  \end{align}
  where $C_1$ only depends on $p$, $\underline{b}$ and $\overline{b}$
  and $C_2$ only depends on $m$. The desired result will then follow by combining \eqref{eq:new:collapse-behavior-general2}, \eqref{eq:new:collapse-behavior-general2bis}, the Minkowski inequality, \eqref{eq:new:collapse-behavior-general1} and the fact that since $\lrav{A_N}\leq \lrav{\tilde{\Y}_{I_N}}$ and
  $\lrav{B_N}\leq\sum_{i\in J_N}\lrav{\tilde{\Y}_{i}}$, we have $\lrav{A_N}+\lrav{B_N}\leq\sum_{i=1}^N\lrav{\tilde{\Y}_{i}}$.
  Let us thus prove separately that (i) the inequalities (\ref{eq:new:collapse-behavior-general2})
  and~(\ref{eq:new:collapse-behavior-general2bis}) hold and (ii) the inequality (\ref{eq:new:collapse-behavior-general1}) holds. \looseness=-1
  \begin{enumerateList}
    \item \textit{Proof
    of~(\ref{eq:new:collapse-behavior-general2}) and~(\ref{eq:new:collapse-behavior-general2bis}).} For all $N\geq2$, set 
    \begin{align} \label{eq:defDN}
    D_N=\sum_{i\in J_N}\frac{W_i}{W_{I_N}}.
    \end{align}  
    Then, for all $N\geq2$,
    \begin{align*}
    \lrav{A_N}& = \lr{1-(1+D_N)^{-\delta}}
    \,\lrav{\tilde{\Y}_{I_N}} \leq\lr{\lr{1+D_N}^\delta-1}\, \lrav{\tilde{\Y}_{I_N}}
    \leq \delta\,D_N\, \tilde{\Y}_N^*,
  \end{align*} 
  where in the last inequality we used the definition of $\tilde{\Y}_N^*$ and the fact that $(1+x)^\delta\leq1+\delta x$ for all
  $\delta\geq1$ and $x\geq0$. Furthermore, for all $N\geq2$,
  $$
  \lrav{B_N}\leq 
  \lr{1+D_N}^{-\delta}\,{\sum_{i\in
      J_N}\lr{\frac{W_i}{W_{I_N}}}^{\delta}\,\lrav{\tilde{\Y}_{i}}}
    \leq \sum_{i\in
              J_N}\lr{\frac{W_i}{W_{I_N}}}^{\delta}\, \lrav{\tilde{\Y}_{i}}
    \leq D_N\,  \tilde{\Y}_N^*\;,
          $$
  where we used that $\delta\geq1$, and $W_{i}\leq W_{I_N}$ for all
  $i\in J_N$. To get~(\ref{eq:new:collapse-behavior-general2})
  and~(\ref{eq:new:collapse-behavior-general2bis}), notice that 
  \begin{align*}
   \PE \lr{ D_N^m\,  (\tilde{\Y}_N^*)^m~\mathbbm{1}_{\lrcb{\underline{b}
              \leq \frac{Z_{I_N}}{\sqrt{\log N}}\leq \overline{b}}\cap\lrcb{\tilde{\Y}_N^*\leq u}}} \leq u^m \PE\lr{D_N^m\,\mathbbm{1}_{\lrcb{\underline{b}
              \leq \frac{Z_{I_N}}{\sqrt{\log N}}\leq \overline{b}}}}, 
  \end{align*}
  thus it only remains to show
  that there exists $C_2>0$ which only depends on $m$ such that \looseness=-1
  \begin{align}\label{eq:new:collapse-behavior-general3}
\PE\lr{D_N^m\,\mathbbm{1}_{\lrcb{\underline{b}
        \leq \frac{Z_{I_N}}{\sqrt{\log N}}\leq \overline{b}}}}\leq C_2\;\sup_{{b\in\lrb{\underline{b},\overline{b}}},{q=1,m}}\lr{\zeta(b\,\sqrt{\log N})}^q.
  \end{align}
  By definition of $D_N$ in \eqref{eq:defDN} and using \eqref{eq:new:YinZrelation}, we have that 
  $$
  D_N =\sum_{k\in J_N}
  \rme^{F^{\leftarrow}\circ\Phi(Z_i)- F^{\leftarrow}\circ\Phi(Z_{I_N})}
  $$
  hence conditioning on $Z_{I_N}$, we get
    \begin{multline}
      \label{eq:new:collapse-behavior-general4}
  \PE\lr{D_N^m\,\mathbbm{1}_{\lrcb{\underline{b}
        \leq \frac{Z_{I_N}}{\sqrt{\log N}}\leq \overline{b}}}}\\ =\PE\lr{\mathbbm{1}_{\lrcb{\underline{b}
        \leq \frac{Z_{I_N}}{\sqrt{\log N}}\leq \overline{b}}}\,\PE\argcond{\lr{
          \sum_{i=1}^{N-1}\rme^{F^{\leftarrow}\circ\Phi(Z_i)-F^{\leftarrow}\circ\Phi(s)}}^m}{Z_i\leq
       Z_{I_N}\,,\,1\leq i\leq N-1}}.
    \end{multline}
  Since the $Z_1, \ldots, Z_N$ are i.i.d. with the same distribution as $Z$, by the Rosenthal inequality \citep[see Therorem~2.12][]{petrov95}, we
  have, for a constant $C_4>0$ only depending on $m$,
  \begin{align*}
   & \PE\argcond{\lr{
      \sum_{i=1}^{N-1}\rme^{F^{\leftarrow}\circ\Phi(Z_i)-F^{\leftarrow}\circ\Phi(s)}}^m}{Z_i\leq
      Z_{I_N}\,,\,1\leq i\leq N-1} \\ 
   & \qquad \qquad \leq C_4\;\left[(N-1)\,\PE\argcond{\rme^{m\,(F^{\leftarrow}\circ\Phi(Z)-F^{\leftarrow}\circ\Phi(s))}}{Z\leq Z_{I_N}} \right. \\
   & \qquad \qquad \quad \left. +\lr{(N-1)\,\PE\argcond{\rme^{F^{\leftarrow}\circ\Phi(Z)-F^{\leftarrow}\circ\Phi(s)}}{Z\leq Z_{I_N}}}^m\right] \\
   & \qquad \qquad \leq C_4\;\sum_{q=1,m}\lr{(N-1)\PE\argcond{\rme^{F^{\leftarrow}\circ\Phi(Z)-F^{\leftarrow}\circ\Phi(s)}}{Z\leq Z_{I_N}}}^q,
  \end{align*}
  where, in the last inequality, we used that $m\geq1$ and that the
  exponent of the exponential in the first line is non-positive when
  conditioning on $\lrcb{Z\leq Z_{I_N}}$. 
  The expectation appearing in the last line is equal to
  $\lr{1-\Phi\lr{Z_{I_N}}}\,\zeta(Z_{I_N})$.
  Using this with~(\ref{eq:new:collapse-behavior-general4}), we get
  that
  \begin{align*}
\PE\lr{D_N^m\,\mathbbm{1}_{\lrcb{\underline{b}
 \frac{Z_{I_N}}{\sqrt{\log N}}\leq \overline{b}}}}&\leq C_4\;\sum_{q=1,m}
                                                            \PE\lr{\mathbbm{1}_{\lrcb{\underline{b}
                                                            \leq \frac{Z_{I_N}}{\sqrt{\log N}}\leq
                                                            \overline{b}}}\,\lrcb{(N-1)
                                                   (1-\Phi(Z_{I_N}))\,\zeta(Z_{I_N})}^q} \\
                                                          &\leq
                                                            2\,C_4\max_{q=1,m}\lr{
                                                            (N-1)^q\,\PE  \lr{\lr{1-\Phi(Z_{I_N})}^q} 
                                                            \;\sup_{b\in\lrb{\underline{b},\overline{b}}}\lr{\zeta(b\,\sqrt{\log N})}^q}\;.
  \end{align*}
    Since $\Phi$ is continuous, we further have that, for any $q\geq1$, 
  $$
  \PE\lr{\lr{1-\Phi(Z_{I_N})}^q}=\PE\lr{(1-U_{(N,N)})^q}=N\int_0^1(1-x)^q\,x^{N-1}\,\rmd x=\frac{\Gamma(q+1)\,N!}{\Gamma(N+q+1)}\leq\frac{\Gamma(q+1)}{(N+1)^q}\;,
$$
where $U_{(N,N)}$ is the maximum of $N$ i.i.d. random variables with
uniform distribution in $[0,1]$. Hence (\ref{eq:new:collapse-behavior-general3}) holds with
$C_2=2\,C_4\,\Gamma(m+1)$, which concludes the proof of~(\ref{eq:new:collapse-behavior-general2}) and~(\ref{eq:new:collapse-behavior-general2bis}).

\item \textit{Proof
  of~(\ref{eq:new:collapse-behavior-general1}).}
Since $m\geq1$, $x\mapsto x^{m}$ is convex and we have
$\lr{\frac1N\sum_{i=1}^N\lrav{\tilde{\Y}_{i}}}^{m}\leq
\frac1N\sum_{i=1}^N\lrav{\tilde{\Y}_{i}}^{m}$, leading to
\begin{align*}
 \PE\lr{\lr{\sum_{i=1}^N\lrav{\tilde{\Y}_{i}}}^{m}\,\mathbbm{1}_{\mathbf{D}_N}}
& \leq N^{m-1}\;\sum_{i=1}^N\PE\lr{\lrav{\tilde{\Y}_{i}}^{m}\,\mathbbm{1}_{\mathbf{D}_N}}\\
&=  N^m\,\PE\lr{\lrav{\tilde{\Y}_{1}}^{m}\,\mathbbm{1}_{\mathbf{D}_N}}\;,
\end{align*}
where we used that the domain $\mathbf{D}_N=\lrcb{\frac{Z_{I_N}}{\sqrt{\log N}}\notin\lrb{\underline{b}
        , \overline{b}}}\cup\lrcb{\tilde{\Y}_N^*>u}$ is stable by permutation of the indices $i$
within $\lrcb{1,\dots,N}$ so that
$\lrav{\tilde{\Y}_{i}}^{m}\,\mathbbm{1}_{\mathbf{D}_N}$ has the same
distribution as $\lrav{\tilde{\Y}_{1}}^{m}\,\mathbbm{1}_{\mathbf{D}_N}$ for
all $i=1,\dots,N$. Now we note that $\mathbbm{1}_{\mathbf{D}_N}\leq \mathbbm{1}_{\lrcb{\frac{Z_{I_N}}{\sqrt{\log N}}\notin\lrb{\underline{b}
      , \overline{b}}}}+\mathbbm{1}_{\lrcb{\tilde{\Y}_N^*>u}}$ and, by
definition of $\tilde{\zeta}_N^{(m,p)}$ in~(\ref{eq:new:def-zeta-tilde}), to
get~(\ref{eq:new:collapse-behavior-general1}) with $C_1=\max\lr{1,C_3^{(p-1)/p}}$, it only remains to show that 
  \begin{align}\label{eq:new:collapse-behavior-general1.1}
\PE\lr{\lrav{\tilde{\Y}_{i}}^{m}\,\mathbbm{1}_{\lrcb{\frac{Z_{I_N}}{\sqrt{\log N}}\notin\lrb{\underline{b}
        , \overline{b}}}}}
    \leq \lr{\PE\lr{\lrav{\tilde{\Y}_{1}}^{p\,m}}}^{1/p}\,\lr{C_3\,\lr{\log N}^{-1/2}\,N^{1-\overline{b}/2}}^{(p-1)/p}\;,
  \end{align}
  for a certain $C_3>0$ only depending on $\underline{b}$ and $\overline{b}$.
  Using~\Cref{lem:MaxProbaBounds} with $c=\overline{b}/2-1$ and
  $c=2/\underline{b}-1$ successively in~(\ref{eq:MaxProbaBoundLarge})
  and~(\ref{eq:MaxProbaBoundSmall}), we indeed find for such a
  constant $C_3$,
$$
\PP\lr{\frac{Z_{I_N}}{\sqrt{\log N}}\notin\lrb{\underline{b},\overline{b}}}
\leq C_3\,\lr{\lr{\log N}^{-1/2}\,N^{1-\overline{b}/2}}\;.
$$
Then we get~(\ref{eq:new:collapse-behavior-general1.1}) by applying
the H{\"o}lder inequality and the proof
of~(\ref{eq:new:collapse-behavior-general1}) is concluded.
\end{enumerateList}

\end{proof}

\begin{rem}  \label{rem:collapse-general}
  \Cref{prop:new:collapse-behavior-general} investigates the behavior of the weighted average of $\lr{\tilde{\Y}_{i}}_{i=1,\dots,N}$ when using self-normalized positive weights $\lr{\overline{W}_{i,N}}_{i=1,\dots,N}$ to some power $\delta\geq1$. More
    precisely,~(\ref{eq:new:collapse-behavior-general}) provides a bound
    of the $L^m$-norm of the error when approximating this average by a
    single $ \tilde{\Y}_{I_N}$, with $I_N$ corresponding to an index
    with maximal weight.

    Under standard moment conditions, and using the law of large numbers, such an average should be well approximated as $N\to\infty$ by $N^{1-\delta}\PE\lr{W_1^\delta\tilde{\Y}_{1}}\lr{\PE\lr{W_1}}^{-\delta}$
    rather than by the (random) $\tilde{\Y}_{I_N}$ as
    in~(\ref{eq:new:collapse-behavior-general}).
    However, we will apply \eqref{eq:new:collapse-behavior-general} with
    $F$ depending on $N$ in such a way that the maximal weight $W_{I_N}$ tends to dominate over all other weights. 
    
    In particular, in our applications
    of \Cref{prop:new:collapse-behavior-general}, we will take
    advantage of the fact that the constant $C$ does not depend on $F$
    in the
    bound~(\ref{eq:new:collapse-behavior-general}). Furthermore,
    since~(\ref{eq:new:collapse-behavior-general}) holds for any
    $u>0$, we will choose $u$ for a given $N$ so that the right-hand
    side of~(\ref{eq:new:collapse-behavior-general}) is as small as
    possible. This will be achieved by compromising between the $u$ in
    the first term (which increases as $u$ increases) and the second
    expectation in~(\ref{eq:new:def-zeta-tilde}) defining
    $\tilde{\zeta}_N^{(m,p)}(u,s)$ which decreases as $u$
    increases. Finally, the first term between the parentheses
    of~(\ref{eq:new:def-zeta-tilde}) will be made small by taking
    $\overline{b}$ large enough to compensate the $N$ in front of the
    parentheses.
    \end{rem}

  \subsection{Additional notation, useful first properties and derivations}
  \label{subsec:addNotation}

  In the remaining of \Cref{sec:deferr-proofs-result}, we let
  $\varepsilon \sim \REPq$ and $\varepsilon_1, \varepsilon_2, \ldots$ be
  i.i.d. copies of $\varepsilon$ and we assume
  that~\ref{hyp:reparamHighDimGaussian} holds and so does
  \ref{hyp:hypZeroREPhighDim}. We now introduce the following helpful notation:
\begin{align}
& \repell[\theta][\phi][\phi'](\cdot;x) = \log \tw[\phi][\cdot][\phi] \\
& B(\theta, \phi;x)^2 = \mathbb{V}_{\varepsilon \sim \REPq}\lr{\repell[\theta][\phi][\phi'](\varepsilon;x)} = 2 \KL \\
& \Wrepell_{j,N} =  \lr{\sum_{k = 1}^N \rme^{(1-\alpha)\,\repell(\varepsilon_k;x)}}^{-1}\,\rme^{(1-\alpha)\,\repell(\varepsilon_j;x)}\;,\qquad 1\leq j\leq N \label{eq:barWjN-def} \\
    \label{eq:repellestandard-def}
 & \repellstandard[\theta][\phi][\phi'](\cdot;x)=\frac{\repell[\theta][\phi][\phi'](\cdot;x)-\PE\lr{\repell[\theta][\phi][\phi'](\varepsilon;x)}}{B(\theta, \phi;x)}  \\
  \label{eq:def-tildeB}
& \tBREP=\sqrt{\mathbb{V}\lr{\partial_{\psi}\repellstandard(\varepsilon;x)}}  \;.
\end{align}
Under \ref{hyp:reparamHighDimGaussian}, the expectation and variance
of $\repell[\theta][\phi][\phi'](\varepsilon;x)$ are well
defined. Furthermore, $\repellstandard[\theta][\phi][\phi']$ is well defined in~(\ref{eq:repellestandard-def}) if $B(\theta, \phi;x)>0$ and $\tBREP^2$ is well defined if
$(\theta',\phi')\mapsto\repellstandard[\theta'][\phi'][\phi''](\varepsilon;x)$
is well defined and differentiable in a neighborhood of
$(\theta,\phi)$. Using \Cref{prop:gaussian-case-REP-df}, all
these conditions hold providing that $B(\theta, \phi;x)>0$ (or
equivalently $\KL>0$), which we assume in the following since we are interested in the setting where $N,\KL\to\infty$.  More precisely, we get the
following assertions:
\begin{enumerate}[labelindent=\parindent, label=(\roman*)]
\item We have that $\partial_\psi B(\theta, \phi;x)$ is well defined;
\item Since $B(\theta, \phi;x)>0$,
    $(\theta',\phi')\mapsto\repellstandard[\theta'][\phi'][\phi''](\varepsilon;x)$
is well defined and differentiable in a neighborhood of $(\theta,\phi)$;
  \item We have that $\tBREP$ is well defined;
\item\label{ass:addit-notat-gauss} We have that $\lr{\lr{\repellstandard(\varepsilon;x),\partial_{\psi}\repellstandard(\varepsilon;x)}}_{(\theta,\phi)\in\Theta\times\Phi}$
  is a Gaussian process;
\item\label{ass:addit-notat-inter} We can interchange the derivative $\partial_\psi$ with the expectation
  sign for $\repellstandard(\varepsilon;x)$ as well as its square. 
\end{enumerate}
Using Assertion~\ref{ass:addit-notat-inter} and since $\repellstandard[\theta][\phi][\phi'](\varepsilon;x)$ is centered with variance 1, we then get that
$$\PE\lr{\partial_{\psi}\repellstandard(\varepsilon;x)}=\mathbb{C}\mathrm{ov}\lr{\repellstandard(\varepsilon;x),\partial_{\psi}\repellstandard(\varepsilon;x)}=0\;.$$
Combining this with Assertion~\ref{ass:addit-notat-gauss}, we thus have the following useful property: 
\begin{enumerate}[align=left,labelindent=\parindent, label=($\mathrm{P}^{\mathrm{REP}}$)]
\item \label{item:propREPest-gaussian}  $\lr{\repellstandard(\varepsilon;x)}_{(\theta,\phi)\in\Theta\times\Phi}$
  and $\lr{\partial_{\psi}\repellstandard(\varepsilon;x)}_{(\theta,\phi)\in\Theta\times\Phi}$
  are two independent and centered Gaussian processes.
\end{enumerate}
Plugging in the notation we introduced in \eqref{estimREP} and using
straightforward algebra, the REP gradient estimator reads
\begin{align}
   \gradREP[1,N]
& = \partial_{\psi}\PE(\repell(\varepsilon;x)) + \partial_{\psi} B(\theta,
      \phi;x) \;\sum_{j=1}^N
     \Wrepell_{j,N}\; \repellstandard(\varepsilon_j;x) \nonumber \\
    \label{eq:gradrepN-nice}
                   &\quad
                     + B(\theta, \phi;x)  \sum_{j=1}^N 
                     \Wrepell_{j,N}\;\partial_{\psi}\repellstandard(\varepsilon_j;x)\;.
\end{align}
Furthermore, differentiating $B^2(\theta,\phi;x)= \mathbb{V}\lr{\repell[\theta][\phi][\phi](\varepsilon;x)}$, we get that
\begin{align*}
B(\theta,
      \phi;x) \, \partial_{\psi} B(\theta,
      \phi;x) &= \frac12\partial_{\psi} B^2(\theta,
      \phi;x)=\mathbb{C}\mathrm{ov}\lr{\repell[\theta][\phi][\phi](\varepsilon;x),\partial_{\psi}
                \repell[\theta][\phi][\phi](\varepsilon;x)}\;.
\end{align*}
Next differentiating the product
$B(\theta, \phi;x)\repellstandard(\varepsilon;x)$, which equals $\repell[\theta][\phi][\phi'](\varepsilon;x)$
up to a added non-random constant, we obtain the sum of two uncorrelated
variables and taking the variance leads to
\begin{align}
 & \mathbb{V}\lr{\partial_{\psi}
  \repell[\theta][\phi][\phi](\varepsilon;x)}=
                                              \lr{\partial_{\psi} B(\theta,
                                              \phi;x)}^2+\lr{B(\theta,\phi;x)\,\tBREP}^2\;. \label{eq:rewritingVar11}
\end{align}
Since we can rewrite the correlation from (\ref{eq:def-corr-parameter-rep}) as
\begin{align*}
 \corrREP =  \frac{\mathbb{C}\mathrm{ov}\lr{\repell[\theta][\phi][\phi](\varepsilon;x),\partial_{\psi}
                \repell[\theta][\phi][\phi](\varepsilon;x)}}{\sqrt{B(\theta,\phi;x)^2  \mathbb{V}\lr{\partial_{\psi}
  \repell[\theta][\phi][\phi](\varepsilon;x)}}}
\end{align*}
we also get using the expressions above that 
\begin{align}
  \label{eq:corr-parameter-rep_Bd-tBd}
  &\corrREP = \frac{\partial_{\psi} B(\theta,
    \phi;x)}{\sqrt{\lr{\partial_{\psi} B(\theta,
    \phi;x)}^2+\lr{B(\theta,
    \phi;x)\,\tBREP}^2}}\;.
\end{align} 
To conclude this section we provide expression for $\mathrm{SNR}[\gradREP[1,1]]$ as well as $\mathrm{SNR}[\gradREP[1,1]-\partial_{\psi} \ell(\theta; x)]$, where the latter is linked to $\corrREP$. Using (\ref{eq:gradrepN-nice}) with $N=1$, we obtain that
\begin{multline}\label{eq:gradREP11}
  \gradREP[1,1]=\partial_{\psi}\PE(\repell(\varepsilon;x))\\ +
  \partial_{\psi} B(\theta, \phi;x)
  \repellstandard(\varepsilon_1;x) + B(\theta, \phi;x) \,
  \partial_{\psi}\repellstandard(\varepsilon_1;x) \;.
\end{multline}
It follows from~(\ref{eq:gradREP11}) and Property~\ref{item:propREPest-gaussian} that                 
\begin{align}
  & \mathbb{V}(\gradREP[1,1]) = \lr{\partial_{\psi} B(\theta, \phi;x)}^2+\lr{B(\theta,
  \phi;x)\tBREP}^2 \label{eq:var-Neq1-rep} \\
  & \label{eq:snr-Neq1-rep}
  \mathrm{SNR}[\gradREP[1,1]] = \frac{\lrav{\partial_{\psi}\PE(\repell(\varepsilon;x))}}{\sqrt{\lr{\partial_{\psi} B(\theta, \phi;x)}^2+\lr{B(\theta,
  \phi;x)\tBREP}^2}}
\end{align}
and, using~(\ref{eq:gaussian-exp-moment-rep-proof})
and~(\ref{eq:corr-parameter-rep_Bd-tBd}), which leads to
\begin{align}
  \mathrm{SNR}[\gradREP[1,1]-\partial_{\psi} \ell(\theta; x)] =B(\theta, \phi;x)\,\lrav{\corrREP}\;. \label{eq:rewritingCorr}
\end{align}
Notice in particular that we also have
\begin{align} 
\frac{\partial_{\psi} \ell(\theta; x)-\PE(\gradREP[1,1])}{\mathbb{V}(\gradREP[1,1])} & = \frac{{B(\theta, \phi;x) \partial_{\psi} B(\theta, \phi;x)}}{\sqrt{\lr{\partial_{\psi} B(\theta, \phi;x)}^2+\lr{B(\theta,
                                                                  \phi;x) \tBREP}^2}} \nonumber \\
                                                                  & = B(\theta, \phi;x)\,{\corrREP} \label{eq:rewritingCorrBis}
\end{align}

\subsection{Proof of \Cref{thm:CollapseSNRGaussian} and \Cref{thm:CollapseSNRGaussian-contd}} We start with two lemmas.

\begin{lem}\label{lem:new:exp-moment-condditional-gaussian}
  Let $Z\sim\mathcal{N}(0,1)$. Then, for all
  $s\in\rset$ and $\sigma>0$,
  \begin{equation}
    \label{eq:new:zeta:log-normal-case:exact}
  \PE\lr{\rme^{\sigma(Z-s)}\,\vert\, Z\leq s} = \frac{\Phi(s-\sigma)}{\Phi(s)}\;\frac{\varphi(s)}{\varphi(\sigma-s)}\;.
\end{equation}
It follows that, defining $\zeta$ by~(\ref{eq:new:def-zeta}) with
  $F^{\leftarrow}=\mu+\sigma \Phi^{-1}$ where $\mu \in \rset$ and $\sigma>0$, there exists a universal
  constant $C>0$ such that, for all $\sigma\geq 2 s\geq1$,
  \begin{equation}
    \label{eq:new:zeta:log-normal-case:bound}
\zeta(s)\leq C\;\frac{s}{\sigma} \;.    
  \end{equation}
\end{lem}
\begin{proof}
We have, for all $s\in\rset$  and $\sigma>0$, 
\begin{align*}
\PE\lr{\rme^{\sigma(Z-s)}\,\vert\, Z\leq s}
  &=\frac1{\Phi(s)}\,\PE\lr{\rme^{\sigma(Z-s)}\,\mathbbm{1}_{\lrcb{Z\leq s}}} \\
  &=\frac1{\Phi(s)}\;\PE\lr{\frac{\varphi(s)\,\varphi(Z-\sigma)}{\varphi(\sigma-s)\,\varphi(Z)}\mathbbm{1}_{\lrcb{Z\leq
    s}}} \\
&=\frac1{\Phi(s)}\,\frac{\varphi(s)}{\varphi(\sigma-s)}\;\PE\lr{\mathbbm{1}_{\lrcb{Z+\sigma\leq
              s}}},
\end{align*}
where, in the last line, we used that $Z+\sigma$ has density $u\mapsto\varphi(u-\sigma)$ hence the likelihood ratio $\varphi(Z-\sigma)/\varphi(Z)$ amounts to change $Z$ into $Z+\sigma$. We thus get~(\ref{eq:new:zeta:log-normal-case:exact}). Now, defining $\zeta$ by~(\ref{eq:new:def-zeta}) with
$F^{\leftarrow}=\mu+\sigma \Phi^{-1}$, we get, for all $s>0$ and $\sigma>0$,
$$
\zeta(s)=\frac{1}{\Phi(s)} \frac{m(\sigma-s)}{m(s)},
$$
where the function $m$ corresponds to the Mills ratio $m(u) = \frac{1-\Phi(u)}{\varphi(u)}$ with $u > 0$. Since
\begin{align*}
  \forall u>0, \quad  \frac{u}{u^2+1} < m(u) < \frac{1}{u},
\end{align*}
\citep[see][]{gordon/aoms/1177731721} we finally deduce that: for all $\sigma> 2s \geq 1$,
\begin{align*}
  \zeta(s) & \leq \frac{1}{\Phi(1)} \frac{s^2 + 1}{s} \frac{1}{\sigma-s} \leq C ~\frac{s}{\sigma},
\end{align*}
where $C = {6}{\Phi(1)}^{-1}$ and where we have used that $s \geq 1$ and $s \leq {\sigma}/{2}$.
\end{proof}
\begin{lem}
  \label{lem:gaussian-collapse-gathered-elementary-statements} Let
  $\xi_1,\xi_2,\dots$ be i.i.d. standard normal random variables
  and set, for all $N=1,2,\dots$, and $\beta\in\mathbb{R}$, 
  \begin{equation}
    \label{eq:weights-def-with-Z}
    \sumW :=  \lr{\sum_{k=1}^{N} \rme^{\beta\,\xi_k}}^{-1}\,\rme^{\beta\,\xi_j}\;,\qquad 1\leq j\leq N\;.   
  \end{equation}
Let $\delta\geq1$, $\lambda \in [0,1]$ and $q\geq1$. Then we have, as $N, \beta \to \infty$ with $\sqrt{\log
N}=o(\beta)$, 
\begin{align}
  \label{eq:collapse-phen-gaussian-case} 
&\PE\lr{\sum_{j=1}^N
    (\sumW)^\delta \; \xi_j} =  \lr{2\,\log N}^{1/2} \lr{1+ o(1)},   \\  
&  \label{eq:thm:gaussian-high-dim-gradient-snr:conclude1}
\mathbb{V}\lr{\sum_{j=1}^N [\lambda \sumW + (1-\lambda) (\sumW)^2]\,\xi_j} = o \lr{ \log N},   \\
\label{eq:thm:gaussian-high-dim-gradient-snr:conclude2}
& \PE\lr{\sum_{j=1}^N
    \lr{\sumW}^\delta} =  1+ o(1).
\end{align}  
\end{lem}
\begin{proof}
  We will first show that, for any $m\in[1,\infty)$,   as $N, \beta \to \infty$ with $\sqrt{\log
N}=o(\beta)$, 
  \begin{align} \label{eq:Prop3ToUseGauss1}
&  \lrn{\sum_{i=1}^N
  (\sumW[i])^{\,\delta}\,\xi_{i}
  - \max(\xi_1,\dots,\xi_N)}_m=  o\lr{\sqrt{\log N}} \;,\\
  \label{eq:Prop3ToUseGauss2}
&  \lrn{\sum_{i=1}^N
  (\sumW[i])^{\,\delta}
  - 1}_m=  O\lr{\lr{\frac{\sqrt{\log N}}{\beta}}^{\frac1m}+N^{-\frac qm}} \;.
\end{align}    
Then we will show that~(\ref{eq:Prop3ToUseGauss1}) implies
 \eqref{eq:collapse-phen-gaussian-case}-\eqref{eq:thm:gaussian-high-dim-gradient-snr:conclude1}, and
  that~(\ref{eq:Prop3ToUseGauss2}) implies
  \eqref{eq:thm:gaussian-high-dim-gradient-snr:conclude2}.
\begin{enumerateList}
  \item \textit{Proof of~(\ref{eq:Prop3ToUseGauss1})
    and~(\ref{eq:Prop3ToUseGauss2})}. We use
  \Cref{prop:new:collapse-behavior-general} successively in the two settings
  \begin{enumerate}[label=Setting~\arabic*),wide=0pt,labelindent=\parindent]
  \item\label{item:setting1}   for all
  $j = 1 \ldots N$,  $W_j = \rme^{\beta\,\xi_j}$ and
  $\tilde{Y}_j = \xi_j$;
\item\label{item:setting2} for all
  $j = 1 \ldots N$,  $W_j = \rme^{\beta\,\xi_j}$ and $\tilde{Y}_j = 1$.
  \end{enumerate}
In both settings, we have $F^{\leftarrow}(u)=\beta
  \, \Phi^{-1}(u)$. Then
  \Cref{prop:new:collapse-behavior-general} gives us that there exists a
  non-decreasing random sequence $\lr{I_N}_{N\geq1}$ such that, for
  all $N\geq1$, $I_N$ is valued in $\lrcb{1,\dots,N}$ with
  $W_{I_N}=\max\lr{W_{1},\dots,W_{N}}$ and, for all
  $m,\delta\in[1,\infty)$, $\underline{b}< 2< \overline{b}$,
  $p\in(1,\infty)$, $u>0$ and $N\in\mathbb{N}^*$, we have
  in~\ref{item:setting1} and~\ref{item:setting2},
    \begin{align*}
      \lrn{\sum_{i=1}^N
      (\sumW[i])^{\,\delta}\,\tilde{\Y}_{i}
      - \tilde{\Y}_{I_N}}_m
  &\leq
      C\,\lr{u\,\sup_{\stackrel{b\in\lrb{\underline{b},\overline{b}}}{q=1,m}}\lr{\zeta(b\,\sqrt{\log N})}^{\frac
                              qm}
        +\tilde{\zeta}_{N}^{(m,p)}\lr{u,\lr{\log
                        N}^{-\frac12}\,N^{1-\frac{\overline{b}}2}}}\;,
    \end{align*}
    where $\zeta$ is defined in \eqref{eq:new:def-zeta}, $\tilde{\zeta}_{N}^{(m,p)}$ is defined in \eqref{eq:new:def-zeta-tilde} and $C>0$ is a constant only depending on $p$, $\underline{b}$,
    $\overline{b}$, $\delta$ and $m$. By
    \Cref{lem:new:exp-moment-condditional-gaussian}, using
    $F^{\leftarrow}$ as above, we have that there exists a universal
    constant  $C_0>0$ such that, for all $b>0$, 
    \begin{align*}
\beta \geq 2 b
    \sqrt{\log N} \geq 1\Longrightarrow      \zeta(b\,\sqrt{\log N}) \leq C_0 ~ \frac{b\,\sqrt{\log N}}{\beta}\;.
    \end{align*}
Note that if $b\leq\overline{b}$ and  $\beta \geq 2\,(\overline{b}\vee C_0)
    \sqrt{\log N} \geq 1$, then the condition on the left-hand side is
    satified and the upper bound in the right-hand side is at most 1.     
    Using this in the previous display for the special case $p =
    2$, we get that for all
  $m,\delta\in[1,\infty)$ and  $\overline{b}>2$, there exists
  $C,C_0>0$ such that, for all $u>0$, if $\beta \geq 2\,(\overline{b}\vee C_0)
    \sqrt{\log N} \geq 1$,  
\begin{align} \label{eq:Prop3ToUseGauss}
  \lrn{\sum_{i=1}^N
  (\sumW[i])^{\,\delta}\,\tilde{\Y}_{i}
  - \tilde{\Y}_{I_N}}_m
&\leq
  C\,\lr{u  \lr{\frac{\overline{b}\, C_0 \sqrt{\log N}}{\beta}}^{\frac1m} +
                          \tilde{\zeta}_{N}^{(m,2)}\lr{u,\lr{\log
                          N}^{-\frac12}\,N^{1-\frac{\overline{b}}2}}}\;.
\end{align}
To prove \eqref{eq:Prop3ToUseGauss1} and \eqref{eq:Prop3ToUseGauss2} we now consider \ref{item:setting1} and~\ref{item:setting2} separately
by taking  $\tilde{Y}_j = \xi_j$ and $\tilde{Y}_j = 1$ respectively in $\tilde{\zeta}_{N}^{(m,2)}\lr{u,\lr{\log
    N}^{-\frac12}\,N^{-1-\frac{\overline{b}}2}}$ and choosing
adequate values
of $\overline{b}$ and $u$ in the upper bound above. Namely, we will show that
\begin{itemize}
  \item If  $\tilde{Y}_j = \xi_j$ for all $j\geq1$, then: as $N \to \infty$,
  \begin{align}
    \label{eq:setting1tildezeta}
  \tilde{\zeta}_{N}^{(m,2)}\lr{\sqrt{10\log N},\lr{\log
      N}^{-\frac12}\,N^{-2}} = o \lr{\log N}\;.
  \end{align}
  \item If  $\tilde{Y}_j = 1$ for all $j\geq1$, then: for all $N\geq2$,
  \begin{align}
    \label{eq:setting2tildezeta}
    \tilde{\zeta}_{N}^{(m,2)}\lr{2,\lr{\log
      N}^{1-\frac{\overline{b}}2}}= (\log N)^{-\frac14}\,N^{\frac32-\frac{\overline{b}}4}\;.
  \end{align}
\end{itemize}
Observe then that (\ref{eq:setting2tildezeta}) follows directly from plugging in $(u,s) = \lr{2,\lr{\log
      N}^{1-\frac{\overline{b}}2}}$ in~(\ref{eq:new:def-zeta-tilde}). As for \eqref{eq:setting1tildezeta}: setting $\tilde{Y}_j = \xi_j$ for all $j\geq1$ and plugging in $(u,s) = (\sqrt{10\log N},\lr{\log
      N}^{-\frac12}\,N^{-2})$ in (\ref{eq:new:def-zeta-tilde}) yields that for all
$N\geq2$,
\begin{align*}
\tilde{\zeta}_{N}^{(m,2)}\lr{\sqrt{10\log N},\lr{\log
    N}^{-\frac12}\,N^{-2}} & = \lrn{\xi_{1}}_{p\,m}^{1/m}(\log
                             N)^{-\frac{1}{4}} +N
                             \lrn{\xi_{1}\,\mathbbm{1}_{\lrcb{\xi_N^*>\sqrt{10
                             \log N}}}}_m \;,
\end{align*}
where
$\xi_N^*=\max\lr{\lrav{\xi_1},\dots,\lrav{\xi_N}}$. Since
$$\lrn{\xi_{1}\,\mathbbm{1}_{\lrcb{\xi_N^*>\sqrt{10 \log
        N}}}}_m\leq\lrn{\xi_{1}}_{2m}
\lrn{\mathbbm{1}_{\lrcb{\xi_N^*>\sqrt{2(1+c) \log
        N}}}}_{2m},$$ the desired result (\ref{eq:setting1tildezeta}) will follow if we can show that: as $N\to\infty$,
$$
\PP\lr{\xi_N^*>\sqrt{10 \log N}}= o\lr{N^{m/2}}\;.
$$
Writing that $\xi_N^*=\max(M_N,M'_N)$ where we have set
$M_N=\max\lr{\xi_1,\dots,\xi_N}$ and $M'_N=\max\lr{-\xi_1,\dots,-\xi_N}$, it holds that
$\PP\lr{\xi_N^*>u}\leq\PP(M_N>u)+\PP(M'_N>u)=2\PP(M_N>u)$ for any $u>0$, hence the last
display follows from~(\ref{eq:MaxProbaBoundLarge}) in
\Cref{lem:MaxProbaBounds} with $c=4$.

Now, for any given $\overline{b}>2$, if $N, \beta \to \infty$ with
$\sqrt{\log N}=o(\beta)$, we eventually have
$\beta \geq 2\,(\overline{b}\vee C_0) \sqrt{\log N} \geq 1$ so
that~(\ref{eq:Prop3ToUseGauss}) eventually applies.  If
$\tilde{Y}_j = \xi_j$ for all $j$, we have
$\tilde{\Y}_{I_N}=\max(\xi_1,\dots,\xi_N)$, and the
bound~(\ref{eq:Prop3ToUseGauss}) with $\overline{b}=6$ among
with~(\ref{eq:setting1tildezeta}), gives us  that, as  $N, \beta \to \infty$ with $\sqrt{\log
N}=o(\beta)$, (\ref{eq:Prop3ToUseGauss1}) hold.
If now
$\tilde{Y}_j = 1$ for all $j$, we have 
  $\tilde{\Y}_{I_N}=1$ and  the
bound~(\ref{eq:Prop3ToUseGauss}) with $\overline{b}$ large enough
among with~(\ref{eq:setting2tildezeta}), gives us that~(\ref{eq:Prop3ToUseGauss2}) holds.
\item \textit{Proofs of~(\ref{eq:collapse-phen-gaussian-case})
    and~(\ref{eq:thm:gaussian-high-dim-gradient-snr:conclude1})}. Since
  the expectation is 1-Lipschitz for the $L^1$-norm, the
  asymptotic behavior~(\ref{eq:Prop3ToUseGauss1}) with
  $m=1$ implies that, for any $\delta\geq1$, 
  $$
  \PE\lr{\sum_{i=1}^N
  (\sumW[i])^{\,\delta}\,\xi_{i}} = \PE\lr{\max(\xi_1,\dots,\xi_N)}_m +  o\lr{\sqrt{\log N}}\;.
$$
Applying \Cref{lem:MaxMoment}  with $m=1$ the first term in the
right-hand side is asymptotically equivalent to $\sqrt{2\log N}$ and
we get~(\ref{eq:collapse-phen-gaussian-case}). The cases $\delta=1,2$
in particular together give that, for any $\lambda\in[0,1]$,
\begin{equation}
  \label{eq:firstMomentsumWwithLambda}
\PE\lr{\sum_{j=1}^N [\lambda \sumW + (1-\lambda)
  (\sumW)^2]\,\xi_j} = \lr{2\,\log N}^{1/2} \lr{1+ o(1)}\;.  
\end{equation}
By the Minkowski inequality for the $L^2$-norm, we obtain that
\begin{align*}
  &\lrav{\lrn{\sum_{j=1}^N [\lambda \sumW + (1-\lambda)
  (\sumW)^2]\,\xi_j}_2-\lrn{\max(\xi_1,\dots,\xi_N)}_2}\\
& \leq\lrn{\sum_{j=1}^N [\lambda \sumW + (1-\lambda)
                                                            (\sumW)^2]\,\xi_j-\max(\xi_1,\dots,\xi_N)}_2\\
& \leq\lambda\lrn{\sum_{j=1}^N  \sumW\,\xi_j-\max(\xi_1,\dots,\xi_N)}_2 + (1-\lambda)\lrn{\sum_{j=1}^N(\sumW)^2\,\xi_j-\max(\xi_1,\dots,\xi_N)}_2.
\end{align*}
Hence, using~(\ref{eq:Prop3ToUseGauss1})
again, this time with $m=2$, we deduce 
\begin{align*}
  &\lrav{\lrn{\sum_{j=1}^N [\lambda \sumW + (1-\lambda)
  (\sumW)^2]\,\xi_j}_2-\lrn{\max(\xi_1,\dots,\xi_N)}_2} =o\lr{\sqrt{\log N}},
\end{align*}
which we can rewrite equivalently as 
$$
\lrn{\sum_{j=1}^N [\lambda \sumW + (1-\lambda)
  (\sumW)^2]\,\xi_j}_2=\lrn{\max(\xi_1,\dots,\xi_N)}_2+o\lr{\sqrt{\log N}}.
$$
Using \Cref{lem:MaxMoment} with $m=2$ to evaluate the first term in
the right-hand side and taking the sqaure, we get that
$$
\PE\lr{\lr{\sum_{j=1}^N [\lambda \sumW + (1-\lambda)
  (\sumW)^2]\,\xi_j}^2}=2\,\log N\,(1+o(1))\;.
$$
This asymptotic behavior with the square
of~(\ref{eq:firstMomentsumWwithLambda}) yields~(\ref{eq:thm:gaussian-high-dim-gradient-snr:conclude1}).

\item \textit{Proof of~(\ref{eq:thm:gaussian-high-dim-gradient-snr:conclude2})}. The
  proof of~(\ref{eq:thm:gaussian-high-dim-gradient-snr:conclude2}) is obtained similarly to the proofs of~(\ref{eq:collapse-phen-gaussian-case})
  and~(\ref{eq:thm:gaussian-high-dim-gradient-snr:conclude1}), but starting from~(\ref{eq:Prop3ToUseGauss2}) instead of (\ref{eq:Prop3ToUseGauss1}).
\end{enumerateList}
\end{proof}
We can now prove \Cref{thm:CollapseSNRGaussian} and \Cref{thm:CollapseSNRGaussian-contd}. 

 \subsubsection{Proof of \Cref{thm:CollapseSNRGaussian}}
\label{app:thm:gaussian-high-dim-gradient-snr}

The proof is made of two steps. Firstly we show that: as $N,B(\theta, \phi;x)\to\infty$ with~(\ref{eq:gaussian-high-dim-growing-cond}) holding, 
   \begin{align}
   \label{eq:gaussian-high-dim-gradient-equiv}
   \PE \lr{\gradREP[1,N]} &=\partial_{\psi}\PE(\repell(\varepsilon;x)) + \sqrt{2\log N} \;\partial_{\psi}
     B(\theta, \phi;x)\; \lr{1+o(1)},\\
\label{eq:lastBeforeNDvarLB}   \mathbb{V}\lr{\gradREP[1,N]} & \geq \lr{B(\theta, \phi;x)\tBREP}^2\lr{1+o(1)}.
   \end{align}
Secondly, we deduce the claim of the proposition.
\begin{enumerateList}
\item \textbf{Proof of~(\ref{eq:gaussian-high-dim-gradient-equiv}) and~(\ref{eq:lastBeforeNDvarLB}).} Recall that
  $\varepsilon_1,\varepsilon_2,\dots$ are i.i.d. copies of
  $\varepsilon$. Since $\Wrepell_{1,N}, \ldots, \Wrepell_{N,N}$ as defined in~(\ref{eq:barWjN-def}) can be expressed as functions of $\repell(\varepsilon_k;x)$ for all $k=1 \dots N$, by Property~\ref{item:propREPest-gaussian} they are independent of the centered random variables $\partial_{\psi}\repellstandard(\varepsilon_1;x), \ldots, \partial_{\psi}\repellstandard(\varepsilon_N;x)$. Furthermore,   Property~\ref{item:propREPest-gaussian} also gives us that $\repellstandard(\varepsilon_1;x), \ldots, \repellstandard(\varepsilon_N;x)$ are i.i.d. $\mathcal{N}(0,1)$. Next noticing that $\Wrepell_{1,N}, \ldots, \Wrepell_{N,N}$ can be rewritten using $\repellstandard(\varepsilon_k;x)$ for all $k=1 \dots N$, we can apply \Cref{lem:gaussian-collapse-gathered-elementary-statements} with $\xi_k=\repellstandard(\varepsilon_k;x)$, and
  $\beta=(1-\alpha)B(\theta, \phi;x)$ to prove \eqref{eq:gaussian-high-dim-gradient-equiv} and \eqref{eq:lastBeforeNDvar}. More precisely, by taking the expectation in (\ref{eq:gradrepN-nice}) we get that
  $$
  \PE\lr{\gradREP[1,N]} = \partial_{\psi}\PE(\repell(\varepsilon;x)) + \partial_{\psi} B(\theta,
  \phi;x) \;\PE\lr{\sum_{j=1}^N
    \Wrepell_{j,N}\; \repellstandard(\varepsilon_j;x) }
  $$
  and, as $N,B(\theta, \phi;x)\to\infty$ with~(\ref{eq:gaussian-high-dim-growing-cond})
  holding, (\ref{eq:collapse-phen-gaussian-case}) with $\delta = 1$ gives us that (\ref{eq:gaussian-high-dim-gradient-equiv}) holds. Similarly, by taking the variance in (\ref{eq:gradrepN-nice}) we have that
\begin{multline}
  \mathbb{V}\lr{\gradREP[1,N]}= \lr{\partial_{\psi} B(\theta,
    \phi;x)}^2 \, \mathbb{V}\lr{\sum_{j=1}^N
    \Wrepell_{j,N}\; \repellstandard(\varepsilon_j;x)} \\
  + \lr{B(\theta, \phi;x)}^2\,\mathbb{V}\lr{\sum_{j=1}^N
    \Wrepell_{j,N}\;\partial_{\psi}\repellstandard(\varepsilon_j;x)}. \label{eq:varToreuse}
\end{multline}
Thus,
\begin{align*}
  \mathbb{V}\lr{\gradREP[1,N]} \geq \lr{B(\theta, \phi;x)}^2\,\mathbb{V}\lr{\sum_{j=1}^N
    \Wrepell_{j,N}\;\partial_{\psi}\repellstandard(\varepsilon_j;x)}
\end{align*}
and to get~(\ref{eq:lastBeforeNDvarLB}) it only remains to
show that the second variance in the right-hand side above is $(B(\theta, \phi;x)\tBREP)^2 \lr{1+o(1)}$. Using Property~\ref{item:propREPest-gaussian} again,
\begin{align}
  \mathbb{V}\lr{\sum_{j=1}^N 
                     \Wrepell_{j,N}\;\partial_{\psi}\repellstandard(\varepsilon_j;x)}&=\PE\lr{\lr{\sum_{j=1}^N 
                                                                                       \Wrepell_{j,N}\;\partial_{\psi}\repellstandard(\varepsilon_j;x)}^2} \nonumber \\
  & =\PE\lr{\sum_{j=1}^N 
    \lr{\Wrepell_{j,N}}^2\;\PE\lr{\lr{\partial_{\psi}\repellstandard(\varepsilon_j;x)}^2}} \nonumber \\
  & =\tBREP^2\, \PE\lr{\sum_{j=1}^N 
    \lr{\Wrepell_{j,N}}^2} \nonumber \\
    & = (B(\theta, \phi;x)\tBREP)^2 \lr{1+o(1)} \label{eq:secondVartoReuse}
\end{align}
where in the last line we have used~(\ref{eq:thm:gaussian-high-dim-gradient-snr:conclude2}) with $\delta=2$.

\item We can now prove the claim of the proposition. By
  Property~\ref{item:propREPest-gaussian}, the stochastic terms in the
  second line of~(\ref{eq:gradREP11}) have mean zero and
  by~(\ref{eq:def-useful-snr-rep})
  and~(\ref{eq:gaussian-high-dim-gradient-equiv}), we have
\begin{equation}
  \label{eq:usnrep-withvar}
\USNRREP= \frac{\sqrt{2\log N} \;\partial_{\psi}
     B(\theta, \phi;x)\; \lr{1+o(1)}}{\sqrt{\mathbb{V}\lr{\gradREP[1,N]}}} \;.
\end{equation}
Applying~(\ref{eq:lastBeforeNDvarLB}), we get that
$$
\USNRREP\leq \frac{\sqrt{2\log N} \;\lrav{\partial_{\psi}
     B(\theta, \phi;x)}\; \lr{1+o(1)}}{B(\theta, \phi;x)\tBREP\lr{1+o(1)}}\;.
$$
Now, by~(\ref{eq:corr-parameter-rep_Bd-tBd}), we have
\begin{align} \label{eq:corrToreuse}
\frac{\lrav{\partial_{\psi}
     B(\theta, \phi;x)}}{B(\theta, \phi;x)\tBREP}=
\lr{\lr{\lrav{\corrREP}}^{-2}-1}^{-1/2}
\end{align}
which is $O\lr{\lrav{\corrREP}}$ under (\ref{eq:SNRcondCollapse-corr}). Hence we get~(\ref{eq:SNRcollapseGaussian-new-UB}).
\end{enumerateList}

\subsubsection{Proof of \Cref{thm:CollapseSNRGaussian-contd}}
\label{app:thm:CollapseSNRGaussian-contd} Since \eqref{eq:SNRcondCollapse-corr} is implied by \eqref{eq:SNRcondCollapse-corr-stronger}, we can build on the equations established in the proof of \Cref{thm:CollapseSNRGaussian} such as \eqref{eq:varToreuse}. We then have that the first variance in the right-hand side of \eqref{eq:varToreuse} is $o\lr{\log N }$ as
$N,B(\theta, \phi;x)\to\infty$
with~(\ref{eq:gaussian-high-dim-growing-cond}) holding thanks to (\ref{eq:thm:gaussian-high-dim-gradient-snr:conclude1}) with $\lambda=1$ in
\Cref{lem:gaussian-collapse-gathered-elementary-statements}. Pairing this up with \eqref{eq:secondVartoReuse} leads to 
   \begin{align}   
  \label{eq:lastBeforeNDvar} \mathbb{V}\lr{\gradREP[1,N]} & = \lr{\partial_{\psi}
                                    B(\theta, \phi;x)}^2\; o\lr{\log N } \\ 
  & \quad + \lr{B(\theta, \phi;x)\tBREP}^2 \lr{1+o(1)} \nonumber.   
  \end{align} 
In addition, the condition \eqref{eq:SNRcondCollapse-corr-stronger} paired up with \eqref{eq:corrToreuse} yields
$$
\frac{\lrav{\partial_{\psi}
     B(\theta, \phi;x)}}{B(\theta, \phi;x)\tBREP}=
|\corrREP|\;\lr{1+o(1)}=O\lr{\lr{\log N}^{-1/2}} \;.
$$
so that~(\ref{eq:lastBeforeNDvar}) becomes
\begin{align}
  \mathbb{V}\lr{\gradREP[1,N]} & =  \lr{B(\theta, \phi;x)\tBREP}^2 \lr{1+o(1)} \nonumber \\
  &  = \mathbb{V}(\gradREP[1,1])\lr{1+o(1)} \label{eq:SNRcollapseGaussianVarEquivalence-new}
\end{align}
Combining this with \eqref{eq:var-Neq1-rep} and \eqref{eq:usnrep-withvar} leads to \eqref{eq:SNRcollapseGaussian-new}. Furthermore, \eqref{eq:rewritingVar11}, \eqref{eq:corr-parameter-rep_Bd-tBd} and \eqref{eq:gaussian-high-dim-gradient-equiv} yield 
\begin{multline}
\PE \lr{\gradREP[1,N]} = \PE(\gradREP[1,1]) \\
  +  \sqrt{\mathbb{V}(\gradREP[1,1])} \sqrt{2\log N}\; \corrREP \lr{1+o(1)}. \label{eq:SBRcollapseExpequiv}
\end{multline}
If we further assume the stronger condition \eqref{eq:SNRcondCollapse-corr-even-stronger}, \eqref{eq:SNR-Equiv-as-old-times} immediately follows from the two above displays paired up with \eqref{eq:var-Neq1-rep} and \eqref{eq:gaussian-high-dim-gradient-equiv}.

\begin{rem} \label{rem:practicalHighDim}
As $N, \KL \to \infty$ under \eqref{eq:gaussian-high-dim-growing-cond} and \eqref{eq:SNRcondCollapse-corr-stronger}, no improvement in variance is obtained by increasing the sample size $N$ beyond the case $N =1$ in $\gradREP[1,N]$ since \eqref{eq:SNRcollapseGaussianVarEquivalence-new} shows that $\mathbb{V}(\gradREP[1,N])$ reverts to $\mathbb{V}(\gradREP[1,1])$. The case of \eqref{eq:SBRcollapseExpequiv} is more subtle: using \eqref{eq:rewritingCorrBis} in \Cref{subsec:addNotation}, \eqref{eq:SBRcollapseExpequiv} reads equivalently \looseness=-1
\begin{multline} \label{eq:decompGrad}
   \PE(\gradREP[1,N]) = \PE(\gradREP[1,1]) \\ + \sqrt{\frac{\log N}{\KL}}\; \lrb{\partial_\psi \ell(\theta;x) - \PE(\gradREP[1,1])}\lr{1+o(1)}.
\end{multline}
Then, $\gradREP[1,N]$ points in the direction of
$ \PE(\gradREP[1,1])$ on average when $\psi \in \{ \phi_1, \ldots, \phi_b \}$, since $\partial_\psi \ell(\theta;x) = 0$ and
\begin{align*}
   \PE(\gradREP[1,N]) = \PE(\gradREP[1,1])\lr{1 - \sqrt{\frac{\log N}{\KL}} +o(1)}.
\end{align*}
There is no thus improvement in terms of gradient direction for learning $\phi$ that follows from taking $N > 1$ instead of $N = 1$. When $\psi \in \{\theta_1, \ldots, \theta_a \}$, the answer depends on which term dominates in \eqref{eq:decompGrad}. Indeed, $\gradREP[1,N]$ points on average (i) in the direction of $\PE(\gradREP[1,1])$ should the first one dominate and (ii) in the direction of a convex combination of $\PE(\gradREP[1,1])$ and $\partial_\psi \ell(\theta; \phi)$ should both terms be of the same order. Note that the case where the second term dominates amounts to assuming that $\mathrm{SNR}[\gradREP[1,1]] = o(1)$ under \eqref{eq:SNRcondCollapse-corr-stronger} so we leave it out.
\end{rem}

\subsection{General Gaussian model}
\label{subsec:genGaussModel}
Consider a general Gaussian model with known positive definite covariance matrix $\Sigma(x)$ obtained by setting $p_\theta(z|x)=\mathcal{N}(z;\mu(\theta;x),\Sigma(x))$ $q_\phi(z|x)=\mathcal{N}(z;\tilde{\mu}(\phi;x),\Sigma(x))$, where $\mu$ and $\tilde{\mu}$ are the mean vectors parameterized by the true and variational parameters respectively and where we use the reparametrization $f(\varepsilon, \phi;x)={\Sigma(x)}^{1/2}\,\varepsilon+\tilde{\mu}(\phi;x)$ with $\varepsilon\sim \mathcal{N}(0,\mathbf{I}_d)$. Then, \ref{hyp:reparamHighDimGaussian} holds and we have
\begin{align*}
  \log \tw & = \ell(\theta; x) +\frac12\langle
             \lr{\tilde{\mu}(\phi;x)-\tilde\mu(\phi';x)}\,,\,\Sigma(x)^{-1/2}\lr{\tilde{\mu}(\phi;x)-\tilde\mu(\phi';x)}
             \rangle\\
    &\quad -\frac12\langle
      \lr{\tilde{\mu}(\phi';x)-\mu(\theta;x)}\,,\,\Sigma(x)^{-1}\lr{\tilde{\mu}(\phi';x)-\mu(\theta;x)}
  \rangle\\  
    &\quad + \langle \varepsilon\,,\,\Sigma(x)^{-1/2}\lr{\tilde{\mu}(\phi;x)-{\mu}(\theta;x)}\rangle.
\end{align*}
Observe then that the stochastic part above does not depend on $\phi'$ hence $\mathbb{V}(\gradDREP[1,1])$ $ = 0$ when $\psi \in \{ \phi_1, \ldots, \phi_n \}$. We also obtain that
\begin{align*}
\KL= \frac12\lrn{\Delta(\theta,\phi;x)}^2\; \quad  \mbox{and} \quad \corrREP=\frac{\langle\partial_\psi\Delta(\theta,\phi;x)\,,\,\Delta(\theta,\phi;x)\rangle}{\lrn{\Delta(\theta,\phi;x)}\,\lrn{\partial_\psi\Delta(\theta,\phi;x)}}
\end{align*}
with $\Delta(\theta,\phi;x):=\Sigma(x)^{-1/2}\lr{\tilde{\mu}(\phi;x)-\mu(\theta;x)}$.

\subsection{Proof of
  \Cref{ex:GaussianHighDim}} \label{app:ex:GaussianHighDim} This
example is a specific case of the general one in
\Cref{subsec:genGaussModel} with $\Sigma(x)=\mathbf{I}_d$,
$\mu(\theta;x)=\theta$ and $\tilde{\mu}(\phi;x)=\phi$. Thus we get
from \Cref{subsec:genGaussModel} that \ref{hyp:reparamHighDimGaussian}
holds, $\Delta(\theta,\phi;x)=\phi-\theta$, wich implies
$\|\partial_\psi(\phi-\theta)\|=1$ and
$\langle\partial_\psi(\phi-\theta)\,,\,(\phi-\theta)\rangle=\psi$ for
any component $\psi$ among
$\{\theta_1, \ldots, \theta_a,\phi_1,\dots,\phi_b\}$, and:
\begin{align}
  \label{eq:normWrewrittenGauss-drep}
  & \log \tw = \ell(\theta;x) -\frac{\|\theta-\phi\|^2}{2}
    -\langle \varepsilon + \phi' -
    \phi, \phi- \theta \rangle\;,\\
  \label{eq:normWrewrittenGauss}
  &\log \tw[\phi][\varepsilon][\phi] = \ell(\theta;x) 
  -\frac{\|\theta-\phi\|^2}{2} -\langle \varepsilon, \phi- \theta \rangle
    \;,\\
  \label{eq:normWrewrittenGauss-kl}
  &\KL= \frac12\,\|\theta-\phi\|^2\;,\\
  \label{eq:normWrewrittenGauss-corr}
  &\corrREP=\frac{\psi}{\|\phi-\theta\|}\;.
\end{align}
We easily get from  \eqref{eq:normWrewrittenGauss} that
\ref{hyp:hypZeroREPhighDim} also holds. Furthermore in the setting $\theta -\phi= \epsilon \cdot
\boldsymbol{u}_d$, we get that
$$
\KL=\frac12\epsilon^2d\quad\text{and}\quad \corrREP=d^{-1/2}\;.
$$
Therefore, in this specific case,~(\ref{eq:gaussian-high-dim-growing-cond}) is equivalent
to~(\ref{eq:Nexpd}),~(\ref{eq:SNRcondCollapse-corr}) to eventually have
$d\geq2$ and~(\ref{eq:SNRcondCollapse-corr-stronger})
and~(\ref{eq:SNRcondCollapse-corr-even-stronger}) to  $\log N=O(d)$
and $\log N\ll d$, respectively. 

Applying \Cref{thm:CollapseSNRGaussian-contd}, \eqref{eq:Nexpd}  then implies \eqref{eq:SNR-Equiv-as-old-times}
which takes the form~(\ref{eq:Nexpd-thm-applied}) since~(\ref{eq:normWrewrittenGauss}) straightforwardly leads to
$$
\mathrm{SNR}[\gradREP[1,1][\phi_k]] = \epsilon\quad\text{for $k=1,\dots,a$.}
$$
Furthermore, as far as the DREP estimator is concerned, 
we get from~(\ref{eq:normWrewrittenGauss-drep}) that 
\begin{align*}
  \lrder{\partial_{\phi'_k} \log \tw}{\phi'=\phi} - 0 = 0 + \phi_k- \theta_k  \;,
\end{align*}
thus $\mathbb{V}(\gradDREP[1,1]) = 0$.

\section{Additional numerical experiments}

\subsection{Gaussian experiment from \Cref{num:GaussEx}}
\label{app:num:GaussEx}

\Cref{ex:GaussianHighDim} predicts that the REP gradient estimator with $N>1$ reverts to the case $N=1$ as $d$ increases, and notably that $\PE(\gradREP[1,N][\phi_k])$ reverts to $\PE(\gradREP[1,1][\phi_k])$. Similarly to the REP case detailed in \Cref{num:GaussEx}, this behavior can be observed in the DREP case by setting a high value for $\epsilon$, see Figure \ref{fig:GaussianExampleDREP}. \looseness=-1 
 \begin{figure}[t]
\begin{tabular}{ccc} 
   \includegraphics[scale=0.28]{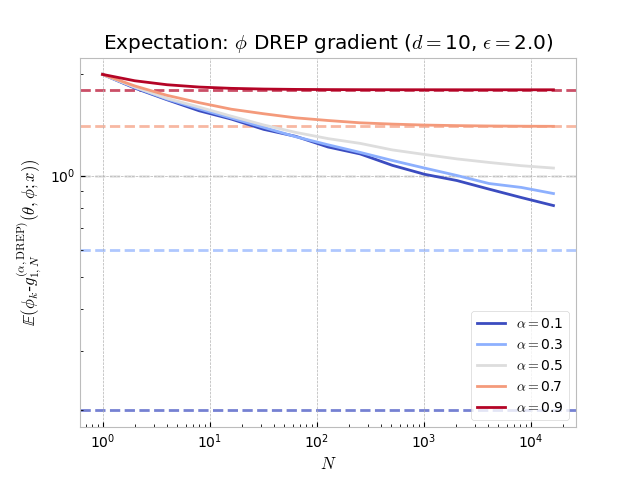} & \includegraphics[scale=0.28]{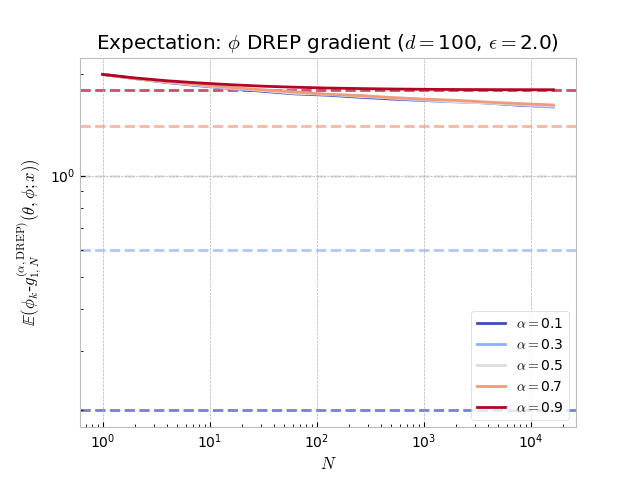} & \includegraphics[scale=0.28]{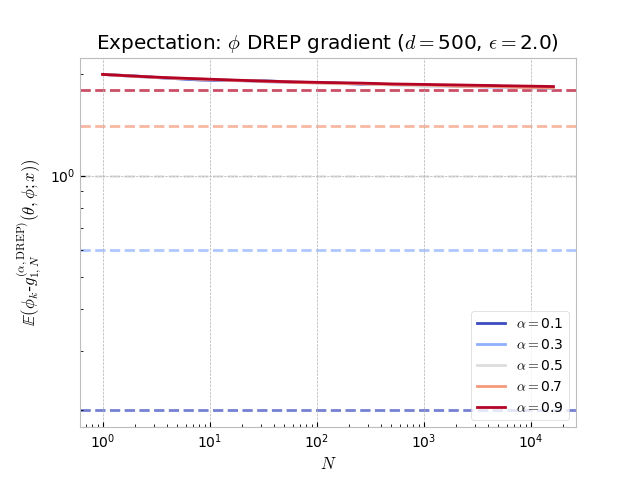} 
  \end{tabular}
    \caption{Plotted is $\PE(\gradDREP[1,N][\phi_k])$ computed over 2000 Monte Carlo samples for the Gaussian example described
    in \Cref{num:GaussEx} as a function of $N$, for varying values of $(\alpha,d)$, $\epsilon = 2$ and a random coordinate $\phi_k$. The solid lines correspond to $\PE(\gradREP[1,N][\phi_k])$, the dashed lines correspond to predictions of the form $y = \epsilon \alpha$. \label{fig:GaussianExampleDREP}}
  \end{figure}

The IWAE case unfolds in the same manner as $d$ increases. By \Cref{ex:Gaussian}: as $N \to \infty$, 
\begin{align}
& \mathrm{SNR}[\gradREP[1,N][\phi_k][0]] = \sqrt{\frac{1}{N}}
  \dfrac{\exp\lr{ \frac{d \epsilon^2}{2}}}{\sqrt{1+\epsilon^2}} (1+o(1)) \label{eq:predictSNRGaussExOne} \\ & \mathrm{SNR}[\gradDREP[1,N][\phi_k][0]] =  \frac{\sqrt {N}(1+o(1))}{\lr{{\exp\lr{4d \epsilon^2} - 4\exp\lr{2d \epsilon^2}+4\exp\lr{d \epsilon^2}-1}}^{\frac{1}{2}}} \label{eq:predictSNRGaussExTwo}
\end{align}
We let $d \in \{10, 100\}$, $\epsilon = 0.2$, $N \in \{ 2^j, j = 1 \ldots 15 \}$ and our results are plotted on Figure \ref{fig:GaussianExampleapp}. Similarly to the REP case with $\alpha \in (0,1)$ detailed in \Cref{num:GaussEx}, in the favourable setting of a low dimension $d = 10$, the behavior of the REP 
and DREP gradient estimators predicted by \eqref{eq:predictSNRGaussExOne} and \eqref{eq:predictSNRGaussExTwo} respectively 
match as $N$ increases. As expected, this is no longer true as $d$ increases.

\begin{figure}[ht!]
  \begin{tabular}{ccc} 
    \includegraphics[scale=0.28]{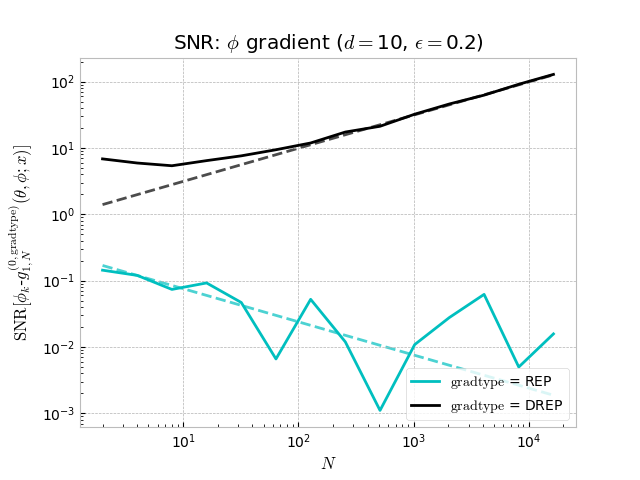} & \includegraphics[scale=0.28]{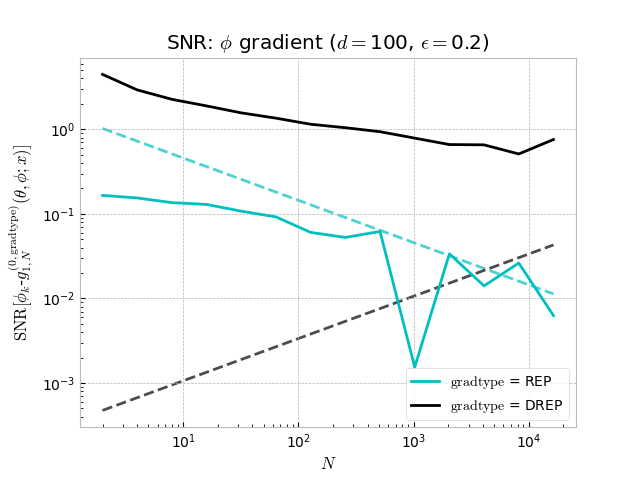} & \includegraphics[scale=0.28]{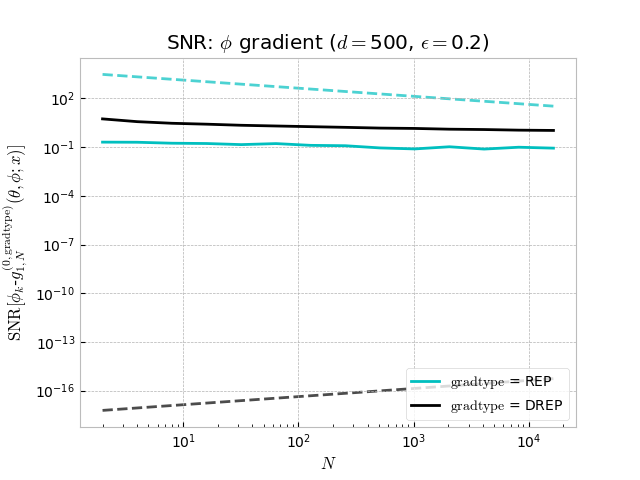}
  \end{tabular}
  \caption{Plotted are 
  $\mathrm{SNR}[\gradREP[1,N][\phi_k][0]]$ and $\mathrm{SNR}[\gradDREP[1,N][\phi_k][0]]$ computed over 2000 Monte Carlo samples for the Gaussian example described in \Cref{num:GaussEx} as a function of $N$ and with $\epsilon = 0.2$. The solid lines correspond to the SNRs, while the dashed lines correspond to predictions of the form \eqref{eq:predictSNRGaussExOne} and \eqref{eq:predictSNRGaussExTwo}.
   \label{fig:GaussianExampleapp}}
\end{figure}

\subsection{Linear Gaussian experiment from \Cref{num:LinGaussEx}}

\label{app:num:LinGaussEx}

We provide here additional experiments for the cases $\psi = \theta_k$ with $\alpha \in [0,1)$ and $\psi = \notationb_k$ with $\alpha = 0$. By \Cref{ex:LinGauss}: as $N \to \infty$, for all $\alpha \in [0,1)$,
\begin{align}
& \mathrm{SNR}[\gradREP[1,N][\theta_k]] = \sqrt{N}~\frac{\left|\frac{x_k-\theta_k}{2} + \frac{3\epsilon\alpha}{4-\alpha}\right| (1+o(1))}{\frac{(4-\alpha)^{d/2}}{(15-6\alpha)^{d/4}} \exp \lr{\frac{12(1-\alpha)^2}{(4-\alpha)(5-2\alpha)} d \epsilon^2} \sqrt{{\frac{2}{5-2\alpha} + \lr{\frac{12(1-\alpha) \epsilon}{(5-2\alpha)(4-\alpha)}}^2}}} \label{eq:predictSNRLinGaussExOne} \\
  & \mathrm{SNR}[\gradDREP[1,N][\notationb_k][0]] =
  \sqrt{N} \frac{\frac{{24\epsilon 4^{d-1}} \exp \lr{\frac{24 d\epsilon^2}{4\cdot 5}}}{3^{d/2}5^{\frac{d}{2} + 1}}}{\sqrt{\genvDREP[0][\notationb_k](\theta, \phi;x)}}(1+o(1)) \label{eq:predictSNRLinGaussExTwo} 
\end{align}
We let $d \in \{10, 100, 500\}$, $\epsilon \in \{0.2, 1\}$, $N \in \{2^j, 1\ldots15 \}$ and our results are plotted on Figures \ref{fig:LinGaussExampleTheta} and \ref{fig:LinGaussExampleB}. Similarly to the cases detailed in \Cref{num:GaussEx}, in the favourable setting of a low dimension $d = 10$, the behavior of the REP 
and DREP gradient estimators predicted by \eqref{eq:predictSNRLinGaussExOne} and \eqref{eq:predictSNRLinGaussExTwo} respectively 
match as $N$ increases. This is no longer true as $d$ increases. Interestingly, we see that when it comes to the learning of $\theta$, the SNR is not monotonic in $\alpha$, which is not surprising given the numerator term $\left|\frac{x_k-\theta_k}{2} + \frac{3\epsilon\alpha}{4-\alpha}\right|$ appearing in \eqref{eq:predictSNRLinGaussExOne}.

\begin{figure}[ht!]
  \begin{tabular}{ccc} 
    \includegraphics[scale=0.28]{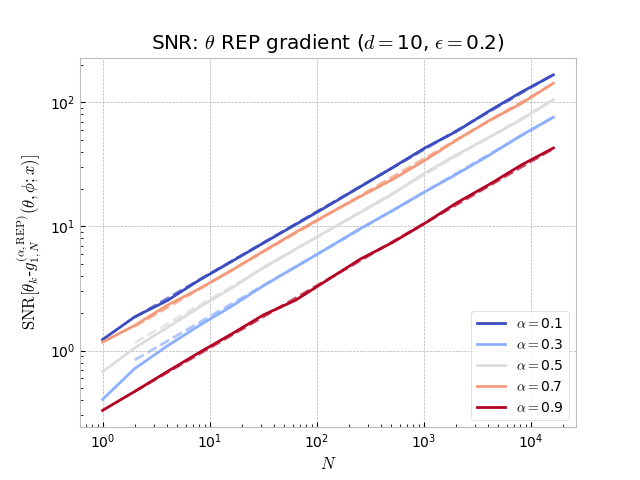} & \includegraphics[scale=0.28]{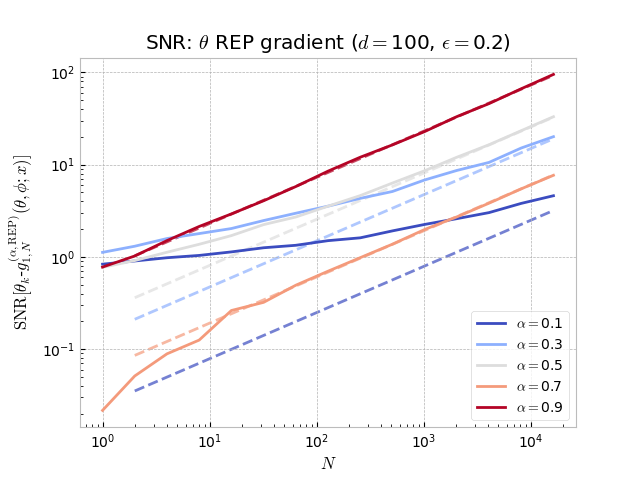} & \includegraphics[scale=0.28]{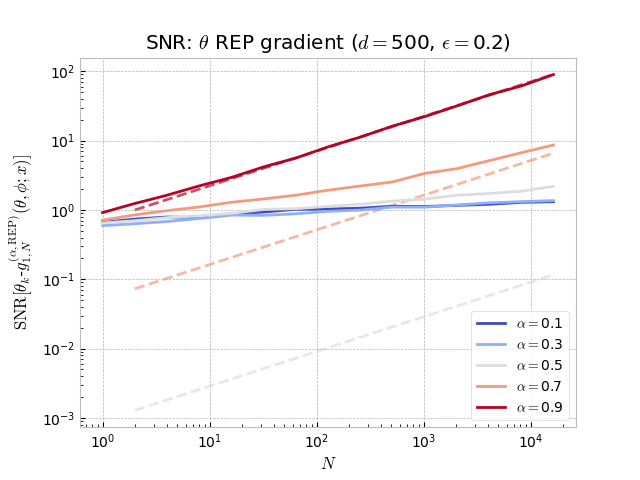} \\
    \includegraphics[scale=0.28]{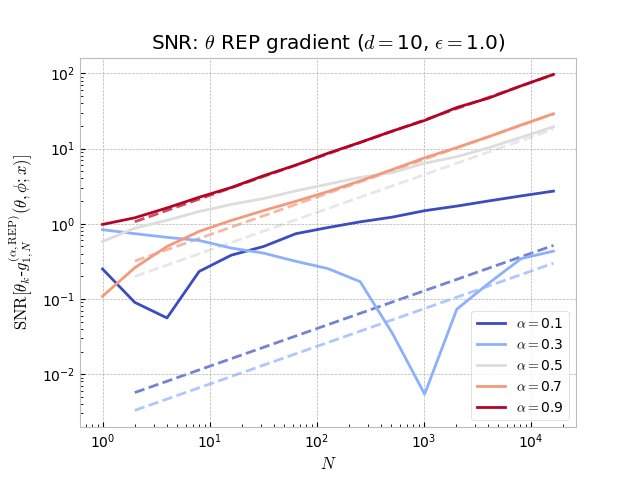} & \includegraphics[scale=0.28]{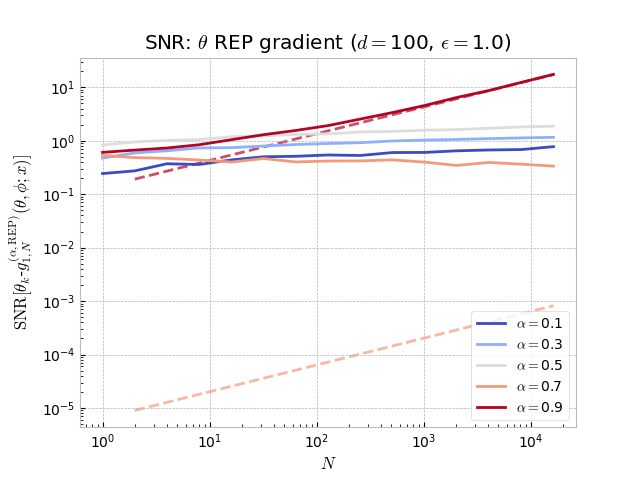} & \includegraphics[scale=0.28]{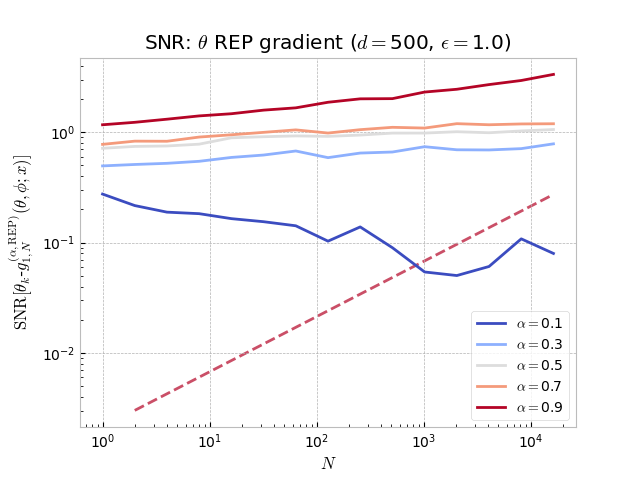}
  \end{tabular}
  \caption{Plotted is $\mathrm{SNR}[\gradREP[1,N][\theta_k]]$ computed over 2000 Monte Carlo samples for the Linear Gaussian example described in \Cref{num:LinGaussEx} as a function of $N$, for varying values of $(\alpha,d, \epsilon)$ and a randomly selected
  datapoint $x$. The solid lines correspond to $\mathrm{SNR}[\gradREP[1,N][\theta_k]]$, the dashed lines correspond to predictions of the form \eqref{eq:predictSNRLinGaussExOne}. \label{fig:LinGaussExampleTheta}}
\end{figure}

\begin{figure}[ht!]
  \begin{tabular}{ccc} 
    \includegraphics[scale=0.28]{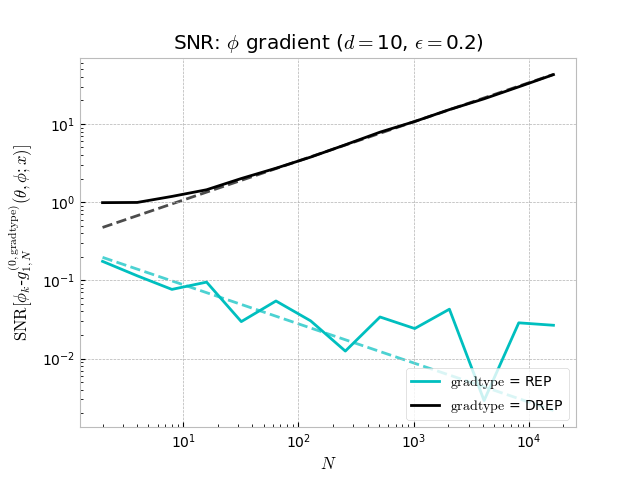} &     \includegraphics[scale=0.28]{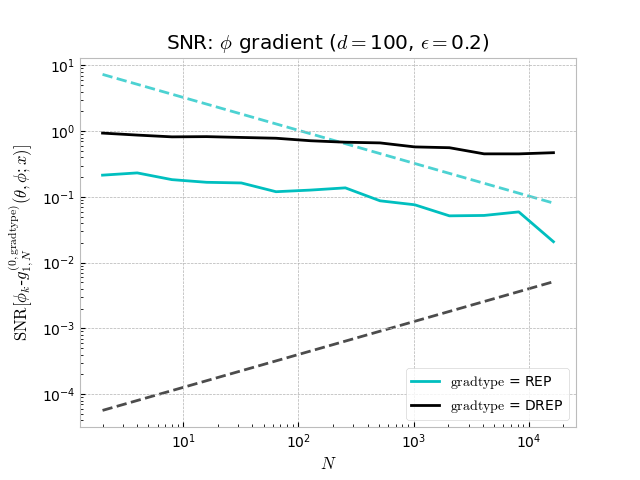} & \includegraphics[scale=0.28]{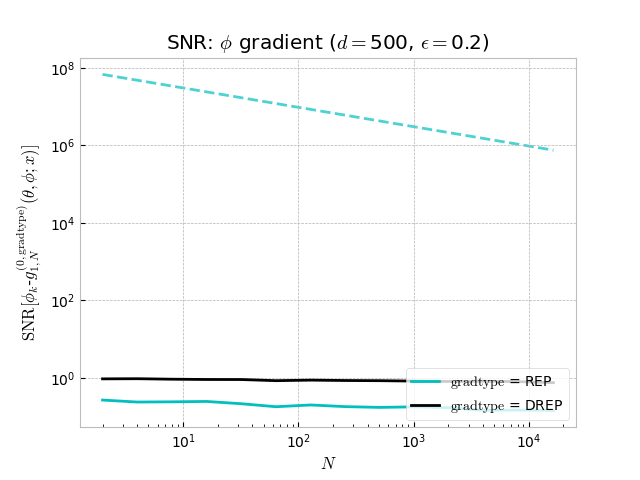}
  \end{tabular}
  \caption{Plotted are 
  $\mathrm{SNR}[\gradREP[1,N][\notationb_k][0]]$ and $\mathrm{SNR}[\gradDREP[1,N][\notationb_k][0]]$ computed over 2000 Monte Carlo samples for the Linear Gaussian example described in \Cref{num:LinGaussEx} as a function of $N$, with $\epsilon = 0.2$ and for a randomly selected
  datapoint $x$. The solid lines correspond to the SNRs, the dashed lines correspond to predictions of the form \eqref{eq:predictSNRLinGaussExOne} and \eqref{eq:predictSNRLinGaussExTwo}. \label{fig:LinGaussExampleB}}
\end{figure}


\printbibliography

\end{document}